\documentclass[10pt]{article}

\usepackage{arxiv}
\usepackage[utf8]{inputenc}
\usepackage{amsmath,amssymb}

\usepackage[noend]{algpseudocode}

\usepackage{amsthm}
\usepackage{xargs}
\usepackage{xcolor}

\usepackage[T1]{fontenc}    % use 8-bit T1 fonts
\usepackage{hyperref}       % hyperlinks
\usepackage{url}            % simple URL typesetting
\usepackage{booktabs}       % professional-quality tables
\usepackage{amsfonts}       % blackboard math symbols
\usepackage{nicefrac}       % compact symbols for 1/2, etc.
\usepackage{microtype}      % microtypography
\usepackage{xcolor}         % colors

\usepackage{times}

\usepackage{lipsum}
\usepackage{graphicx}

\usepackage{mathrsfs}
\usepackage{xargs}
\usepackage{enumitem}

\usepackage{aliascnt}
\usepackage{appendix}
\usepackage[english]{babel}

\usepackage{amsmath}
\usepackage{amsthm}
\usepackage{cleveref}
\usepackage{amssymb}
\usepackage{caption}
\usepackage{cancel}
\usepackage{array}
\usepackage{framed}
\usepackage{multirow}

\usepackage{bbm}
\usepackage[ruled,vlined]{algorithm2e}

\newtheorem{theorem}{Theorem}

\usepackage[english]{babel}
\usepackage{natbib}
\usepackage{hyperref}
\hypersetup{colorlinks = true}
\hypersetup{linkcolor=blue}
\usepackage{cleveref}

\newtheoremstyle{styledef}
  {6pt}% space above
  {6pt}% space below
  {\itshape}%body font
  {0em}%indent amount
  {\bfseries}%Theorem head
  {}%punctuation
  {1.5em}%space after theorem head
  {}

\theoremstyle{styledef}

\newtheorem{thm}{Theorem}
\crefname{thm}{theorem}{theorems}
\Crefname{thm}{Theorem}{Theorems}

\newtheorem{ex}{Example}
\crefname{ex}{example}{examples}
\Crefname{ex}{Example}{Examples}

\newtheorem{lem}{Lemma}
\crefname{lem}{Lemma}{lemmas}
\Crefname{lem}{Lemma}{Lemmas}

\newtheorem{prop}{Proposition}
\crefname{prop}{Proposition}{Propositions}
\Crefname{Prop}{Proposition}{Propositions}

\newtheorem{coro}{Corollary}
\crefname{coro}{corollary}{corollaries}
\Crefname{Coro}{Corollary}{Corollaries}

\crefname{defi}{definition}{definitions}
\Crefname{Def}{Definition}{Definitions}

\newtheorem{rem}{Remark}
\crefname{rem}{remark}{remarks}
\Crefname{rem}{Remark}{Remarks}

\newenvironment{hyp}[1]{%\begin{sf}\refstepcounter{hyp#1}
  \begin{enumerate}[label=\textbf{\sf(#1\arabic*)},resume=hyp#1]\begin{sf}}
  {\end{sf}\end{enumerate}}%\end{sf}}
  \crefname{hyp}{}{ass}
  \Crefname{hyp}{}{Ass}

\newenvironment{hypxi}[2]{%\begin{sf}\refstepcounter{hyp#1}
    \begin{enumerate}[label=\textbf{\sf(#1-#2)},resume=hypxi#1]\begin{sf}}
    {\end{sf}\end{enumerate}}%\end{sf}}
\crefname{hyp}{}{ass}
\Crefname{hyp}{}{Ass}

\renewenvironment{leftbar}[2][\hsize]
{
    \def\FrameCommand
    {
        {\color{#2}\vrule width 3pt}
        \hspace{0pt}
    }
    \MakeFramed{\hsize#1\advance\hsize-\width\FrameRestore}
}
{\endMakeFramed}

\newcommandx{\barw}[2][1=\theta, 2=\phi]{\overline{w}_{#1, #2}}
\newcommandx{\barwa}[2][1=\theta, 2=\phi]{\overline{w}^{(\alpha)}_{#1, #2}}
\newcommand{\data}{\mathcal{D}}
\newcommand{\eqdef}{:=}

\newcommand{\lr}[1]{\left(#1 \right)}
\newcommand{\lrb}[1]{\left[#1 \right]}

\newcommand{\lrcb}[1]{\left\{#1 \right\}}

\newcommandx{\liwae}[1][1=N]{\ell^{(\mathrm{IWAE})}_{#1}}
\newcommandx{\liren}[2][1=\alpha, 2=N]{\ell^{(#1)}_{#2}}
\newcommandx{\lirenBB}[2][1=\alpha, 2=N]{\tilde{\ell}^{(#1)}_{#2}}
\newcommand{\PE}{\mathbb E}
\newcommand{\PP}{\mathbb P}

\newcommand{\rmd}{\mathrm d}
\newcommand{\rset}{\mathbb{R}}

\newcommand{\Ralpha}{R_{\alpha}}
\newcommandx{\RalphaN}[1][1=\alpha]{R_{#1,N}}
\newcommandx{\Rtalpha}[1][1=\alpha]{\overline{w}_{\theta, \phi}^{(#1)}}

\newcommand{\sd}{\boldsymbol{S}_d}
\newcommand{\sigmapert}{\sigma_{\mathrm{perturb}}}
\newcommandx{\w}[2][1=\theta, 2=\phi]{w_{#1, #2}}

\DeclareMathOperator{\cov}{Cov}

\begin{document}
%\title{Challenges and Opportunities in Scalable Alpha-divergence Variational Inference: Application to Importance Weighted Auto-Encoders}
\title{Alpha-divergence Variational Inference Meets Importance Weighted Auto-Encoders: Methodology and Asymptotics}

\author{Kamélia Daudel %\\ kamelia.daudel@stats.ox.ac.uk 
\And Joe Benton* %\\ benton@stats.ox.ac.uk  
\And Yuyang Shi* %\\ yshi@stats.ox.ac.uk 
\And Arnaud Doucet %\\ doucet@stats.ox.ac.uk 
\AND \normalfont{Department of Statistics, University of Oxford, United Kingdom}}
\maketitle
{\small{}{}{}}{\small\par}

\def\thefootnote{*}\footnotetext{: Equal contribution}

\begin{abstract}
Several algorithms involving the Variational Rényi (VR) bound have been proposed to minimize an alpha-divergence between a target posterior distribution and a variational distribution. Despite promising empirical results, those algorithms resort to biased stochastic gradient descent procedures and thus lack theoretical guarantees. In this paper, we formalize and study the VR-IWAE bound, a generalization of the Importance Weighted Auto-Encoder (IWAE) bound. We show that the VR-IWAE bound enjoys several desirable properties and notably leads to the same stochastic gradient descent procedure as the VR bound in the reparameterized case, but this time by relying on unbiased gradient estimators. We then provide two complementary theoretical analyses of the VR-IWAE bound and thus of the standard IWAE bound. Those analyses shed light on the benefits or lack thereof of these bounds. Lastly, we illustrate our theoretical claims over toy and real-data examples.
\end{abstract}


\section{Introduction}

Variational inference methods aim at finding the best approximation to a target posterior density within a so-called variational family of probability densities. This best approximation is traditionally obtained by minimizing the exclusive Kullback--Leibler divergence \citep{wainwright2008graphical, blei2017variational}, however this divergence is known to have some drawbacks \citep[for instance variance underestimation, see][]{minka2005divergence}.

As a result, alternative divergences have been explored \citep{minka2005divergence, li2016renyi, Bui2016BlackboxF, dieng2017variational, li2017dropout, wang2018variational, daudel2021infinite, https://doi.org/10.48550/arxiv.2103.05684, daudel2021mixture,  RODRIGUEZSANTANA2022260}, in particular the class of alpha-divergences. This family of divergences is indexed by a scalar $\alpha$. It provides additional flexibility that can in theory be used to overcome the obstacles associated to the exclusive Kullback--Leibler divergence (which is recovered by letting $\alpha \to 1$).

Among those methods, techniques involving the Variational Rényi (VR) bound introduced in \cite{li2016renyi} have led to promising empirical results and have been linked to key algorithms such as the Importance Weighted Auto-encoder (IWAE) algorithm \citep{burda2015importance} in the special case $\alpha = 0$ and the Black-Box Alpha (BB-$\alpha$) algorithm \citep{hernandez2016black}. 

Yet methods based on the VR bound are seen as lacking theoretical guarantees. This comes from the fact that they are classified as biased in the community: by selecting the VR bound as the objective function, those methods indeed resort to biased gradient estimators \cite[]{li2016renyi, hernandez2016black, Bui2016BlackboxF, li2017dropout, GeffnerDomke2020Biased, geffner21a, pmlr-v130-zhang21o, RODRIGUEZSANTANA2022260}.

\cite{GeffnerDomke2020Biased} have recently provided insights from an empirical perspective regarding the magnitude of the bias and its impact on the outcome of the optimization procedure when the (biased) reparameterized gradient estimator of the VR bound is used. They observe that the resulting algorithm appears to require an impractically large amount of computations to actually optimise the VR bound as the dimension increases (and otherwise seems to simply return minimizers of the exclusive Kullback--Leibler divergence). They postulate that this effect might be due to a weight degeneracy behavior \citep{bengtsson2008curse}, but this behavior is not quantified precisely from a theoretical point of view. 

In this paper, our goal is to (i) develop theoretical guarantees for VR-based variational inference methods and (ii) construct a theoretical framework elucidating the weight degeneracy behavior that has been empirically observed for those techniques. The rest of this paper is organized as follows:

\begin{itemize}
  \item In \Cref{sec:background}, we provide some background notation and we review the main concepts behind the VR bound.
  
  \item In \Cref{sec:VR-IWAEbound}, we introduce the VR-IWAE bound. We show in \Cref{lem:generalBound} that this bound, previously defined by \cite{li2016renyi} as the expectation of the biased Monte Carlo approximation of the VR bound, can be  actually interpreted as a variational bound which depends on an hyperparameter $\alpha$ with $\alpha \in [0,1)$. In addition, we obtain that the VR-IWAE bound leads to the same stochastic gradient descent procedure as the VR bound in the reparameterized case. Unlike the VR bound, the VR-IWAE bound relies on \textit{unbiased} gradient estimators and coincides with the IWAE bound for $\alpha = 0$, fully bridging the gap between both methodologies. 
  
  We then generalize the approach of \cite{rainforth2018tighter} -- which characterizes the Signal-to-Noise Ratio (SNR) of the reparameterized gradient estimators of the IWAE -- to the VR-IWAE bound and establish that the VR-IWAE bound with $\alpha \in (0,1)$ enjoys better theoretical properties than the IWAE bound (\Cref{prop:SNRconvergence}). To further tackle potential SNR difficulties, we also extend the doubly-reparameterized gradient estimator of the IWAE \citep{Tucker2019DoublyRG} to the VR-IWAE bound (\Cref{prop:drepIwaeAlpha}). 

  \item In \Cref{sec:HighDim}, we provide a thorough theoretical study of the VR-IWAE bound. Following \cite{domke2018}, we start by investigating the case where the dimension of the latent space $d$ is fixed and the number of Monte Carlo samples $N$ in the VR-IWAE bound goes to infinity (\Cref{prop:GenDomke}). Our analysis shows that the hyperparameter $\alpha$ allows us to balance between an error term depending on both the encoder and the decoder parameters $(\theta, \phi)$ and a term going to zero at a $1/N$ rate. This suggests that tuning $\alpha$ can be beneficial to obtain the best empirical performances.
  
  However, the relevance of such analysis can be limited for a high-dimensional latent space $d$ (Examples \ref{lem:VarianceExponentialWithD} and \ref{ex:LinGaussThm3}). We then propose a novel analysis where $N$ does not grow as fast as exponentially with $d$ (Theorems \ref{thm:MainGauss} and \ref{thm:MainGauss2}) or sub-exponentially with $d^{1/3}$ (\Cref{thm:iidRv}), which we use to revisit Examples \ref{lem:VarianceExponentialWithD} and \ref{ex:LinGaussThm3} in Examples \ref{ex:GaussianLogNormal} and \ref{ex:linGauss} respectively. This analysis suggests that in these regimes the VR-IWAE bound, and hence in particular the IWAE bound, are of limited interest.

  \item In \Cref{sec:relatedWork}, we detail how our work relates to the existing litterature.
   
  \item Lastly, \Cref{sec:numericalExperiments} provides empirical evidence illustrating our theoretical claims for both toy and real-data examples.
\end{itemize}

\section{Background}

\label{sec:background}

Given a model with joint distribution $p_\theta(x, z)$ parameterized by $\theta$, where $x$ denotes an observation and $z$ is a latent variable valued in $\rset^d$, one is interested in finding the parameter $\theta$ which best describes the observations $\data = \{x_1, \ldots, x_T \}$. This will be our running example. The corresponding posterior density satisfies:
\begin{align} \label{eq:posteriorOne}
p_\theta(\boldsymbol{z}|\data) \propto \prod_{i=1}^T p_\theta(x_i, z_i)
\end{align}
with $\boldsymbol{z} = (z_1, \ldots, z_T )$, so that the marginal log likelihood reads
\begin{align}  \label{eq:def:loglik}
  \ell(\theta; \data) = \sum_{i = 1}^T \ell(\theta; x_i) \quad \mbox{with} \quad 
\ell(\theta; x) \eqdef \log p_\theta(x) =  \log \lr{\int p_\theta(x,z)\rmd z}.
\end{align}
Unfortunately as this marginal log likelihood is typically intractable, finding $\theta$ maximizing it is difficult. Variational bounds are then designed to act as surrogate objective functions more amenable to optimization. 

Let $q_\phi(z|x)$ be a variational encoder parameterized by $\phi$, common variational bounds are the Evidence Lower BOund (ELBO) and the IWAE  bound \citep{burda2015importance}:
\begin{align*} 
&\mathrm{ELBO} (\theta, \phi; x) = \int q_\phi(z|x) \log \w(z;x)~ \rmd z, \\ %\label{eq:ELBO} \\
&\liwae (\theta,\phi; x) =\int\int\prod_{i=1}^{N}q_{\phi}(z_{i}|x)\log\left(\frac{1}{N}\sum_{j=1}^{N} \w(z_{j}; x)\right)\rmd z_{1:N} , \quad  N \in \mathbb{N}^\star %\label{eq:ELBOIWAE}
\end{align*}
where for all $z \in \rset^d$,
\begin{align*}
\w(z;x) & =\frac{p_{\theta}(x,z)}{q_{\phi}(z|x)}.
\end{align*}
The IWAE bound generalizes the ELBO (which is recovered for $N = 1$) and acts as a lower bound on $\ell(\theta; x)$ that can be estimated in an unbiased manner. Instead of maximizing $ \ell(\theta; \data)$ defined in \eqref{eq:def:loglik}, one then considers the surrogate objective
$$
\sum_{i = 1}^T \liwae (\theta,\phi; x_i)
$$
which is optimized by performing stochastic gradient descent steps w.r.t. $(\theta, \phi)$ on it combined to mini-batching. Optimizing this objective w.r.t. $\phi$ is difficult due to high-variance gradients with low Signal-to-Noise Ratio \citep{rainforth2018tighter}. To mitigate this problem,  reparameterized \citep{kingma2014autoencoding, burda2015importance} and doubly-reparameterized gradient estimators \citep{Tucker2019DoublyRG} have been proposed.

Crucially, stochastic gradient schemes on the IWAE bound (and hence on the ELBO) only resort to unbiased estimators in both the reparameterized \citep{kingma2014autoencoding, burda2015importance} and the doubly-reparameterized \citep{Tucker2019DoublyRG} cases, providing theoretical justifications behind those approaches. In particular, under the assumption that $z$ can be reparameterized (that is $z = f(\varepsilon, \phi; x) \sim q_\phi(\cdot|x)$ where $\varepsilon \sim q$) and under common differentiability assumptions, the reparameterized gradient w.r.t. $\phi$ of the IWAE bound is given by
\begin{align*}%\label{reparam:IWAE}
  \frac{\partial}{\partial \phi} \liwae  (\theta,\phi; x)  =  \int \int \prod_{i = 1}^N q(\varepsilon_i) \lr{\sum_{j = 1}^N  \frac{\w(z_j; x)}{\sum_{k = 1}^N \w(z_k; x)}  \frac{\partial}{\partial \phi} \log \w(f(\varepsilon_j, \phi; x); x)}   \rmd \varepsilon_{1:N}
\end{align*} 
and the doubly-reparameterized one by
\begin{multline} \label{doublyreparam:IWAE}
  \frac{\partial}{\partial \phi} \liwae  (\theta,\phi; x) \\ = \int \int \prod_{i = 1}^N q(\varepsilon_i) \lr{\sum_{j=1}^N ~\lr{\frac{\w(z_j;x) }{\sum_{k = 1}^N \w(z_k;x)}}^2 \frac{\partial}{\partial \phi} \log \w[\theta][\phi'](f(\varepsilon_j, \phi;x);x)|_{\phi' = \phi}}  \rmd \varepsilon_{1:N}.
  \end{multline}
Unbiased Monte Carlo estimators of both gradients are hence respectively given by 
\begin{align}\label{reparam:IWAE:MC}
  \sum_{j = 1}^N  \frac{\w(z_j; x) }{\sum_{k = 1}^N \w(z_k; x)}  \frac{\partial}{\partial \phi} \log \w(f(\varepsilon_j, \phi;x);x)
\end{align} 
and
\begin{align*} %\label{doublyreparam:IWAE:MC}
\sum_{j=1}^N ~\lr{\frac{\w(z_j;x)}{\sum_{k = 1}^N \w(z_k;x)}}^2 \frac{\partial}{\partial \phi} \log \w[\theta][\phi'](f(\varepsilon_j, \phi;x);x)|_{\phi' = \phi},
\end{align*}
with $\varepsilon_1, \ldots, \varepsilon_N$ being i.i.d. samples generated from $q$ and $z_j = f(\varepsilon_j, \phi;x)$ for all $j = 1 \ldots N$. \cite{maddison2017} and \cite{domke2018} in particular established that the variational gap - that is the difference between the IWAE bound and the marginal log-likelihood - goes to zero at a fast $1/N$ rate when the dimension of the latent space $d$ is fixed and the number of samples $N$ goes to infinity.

Another example of variational bound is the Variational Rényi (VR) bound introduced by \cite{li2016renyi}: it is defined for all $\alpha \in \rset \setminus \lrcb{1}$ by
\begin{align}\label{eq:DefiVRBound}
\mathcal{L}^{(\alpha)}(\theta, \phi; x) = \frac{1}{1-\alpha} \log \lr{\int q_\phi(z|x)~\w(z;x)^{1-\alpha}~\rmd z}
\end{align}
and it generalizes the ELBO \citep[which corresponds to the extension by continuity of the VR bound to the case $\alpha = 1$, see][Theorem 1]{li2016renyi}. It is also a lower (resp. upper) bound on the marginal log-likelihood $\ell(\theta; x)$ for all $\alpha > 0$ (resp. $\alpha < 0$).

In the spirit of the IWAE bound optimisation framework, the VR bound is used for variational inference purposes in \cite[Section 4.1, 4.2 and 5.2]{li2016renyi} to optimise the marginal log-likelihood $\ell(\theta, \data)$ defined in \eqref{eq:def:loglik} by considering the global objective function
$$
\sum_{i=1}^T \mathcal{L}^{(\alpha)}(\theta, \phi; x_i)
$$
and by performing stochastic gradient descent steps w.r.t. $(\theta, \phi)$ on it paired up with mini-batching and reparameterization. This VR bound methodology has provided positive empirical results compared to the usual case $\alpha = 1$ and has been widely adopted in the literature \cite[]{li2016renyi,Bui2016BlackboxF, hernandez2016black, li2017dropout, pmlr-v130-zhang21o, RODRIGUEZSANTANA2022260}. As discussed in the remark below, this methodology is obviously not limited to the choice of posterior density defined in \eqref{eq:posteriorOne} and is more broadly applicable.
\begin{rem}[Black-box Alpha energy function] \label{rem:BBalphaEnergyFunc}
  Let $p_0(z)$ be a prior on a latent variable $z$ valued in $\rset^d$ and by $p(x|z)$ the likelihood of the observation $x$ given $z$, we might consider the posterior density
  \begin{align} \label{eq:modelRenyi}
  p(z|\data) \propto p_0(z) \prod_{i=1}^T p(x_i| z),
  \end{align} 
leading to the marginal log-likelihood   
$$
\tilde{\ell}(\data) = \log \lr{\int p(\data, z) \rmd z} = \log \lr{ p_0(z) \prod_{i=1}^T p(x_i| z) \rmd z}. 
$$
Here, the latent variable $z$ valued in $\rset^d$ is shared across all the observations. Now further assume that the prior density $p_0(z) = \exp(s(z)^T \phi_0 - \log Z(\phi_0))$ has an exponential form, with $\phi_0$ and $s$ being the natural parameters and the sufficient statistics respectively and $Z(\phi_0)$ being the normalizing constant ensuring that $p_0$ is a probability density function. 

In order to find the best approximation to the posterior density \eqref{eq:modelRenyi}, \cite{hernandez2016black} offers to minimize the Black-Box Alpha (BB-$\alpha$) energy function, which is defined by: for all $\alpha \in \rset \setminus \lrcb{1}$,
\begin{align*} %\label{eq:PowerEPobjfunc}
    \mathcal{E}(\phi) = \log Z(\phi_0) - \log Z(\tilde{\phi}) - \frac{1}{1-\alpha} \sum_{i = 1}^T \log \lr{ \int q_\phi(z) \lr{\frac{p(x_i|z)}{f_\phi(z)}}^{1-\alpha}  \rmd z }
\end{align*}
where $f_\phi(z) = \exp(s(z)^T \phi)$ is within the same exponential family as the prior and $q_\phi(z) = \exp(s(z)^T \tilde{\phi} - \log Z(\tilde{\phi}))$ with $\tilde{\phi} = T \phi + \phi_0$ denoting the natural parameters of $q_\phi$ and $Z(\tilde{\phi})$ its normalizing constant. Here, the minimisation is carried out via stochastic gradient descent w.r.t. $\phi$ combined with mini-batching and reparameterization.

As observed in \cite{li2017dropout}, minimizing $\mathcal{E}(\phi)$ w.r.t. $\phi$ is equivalent to maximizing the sum of VR bounds
  $$
  \sum_{i=1}^T \frac{T}{1-\alpha} \log \lr{ \int q_\phi(z) w_{\theta, \phi}(z; x)^{\frac{1-\alpha}{T}}  \rmd z }
  $$
w.r.t. $\phi$, where this time $w_{\theta, \phi}(z; x) = {p(x_i|z)^T p_0(z)}/{q_\phi(z)}$.
\end{rem}
However, the stochastic gradient descent scheme originating from having selected the VR bound as the objective function suffers from one important shortcoming: it relies on biased gradient estimators for all $\alpha \notin \{0, 1\}$, meaning that there exists no convergence guarantees for the whole scheme. Indeed, \cite{li2016renyi} show that the gradient of the VR bound w.r.t. $\phi$ satisfies
  \begin{align*} 
    \frac{\partial}{\partial \phi} \mathcal{L}^{(\alpha)}(\theta, \phi; x)  & = \frac{\int q(\varepsilon)~\w(z;x)^{1-\alpha}~  \frac{\partial}{\partial \phi} \log \w(f(\varepsilon, \phi;x);x) ~\rmd \varepsilon}{\int q(\varepsilon)~\w(z;x)^{1-\alpha}~\rmd \varepsilon},
  \end{align*} 
  with $z = f(\varepsilon, \phi;x) \sim q_\phi(\cdot|x)$ where $\varepsilon \sim q$. The gradient above being intractable, they approximate it using
  \begin{align}\label{eq:reparamVR} 
    \sum_{j = 1}^N  \frac{\w(z_j; x)^{1-\alpha} }{\sum_{k = 1}^N \w(z_k; x)^{1-\alpha}} \frac{\partial}{\partial \phi} \log \w(f(\varepsilon_j, \phi;x);x),
  \end{align} 
  where $\varepsilon_1, \ldots, \varepsilon_N$ are i.i.d. samples generated from $q$ and $z_j = f(\varepsilon_j, \phi;x)$ for all $j = 1 \ldots N$. The cases $\alpha = 0$ and $\alpha = 1$ recover the stochastic reparameterized gradients of the IWAE bound \eqref{reparam:IWAE:MC} and of the ELBO (consider \eqref{reparam:IWAE:MC} with $N = 1$). As a result, we can trace them back to unbiased stochastic gradient descent schemes for IWAE bound and ELBO optimisation respectively. Yet, this is no longer the case when $\alpha \notin \{0, 1 \}$, hence impeding the theoretical guarantees of the scheme. 
  
  In addition, due to the log function, the VR bound itself can only be approximated using biased Monte Carlo estimators, with \citep[Section 4.1]{li2016renyi} using 
  \begin{align} \label{eq:biasedVRbound}
  \frac{1}{1-\alpha} \log \lr{ \frac{1}{N} \sum_{j = 1}^N \w(Z_j;x)^{1-\alpha} }
  \end{align}
  where $Z_1, \ldots, Z_N$ are i.i.d. samples generated from $q_\phi$. Furthermore, while the VR bound and the IWAE bound approaches are linked via the gradient estimator \eqref{eq:reparamVR}, the VR bound does not recover the IWAE bound when $\alpha = 0$.

  The next section aims at overcoming the theoretical difficulties regarding the VR bound mentioned above.
  
\section{The VR-IWAE bound}
\label{sec:VR-IWAEbound}

For all $\alpha \in \rset \setminus \lrcb{1}$, let us introduce the quantity
\begin{align}\label{eq:liren:bound}
\liren(\theta, \phi; x) \eqdef \frac{1}{1-\alpha} \int \int \prod_{i = 1}^N q_\phi (z_i|x) \log \lr{ \frac{1}{N} \sum_{j = 1}^N \w(z_j;x) ^{1-\alpha} } \rmd z_{1:N},
\end{align}
which we will refer to as the \textit{VR-IWAE bound}. Note that
 for the VR-IWAE bound to be well-defined we will assume that the following assumption holds in the rest of the paper.
\begin{hyp}{A} 
\item \label{hyp:VRIWAEwell-defined} It holds that $0 < p_\theta(x) < \infty$ and the support of $q_\phi(\cdot|x)$ and of $p_\theta(\cdot|x)$ are equal.
\end{hyp}
We may omit the dependency on $x$ in $z \mapsto q_\phi(z|x)$ and $z \mapsto \w(z;x)$ for notational convenience and we now make two remarks regarding the VR-IWAE bound defined in \eqref{eq:liren:bound}.

\begin{itemize}
  \item Contrary to the VR bound, VR-IWAE bound (i) can be approximated using an unbiased Monte Carlo estimator and (ii) recovers the IWAE bound by setting $\alpha = 0$. Under common differentiability assumptions, we also have that
  $$
  \lim_{\alpha \to 1} \liren  (\theta,\phi;x) = \mathrm{ELBO}(\theta, \phi;x)
  $$
  (see \Cref{subsec:extAlpha1} for details), meaning that the VR-IWAE bound interpolates between the IWAE bound and the ELBO.
  \item  \cite{li2016renyi} interpreted the quantity defined in \eqref{eq:liren:bound} as the expectation of the biased Monte Carlo approximation of the VR bound \eqref{eq:biasedVRbound}. They established in \cite[Theorem 2]{li2016renyi} some properties on this quantity. In particular, they showed that (i) for all $\alpha \leq 1$ and all $N \in \mathbb{N}^\star$,
$$
\liren(\theta, \phi;x) \leq \liren[\alpha][N+1](\theta, \phi;x) \leq \mathcal{L}^{(\alpha)}(\theta,\phi;x)
$$
and (ii) for all $\alpha \in \rset$, $\liren(\theta, \phi;x)$ approaches the VR bound $\mathcal{L}^{(\alpha)}(\theta,\phi;x)$ as $N$ goes to infinity if the function $z \mapsto \w(z)$ is assumed to be bounded. 

\end{itemize}
Based on the two previous remarks, $\liren (\theta, \phi;x)$ seems to be an interesting candidate as a variational bound which generalizes the IWAE bound. We take here another perspective on the quantity $\liren (\theta, \phi;x)$ by wanting to frame it as a variational bound with ties to the Rényi's $\alpha$-divergence variational inference methodology of \cite{li2016renyi} and to the IWAE bound, hence the name VR-IWAE bound. We now need to check that the VR-IWAE bound can indeed be used as a variational bound for marginal log-likelihood optimisation in the context of our running example. 

\subsection{The VR-IWAE bound as a variational bound}
\label{subsec:debiasing}

As underlined in the following proposition, the VR-IWAE bound is a variational bound for all $\alpha \in [0,1)$ which enjoys properties akin to those obtained for the IWAE bound \cite[]{burda2015importance} and which becomes looser as $\alpha$ increases.

\begin{prop}[Properties of the VR-IWAE bound]\label{lem:generalBound} The following properties hold for the VR-IWAE bound.
\begin{enumerate}
  \item For all $\alpha \in [0,1)$ and all $N \in \mathbb{N}^\star$,
  \begin{align}\label{eq:ineqELBOLnAlpha}
  \mathrm{ELBO} (\theta, \phi; x) \leq \liren(\theta, \phi;x) \leq \liren[\alpha][N+1](\theta, \phi;x) \leq \mathcal{L}^{(\alpha)}(\theta, \phi; x) \leq \ell(\theta;x).
  \end{align}
  %where the case of equality is reached if and only if $p_\theta(z|x) = q_\phi(z)$ for $\nu$-almost all $z \in \rset^d$ (with $\nu$ denoting the Lebesgue measure). 
  %In particular, the case $N = 1$ recovers the Evidence Lower BOund (ELBO). 
  Moreover, if the function $z \mapsto \w(z)$ is bounded, then $\liren(\theta, \phi;x)$ approaches the VR bound $\mathcal{L}^{(\alpha)}(\theta,\phi;x)$ as $N$ goes to infinity.
  \item For all $\alpha_1, \alpha_2 \in (0,1)$ such that $\alpha_1 > \alpha_2$ and all $N \in \mathbb{N}^\star$,
  \begin{align}\label{eq:looseness}
  \liren[\alpha_1](\theta, \phi;x) \leq \liren[\alpha_2](\theta, \phi;x) \leq \liwae (\theta, \phi;x),
  \end{align}
  where the case of equality is reached if and only if $z \mapsto \w(z)$ is constant for $\nu$-almost all $z \in \rset^d$ (with $\nu$ denoting the Lebesgue measure).
  \item Further assuming that $z$ can be reparameterized, that is $z = f(\varepsilon, \phi) \sim q_\phi$ where $\varepsilon \sim q$, we have under common differentiability assumptions that
    \begin{multline}\label{reparam:liren}
      \frac{\partial}{\partial \phi} \liren  (\theta,\phi;x) \\ =  \int \int \prod_{i = 1}^N q(\varepsilon_i) \lr{\sum_{j = 1}^N  \frac{\w(z_j)^{1-\alpha} }{\sum_{k = 1}^N \w(z_k)^{1-\alpha}}  \frac{\partial}{\partial \phi} \log \w(f(\varepsilon_j, \phi))}   \rmd \varepsilon_{1:N}
    \end{multline} 
  and an unbiased estimator of $\partial \liren (\theta,\phi;x) / {\partial \phi} $ is given by
  \begin{align}\label{reparam:liren:MC}
    \sum_{j = 1}^N  \frac{\w(z_j)^{1-\alpha} }{\sum_{k = 1}^N \w(z_k)^{1-\alpha}}  \frac{\partial }{\partial \phi}\log \w(f(\varepsilon_j, \phi))
  \end{align} 
  where $\varepsilon_1, \ldots, \varepsilon_N$ are i.i.d. samples generated from $q$ and $z_j = f(\varepsilon_j, \phi)$ for all $j = 1 \ldots N$. \looseness = -1
\end{enumerate}
\end{prop}
The proof of \Cref{lem:generalBound} is deferred to \Cref{sec:prooflemgeneralbound} and we now comment on \Cref{lem:generalBound}. 
Observe that both the VR and VR-IWAE bounds share the same estimated reparameterized gradient w.r.t. $\phi$, that is \eqref{eq:reparamVR} is exactly \eqref{reparam:liren:MC}, hence they lead to the same stochastic gradient descent algorithm. However, when $\alpha \in (0,1)$, a key difference is that in the VR bound case this estimator is biased while it is unbiased for VR-IWAE bound. 

This motivates the VR-IWAE bound as a generalization of the IWAE bound that overcomes the theoretical difficulties of the VR bound, as unbiased gradient estimates provide the convergence of the stochastic gradient descent procedure (under proper conditions on the learning rate). In fact, and as we shall see next, the estimated reparameterized gradient w.r.t. $\phi$ written in \eqref{eq:reparamVR} and \eqref{reparam:liren:MC} - that we have now properly justified using the VR-IWAE bound - also enjoys an advantageous Signal-to-Noise Ratio behavior when $\alpha \in (0,1)$.%, which is in contrast with the IWAE bound case ($\alpha = 0$) studied in \cite{rainforth2018tighter}.

\subsection{Signal-to-Noise Ratio (SNR) analysis}
\label{subsecSNRanalysis}

\cite{rainforth2018tighter} identified some issues associated to using reparameterized gradient estimators of the IWAE bound. They did so by looking at the Signal-to-Noise Ratio (SNR) of those estimates: their main theorem \cite[Theorem 1]{rainforth2018tighter} shows that while increasing $N$ leads to a tighter IWAE bound and improves the SNR for learning $\theta$, it actually worsens the SNR for learning $\phi$. %for unbiased estimates of the reparameterized gradient of the IWAE bound w.r.t. each component of $\theta$ and of $\phi$

Let us now investigate if and how the conclusions of \cite[Theorem 1]{rainforth2018tighter} extend to the VR-IWAE bound. To this end, we first recall the definition of the SNR used in \cite{rainforth2018tighter}. Given a random vector $X = (X_1 , \ldots, X_L)$ of dimension $L \in \mathbb{N}^\star$, the SNR may be defined as follows: %, a quantity which can be used as a measure of an estimator's quality
$$
\mathrm{SNR}[X] = \lr{ \frac{|\PE(X_1)|}{\sqrt{\mathbb{V}(X_1)}}, \ldots,  \frac{|\PE(X_L)|}{\sqrt{\mathbb{V}(X_L)}}}.%\frac{\|\PE(X)\|^2}{\|\PE(X)\|^2 + \mathbb{V}(X)} = \frac{1}{1+ \frac{\mathbb{V}(X)}{\|\PE(X)\|^2}}.
$$
Writing $\theta = (\theta_1, \ldots, \theta_L)$ and $\phi = (\phi_1, \ldots, \phi_{L'})$ and with $L, L' \in \mathbb{N}^\star$, we now consider for all $\ell = 1 \ldots L$ and all $\ell' = 1 \ldots L'$ the unbiased estimates of the reparameterized gradient of the VR-IWAE bound w.r.t. $\theta_\ell$ and w.r.t. $\phi_{\ell'}$ given by: for all $M, N \in \mathbb{N}^\star$ and all $\alpha \in [0,1)$, 
\begin{align}
  & \delta_{M, N}^{(\alpha)}(\theta_\ell) = \frac{1}{(1-\alpha)M} \sum_{m=1}^M \frac{\partial}{\partial \theta_\ell} \log \lr{\frac{1}{N} \sum_{j=1}^N \w(f(\varepsilon_{m,j}, \phi))^{1-\alpha}}, \label{eq:deltaMNthetaell} \\
  & \delta_{M, N}^{(\alpha)}(\phi_{\ell'}) = \frac{1}{(1-\alpha)M} \sum_{m=1}^M \frac{\partial}{\partial \phi_{\ell'}} \log \lr{\frac{1}{N} \sum_{j=1}^N \w(f(\varepsilon_{m,j}, \phi))^{1-\alpha}}, \label{eq:deltaMNphiell}
 \end{align}
where $(\varepsilon_{m,j})_{1 \leq m \leq M, 1 \leq j \leq N}$ are i.i.d. samples generated from $q$ and $z_{m,j} = f(\varepsilon_{m, j}, \phi)$ for all $m = 1 \ldots M$ and all $n = 1 \ldots N$. Note that the link with the reparameterized gradient estimator \eqref{reparam:liren:MC} from \Cref{lem:generalBound} can be made by considering the case $M = 1$ in \eqref{eq:deltaMNphiell}. % (resp. $\delta_{M, N}^{(0)}(\theta_{\ell'})$).
We then have the following theorem.

\begin{thm}[SNR analysis] \label{prop:SNRconvergence} Let $\alpha \in [0,1)$ and for all $N \in \mathbb{N}^\star$ and all $j = 1 \ldots N$, define $\tilde{w}_{1,j} = \w(f(\varepsilon_{1,j}, \phi))$ and $\hat{Z}_{1, N, \alpha} = N^{-1} \sum_{j=1}^N \tilde{w}_{1,j}^{1-\alpha}$. Assume that the eighth moments of $\tilde{w}_{1,1}^{1-\alpha}$, ${\partial} \tilde{w}_{1,1}^{1-\alpha} / {\partial \theta_\ell}$ and ${\partial} \tilde{w}_{1,1}^{1-\alpha} / {\partial \phi_{\ell'}}$ are finite, where $\ell$ is an integer between $1$ and $L$ and $\ell'$ is an integer between $1$ and $L'$. Furthermore, assume that there exists some $N \in \mathbb{N}^\star$ for which $\PE((1/\hat{Z}_{1, N, \alpha})^4) < \infty$. Lastly, assume that $\partial \PE(\tilde{w}_{1,1}^{1-\alpha}) / \partial \theta_\ell \neq 0$ and that
  \begin{align}
    & {\partial} \mathbb{V}(\tilde{w}_{1,1}^{1-\alpha}) / {\partial \phi_{\ell'}} > 0, \quad \mbox{if $\alpha = 0$} \nonumber \\
    & \partial \PE(\tilde{w}_{1,1}^{1-\alpha}) / \partial \phi_{\ell'} \neq 0 , \quad \mbox{if $\alpha \in (0,1)$}. \label{eq:leadingOrder}
  \end{align}
  Then, under common differentiability assumptions, the SNR of the VR-IWAE bound reparameterized gradient estimates w.r.t $\theta_\ell$ and w.r.t $\phi_{\ell'}$ defined in \eqref{eq:deltaMNthetaell} and \eqref{eq:deltaMNphiell} respectively satisfy
  \begin{align}
    & \mathrm{SNR}[\delta_{M, N}^{(\alpha)}(\theta_\ell)] =  \Theta(\sqrt{MN}) \label{eq:SNRbehavior1} \\
    & \mathrm{SNR}[\delta_{M, N}^{(\alpha)}(\phi_{\ell'})] = \begin{cases}
      \Theta(\sqrt{M/N}) & \mbox{if $\alpha = 0$}, \\
      \Theta(\sqrt{MN}) & \mbox{if $\alpha \in (0,1)$}.
    \end{cases} \label{eq:SNRbehavior2}
  \end{align}
\end{thm}
The proof of \Cref{prop:SNRconvergence} can be found in \Cref{subsec:SNRanalysis:app}. \Cref{prop:SNRconvergence} states that for $\alpha \in (0,1)$, the SNR for learning the generative network ($\theta$) and for learning the inference network ($\phi$) both improve as $N$ increases, unlike the IWAE bound case $\alpha = 0$ where the second SNR worsens as $N$ increases. This provides theoretical support suggesting that taking $\alpha > 0$ in the VR-IWAE bound may help to ensure a good training signal, thus leading to improved empirical performances compared to the IWAE bound.

In the following, we investigate another way to provide gradient estimators of the VR-IWAE bound with an advantageous SNR behavior in practice. 
\subsection{Doubly-reparameterized gradient for the VR-IWAE bound}

To remedy the SNR issue identified in \cite{rainforth2018tighter}, \cite{Tucker2019DoublyRG} proposed a new estimator of the gradient of the IWAE bound \eqref{doublyreparam:IWAE} under the name doubly-reparameterized gradient estimator. As written in the theorem below, the doubly-reparameterized gradient estimator of the IWAE bound \eqref{doublyreparam:IWAE} in fact generalizes to the case $\alpha \in (0,1)$.

\begin{thm}[Generalized doubly-reparameterized gradient] \label{prop:drepIwaeAlpha} Under common differentiability assumptions and assuming that $z$ can be reparameterized, that is $z = f(\varepsilon, \phi) \sim q_\phi$ where $\varepsilon \sim q$, we have that: for all $\alpha \in [0,1]$,
  \begin{align} \label{eq:partialVRIWAEdrep}
    \frac{\partial}{\partial \phi} \liren  (\theta,\phi;x) = \int \int \prod_{i = 1}^N q(\varepsilon_i) \lr{\sum_{j=1}^N h_j(\alpha) \frac{\partial}{\partial \phi} \log \w[\theta][\phi'](f(\varepsilon_j, \phi))|_{\phi' = \phi}}  \rmd \varepsilon_{1:N},
    \end{align}
    with $z_j = f(\varepsilon_j, \phi)$ for all $j = 1 \ldots N$ and 
    $$
    h_j(\alpha) = \alpha ~ \frac{\w(z_j)^{1-\alpha} }{\sum_{k = 1}^N \w(z_k)^{1-\alpha}} + (1-\alpha)~\lr{\frac{\w(z_j)^{1-\alpha} }{\sum_{k = 1}^N \w(z_k)^{1-\alpha}}}^2.
    $$
  An unbiased estimator of ${\partial \liren  (\theta,\phi;x)}/{\partial \phi} $ is then given by
  \begin{align} \label{dreparam:liren:MC}
  {\sum_{j=1}^N h_j(\alpha) \frac{\partial}{\partial \phi} \log \w[\theta][\phi'](f(\varepsilon_j, \phi))|_{\phi' = \phi}} 
  \end{align}
  where $\varepsilon_1, \ldots, \varepsilon_N$ are i.i.d. samples generated from $q$ and $z_j = f(\varepsilon_j, \phi)$ for all $j = 1 \ldots N$.
   % \begin{multline*}
   %   \frac{\partial}{\partial \phi} \liren  (\theta,\phi) = \\ \int \int \prod_{i = 1}^N q(\varepsilon_i) \lr{\sum_{j=1}^N \lrb{\alpha ~ \frac{\w(f(\varepsilon_j, \phi))^{1-\alpha} }{\sum_{k = 1}^N \w(f(\varepsilon_k, \phi))^{1-\alpha}} + (1-\alpha) ~ \lr{\frac{\w(f(\varepsilon_j, \phi))^{1-\alpha} }{\sum_{k = 1}^N \w(f(\varepsilon_k, \phi))^{1-\alpha}}}^2} \frac{\partial}{\partial \phi} \log \w[\theta][\phi'](f(\varepsilon_j, \phi))|_{\phi' = \phi}}  \rmd \varepsilon_{1:N},
   %   \end{multline*}
\end{thm}
The proof of \Cref{prop:drepIwaeAlpha} is deferred to \Cref{subsec:proof:drepIwaeAlpha}. One can then check that we recover the usual doubly-reparameterized gradient estimator of the IWAE bound (resp. the ELBO) when $\alpha = 0$ (resp. $\alpha = 1$). Like the reparameterized  gradient estimator \eqref{reparam:liren:MC}, which we have studied in \Cref{subsecSNRanalysis} as it corresponds to the special case $M = 1$ in \Cref{prop:SNRconvergence}, this second (doubly-reparameterized) gradient estimator too may lead to improved empirical performances. From there, large-scale learning occurs by using that \Cref{lem:generalBound} implies
  \begin{align*} %\label{eq:LinkToRenyi}
  \ell(\theta; \data) = \sum_{i = 1}^T \ell(\theta; x_i) \geq \sum_{i = 1}^T  \mathcal{L}^{(\alpha)}(\theta, \phi;x_i) \geq \sum_{i=1}^T \liren (\theta,\phi; x_i)
  \end{align*}
  and by following the training procedure for the IWAE bound. Indeed, we have access to an unbiased estimator of the lower bound of the full dataset $\sum_{i=1}^T \liren (\theta,\phi; x_i)$ (as well as an unbiased estimator of its reparameterized/doubly-reparameterized gradient) using mini-batching. Seeking to maximize the objective function $\sum_{i = 1}^T  \mathcal{L}^{(\alpha)}(\theta, \phi;x_i)$ by optimising $\sum_{i=1}^T \liren (\theta,\phi; x_i)$ in fact amounts to seeking to minimize a specific Rényi's $\alpha$-divergence with a mean-field assumption on the variational approximation (see \Cref{rem:objfuncRenyi} for detail).
  
  \begin{rem} \label{rem:objfuncRenyi} Define $\boldsymbol{z} = (z_1, \ldots, z_T )$ and $q_\phi(\boldsymbol{z}) = \prod_{i=1}^T q_\phi(z_i)$. %$p_\theta(\data, \boldsymbol{z}) = \prod_{i=1}^T p(x_i, z_i)$  and $\w(\boldsymbol{z}; \data) = p_\theta(\data, \boldsymbol{z}) / q_\phi(\boldsymbol{z})$, 
   Then, for all $\alpha \in (0,1)$: %By maximizing $\sum_{i = 1}^T \liren (\theta,\phi; x_i)$, we are also trying to maximise $\sum_{i = 1}^T  \mathcal{L}^{(\alpha)}(\theta, \phi;x_i)$, where : 
  \begin{align*}
    \sum_{i = 1}^T  \mathcal{L}^{(\alpha)}(\theta, \phi;x_i) & = \sum_{i=1}^T \frac{1}{1-\alpha} \log\lr{  \int q_\phi(z) \w(z; x_i)^{1-\alpha} \rmd z} \\
    & = \sum_{i=1}^T \frac{1}{1-\alpha} \log\lr{  \int q_\phi(z_i) \w(z_i; x_i)^{1-\alpha} \rmd z_i} \\
    & = \frac{1}{1-\alpha} \log\lr{ \int \int \prod_{i=1}^T q_\phi(z_i) \prod_{j=1}^T \w(z_j; x_j)^{1-\alpha} \rmd z_{1:T}} \\
    & = \frac{1}{1-\alpha}  \log\lr{ \int q_\phi(\boldsymbol{z}) \w(\boldsymbol{z}; \data)^{1-\alpha} \rmd \boldsymbol{z}}
  \end{align*}
  where $p_\theta(\data, \boldsymbol{z}) = \prod_{i=1}^T p(x_i, z_i)$  and $\w(\boldsymbol{z}; \data) = p_\theta(\data, \boldsymbol{z}) / q_\phi(\boldsymbol{z})$. Observe that the last equality is a VR bound, meaning that maximizing the global objective function $\sum_{i = 1}^T  \mathcal{L}^{(\alpha)}(\theta, \phi;x_i)$ is equivalent to minimizing the Rényi's $\alpha$-divergence between the two probability distributions with associated probability densities $q_\phi(\boldsymbol{z})$ and $p(\boldsymbol{z}|\data)$ respectively w.r.t. the Lebesgue measure. Hence, this approach belongs to alpha-divergence variational inference methods with the particularity that it makes a mean-field assumption on the variational approximation $q_\phi(\boldsymbol{z})$.
\end{rem}
At this stage, we have formalized and motivated the VR-IWAE bound. We now want to get an understanding of its theoretical properties.
\section{Theoretical study of the VR-IWAE bound}

\label{sec:HighDim}

The starting point of our approach is to exploit the fact that prior theoretical works study the particular case $\alpha = 0$ (corresponding to the IWAE bound) when the dimension of the latent space $\dim(z)=d$ is fixed and the number of samples $N$ goes to infinity. 

\subsection{Behavior of the VR-IWAE bound when $d$ is fixed and $N$ goes to infinity}
\label{subsec:behaviourfixedd}

A quantity that has been of interest to assess the quality of the IWAE bound is the \textit{variational gap}, which is defined as the difference between the IWAE bound and the marginal log-likelihood:
\begin{align}
\Delta_N(\theta,\phi; x):=\liwae (\theta,\phi;x) -\ell(\theta;x)=\int\int\prod_{i=1}^{N}q_{\phi}(z_{i})\log\left(\frac{1}{N}\sum_{j=1}^{N}\barw(z_{j})\right) \rmd z_{1:N}\label{eq:variationalgap}
\end{align}
where for all $z \in \rset^d$
\begin{align*} %\label{eq:defRelativeWeights}
\barw(z) & \eqdef \frac{\w(z)}{\mathbb{E}_{Z\sim q_{\phi}}\left(\w(Z)\right)}=\frac{\w(z)}{p_{\theta}(x)},
\end{align*}
so that $\barw(z_{1}), \ldots, \barw(z_{N})$ correspond to the relative weights. The analysis of the variational gap \eqref{eq:variationalgap}, first performed in \cite{maddison2017} and then refined in \cite{domke2018}, investigated the case where $\dim(z)=d$ is fixed and $N$ goes to infinity. 
Informally, they obtained in their Theorem 3 that the variational gap behaves as follows
\begin{align*}
\Delta_N(\theta,\phi;x)=-\frac{\gamma_0^{2}}{2N}+o\left(\frac{1}{N}\right) % \label{eq:Variationalgapclassic}
\end{align*}
with $\gamma_0$ denoting the variance of the relative weights, that is
\begin{equation*}
    \gamma_0^2:=\mathbb{V}_{Z \sim q_\phi}(\barw(Z)).
\end{equation*}
This result suggests that using $N$ is very beneficial to reduce the variational gap, as it goes to zero at a fast $1/N$ rate. It motivates a study - in a regime where $d$ is fixed and $N$ goes to infinity - of the more general variational gap defined for all $\alpha \in [0,1)$ by
\begin{align*}
  \Delta_{N}^{(\alpha)}(\theta, \phi; x) & \eqdef \liren  (\theta,\phi;x) - \ell(\theta;x).
\end{align*}
The following result generalizes \cite[Theorem 3]{domke2018} to the VR-IWAE bound.

\begin{thm}\label{prop:GenDomke}
Let $\alpha \in [0,1)$. Then, it holds that
\begin{align} \label{eq:boundedExpectationWeights}
0 < \PE_{Z \sim q_\phi}(\w(Z)^{1-\alpha}) < \infty.
\end{align}
Further assume that there exists $\beta > 0$ such that
\begin{align}\label{eq:conditionVariance}
\PE_{Z \sim q_\phi}(|\Rtalpha(Z)  - 1|^{2+\beta}) < \infty,
\end{align}
where we have defined $\Rtalpha(z) = {\w(z)^{1-\alpha}}/{\PE_{Z \sim q_\phi}(\w(Z)^{1-\alpha})}$ for all $z \in \rset^d$. Lastly, assume that the following condition holds
\begin{align}\label{eq:conditionLimsup}
\limsup_{N \to \infty} \PE(1/\RalphaN) < \infty,
\end{align}
where, for all $N \in \mathbb{N}^\star$, $\RalphaN = N^{-1} \sum_{i=1}^N \w(Z_i)^{1-\alpha}$ and $Z_1, \ldots, Z_N$ are i.i.d. samples generated according to $q_\phi$. Then, denoting $\gamma_\alpha^2 = (1-\alpha)^{-1}\mathbb{V}_{Z \sim q_\phi}(\Rtalpha(Z))$, we have:
\begin{align} \label{eq:OneOverNGenDomke}
\Delta_{N}^{(\alpha)}(\theta, \phi; x) = \mathcal{L}^{(\alpha)}(\theta, \phi; x) - \ell(\theta;x) - \frac{\gamma_{\alpha}^2}{2 N}+o\left(\frac{1}{N}\right).
\end{align}
\end{thm}
The proof of this result is deferred to \Cref{subsec:proofGenDomke} and we now aim at interpreting \Cref{prop:GenDomke}, starting with the conditions \eqref{eq:conditionVariance} and \eqref{eq:conditionLimsup}.

\subsubsection{Conditions (\ref{eq:conditionVariance}) and (\ref{eq:conditionLimsup})}

A first remark is that the conditions \eqref{eq:conditionVariance} and \eqref{eq:conditionLimsup} stated in \Cref{prop:GenDomke} exactly generalize the ones from \cite[Theorem 3]{domke2018}, which are recovered by setting $\alpha = 0$. This then prompt us to investigate in the following proposition how restrictive the conditions \eqref{eq:conditionVariance} and \eqref{eq:conditionLimsup} are as a function of $\alpha$. 

\begin{prop} \label{lem:discussConditions}
Let $\alpha_1, \alpha_2 \in [0,1)$ with $\alpha_1 > \alpha_2$. Then, the two following assertions hold.
\begin{enumerate}
  \item \label{discussConditionsitemOne} If \eqref{eq:conditionVariance} holds with $\alpha = \alpha_2$, then \eqref{eq:conditionVariance} holds with $\alpha = \alpha_1$.
  \item \label{discussConditionsitemTwo} If \eqref{eq:conditionLimsup} holds with $\alpha = \alpha_2$, then \eqref{eq:conditionLimsup} holds with $\alpha = \alpha_1$.
\end{enumerate}
\end{prop}
The proof of this result is deferred to \Cref{subsec:prooflem:discussConditions}. It notably relies the fact that the condition \eqref{eq:conditionLimsup} is equivalent to the statement that there exists some $N \in \mathbb{N}^\star$ for which $\PE(1/\RalphaN) < \infty$, which follows from \Cref{lem:equivlimsupconditionGen} in \Cref{subsec:SNRanalysis:app} with $k = 1$. Notice that this provides an interesting equivalent condition to \eqref{eq:conditionLimsup} that might be easier to check empirically.

\Cref{lem:discussConditions} then states that the conditions \eqref{eq:conditionVariance} and \eqref{eq:conditionLimsup} with $\alpha = \alpha_1$ are at worse as restrictive as the case $\alpha = \alpha_2$, where $\alpha_1 > \alpha_2$. 
Putting this into perspective with \cite{domke2018}, the conditions \eqref{eq:conditionVariance} and \eqref{eq:conditionLimsup} when $\alpha > 0$ are hence not more restrictice than the conditions presented in \cite[Theorem 3]{domke2018} for the more usual IWAE bound case $\alpha = 0$. In fact, one would even be inclined to think that those conditions become easier to satisfy as $\alpha$ increases, motivating once again the use of $\alpha \in (0,1)$ in practice to be in the conditions of application of \Cref{prop:GenDomke}.

\subsubsection{Interpreting (\ref{eq:OneOverNGenDomke})}

Under the assumptions of \Cref{prop:GenDomke}, \eqref{eq:OneOverNGenDomke} states: for all $\alpha \in [0,1)$, % that the variational gap $\Delta_{N}^{(\alpha)}(\theta, \phi; x)$ can be written as:
\begin{align*}
\Delta_{N}^{(\alpha)}(\theta, \phi; x) = \mathcal{L}^{(\alpha)}(\theta, \phi; x) - \ell(\theta;x) - \frac{\gamma_{\alpha}^2}{2 N}+o\left(\frac{1}{N}\right).
  \end{align*}
The variational gap $\Delta_{N}^{(\alpha)}(\theta, \phi; x)$ is hence composed of two main terms: 
\begin{itemize}
  \item A term going to zero at a $1/N$ rate that depends on $\gamma_\alpha^2$. Here $\gamma_\alpha^2$ is controlled thanks to \eqref{eq:conditionVariance}, as \eqref{eq:conditionVariance} implies that $\mathbb{V}_{Z \sim q_\phi}(\Rtalpha(Z)) < \infty$ or equivalently that $\gamma_\alpha^2 < \infty$.

  \item An error term $\mathcal{L}^{(\alpha)}(\theta, \phi; x) - \ell(\theta;x)$. This term decreases away from zero as $\alpha$ increases due to the fact that $\mathcal{L}^{(\alpha)}(\theta, \phi; x)$ decreases away from its upper bound $\ell(\theta;x)$ as $\alpha$ increases \cite[see for example][Theorem 1]{li2016renyi}. It is equal to zero when $\alpha = 0$ or when the posterior and the encoder distributions are equal to one another. 
  
  Unless $\alpha = 0$ or the posterior and encoder distributions are matching, the error term $\mathcal{L}^{(\alpha)}(\theta, \phi; x) - \ell(\theta;x)$ hence maintains a dependency in $(\theta,\phi)$ in the variational gap even as $N$ goes to infinity. This is coherent with \Cref{prop:SNRconvergence}, in the sense that the case $\alpha \in (0,1)$ might ensure a better learning of both $\theta$ and $\phi$ in practice compared to the case $\alpha = 0$ (as the latter does not keep a dependency in $\phi$ as $N$ goes to infinity). 
\end{itemize}
Since the error term $\mathcal{L}^{(\alpha)}(\theta, \phi; x) - \ell(\theta;x)$ is decreasing away from zero as $\alpha$ increases and the term going to zero at a $1/N$ rate depends on the behavior of $\gamma_\alpha^2$ (with $\gamma_\alpha^2$ going to $0$ as $\alpha$ goes to $1$, see \Cref{lem:gammaAlphaToZero} of \Cref{subsec:behaviorGammaAlpha}), there might then be a tradeoff to achieve when choosing $\alpha$ in order to obtain the best empirical performances.

To the best of our knowledge, and by appealing to the link between the VR-bound and the VR-IWAE bound methodologies established in \Cref{subsec:debiasing}, \Cref{prop:GenDomke} is the first result shedding light via \eqref{eq:OneOverNGenDomke} on how the quantity $\gamma_\alpha^2$ alongside with the error term $\mathcal{L}^{(\alpha)}(\theta, \phi; x) - \ell(\theta;x)$ may play a role to guarantee the success of gradient-based methods involving the VR-bound. 
While the result obtained in \Cref{prop:GenDomke} is encouraging and might further motivate the use of $\alpha \in (0,1)$ in practice, one may seek to identify potential limitations of \Cref{prop:GenDomke}.

\subsubsection{Limitations of \Cref{prop:GenDomke}}
\label{subsubsecLimitation}

To investigate the limitations of \Cref{prop:GenDomke}, let us provide below two insightful examples in which all the terms appearing in \eqref{eq:OneOverNGenDomke} are tractable.

\begin{ex}\label{lem:VarianceExponentialWithD} Let $\sigma > 0$, $S_1, \ldots, S_N$ be i.i.d. normal random variables and assume that the distribution of the relative weights $\barw(z_1), \ldots,$ $\barw(z_N)$ is log-normal of the form
  \begin{align} \label{eq:logNormalInLem}
  \log \barw(z_i)  ={-\frac{\sigma^2 d}{2}- \sigma \sqrt{d} {S_i}}, \quad i = 1 \ldots N,
  \end{align}  
where the relationship between mean and variance ensures that the relative weights have expectation~$1$. Then, we can apply \Cref{prop:GenDomke}: %and 
%Under \eqref{eq:logNormalInLem}, both assumptions \eqref{eq:conditionVariance} \eqref{eq:conditionLimsup} in  are satisfied. 
%Furthermore, 
for all $\alpha \in [0,1)$, 
\begin{align*} %\label{eq:OneOverNGenDomke}
  \Delta_{N}^{(\alpha)}(\theta, \phi; x) = \mathcal{L}^{(\alpha)}(\theta, \phi; x) - \ell(\theta;x) - \frac{\gamma_\alpha^2}{2 N}+o\left(\frac{1}{N}\right)
  \end{align*}
with
\begin{align*}
& \mathcal{L}^{(\alpha)}(\theta, \phi; x) - \ell(\theta;x) %&= \frac{(1-\alpha) \sigma^2 d}{2} -\frac{\sigma^2 d}{2} 
= - \frac{\alpha \sigma^2 d}{2} \quad \mbox{and} \quad \gamma_\alpha^2 = \frac{\exp\lrb{(1-\alpha)^2\sigma^2 d} - 1}{1-\alpha}.
\end{align*}
In particular, we can write the weights under the form \eqref{eq:logNormalInLem} with $\sigma = 1$ by setting $p_\theta(z|x) = \mathcal{N}(z;\theta, \boldsymbol{I}_d)$, $q_\phi(z|x) = \mathcal{N}(z; \phi, \boldsymbol{I}_d)$, $\theta = 0 \cdot \boldsymbol{u}_d$ and $\phi = \boldsymbol{u}_d$, where $\boldsymbol{I}_d$ is the $d$-dimensional identity matrix and $\boldsymbol{u}_d$ the $d$-dimensional vector whose coordinates are all equal to $1$.  
\end{ex}
The proof of \Cref{lem:VarianceExponentialWithD} is deferred to \Cref{subsec:lem:VarianceExponentialWithD} and we now comment on \Cref{lem:VarianceExponentialWithD}. A first comment is that as $\alpha$ increases, the error term $\mathcal{L}^{(\alpha)}(\theta, \phi; x) - \ell(\theta;x)$ worsens linearly with $\alpha$ while $\gamma_\alpha^2$ decreases with $\alpha$, which supports our claim that there might exist an optimal $\alpha$ that balances between the two terms appearing in the variational gap as a rule of thumb. 

Furthermore, the variance of the relative weights and more generally $\gamma_\alpha^2$ is exponential with $d$. This means that the analysis of \cite{domke2018} -- that we extended to $\alpha \in [0,1)$ in \Cref{prop:GenDomke} -- may not capture what is happening in some high-dimensional scenarios as we may never use $N$ large enough in high-dimensional settings for the asymptotic regime of \Cref{prop:GenDomke} to kick in. We now present our second example. 

\begin{ex} \label{ex:LinGaussThm3} We consider the linear Gaussian example from \cite{rainforth2018tighter}, that is $p_{\theta}(z)=\mathcal{N}(z;\theta, \boldsymbol{I}_d)$, $p_\theta(x|z)=\mathcal{N}(x;z, \boldsymbol{I}_d)$ with $\theta \in \rset^d$, and $q_\phi(z|x)=\mathcal{N}(z;Ax+b, 2/3 ~ \boldsymbol{I}_d)$ with $A=\mathrm{diag}(\tilde{a})$ and $\phi=(\tilde{a},b) \in \rset^d \times \rset^d$. Here, the optimal parameter values $(\theta^\star, \phi^\star)$ are given by $\theta^\star = T^{-1} \sum_{t = 1}^T x_t$ and $\phi^\star = (a^\star, b^\star)$ with $a^\star = {1}/{2} \boldsymbol{u}_d$ and $b^\star = {\theta^\star}/{2}$ \cite[see][Appendix B]{rainforth2018tighter}. Furthermore, the true marginal likelihood and true posterior density are given by $p_\theta(x)=\mathcal{N}(x;\theta, 2\boldsymbol{I}_d)$ and $p_\theta(z|x)=\mathcal{N}(z; (\theta+x)/2, 1/2 ~ \boldsymbol{I}_d)$ respectively. Then, we can apply \Cref{prop:GenDomke}: for all $\alpha \in [0,1)$, 
\begin{align*} 
\Delta^{(\alpha)}_{N}(\theta, \phi;x) =   \mathcal{L}^{(\alpha)}(\theta, \phi; x) - \ell(\theta; x) - \frac{\gamma_{\alpha}^2}{2N} +o\left(\frac{1}{N}\right), 
\end{align*}
with  
\begin{align*}
    &\mathcal{L}^{(\alpha)}(\theta, \phi; x) - \ell(\theta; x) =    \frac{d}{2} \lrb{\log \lr{\frac{4}{3}} + \frac{1}
    {1-\alpha} \log \lr{\frac{3}{4-\alpha}}} - \frac{3\alpha}{4-\alpha}\Big\|Ax+b-\frac{\theta+x}{2}\Big\|^2 \\
  &\gamma_{\alpha}^2 = \frac{1}{1-\alpha}\left[(4-\alpha)^{d}(15-6\alpha)^{-\frac{d}{2}}\exp\left(\frac{24(1-\alpha)^2}{(5-2\alpha)(4-\alpha)}\Big\|Ax+b-\frac{\theta+x}{2}\Big\|^2\right)-1\right].
\end{align*}
\end{ex}
The proof of \Cref{ex:LinGaussThm3} is deferred to \Cref{subsec:linGaussExApp2}. To interpret \Cref{ex:LinGaussThm3}, observe that the case of optimality is particularly telling in this example, since when $(\theta, \phi) = (\theta^\star, \phi^\star)$ it holds that 
$\gamma_{\alpha}^2 = (1-\alpha)^{-1}[(4-\alpha)^{d}(15-6\alpha)^{-{d}/{2}}-1]$ and $\gamma_\alpha^2$ is thus exponential in $d$ despite the parameters $(\theta, \phi)$ being optimal for the setting considered. 

Hence, and in line with our conclusions for \Cref{lem:VarianceExponentialWithD}, the relevance of \Cref{prop:GenDomke} can be limited for a high-dimensional latent space $d$. This calls for an in-depth study of the variational gap as both $d$ and $N$ go to infinity.
\subsection{Behavior of the VR-IWAE bound when both $d$ and $N$ go to infinity}
\label{subsec:VRIWAEndInfty}

To better capture what is happening to the VR-IWAE bound in high-dimensional scenarios, we now let $d,N \rightarrow \infty$ in the variational gap
\begin{align*}
  \Delta_{N,d}^{(\alpha)}(\theta, \phi; x) & \eqdef \liren[\alpha][N,d]  (\theta,\phi;x) - \ell_d(\theta;x),%\\
  %& = \frac{1}{1-\alpha} \int \int \prod_{i = 1}^N q_\phi (z_i) \log \lr{ \frac{1}{N} \sum_{j = 1}^N \barw(z_j) ^{1-\alpha} }  \rmd z_{1:N}.
\end{align*}
where we have emphasized notationally the dependence on $d$ in the VR-IWAE bound \eqref{eq:liren:bound}, the log-likelihood \eqref{eq:def:loglik} and in the variational gap. 
We will consider the two cases:
\begin{align*}
\mbox{(i)} & \quad d,N \rightarrow \infty \quad \mbox{with} \quad \frac{\log N}{d}\rightarrow 0, \\ %\label{eq:rateNdExpo} \\
%\label{eq:rateNd}
\mbox{(ii)} & \quad d,N \rightarrow \infty \quad \mbox{with} \quad \frac{\log N}{d^{1/3}}\rightarrow 0,
\end{align*}
that is, $N$ grows slower than exponentially with $d$ as in (i) or slower than sub-exponentially with $d^{1/3}$ as in (ii). As we shall see, those two cases will rely on a different set of assumptions each in order to carry out the analysis. In both scenarios, we will prove that a single importance weight dominates all the others, which strongly impacts the variational gap. To this end, let us rewrite the variational gap $\Delta_{N,d}^{(\alpha)}(\theta, \phi;x)$ under a more convenient form. Writing $\overline{w}_{i}=\barw(z_{i})$ for all $i = 1 \ldots N$, we first re-order the weights $\overline{w}_1, \ldots, \overline{w}_N$ as 
\begin{align*}
\overline{w}^{\left(1\right)} < \overline{w}^{\left(2\right)}< & \cdots < \overline{w}^{\left(N-1\right)} < \overline{w}^{\left(N\right)},
\end{align*}
where we have made the assumption that the weights have no tie almost surely. Now denoting by $q^{(N)}_{\phi}$ the density of $\overline{w}^{\left(N\right)}$ and defining for all $\alpha \in [0,1)$
  \begin{align}
  T^{(\alpha)}_{N,d} = \sum_{j = 1}^{N-1} \lr{\frac{\overline{w}^{(j)}}{\overline{w}^{(N)}}}^{1-\alpha} \label{eq:expressionT_NdExt}
\end{align}
(we have dropped the dependency in $x$ appearing in $ T^{(\alpha)}_{N,d}$ for notational ease here), we then have the following proposition.

\begin{prop}\label{lem:rewritingrenyi} For all $\alpha \in [0,1)$, the variational gap $\Delta^{(\alpha)}_{N,d}(\theta,\phi;x)$ can be rewritten as
  \begin{align} \label{eq:DeltaNDSum}
    \Delta_{N,d}^{(\alpha)}(\theta,\phi;x) & =\Delta^{(\alpha,MAX)}_{N,d}(\theta,\phi;x)+R_{N,d}^{(\alpha)}(\theta,\phi;x)
    \end{align}
  where
  \begin{align}
  &\Delta^{(\alpha, MAX)}_{N,d}(\theta,\phi;x) = \int q^{(N)}_{\phi}(\overline{w}^{\left(N\right)})\log\left(\overline{w}^{\left(N\right)}\right)\rmd\overline{w}^{\left(N\right)} +\frac{\log N}{\alpha -1}\label{eq:DeltaMaxNDExt} \\ 
  &0 \leq R^{(\alpha)}_{N,d}(\theta,\phi;x) \leq \frac{1}{1-\alpha} \mathbb{E}(T^{(\alpha)}_{N,d}) \label{eq:BoundRNdExt}.
  \end{align}
  \end{prop}
  The proof of \Cref{lem:rewritingrenyi} can be found in \Cref{sec:prooflemrewriting}. To continue the analysis, the key intuition will be that the log weights typically satisfy a central limit theorem (CLT), hence the weights are approximately log-normal as the dimension $d$ increases. One such case for instance arises when the posterior and variational distributions are such that the log weights satisfy \looseness=-1
\begin{align} \label{eq:weightsIIDstart}
\log \overline{w}_{i} = \sum_{j=1}^d X_{i,j}, \quad i = 1 \ldots N,
\end{align}
where, for all $i = 1 \ldots N$, $X_{i,1}, \ldots X_{i,d}$ are i.i.d. random variables and $\PE(\exp(\sum_{j=1}^d X_{i,j})) = 1$ (since the relative weights satisfy $\PE(\overline{w}_i) = 1$). Indeed, denoting $\xi_{i,j} = -(X_{i,j} - \PE(X_{1,1}))$, $\sigma^2 = \mathbb{V}(\xi_{1, 1})$ and $S_i = \sum_{j = 1}^d \xi_{i,j}/(\sigma \sqrt{d})$, \eqref{eq:weightsIIDstart} can equivalently be rewritten as
\begin{align} \label{eq:weightsIIDstart2}
\log \overline{w}_{i} = - \log \PE(\exp(- \sigma \sqrt{d} S_1)) - \sigma \sqrt{d} S_i, \quad i = 1 \ldots N, 
\end{align}
where under the assumption that $\sigma^2< \infty$, $S_i$ converges in distribution to the standard normal distribution by the CLT for all $i = 1 \ldots N$. Consequently, the distribution of the weights originating from \eqref{eq:weightsIIDstart2} can be approximated in high-dimensional settings by the log-normal distribution from \Cref{lem:VarianceExponentialWithD}, that is
\begin{equation*} %\label{eq:initWeightsLogNormal}
    \log \overline{w}_{i} = -\frac{\sigma^2 d}{2}- \sigma \sqrt{d} {S_i},\quad S_i \sim \mathcal{N}(0,1), \quad i = 1 \ldots N.% \sim \mathcal{N}\left(-\frac{\sigma^2 d}{2}, \sigma^2 d\right), \quad i = 1 \ldots N.
\end{equation*}
For this reason, we first show in the following how the rest of the analysis unfolds when the distribution of the weights is assumed to be exactly log-normal. We will then use this analysis as a stepping stone to treat the more general case where the distribution of the weights is approximately log-normal of the form \eqref{eq:weightsIIDstart2}. %Our analyses permit us to improve our understanding of the behavior of the VR-IWAE bound in high dimensions and to revisit Examples~\ref{lem:VarianceExponentialWithD} and \ref{ex:LinGaussThm3}.

\subsubsection{Log-normal distribution assumption for the weights}
\label{subsec:lognormalExporegime}

Let $S_1, \ldots, S_N$ be i.i.d. random variables and let the weights $\overline{w}_{1}, \ldots, \overline{w}_{N}$ be of the form
\begin{equation} \label{eq:lognormalweights}
  \log \overline{w}_{i} =-\frac{\sigma^2 d}{2}- \sigma \sqrt{d} {S_i},\quad S_i \sim \mathcal{N}(0,1), \quad i = 1 \ldots N,
\end{equation}
that is we consider the case where the distribution of the weights is log-normal. Let $S^{(1)} \leq \ldots \leq S^{(N)}$ denote the ordered sequence of $S_1, \ldots, S_N$ and recall that $\overline{w}^{(1)} \leq \ldots \leq \overline{w}^{(N)}$ denotes the ordered sequence of $\overline{w}_1, \ldots, \overline{w}_N$. We then have the following lemma, which provides asymptotic results on the expectation of $S^{(1)}$ as $N \to \infty$.

\begin{lem} \label{lem:S1approxGauss} 
  Let $S_1, \ldots, S_N$ be i.i.d. normal random variables. Then,
  \begin{align}\label{eq:ExpS1approxGen}
    \mathbb{E}(S^{(1)})%&=-\sqrt{2\log N} +\frac{\log\log N+\log4\pi}{2\sqrt{2\log N}} +\frac{\mathbb{E}(U)}{\sqrt{2\log N}}+O\left(\frac{1}{\sqrt{\log N}}\right) \nonumber \\
    &=-\sqrt{2\log N}+O\left(\frac{\log \log N}{\sqrt{\log N}}\right).
  \end{align}
  % \begin{align}\label{eq:S1approxGen}
  %   S^{(1)} = - \sqrt{2 \log N} + o_P(1).
  % \end{align}
\end{lem}
The proof of this lemma can be found in \Cref{subsec:proofS1approxGauss}. Intuitively, \Cref{lem:S1approxGauss} will serve as the basis to study the two terms appearing in Equation \eqref{eq:DeltaNDSum} of \Cref{lem:rewritingrenyi}, as both $\Delta_{N,d}^{(\alpha, MAX)}(\theta, \phi;x)$ and $\PE(T^{(\alpha)}_{N,d})$ depend on $S^{(1)}$ through the relation
$$
\log \overline{w}^{(N)} =-\frac{\sigma^2 d}{2}- \sigma \sqrt{d} {S^{(1)}}.
$$
From there, we can derive the two propositions below.
\begin{prop}\label{prop:limDeltaGaussianRenyi} Let $S_1, \ldots, S_N$ be i.i.d. normal random variables. % and let $S^{(1)} \leq \ldots \leq S^{(N)}$  be the ordered sequence of $S_1, \ldots, S_N$. 
  Further assume that the weights $\overline{w}_1, \ldots, \overline{w}_N$ satisfy \eqref{eq:lognormalweights}. Then, for all $\alpha \in [0,1)$, 
  \begin{align*}
  \lim_{N,d \to \infty} \Delta_{N,d}^{(\alpha, MAX)}(\theta, \phi;x) + \frac{d \sigma^2}{2} \lr{1 - 2 \sqrt{\frac{2 \log N}{ d\sigma^2}} + \frac{1}{1-\alpha} \frac{2\log N}{d \sigma^2}+ O\lr{\frac{\log \log N}{\sqrt{d \log N}}}} = 0.
  \end{align*}
\end{prop}

\begin{prop}\label{prop:limitTGaussianRenyi} Let $S_1, \ldots, S_N$ be i.i.d. normal random variables. %and let $S^{(1)} \leq \ldots \leq S^{(N)}$  be the ordered sequence of $S_1, \ldots, S_N$. 
Further assume that the weights $\overline{w}_1, \ldots, \overline{w}_N$ satisfy \eqref{eq:lognormalweights}. Then, for all $\alpha \in [0,1)$, we have
  \begin{align} \label{eq:TNDtoZero}
    \lim_{\substack{N, d \to \infty \\ \log N / d \to 0}} \PE(T^{(\alpha)}_{N,d}) = 0.
  \end{align}
\end{prop}
The proof of these two propositions are deferred to \Cref{subsec:prooflimDeltaGaussian} and \Cref{sec:Proofgaussian} respectively. Importantly, \Cref{prop:limitTGaussianRenyi} implies that the largest weight $\overline{w}^{(N)}$ converges to $1$ in probability, meaning that there is a weight collapse when $N,d \to \infty$ with $\log N /d \to 0$ \cite[following the definition of weight collapse given in][]{bengtsson2008curse}. By using \eqref{eq:TNDtoZero} with $\alpha = 0$, this weight collapse indeed follows from Markov's inequality (in order to get that $T_{N,d}^{(0)}$ converges to $0$ in probability) combined with the fact that $\overline{w}^{(N)} = (1+ T^{(0)}_{N,d})^{-1}$.

Building on \Cref{lem:rewritingrenyi}, \Cref{prop:limDeltaGaussianRenyi} and \Cref{prop:limitTGaussianRenyi}, we now deduce the following theorem, which describes the asymptotic behavior of the variational gap as $N, d \to \infty$ in the log-normal distribution case for values of $\alpha$ in $[0,1)$.

  \begin{thm}[i.i.d. normal random variables]\label{thm:MainGauss}
    Let $S_1, \ldots, S_N$ be i.i.d. normal random variables. % and let $S^{(1)} \leq \ldots \leq S^{(N)}$  be the ordered sequence of $S_1, \ldots, S_N$. 
    Further assume that the weights $\overline{w}_1, \ldots, \overline{w}_N$ satisfy \eqref{eq:lognormalweights}. 
    Then, for all $\alpha \in [0,1)$, 
    % and as $N,d \to \infty$ with $\frac{\log N}{d} \to 0$, 
    we have
    \begin{align*}%\label{eq:MainResultGauss}
      \lim_{\substack{N, d \to \infty \\ \log N / d \to 0}} \Delta^{(\alpha)}_{N,d}(\theta, \phi;x) + \frac{d \sigma^2}{2} \lr{1 - 2 \sqrt{\frac{2 \log N}{ d\sigma^2}} + \frac{1}{1-\alpha} \frac{2\log N}{d \sigma^2} + O\lr{\frac{\log \log N}{\sqrt{d \log N}}}} = 0.
    \end{align*}
    \end{thm} %{eq:OneOverNGenDomke}
  While \Cref{thm:MainGauss} states that increasing $N$ decreases the variational gap $\Delta^{(\alpha)}_{N,d}(\theta, \phi;x) $ for $N$ large enough, it does so by a factor which is negligible compared to the term $- d\sigma^2/2$. This is in sharp contrast to \Cref{prop:GenDomke} and more specifically to \Cref{lem:VarianceExponentialWithD}, which predicts that for log-normal weights the variational gap decreases in $1/N$ in the fixed $d$, large $N$ regime. 

   Contrary to \Cref{lem:VarianceExponentialWithD}, the term $- d\sigma^2/2$ does not depend on $\alpha$ here. In fact, by taking the expectation in \eqref{eq:lognormalweights}, $\mathrm{ELBO} (\theta, \phi; x) - \ell(\theta ; x)  = - {d \sigma^2 }/ {2}$, meaning that the following approximation of the variational gap in the context of \Cref{thm:MainGauss} holds: for all $\alpha \in [0,1)$,
    \begin{align*}
      \Delta^{(\alpha)}_{N,d}(\theta, \phi;x) \approx \mathrm{ELBO} (\theta, \phi; x) - \ell(\theta ; x), \quad \mbox{as $N, d \to \infty$ with $\frac{\log N}{d} \to 0$.}
    \end{align*}
    Hence, \Cref{thm:MainGauss} shows that in high-dimensional scenarios and under the log-normal distribution assumption \eqref{eq:lognormalweights}, we cannot expect to gain much from the VR-IWAE bound unless $N$ grows exponentially with $d$, in the sense that the improvement is negligible compared to the ELBO. This result holds \textit{for all values of $\alpha$ in $[0,1)$}, thus it holds for the IWAE bound ($\alpha = 0$) as well. 
    
    We obtain the following slightly more general result by building on the proof of \Cref{thm:MainGauss}.
    \begin{thm}[General i.i.d. normal random variables] \label{thm:MainGauss2}
      Let $S_1, \ldots, S_N$ be i.i.d. normal random variables. % and let $S^{(1)} \leq \ldots \leq S^{(N)}$  be the ordered sequence of $S_1, \ldots, S_N$. 
      Further assume that the weights $\overline{w}_1, \ldots, \overline{w}_N$ satisfy 
      \begin{equation} \label{eq:lognormalweights2}
        \log \overline{w}_{i} =-\frac{B_d^2}{2}- B_d {S_i}, \quad i = 1 \ldots N,
      \end{equation} 
      and that there exists $\sigma_- > 0$ such that $B_d \geq \sigma_- \sqrt{d}$. Then, for all $\alpha \in [0,1)$, 
      % and as $N,d \to \infty$ with $\frac{\log N}{d} \to 0$, 
      we have
      \begin{align*}%\label{eq:MainResultGauss}
        \lim_{\substack{N, d \to \infty \\ \log N / d \to 0}} \Delta^{(\alpha)}_{N,d}(\theta, \phi;x) + \frac{B_d^2}{2} \lrcb{ 1 -  2 ~ \frac{\sqrt{2 \log N}}{B_d} + \frac{1}{1-\alpha} ~\frac{2\log N}{B_d^2} + O \lr{\frac{\log \log N}{B_d \sqrt{\log N}}} } = 0.
      \end{align*}
      \end{thm} %{eq:OneOverNGenDomke}
      The proof of \Cref{thm:MainGauss2} can be found in \Cref{subsub:proof:thm:MainGauss2}. We now revisit the Gaussian example given in \Cref{lem:VarianceExponentialWithD} in the context of \Cref{thm:MainGauss2}.

  \begin{ex}\label{ex:GaussianLogNormal} Set $p_\theta(z|x) = \mathcal{N}(z;\theta, \boldsymbol{I}_d)$ and $q_\phi(z) = \mathcal{N}(z; \phi, \boldsymbol{I}_d)$, with $\theta, \phi \in \rset^d$. Denoting $B_d = \|\theta - \phi\|$, we can write the weights $\overline{w}_1, \ldots, \overline{w}_N$ under the form \eqref{eq:lognormalweights2} (see \eqref{eq:logNormalBdProof} of \Cref{subsec:lem:VarianceExponentialWithD}). Hence, \Cref{thm:MainGauss2} applies if there exists $\sigma_- > 0$ such that $B_d \geq \sigma_- \sqrt{d}$. This is for example the case if $\theta = 0 \cdot \boldsymbol{u}_d$ and $\phi = \boldsymbol{u}_d$ with $\sigma_- = 1$.
  \end{ex}
  As we shall see next, our conclusion regarding the behavior of the VR-IWAE bound in high-dimensional settings extends to cases where the log-normal assumption does not necessarily hold exactly, that is if we assume instead that \eqref{eq:weightsIIDstart2} holds, 
where $S_1, \ldots S_N$ are i.i.d. random variables whose distribution is close to a normal as $N, d \rightarrow \infty$.

\subsubsection{Beyond the log-normal distribution assumption}
\label{subsec:BeyondLogNormal}
Following \eqref{eq:weightsIIDstart2}, let us set 
\begin{equation}\label{eq:approxlognormalweights}
  \log \overline{w}_i = - \log { \PE(\exp(-\sigma \sqrt{d} S_1))} - \sigma \sqrt{d} S_i, \quad i = 1 \ldots N,
\end{equation} 
where the i.i.d. random variables $S_1, \ldots, S_N$ are defined as follows:
\begin{align}
S_i = \frac{1}{\sigma \sqrt{d}} \sum_{j = 1}^d \xi_{i,j}, \quad i = 1 \ldots N \label{eq:expressionSi}.
\end{align}
The assumption \ref{hypBickel} below ensures that $S_1, \ldots S_N$ have a distribution that is close to a normal as $N, d \rightarrow \infty$, so that \eqref{eq:lognormalweights} is recovered in the limit.
\begin{hyp}{A}
\item \label{hypBickel}  For all $ i = 1 \ldots N$, \begin{enumerate}
  \item \label{hypBickel:a} $\xi_{i,1}, \ldots , \xi_{i,d}$ are i.i.d. random variables which are absolutely continuous with respect to the Lebesgue measure and satisfy $\PE(\xi_{i,1}) = 0$ and $\mathbb{V}(\xi_{i,1}) = \sigma^2 < \infty$. 
\item \label{hypBickel:b} There exists $K > 0$ such that:
$$
|\PE(\xi_{i,1}^k)| \leq k! K^{k-2} \sigma^2, \quad k \geq 3.
$$
\end{enumerate}
\end{hyp}
Here, the condition \ref{hypBickel:b} corresponds to the well-known Bernstein condition. Paired up with \ref{hypBickel:a}, this condition permits us to appeal to classical limit theorems for large deviations in order to enlarge the so-called zone of normal convergence beyond the CLT \citep{Petrov2000, saulis2000}. This enables us to establish preliminary results which are used to prove the results of \Cref{subsec:BeyondLogNormal} we will now present (we refer to \Cref{app:prelimResults} for the statement of those preliminary results). We first provide the equivalent of \Cref{lem:S1approxGauss} in the more general context of \eqref{eq:expressionSi} and under \ref{hypBickel}.
\begin{lem} \label{lem:S1approxGen}
  Assume \ref{hypBickel}. Let $S_1, \ldots, S_N$ be as in \eqref{eq:expressionSi}. %and let $S^{(1)} \leq \ldots \leq S^{(N)}$  be the ordered sequence of $S_1, \ldots, S_N$. 
  Then, as $N,d \to \infty$, with $\frac{\log N}{d^{1/3}} \to 0$, \eqref{eq:ExpS1approxGen} holds. 
\end{lem}
The proof of this result is deferred to \Cref{subsec:prooflemS1approxGen}. Notice that we are now assuming that $N$ grows slower than sub-exponentially with $d^{1/3}$ in \Cref{lem:S1approxGen}. The following two propositions give results akin to those obtained in \Cref{prop:limDeltaGaussianRenyi} and \Cref{prop:limitTGaussianRenyi}.

\begin{prop} \label{prop:DeltaGen}
Assume \ref{hypBickel}. Let $S_1, \ldots, S_N$ be as in \eqref{eq:expressionSi}. % and let $S^{(1)} \leq \ldots \leq S^{(N)}$  be the ordered sequence of $S_1, \ldots, S_N$. 
Further assume that the weights $\overline{w}_1, \ldots, \overline{w}_N$ satisfy \eqref{eq:approxlognormalweights}. Then, setting
\begin{align} \label{eq:defALognormal}
  a \eqdef \log \PE(\exp(-\xi_{1,1})),
\end{align} 
we have that $a > 0$ and that for all $\alpha \in [0,1)$, %as $N,d \to \infty$ with $\frac{\log N}{d^{1/3}} \to 0$, and , 
\begin{align*}
 \lim_{\substack{N,d \to \infty \\ \log N / d^{1/3} \to 0}} \Delta_{N,d}^{(\alpha, MAX)}(\theta, \phi;x) + d a \lrcb{1 -  \frac{\sigma}{a} \sqrt{\frac{ \log N}{ d}} + O \lr{\frac{\log \log N}{\sqrt{d \log N}}}} = 0.
\end{align*}
%\begin{align*}
%  \lim_{\substack{N,d \to \infty \\ \log N / d^{1/3} \to 0}} \Delta_{N,d}^{(\alpha, MAX)}(\theta, \phi;x) + d a \lrcb{1 -  \frac{\sigma}{a} \sqrt{\frac{ \log N}{ d}} + \frac{1}{1-\alpha} \frac{2\log N}{d a} + O \lr{\frac{\log \log N}{\sqrt{d \log N}}}} = 0.
% \end{align*}
\end{prop}
\begin{prop}\label{Prop:limitT}
  Assume \ref{hypBickel}. Let $S_{1}, \ldots, S_{N}$ be as in \eqref{eq:expressionSi}. % and let $S^{(1)} \leq \ldots \leq S^{(N)}$ be the ordered sequence of $S_1, \ldots, S_N$. 
  Further assume that the weights $\overline{w}_1, \ldots, \overline{w}_N$ satisfy \eqref{eq:approxlognormalweights}. Then, for all $\alpha \in [0,1)$, %as $N,d\rightarrow \infty$ with $\frac{\log N}{d^{1/3}}\rightarrow 0$, we have 
$$
\lim_{\substack{N,d \to \infty \\ \log N / d^{1/3} \to 0}} \mathbb{E}(T^{(\alpha)}_{N,d}) = 0.
$$
\end{prop}
The proof of \Cref{prop:DeltaGen} and \Cref{Prop:limitT} can be found in \Cref{subsec:proof:prop:deltaGen} and \Cref{subsec:ProofLimitT} respectively.

\begin{rem}

The log-normal case corresponds to setting $a = \sigma^2/2$ in \Cref{prop:DeltaGen} (this can be checked using the definition of $a$ in \eqref{eq:defALognormal} combined with \eqref{eq:checkALognormal} from the proof of \Cref{prop:DeltaGen} in \Cref{subsec:proof:prop:deltaGen}). Contrary to \Cref{prop:limDeltaGaussianRenyi}, the $(1-\alpha)^{-1} (d \sigma^2)^{-1} 2 \log N $ term is now subsumed by the final $O(\log \log N / \sqrt{d \log N})$ term in \Cref{prop:DeltaGen}, which comes from the fact that \Cref{prop:DeltaGen} makes the additional assumption $\log N/d^{1/3} \to 0$ as $N,d \to \infty$. Hence, \Cref{prop:limDeltaGaussianRenyi} and \Cref{prop:DeltaGen} agree with each other in the log-normal case. 
\end{rem}
\Cref{lem:rewritingrenyi}, \Cref{prop:DeltaGen} and \Cref{Prop:limitT} lead to the theorem below, which characterizes the asymptotics of the variational gap as $N, d \to \infty$ for $\alpha \in [0,1)$ in the more general case where the distribution of the weights is approximately log-normal according to \eqref{eq:approxlognormalweights}.  
\begin{thm}[i.i.d. random variables] \label{thm:iidRv}
  Assume \ref{hypBickel}. Let $S_1, \ldots, S_N$ be as in \eqref{eq:expressionSi}. % and let $S^{(1)} \leq \ldots \leq S^{(N)}$  be the ordered sequence of $S_1, \ldots, S_N$. 
  Further assume that the weights $\overline{w}_1, \ldots, \overline{w}_N$ satisfy \eqref{eq:approxlognormalweights} and let $a > 0$ be defined as in \eqref{eq:defALognormal}. Then, for all $\alpha \in [0,1)$,
  \begin{align*}
    \lim_{\substack{N,d \to \infty \\ \log N / d^{1/3} \to 0}} \Delta^{(\alpha)}_{N,d}(\theta, \phi;x) + d a \lr{1 - \frac{\sigma}{a} \sqrt{\frac{2 \log N}{ d}} + O \lr{\frac{\log \log N}{\sqrt{d \log N}}}} = 0. %+ \frac{1}{1-\alpha} \frac{\log N}{d a}
  \end{align*}
  \end{thm}
We have thus obtained that, under the assumptions of \Cref{thm:iidRv}, the VR-IWAE bound is of limited interest for all values of $\alpha \in [0,1)$ unless $N$ grows at least sub-exponentially with $d^{1/3}$. In fact, by taking the expectation in the expression of the log-weights, we have that $\mathrm{ELBO} (\theta, \phi; x) - \ell(\theta ; x)  = - d a$ (using for example \eqref{eq:WNGen} from the proof of \Cref{prop:DeltaGen} in \Cref{subsec:proof:prop:deltaGen}). Hence, the following approximation of the variational gap holds in the context of \Cref{thm:iidRv}: for all $\alpha \in [0,1)$,
\begin{align*}
  \Delta^{(\alpha)}_{N,d}(\theta, \phi;x) \approx \mathrm{ELBO} (\theta, \phi; x) - \ell(\theta ; x), \quad \mbox{as $N, d \to \infty$ with $\frac{\log N}{d^{1/3}} \to 0$.}
\end{align*}
Since the weights are assumed to be \textit{approximately} log-normal this time as opposed to \Cref{subsec:lognormalExporegime}, the condition that $N$ should grow at least exponentially with $d$ to avoid a weight collapse effect has now been replaced by the less restrictive yet still stringent condition that $N$ should grow at least sub-exponentially with $d^{1/3}$.

As described below, the assumptions on the distribution of the weights -- that is, on the ratio between the posterior and the variational distributions -- appearing in \Cref{thm:iidRv} are met for the linear Gaussian setting from \Cref{ex:LinGaussThm3}.
\begin{ex}
\label{ex:linGauss}
We consider the linear Gaussian setting from \cite{rainforth2018tighter} that we recalled in \Cref{ex:LinGaussThm3}. Denoting $\lambda = \big\| \frac{\theta + x}{2} - Ax - b \big\| / \sqrt{d}$, the weights can be written in the form of \eqref{eq:approxlognormalweights} with $\sigma^2 = {1}/{18} + {8}/{3} \lambda^2$ and we also have $a = \lambda^2 + 1/6 + 1/2 \log (3/4)$. As a result, we can apply \Cref{thm:iidRv} if \ref{hypBickel} holds. This is for example the case at optimality when $(\theta, \phi) = (\theta^\star, \phi^\star)$ (and the derivation details for this example can be found in \Cref{subsec:ex:linGaussApp}).
\end{ex}
\Cref{ex:linGauss} states that we are in the conditions of application of \Cref{thm:iidRv} when the parameters $(\theta, \phi)$ are optimal, with corresponding optimal posterior density $p_{\theta^\star}(z|x) = \mathcal{N}(z; (\theta^\star + x)/2, 1/2 \boldsymbol{I}_d)$ and optimal variational density $q_{\phi^\star}(z|x) = \mathcal{N}(z; (\theta^\star + x)/2, 2/3 \boldsymbol{I}_d)$. 

This example showcases how, in some instances where the variational family is not large enough to contain the target density, there can be a weight collapse phenomenon as $d$ increases that severely impacts the VR-IWAE bound, even when the parameters $(\theta, \phi)$ are set to be the optimal ones for the problem considered. This concludes our theoretical study of the VR-IWAE bound, which sheds lights on the conditions behind the success or failure of this bound. %Note in particular that by virtue of \Cref{rem:BBalphaEnergyFunc}, all our theoretical results also apply to the Black-Box alpha methodology. 
In the next section, we describe how our theoretical results relate to the existing literature.

\section{Related work}
\label{sec:relatedWork}

\textbf{Alpha-divergence variational inference.} Our work provides the theoretical grounding behind VR-bound gradient-based schemes \citep{hernandez2016black,Bui2016BlackboxF, li2016renyi,dieng2017variational, li2017dropout, pmlr-v130-zhang21o,RODRIGUEZSANTANA2022260}. It also unifies the VR and IWAE bound methodologies \citep{burda2015importance,rainforth2018tighter,Tucker2019DoublyRG,domke2018,maddison2017} and serves as a foundation for improving on both methodologies. \newline

\noindent \textbf{Proof techniques.} Several of our theoretical results generalize known findings from the literature in order to build the VR-IWAE bound methodology and to characterize its asymptotics. Some of our proofs are straightforwardly derived from existing ones, such as the proofs of Theorems \ref{prop:drepIwaeAlpha} and \ref{prop:GenDomke} (which are established by directly adapting the proofs written in \cite{Tucker2019DoublyRG} and in \cite{domke2018} respectively). However, a number of our proof techniques differs significantly from/alter parts of known proofs (see \Cref{app:sec:relatedWork} for details). Lastly, the derivations made in \Cref{subsec:VRIWAEndInfty} for the asymptotics of the VR-IWAE bound when $N, d \to \infty$ are, to the best of our knowledge, the first of their kind. \newline 

\noindent \textbf{Importance sampling.} Common variational bounds and their gradients can often be expressed in terms of the importance weights $\w$ (with our novel VR-IWAE bound being no exception to that rule). As such, the success of gradient-based variational inference has been known to depend on the behavior of the importance weights and there has been a growing interest in understanding this behavior through the use of insights and tools from the importance sampling (IS) literature \citep{maddison2017, domke2018, dhaka2021challenges, geffner21a}. In particular, it is well-known that IS can perform poorly in high dimensions unless the target and reference/proposal distributions are close. \looseness=-1

\cite{picklands1975statistical} for instance showed that, under commonly satisfied assumptions, the right tail of the importance weights distribution approximates a generalized Pareto distribution, that is, a heavy-tailed distribution with three parameters $(u, \sigma, k)$ and moments of order up to $\lfloor 1/k \rfloor$. This behavior is typical in high dimensions and it makes IS fail, as the IS estimators are dominated by the few largest terms. Leveraging this result, \cite{dhaka2021challenges} considered the case of black-box variational inference and viewed the importance weights as approximately drawn from a generalized Pareto distribution with tail index $k$. The importance weights taken to an exponent $1-\alpha$ are then approximately distributed according to a generalized Pareto distribution with tail index $(1-\alpha)k$ and they deduced that the estimates should be more stable as $\alpha$ increases towards $1$ due to lighter tails. 

The analysis from \cite{dhaka2021challenges} goes hand in hand with our findings, as (i) Theorems~\ref{prop:SNRconvergence} and \ref{prop:GenDomke} predict improvements in terms of SNR and variance as $\alpha$ increases, at the cost of an increasing bias and (ii) our results from \Cref{subsec:VRIWAEndInfty} show that, as $d$ increases, the VR-IWAE bound fails regardless of the value of $\alpha \in [0,1)$ and provides negligible improvements compared to the ELBO $(\alpha = 1)$. However, one main specificity of our work is our precise characterization of how the distribution of the importance weights impacts the tightness of the VR-IWAE bound. Specifically, \Cref{prop:GenDomke} generalizes \cite{domke2018} to the VR-IWAE bound, while the results from \Cref{subsec:VRIWAEndInfty} provide the first theoretical justification behind the empirical findings from \cite{geffner21a} regarding the impact of weight collapse on the tightness of variational bounds. 

The next section is devoted to illustrating the theoretical claims we have made thus far over toy and real-data experiments.

\section{Numerical Experiments}
\label{sec:numericalExperiments}

In this section, our goal is to verify the validity of the theoretical results we established over several numerical experiments, starting with a Gaussian example in which the distribution of the weights is exactly log-normal.

\subsection{Gaussian example}
\label{subsec:Toy}

We consider the Gaussian example described in \Cref{ex:GaussianLogNormal}, for which the weights $\overline{w}_1, \ldots, \overline{w}_N$ can be written under the form \eqref{eq:lognormalweights2} %, that is
%\begin{align*}
%  \log \overline{w}_i & = - \frac{B_d^2}{2} - B_d S_i, \quad S_i \sim \mathcal{N}(0,1), \quad i = 1 \ldots N,
%\end{align*}
with $B_d = \| \theta - \phi \|$, meaning that the distribution of the weights is log-normal. On the one hand, \Cref{prop:GenDomke} predicts that for all $\alpha \in [0,1)$,
\begin{align} 
&\Delta^{(\alpha)}_{N,d}(\theta, \phi;x) = - \frac{\alpha B_d^2}{2}  - \frac{\exp\lrb{(1-\alpha)^2 B_d^2} - 1}{2(1-\alpha)N} +o\left(\frac{1}{N}\right) \label{eq:OneOverNGenDomkeToy}
\end{align}
(this follows from a straightforward adaptation of \Cref{lem:VarianceExponentialWithD}). On the other hand, \Cref{thm:MainGauss2} tells us that if there exists $\sigma_- > 0$ such that $B_d \geq \sigma_- \sqrt{d}$, then: for all $\alpha \in [0,1)$,
\begin{align} \label{eq:MainResultGaussToy}
  \lim_{\substack{N, d \to \infty \\ \log N / d \to 0}} \Delta^{(\alpha)}_{N,d}(\theta, \phi;x) + \frac{B_d^2}{2} \lrcb{ 1 -  2 ~ \frac{\sqrt{2 \log N}}{B_d} + \frac{1}{1-\alpha} ~\frac{2\log N}{B_d^2} + O \lr{\frac{\log \log N}{B_d \sqrt{\log N}}} } = 0. 
\end{align}
We now want to check the validity of the two asymptotic results above. To do so, we need to be able to approximate the variational gap $\Delta^{(\alpha)}_{N,d}(\theta, \phi;x)$, which can be done using the unbiased Monte Carlo (MC) estimator given for all $N \in \mathbb{N}^\star$ by
$$
\frac{1}{1-\alpha} \log\left(\frac{1}{N}\sum_{j=1}^{N}\barw(Z_{j})^{1-\alpha}\right), %\label{eq:variationalgap},
$$
with $Z_1, \ldots, Z_N$ being i.i.d. samples generated according to $q_\phi$. As for the approximation returned by \Cref{prop:GenDomke}, we will represent it according to \eqref{eq:OneOverNGenDomkeToy} through functions of the form 
\begin{align} \label{eq:functionFormDomke}
c_1 \mapsto - \frac{\alpha B_d^2}{2} - \frac{\exp\lrb{(1-\alpha)^2 B_d^2} - 1}{2(1-\alpha) N} + \frac{c_1}{N}
\end{align}
and for the approximation returned by \Cref{thm:MainGauss2}, we will represent it according to \eqref{eq:MainResultGaussToy} through functions of the form 
\begin{align} \label{eq:functionFormLogNormal}
c_2 \mapsto - \frac{B_d^2}{2} + B_d \sqrt{2 \log N} + \frac{\log N}{\alpha-1} + \frac{c_2 B_d \log \log N}{\sqrt{\log N}}.
\end{align}
We first consider the case where $\theta = 0 \cdot \boldsymbol{u}_d$ and $\phi = \boldsymbol{u}_d$. In that setting, $B_d = \sqrt{d}$ and we have that (i) the $1/N$ term from \eqref{eq:OneOverNGenDomkeToy} is exponential in $(1-\alpha)^2 d$ and (ii) we are in the conditions of application of \Cref{thm:MainGauss2} by setting $\sigma_- = 1$. 

Consequently, for this choice of $(\theta, \phi)$ and \textit{regardless of the value of $\alpha \in [0,1)$}, we are expecting \Cref{thm:MainGauss2} to capture the behavior of the variational gap as $d$ and $N$ increase in such a way that $\log N/d$ decreases. This is indeed what we observe in Figure~\ref{fig:ToyLogNormal}, in which we let $d \in \lrcb{10, 100, 1000}$, $\alpha \in \lrcb{0., 0.2, 0.5}$, $N \in \lrcb{2^j ~ : ~ j = 1 \ldots 9}$ and we compare the behavior of the variational gap to the behavior predicted by \Cref{thm:MainGauss2} through curves of the form \eqref{eq:functionFormLogNormal}. 

Unsurprisingly, although valid in low dimensions for a proper choice of $\alpha$, the analysis of \Cref{prop:GenDomke} requires an unpractical amount of samples $N$ to properly capture the behavior of the variational gap as $d$ increases (additional plots providing the comparison with \Cref{prop:GenDomke} are made available in \Cref{subseb:app:toy} for the sake of completeness). \newline

\begin{figure}[!ht]
  \begin{tabular}{ccc} 
    \includegraphics[scale=0.3]{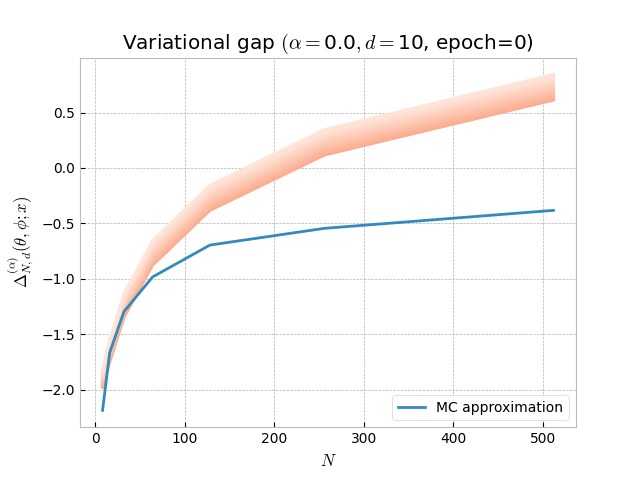} &
    \includegraphics[scale=0.3]{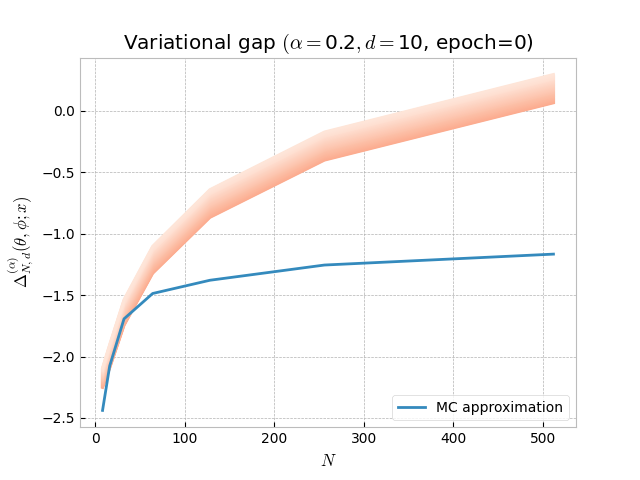} &
    \includegraphics[scale=0.3]{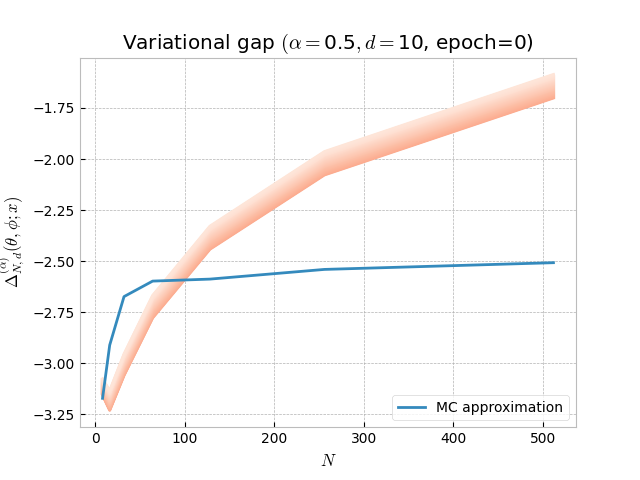} \\
    \includegraphics[scale=0.3]{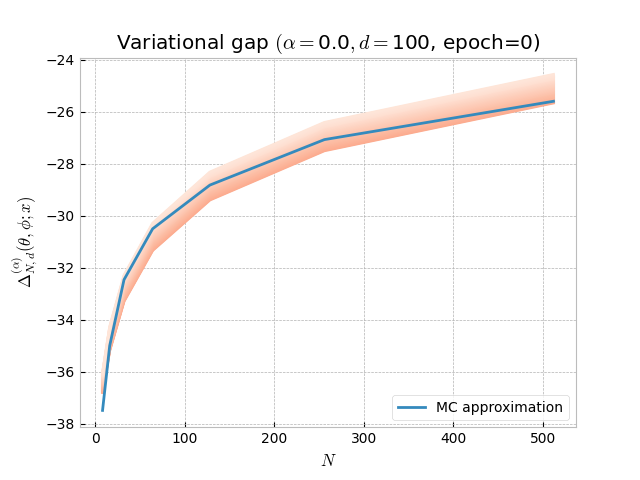} &
    \includegraphics[scale=0.3]{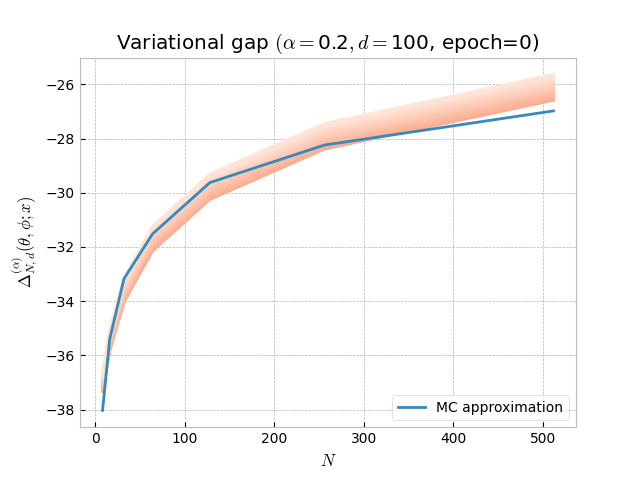} &
    \includegraphics[scale=0.3]{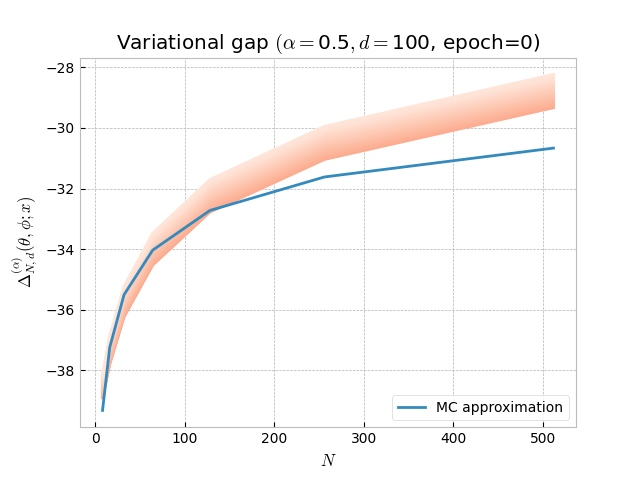} \\
    \includegraphics[scale=0.3]{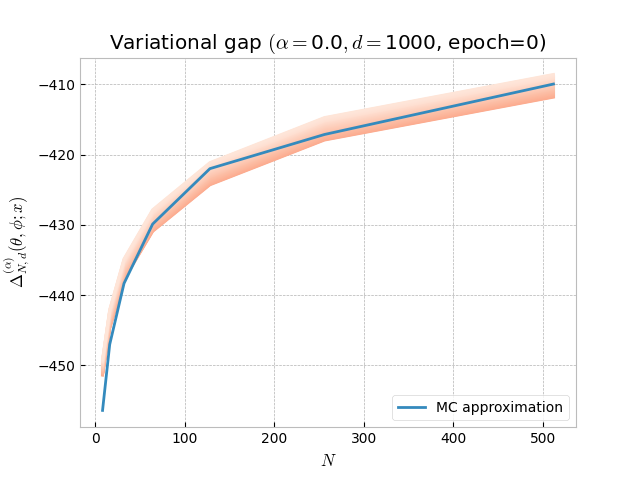} &    \includegraphics[scale=0.3]{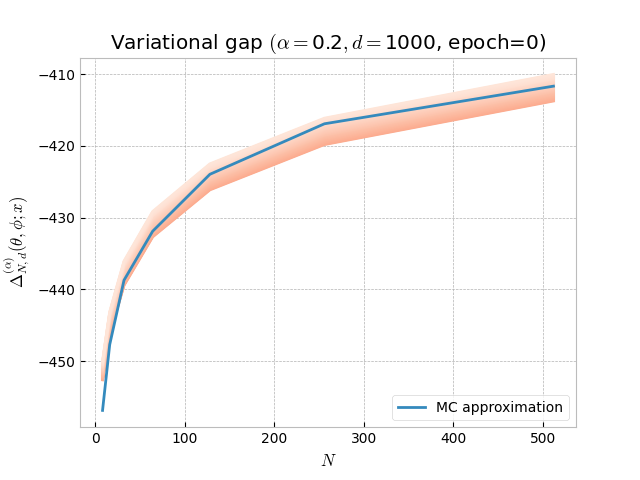} &
    \includegraphics[scale=0.3]{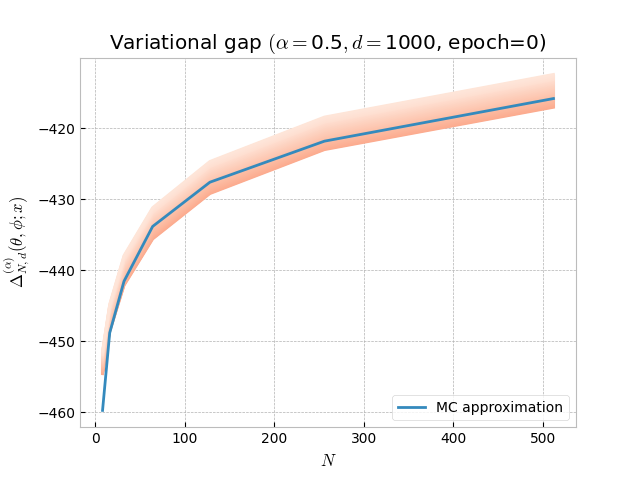} \\
  \end{tabular}
  \caption{Plotted in blue is the MC estimate of the variational gap $\Delta_{N,d}^{(\alpha)}(\theta, \phi; x)$ (averaged over 1000 MC samples) for the toy example described in \Cref{subsec:Toy} as a function of $N$, for varying values of $\alpha$ and of $d$ and with $(\theta, \phi) = (0 \cdot \boldsymbol{u_d}, \boldsymbol{u}_d)$ so that $B_d = \sqrt{d}$. Plotted in orange are curves of the form \eqref{eq:functionFormLogNormal} with tailored values of $c_2$.} \label{fig:ToyLogNormal}
\end{figure}

We next train the parameter $\phi$ in order to measure the impact of the training procedure on the validity of our asymptotic results. Here, this impact is reflected in the quantity $B_d$ through the simple relation $B_d = \| \theta - \phi\|$. In case the training is successful, $B_d / \sqrt{d}$ is then anticipated to decrease from $1$ to $0$ (having set $\theta = 0 \cdot \boldsymbol{u}_d$ and initialized with $\phi = \boldsymbol{u_d}$). 

Hence, as the training progresses, we will be less and less able to find $\sigma_- > 0$ such that $B_d \geq \sigma_- \sqrt{d}$, which will contradict the assumption we make in \Cref{thm:MainGauss2}. At the same time, the $1/N$ term from \eqref{eq:OneOverNGenDomkeToy} will decrease thanks to its dependency in $B_d$, meaning that~\eqref{eq:OneOverNGenDomkeToy} may become a better approximation than \eqref{eq:MainResultGaussToy} during the training procedure. 

This behavior is empirically confirmed in Figure~\ref{fig:ToyLogNormal2} (and we also check in Figure~\ref{fig:ToyLogNormal3} of \Cref{subseb:app:toy} that $B_d/\sqrt{d}$ indeed goes from $1$ to $0$ during the training procedure). In those plots, the parameter $\phi$ was optimised via stochastic gradient descent using the reparameterized gradient estimator~\eqref{reparam:liren:MC} with $N = 100$ and we set $\alpha = 0.2$ and $d = 1000$ (and a similar trend can be observed for other values of $\alpha$ and $d$). 

\begin{figure}[!ht]
  \begin{tabular}{ccc}
    \includegraphics[scale=0.3]{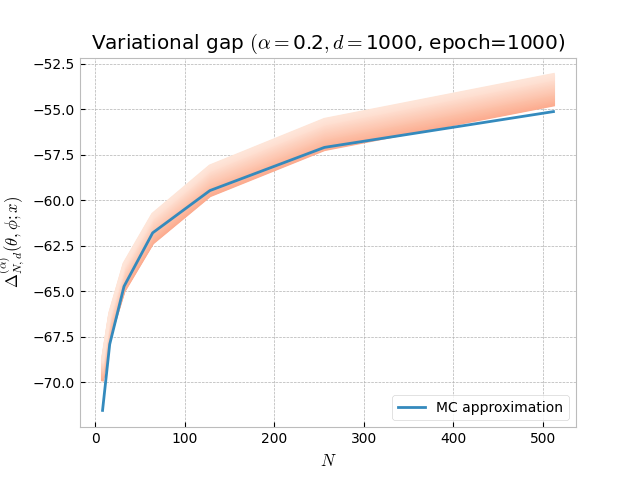}  
  &  \includegraphics[scale=0.3]{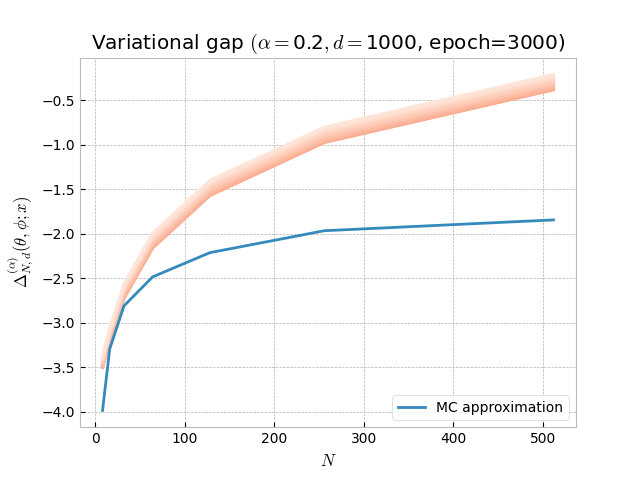}
   &  \includegraphics[scale=0.3]{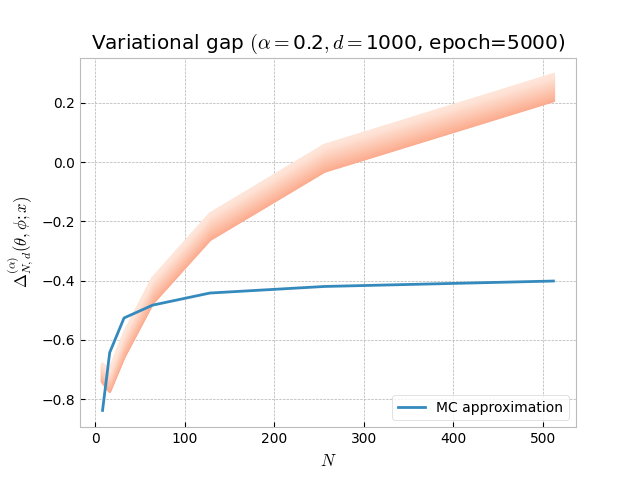} \\
 \includegraphics[scale=0.3]{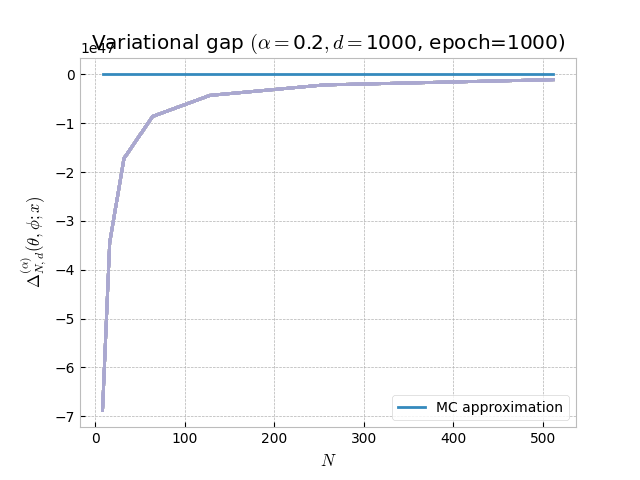} & \includegraphics[scale=0.3]{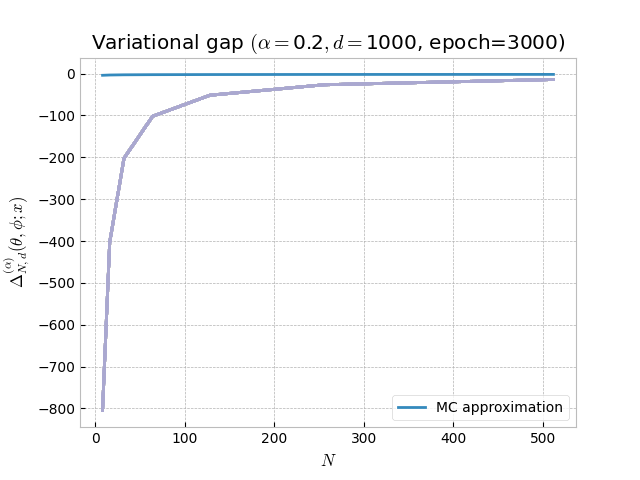}
 & \includegraphics[scale=0.3]{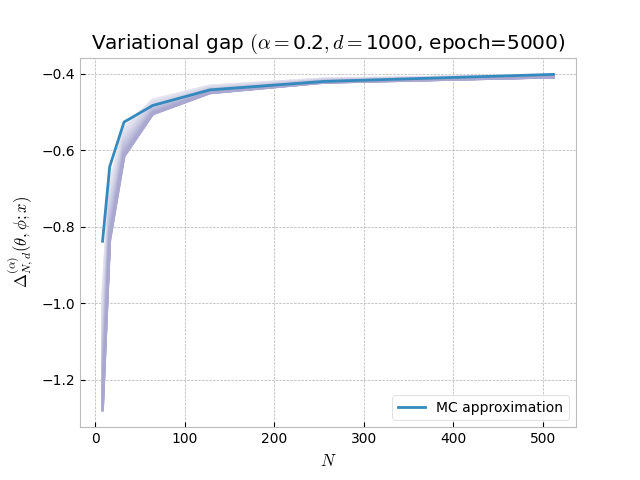}  
  \end{tabular}
  \caption{Plotted in blue is the MC estimate of the variational gap $\Delta_{N,d}^{(\alpha)}(\theta, \phi; x)$ (averaged over 1000 MC samples) at epochs $\lrcb{1000, 3000, 5000}$ for the toy example described in \Cref{subsec:Toy} as a function of $N$, for $\alpha = 0. 2$ and $d = 1000$. Plotted in orange (resp. in purple) are curves of the form \eqref{eq:functionFormLogNormal} with tailored values of $c_2$ (resp. of the form \eqref{eq:functionFormDomke} with tailored values of $c_1$).}
  \label{fig:ToyLogNormal2}
\end{figure}

The main insight we get from our first numerical experiment is then that: as the dimension $d$ increases and $N$ does not grow faster than exponentially with $d$, we should not expect much empirically from the VR-IWAE bound as a lower bound to the marginal log-likelihood when the distribution of the weights is log-normal. This is true unless the encoder and decoder distributions become very close to one another, in which case \Cref{prop:GenDomke} does apply instead of \Cref{thm:MainGauss2}.  

Thus, while this limitation of the VR-IWAE bound holds for all $\alpha \in [0,1)$, it may be mitigated by (i) proposing successful training procedures (further shedding light on the importance of finding gradient estimators with good SNR properties) and (ii) selecting suitable variational families which can capture the complexity within the target posterior density. Furthermore, the analysis provided by \Cref{prop:GenDomke} may also apply in lower dimensional settings, under the condition that the variance term appearing in \Cref{prop:GenDomke} is well-behaved and that the value of $\alpha$ is properly tuned. We next present a second numerical experiment, where this time the weights are not exactly log-normal.

\subsection{Linear Gaussian example}
\label{subsec:linGaussEx}

We are interested in the linear Gaussian example from \cite{rainforth2018tighter}, which we already highlighted in Examples \ref{ex:LinGaussThm3} and \ref{ex:linGauss}. 
The dataset $\data = \{x_1, \ldots, x_T \}$ is generated by sampling $T=1024$ datapoints from $\mathcal{N}(0, 2 \boldsymbol{I}_d)$ and we will consider three initializations for the parameters $(\theta, \phi)$ involving a Gaussian perturbation of standard deviation $\sigmapert$ of the ground truth values $(\theta^\star, \phi^\star)$: \newline

(i) ~~$\sigmapert=0.5$: the parameters are initialized far from $(\theta^\star, \phi^\star)$,

(ii) ~$\sigmapert=0.01$: the parameters are initialized close to $(\theta^\star, \phi^\star)$,

(iii) $\sigmapert = 0.$: the parameters are equal to $(\theta^\star, \phi^\star)$. \newline

\noindent The first two initializations follow from \cite{rainforth2018tighter} and should notably permit us to approximately characterize the behavior of the linear model before and after training.

Our first step is to check that, as written in \Cref{ex:linGauss}, the distribution of the weights is approximately log-normal as $d$ increases for the initializations above. To do so, we randomly select a datapoint $x$, draw $N = 1 000 000$ weight samples in dimension $d = \lrcb{20, 100, 1000}$ for $\sigmapert \in \lrcb{0.5, 0.01, 0.}$, before plotting for each $d$ a histogram of the resulting log-weight distribution as well as a Q-Q plot to test the normality assumption of those log-weights.

The results are shown on \Cref{fig:linear_gaussian_weight_distr}  and we see that while the log-normality phenomenon happens in dimension $d = 100$ when a large perturbation is being considered, even a small perturbation to no perturbation at all can induce some log-normality of the weights as $d$ further increases, which is in line with the theory (and similar plots can be observed for other randomly selected datapoints).

\begin{figure}[!ht]
  \begin{tabular}{ccc}
   \includegraphics[scale=0.27]{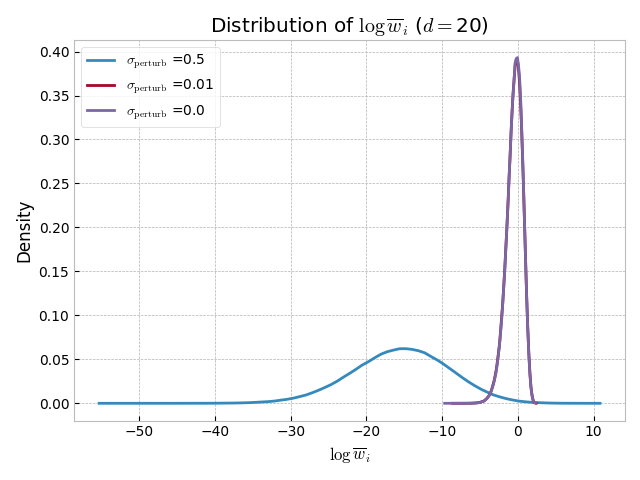} 
   & \includegraphics[scale=0.27]{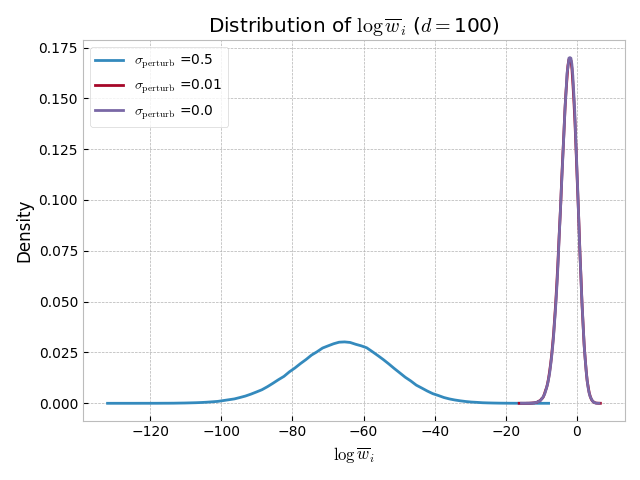} 
   & \includegraphics[scale=0.27]{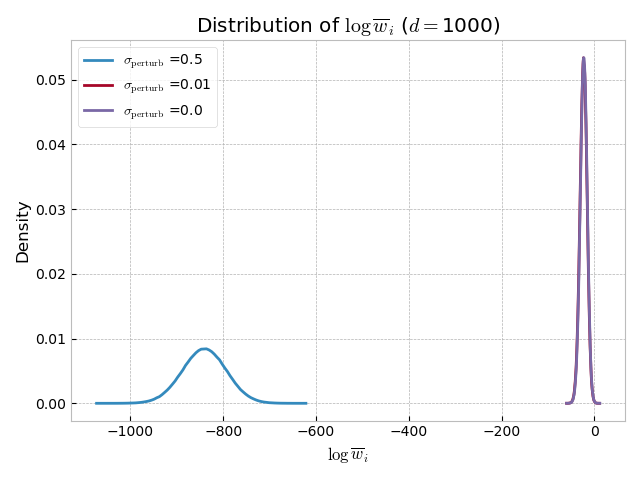} \\
   \includegraphics[scale=0.3]{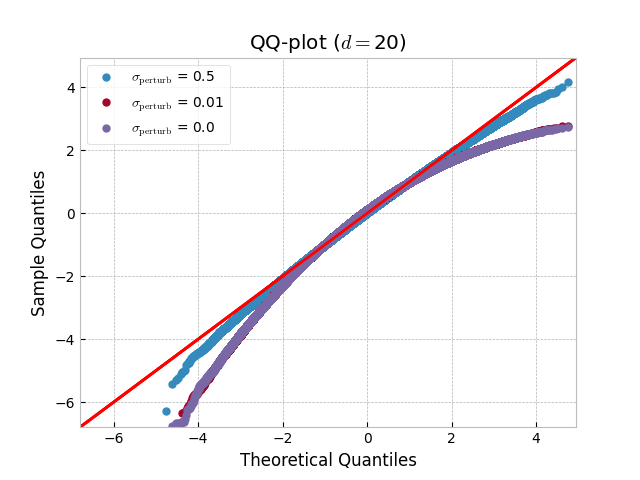} & \includegraphics[scale=0.3]{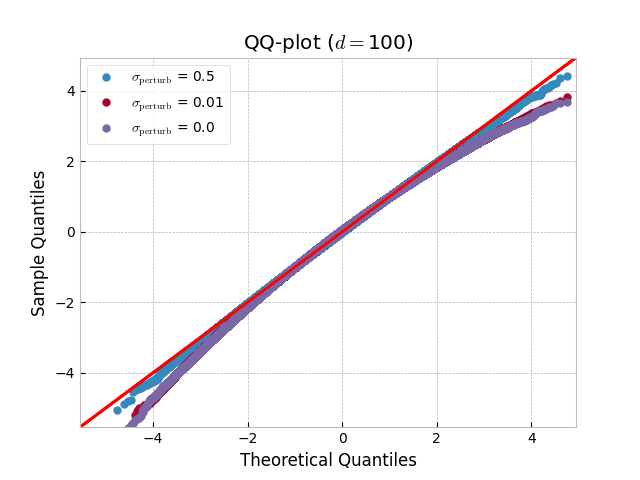} & \includegraphics[scale=0.3]{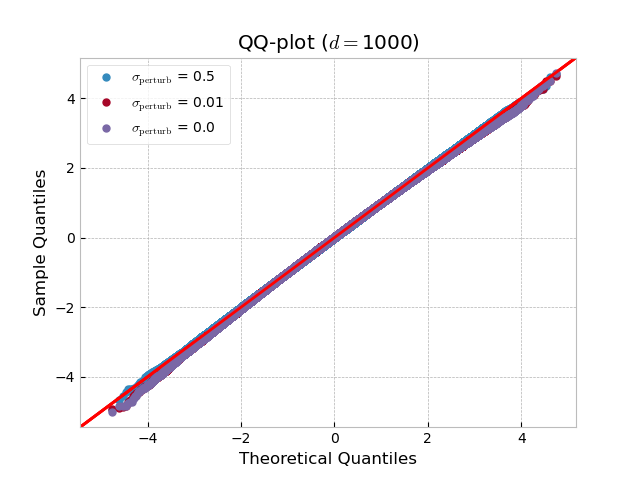}
  \end{tabular}
  \caption{Plotted is the distribution of $\log \overline{w}_i$ and the corresponding QQ-plot for the linear Gaussian example described in \Cref{subsec:linGaussEx} for a randomly selected datapoint $x$, for varying values of $d$ and for three different initializations of the parameters $(\theta, \phi)$. \label{fig:linear_gaussian_weight_distr}}
\end{figure}
We next want to test the validity of our asymptotic results. On the one hand, \Cref{prop:GenDomke} predicts that: for all $\alpha \in [0,1)$, 
\begin{align} 
\Delta^{(\alpha)}_{N,d}(\theta, \phi;x) =   \mathcal{L}_d^{(\alpha)}(\theta, \phi; x) - \ell_d(\theta; x) - \frac{\gamma_{\alpha,d}^2}{2N} +o\left(\frac{1}{N}\right), \label{eq:OneOverNGenDomkeLinGauss}
\end{align}
where $\mathcal{L}_d^{(\alpha)}(\theta, \phi; x) - \ell_d$ and $\gamma_{\alpha, d}^2$ can be analytically computed using \Cref{ex:LinGaussThm3} (and we have emphasized the dependency in $d$ in each of those terms). On the other hand, \Cref{thm:iidRv} predicts under \ref{hypBickel} that: for all $\alpha \in [0,1)$,
\begin{align*}
\lim_{\substack{N,d \to \infty \\ \log N / d^{1/3} \to 0}} \Delta^{(\alpha)}_{N,d}(\theta, \phi;x) + d a \lr{1 - \frac{\sigma}{a} \sqrt{\frac{2 \log N}{ d}} + O \lr{\frac{\log \log N}{\sqrt{d \log N}}}} = 0,
\end{align*}
where $\sigma^2$ and $a$ can be computed analytically according to \Cref{ex:linGauss}. Hence, to check whether these results apply, we want to look at functions of the form
\begin{align}
  \mbox{(\Cref*{prop:GenDomke})} \quad & c_1 \mapsto \mathcal{L}_d^{(\alpha)}(\theta, \phi; x) - \ell_d(\theta;x) - \frac{\gamma_{\alpha,d}^2}{2 N}+ \frac{c_1}{N} \label{eq:funcThm3LinGauss} \\
  \mbox{(\Cref*{thm:iidRv})} \quad & c_2 \mapsto - d a + \sqrt{d} \sigma {\sqrt{2 \log N}} + {\frac{c_2 \sqrt{d} \log \log N}{\sqrt{\log N}}} \label{eq:funcThm6LinGauss}
\end{align}
and see how well they approximate the behavior of the variational gap $\Delta^{(\alpha)}_{N,d}(\theta, \phi;x)$. Based on \Cref{ex:linGauss}, we are expecting the regime predicted by \Cref{thm:iidRv} to apply as $d$ increases if $N$ does not grow faster than $d^{1/3}$ and this is indeed what we observe in \Cref{fig:linGaussThm6}. 

\begin{figure}[t]
  \begin{tabular}{ccc} 
    \includegraphics[scale=0.29]{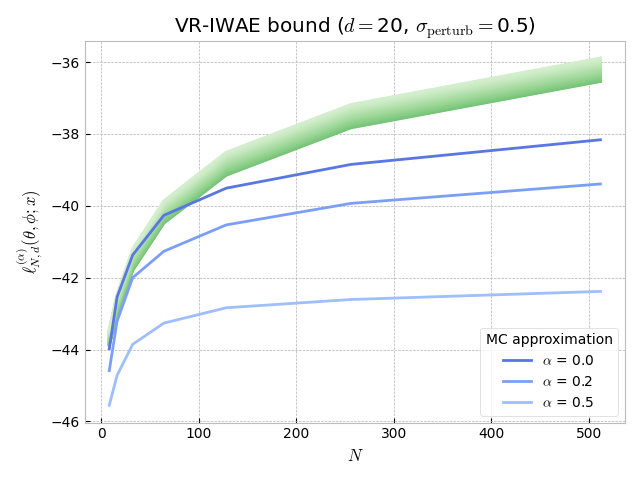} &    
    \includegraphics[scale=0.29]{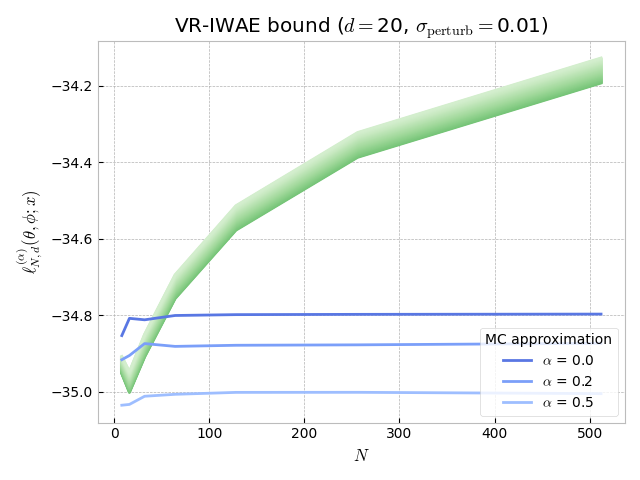} &
    \includegraphics[scale=0.29]{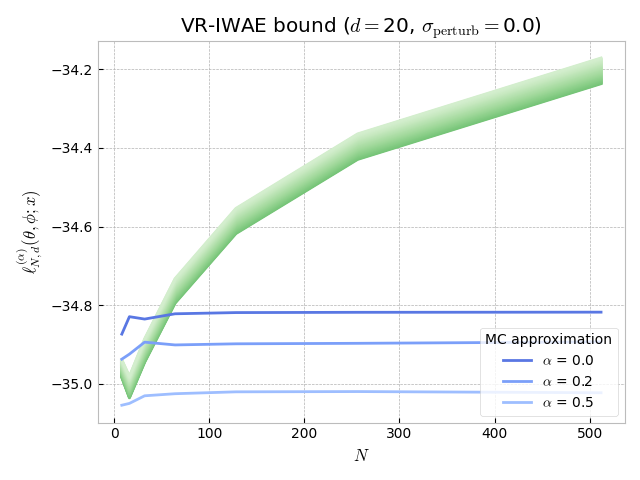} \\
    \includegraphics[scale=0.29]{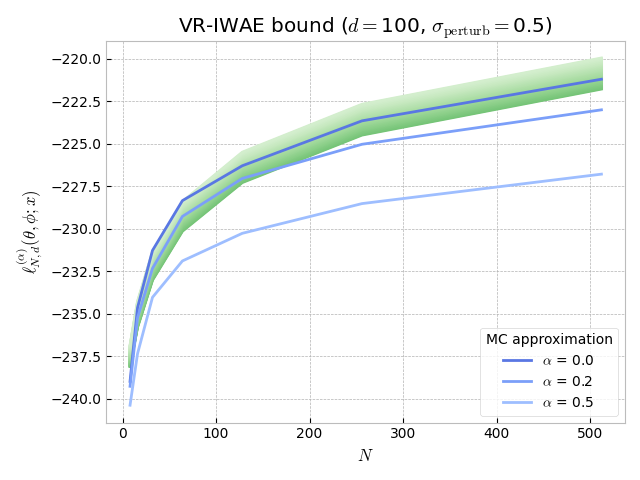} &   
    \includegraphics[scale=0.29]{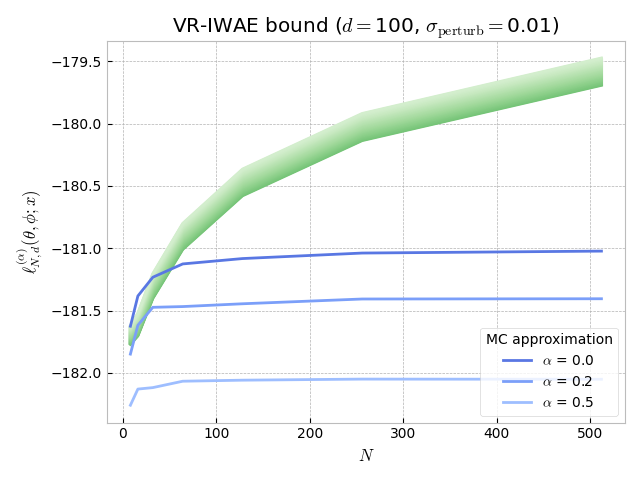} &
    \includegraphics[scale=0.29]{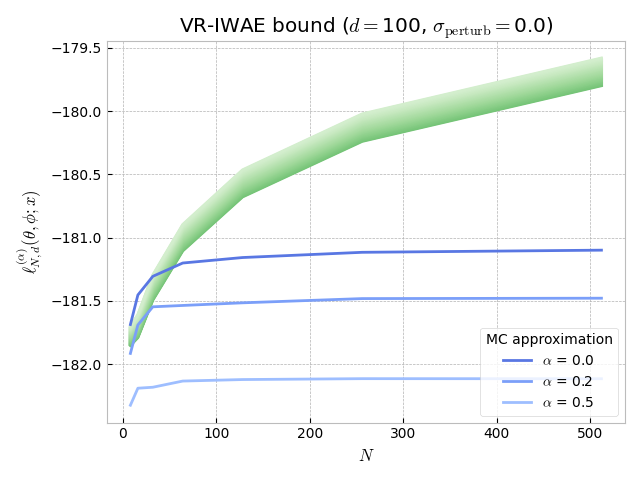} \\
    \includegraphics[scale=0.29]{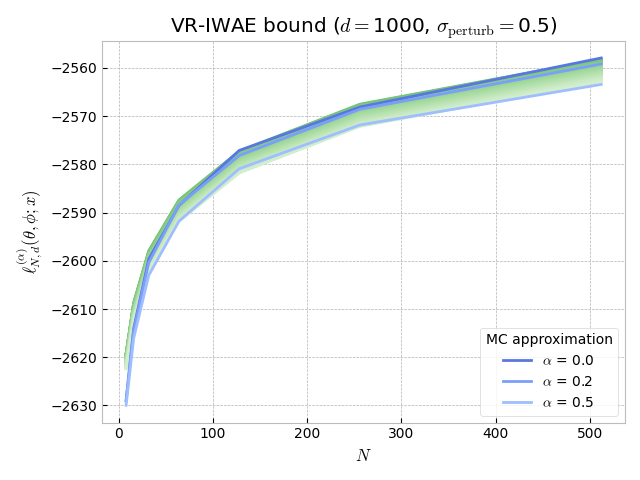} &   
    \includegraphics[scale=0.29]{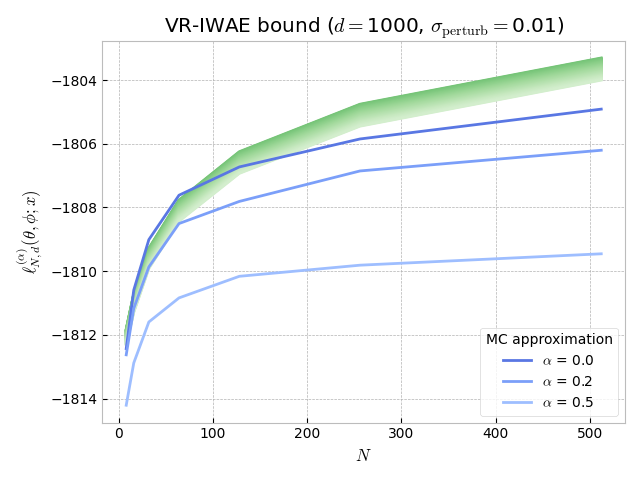} &
    \includegraphics[scale=0.29]{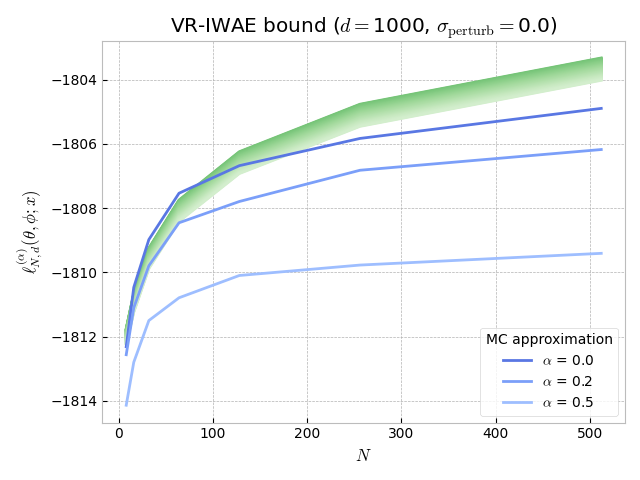} 
  \end{tabular}
  \caption{Plotted in blue is the MC estimate of the VR-IWAE bound $\liren[\alpha][N,d](\theta, \phi; x)$ (averaged over 1000 MC samples) for the linear Gaussian example described in \Cref{subsec:linGaussEx} as a function of $N$, for varying values of $\alpha$ and of $d$ and with three different initializations of $(\theta, \phi)$. Plotted in green are curves of the form \eqref{eq:funcThm6LinGauss} with tailored values of $c_2$.} \label{fig:linGaussThm6}
\end{figure}

While this process is noticeably quicker for an initialization that is far from the optimum $(\sigmapert = 0.5)$, all three initializations considered here eventually exhibit the behavior predicted by \Cref{thm:iidRv} as $d$ further increases. As already mentioned in \Cref{subsec:BeyondLogNormal}, this sends the important message that the VR-IWAE bound can strongly deteriorate as $d$ increases due to a mismatch between the targeted density and its variational approximation, even though the parameters themselves are optimal. 

As for \Cref{prop:GenDomke}, we obtain that this theorem applies in low to medium dimensions when the value of $\alpha$ is well-chosen and/or the parameters are close to being optimal, but fails as $d$ increases unless we use an unpractical amount of samples $N$ (see \Cref{subsub:DomkeLinGaussExApp}). \newline

We now want to get insights regarding the training of the VR-IWAE bound in practice. We follow the methodology used in \cite{rainforth2018tighter}, which looked into the convergence of the SNR for the numerical example considered here in the specific case of the IWAE bound ($\alpha = 0$). Our goal is thus to check whether we can observe the SNR advantages when $\alpha > 0$ predicted by \Cref{prop:SNRconvergence} in the reparameterized case. 

Let us decompose $\theta$ as $(\theta_\ell)_{1 \leq \ell \leq d}$ and $\phi$ as $(\phi_{\ell'})_{1 \leq \ell' \leq d+1}$. We then look at the reparameterized estimated gradients of the VR-IWAE bound $(\delta_{1, N}^{(\alpha)}(\theta_{\ell}))_{1 \leq \ell \leq d}$ and $(\delta_{1, N}^{(\alpha)}(\phi_{\ell'}))_{1 \leq \ell' \leq d +1}$ defined in \eqref{eq:deltaMNthetaell} and \eqref{eq:deltaMNphiell} respectively as a function of $N$, for varying values of $\alpha$, varying values of $d$ and for the two initializations $\sigmapert = 0.01$ and $\sigmapert = 0.5$. The results are shown in Figure \ref{fig:linear_gaussian_vr_iwae_grad_snr_theta} (resp. Figure \ref{fig:linear_gaussian_vr_iwae_grad_snr_phi_rep}) and they have been obtained by randomly selecting $10$ indexes $\ell$ ranging between $1$ and $d$ and averaging over the resulting SNR$(\delta_{1, N}^{(\alpha)}(\theta_{\ell}))$ values (resp. by randomly selecting $10$ indexes $\ell'$ ranging between $1$ and $d+1$ and averaging over the resulting SNR$(\delta_{1, N}^{(\alpha)}(\phi_{\ell'}))$ values). Theoretical lines have also been added to Figures \ref{fig:linear_gaussian_vr_iwae_grad_snr_theta} and \ref{fig:linear_gaussian_vr_iwae_grad_snr_phi_rep} in order to reflect the asymptotic regimes predicted by \Cref{prop:SNRconvergence}.
\begin{figure}
  \begin{tabular}{ccc}
    \includegraphics[scale=0.29]{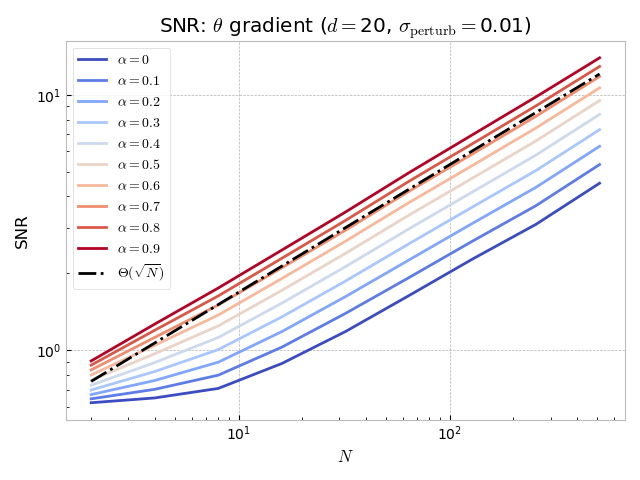} & 
    \includegraphics[scale=0.29]{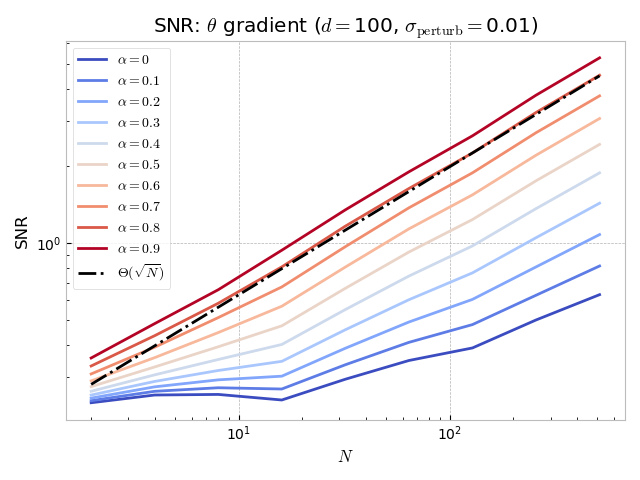} & 
    \includegraphics[scale=0.29]{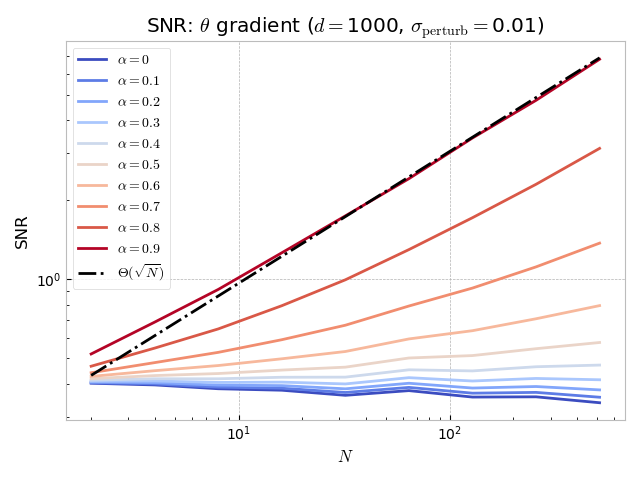} \\
    \includegraphics[scale=0.29]{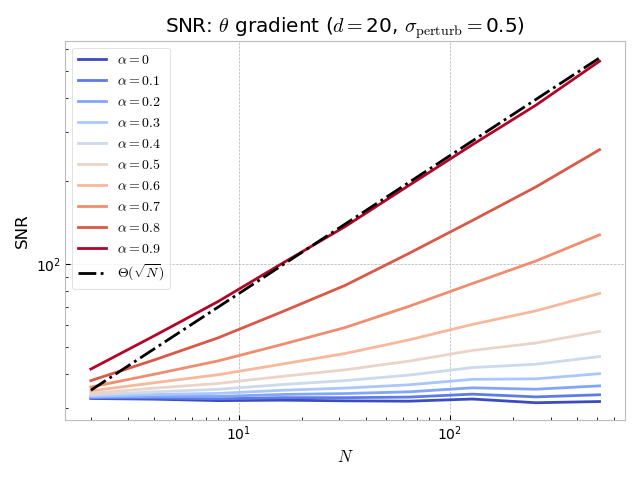} & 
    \includegraphics[scale=0.29]{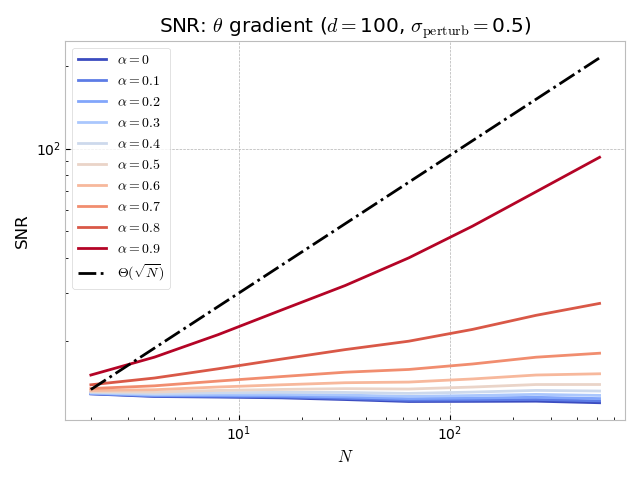} & 
    \includegraphics[scale=0.29]{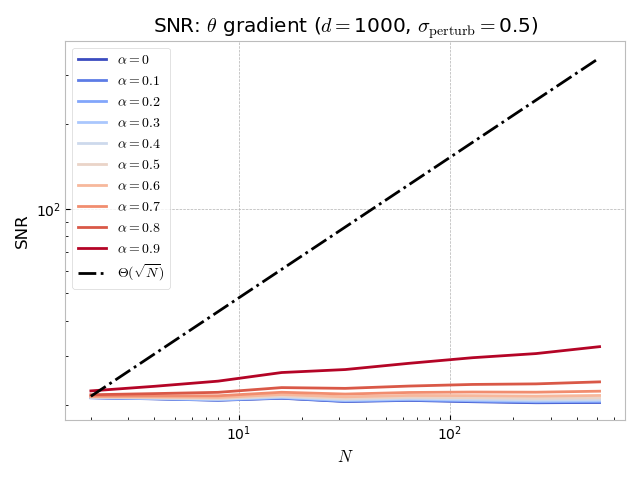}
  \end{tabular}
  \caption{Plotted is the SNR of the generative network ($\theta$) gradients in the reparameterized case (computed over 1000 MC samples) for the linear Gaussian example described in \Cref{subsec:linGaussEx} as a function of $N$, for varying values of $\alpha$ and of $d$, for a randomly selected datapoint $x$ and for 10 different initializations of the parameters $(\theta, \phi)$. \label{fig:linear_gaussian_vr_iwae_grad_snr_theta}}
\end{figure}

\begin{figure}[ht!]
  \begin{tabular}{ccc}
    \includegraphics[scale=0.29]{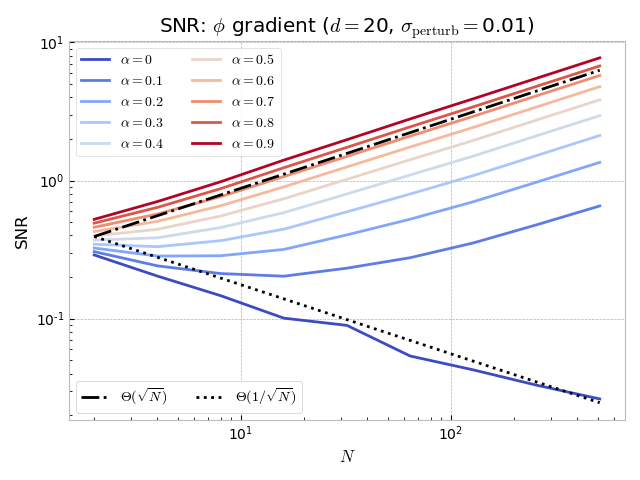} & 
    \includegraphics[scale=0.29]{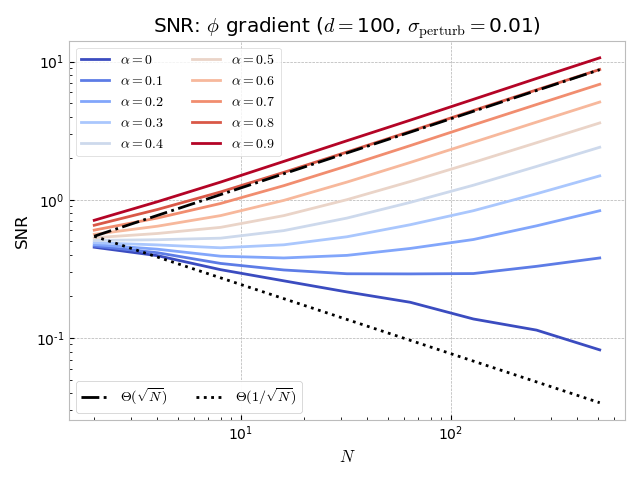} & 
    \includegraphics[scale=0.29]{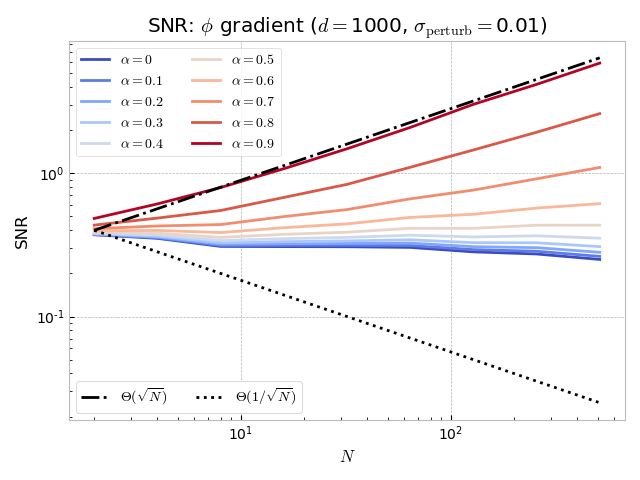} \\
    \includegraphics[scale=0.29]{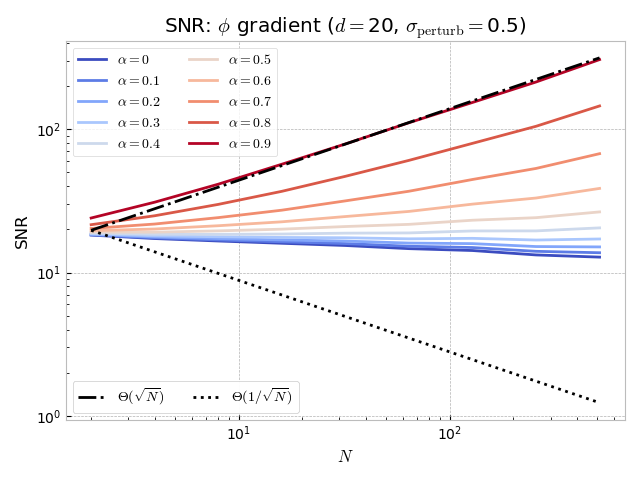} & 
    \includegraphics[scale=0.29]{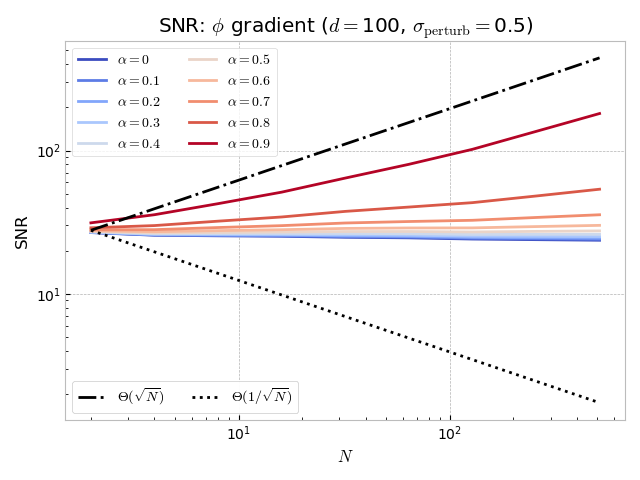} & 
    \includegraphics[scale=0.29]{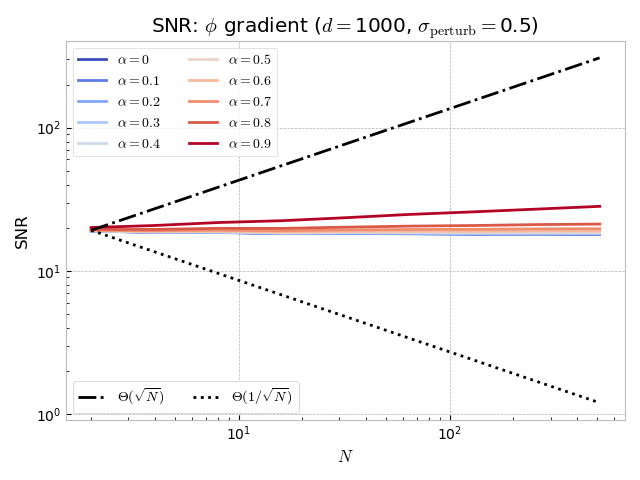}
  \end{tabular}
  \caption{Plotted is the SNR of the inference network ($\phi$) gradients in the reparameterized case (computed over 1000 MC samples) for the linear Gaussian example described in \Cref{subsec:linGaussEx} as a function of $N$, for varying values of $\alpha$ and of $d$, for a randomly selected datapoint $x$ and for 10 different initializations of the parameters $(\theta, \phi)$. \label{fig:linear_gaussian_vr_iwae_grad_snr_phi_rep}}
\end{figure}
Observe that, in the favourable setting of low to medium dimensions with a small perturbation near the optimum (that is $d \in \lrcb{20, 100}$ with $\sigmapert = 0.01$), the asymptotic rates predicted by \Cref{prop:SNRconvergence} for the SNR match the observed rates. In particular, the SNR of the inference network gradients vanishes for $\alpha = 0$ while it does not for $\alpha > 0$, which showcases the potential benefits of using the VR-IWAE bound with $\alpha>0$ instead of the IWAE bound. More generally, increasing $\alpha$ increases the SNR of both the generative and the inference networks, with what seems to be a monotonic increase with $\alpha$. 

However, the improvement in SNR for both the generative and inference networks becomes less pronounced as we get further away from the optimum ($\sigmapert = 0.5$) and/or increase $d$ ($d = 1000$). We relate this behavior to the weight collapse effect established in \Cref{thm:iidRv} and anticipate that observing the asymptotic rates predicted by \Cref{prop:SNRconvergence} requires an unpractical amount of samples $N$ as $d$ increases, regardless of the value of $\alpha \in [0,1)$. Note that the use of doubly-reparameterized gradient estimators for $\phi$ mitigates the decay in SNR (see Figure \ref{fig:linear_gaussian_vr_iwae_grad_snr_phi_drep} of \Cref{subsub:AddExpLinGauss}).

Lastly, the behavior of the VR-IWAE bound as well as the SNR behavior of its gradient estimators are not the only way to measure the success of gradient-based methods involving the VR-IWAE bound. For example, we observe that while increasing $\alpha$ does not lower the Mean Squared Error (MSE) for log-likelihood estimation, it can be useful in lowering the MSE of the $\theta$ gradient estimates (see Figures \ref{fig:linear_gaussian_vr_iwae_p_grad_mse_against_N} and \ref{fig:linear_gaussian_vr_iwae_mse_against_N} of \Cref{subsub:AddExpLinGauss}). We now move on to our third and final numerical experiment, in which we examine a real-data scenario.

\subsection{Variational auto-encoder}
\label{subsec:RealData}

We consider the case of a variational auto-encoder (VAE) model designed to generate MNIST digits with a $d$-dimensional latent space, where $p_\theta(z)$ is a fixed standard Gaussian distribution, $p_\theta(x | z)$ is a product over the output dimensions of independent Bernoulli random variables with logits $\pi_\theta(z)$, $q_\phi(z | x) = \mathcal{N}(z; \mu_\phi(x), \sigma_\phi(x))$ and the functions $\pi_\theta(z)$ and $(\mu_\phi(x), \sigma_\phi(x))$ are parameterized by neural networks. More precisely, both the encoding and decoding networks are MLPs with two hidden layers of size 200 and $\tanh$ nonlinearities.

We first want to investigate whether the distribution of the weights appears to become log-normal in this setting as the dimension of the latent space $d$ increases. 

To verify this claim empirically, we randomly select a datapoint $x$ in the testing set and for $d \in \lrcb{5, 10, 50, 100, 1000, 5000}$ we randomly generate some model parameters $(\theta, \phi)$, before drawing $N=1 000 000$ (unnormalized) weight samples. For each $d$, we then normalize the weights and plot a histogram of the resulting log-weight distribution, alongside with a QQ-plot to test the normality assumption of those log-weights. The results are shown in \Cref{fig:lognormalweightsdistribution} and they illustrate the fact that the weights tend to become log-normal as $d$ increases (and similar plots can be obtained for other randomly selected datapoints and other initializations of the parameters $(\theta, \phi)$).

\begin{figure}[!ht]
  \begin{tabular}{ccc}
    \includegraphics[scale=0.3]{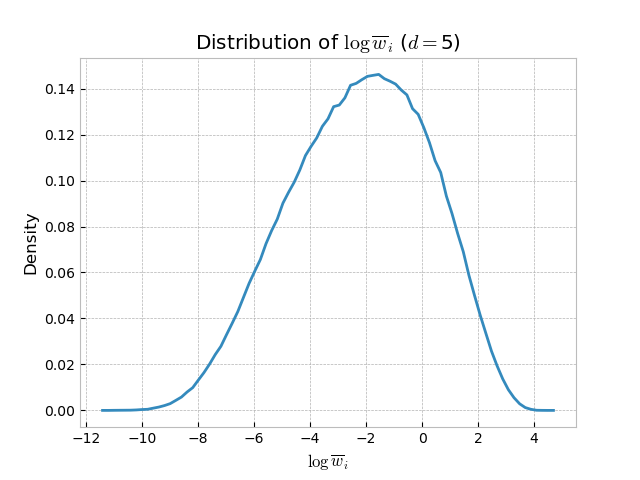} & \includegraphics[scale=0.3]{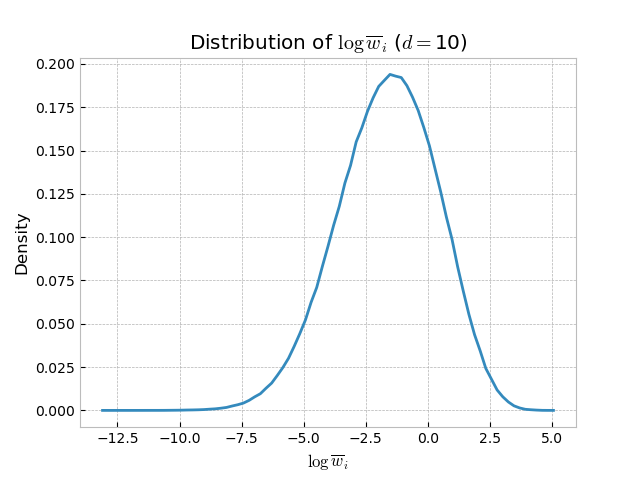} & \includegraphics[scale=0.3]{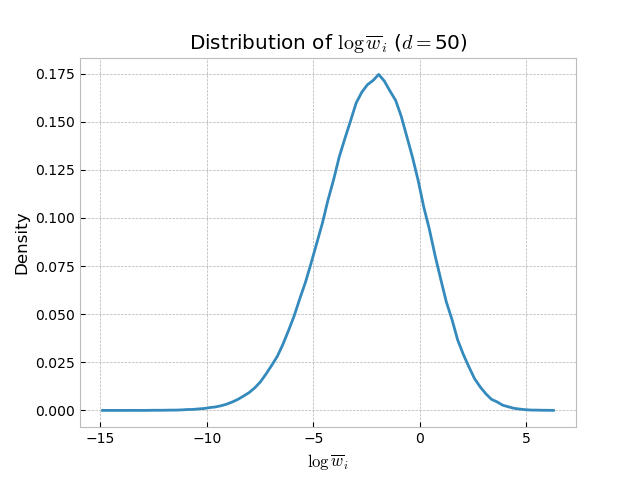} \\ 
    \includegraphics[scale=0.3]{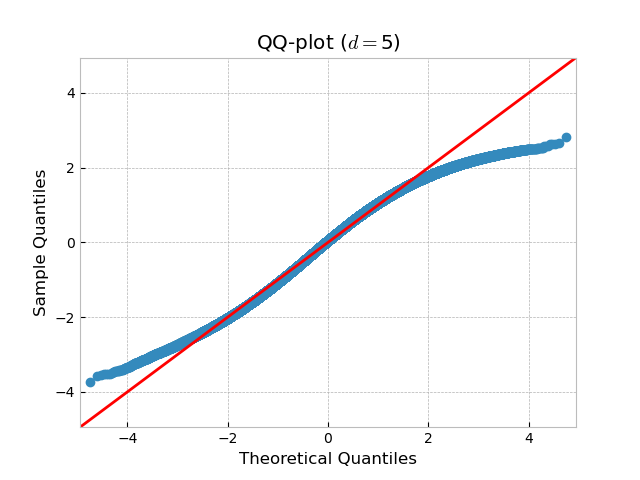} & \includegraphics[scale=0.3]{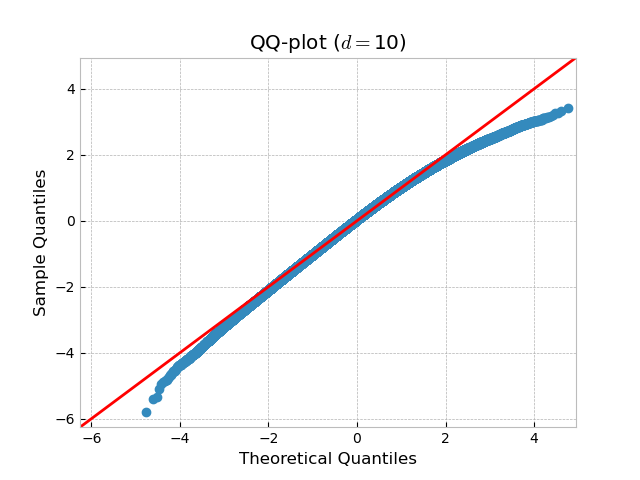} & \includegraphics[scale=0.3]{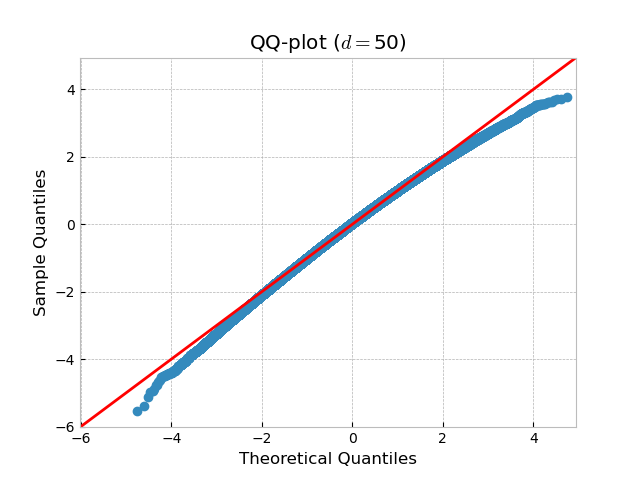} \\ 
   \includegraphics[scale=0.3]{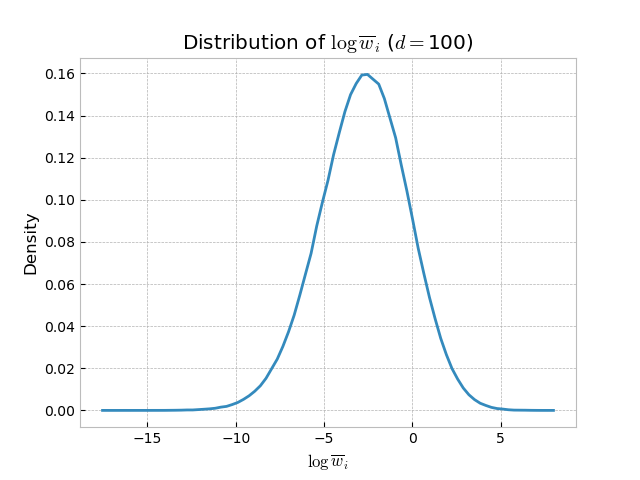} & \includegraphics[scale=0.3]{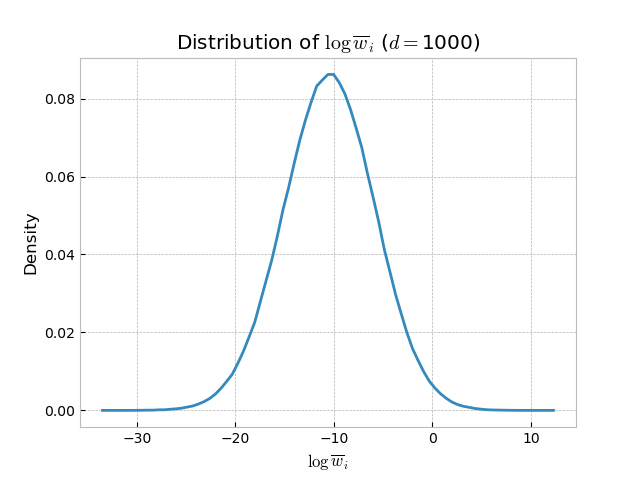} & \includegraphics[scale=0.3]{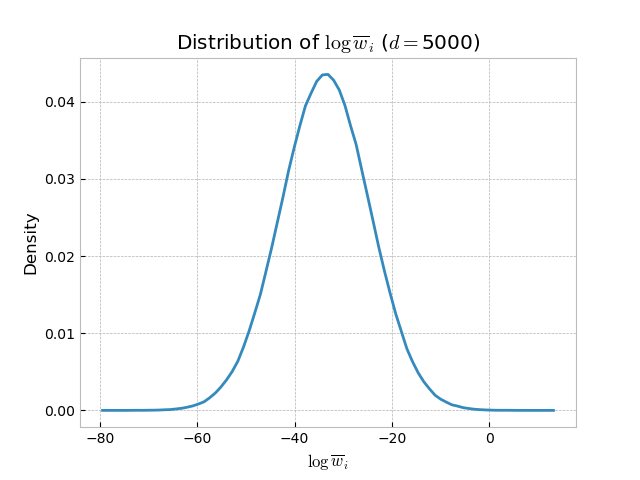}  \\
    \includegraphics[scale=0.3]{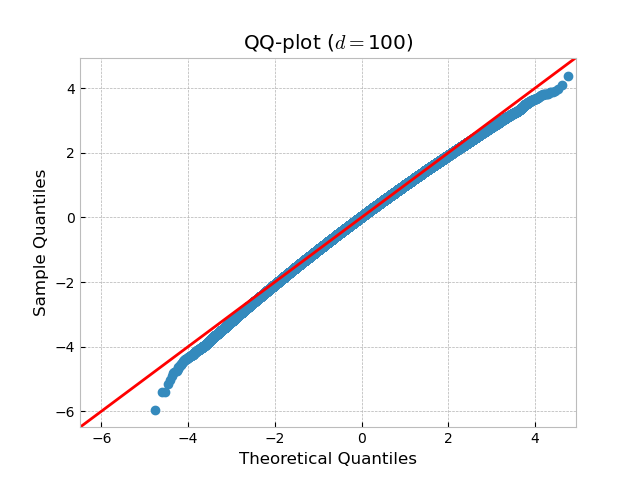} & \includegraphics[scale=0.3]{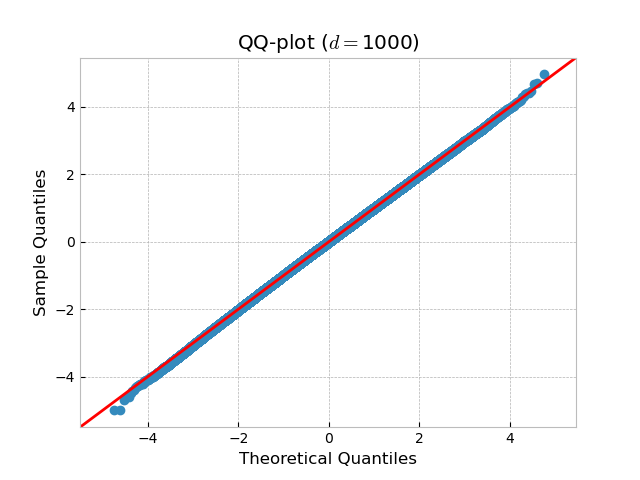} & \includegraphics[scale=0.3]{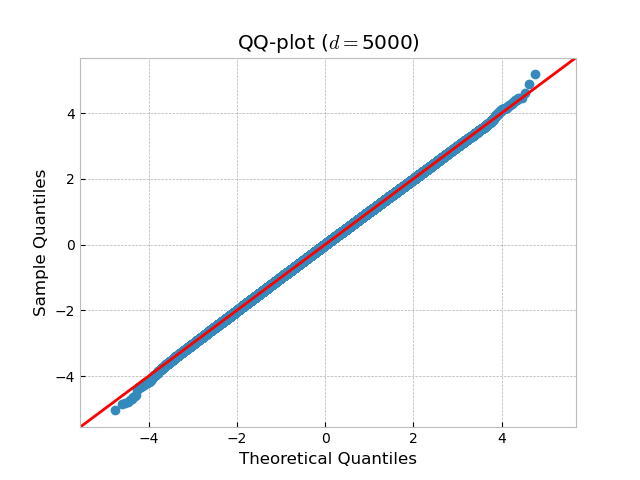}
  \end{tabular}
  \caption{Plotted is the distribution of $\log \overline{w}_i$ and the corresponding QQ-plot for the VAE considered in \Cref{subsec:RealData}, for a randomly selected datapoint $x$ in the testing set, randomly generated model parameters $(\theta, \phi)$ and for varying values of $d$.}
  \label{fig:lognormalweightsdistribution}
\end{figure}

From there, we want to check the validity of our asymptotic results. To do so, a first comment is that, regardless of the distribution of the weights, \Cref{prop:GenDomke} predicts the following: for all $\alpha \in [0,1)$,
%\begin{align} 
%  &\Delta^{(\alpha)}_{N,d}(\theta, \phi;x) = \mathcal{L}_d^{(\alpha)}(\theta, \phi; x) - \ell_d(\theta;x) - \frac{\gamma_{\alpha,d}^2}{2 N}+o\left(\frac{1}{N}\right), \label{eq:OneOverNGenDomkeVAE}
%  \end{align}
%or equivalently
\begin{align} 
  &\liren[\alpha][N,d](\theta, \phi;x) = \mathcal{L}_d^{(\alpha)}(\theta, \phi; x) - \frac{\gamma_{\alpha,d}^2}{2 N}+o\left(\frac{1}{N}\right),  \label{eq:OneOverNGenDomkeVAE2}
  \end{align}
where, $\liren[\alpha][N,d](\theta, \phi;x)$ denotes the VR-IWAE bound, $\mathcal{L}_d^{(\alpha)}(\theta, \phi; x)$ the VR-bound and $\gamma_{\alpha,d}^2 = (1-\alpha)^{-1} \mathbb{V}_{Z \sim q_\phi} (\Rtalpha(Z))$ (and we have emphasized the dependency in $d$ in each of those terms). If we further make the assumption that the weights are of the form \eqref{eq:approxlognormalweights} (which appears to approximately be the case as the dimension $d$ increases as per Figure~\ref{fig:lognormalweightsdistribution}), then \Cref{thm:iidRv} predicts under \ref{hypBickel} that: for all $\alpha \in [0,1)$,
%\begin{align*}
%  \lim_{\substack{N,d \to \infty \\ \log N / d^{1/3} \to 0}} \Delta^{(\alpha)}_{N,d}(\theta, \phi;x) + d a \lr{1 - \frac{\sigma}{a} \sqrt{\frac{2 \log N}{ d}} + O \lr{\frac{\log \log N}{\sqrt{d \log N}}}} = 0
%  \end{align*}
%or equivalently that 
\begin{align}
  \lim_{\substack{N,d \to \infty \\ \log N / d^{1/3} \to 0}} \liren[\alpha][N,d](\theta, \phi;x) - \mathrm{ELBO}_d(\theta, \phi;x) - \sqrt{d} \sigma \sqrt{2 \log N} + O\lr{\frac{\sqrt{d} \log \log N}{\sqrt{\log N}}} = 0. \label{eq:LogNormGenDomkeVAE2}
\end{align}
Here we have emphasized the dependency in $d$ in $\mathrm{ELBO}_d (\theta, \phi; x)$ and we have also used the fact that $\mathrm{ELBO}_d(\theta, \phi;x) - \ell_d(\theta;x) = - d a$ (as previously stated, this follows from taking the expectation in \eqref{eq:WNGen} from the proof of \Cref{prop:DeltaGen} in \Cref{subsec:proof:prop:deltaGen}). Hence, to check whether these results apply, we want to look at functions of the form
\begin{align}
  \mbox{(\Cref*{prop:GenDomke})} \quad & c_1 \mapsto \mathcal{L}_d^{(\alpha)}(\theta, \phi; x) - \frac{\gamma_{\alpha,d}^2}{2 N}+ \frac{c_1}{N} \quad \label{eq:funcThm3VAE} \\
  \mbox{(\Cref*{thm:iidRv})} \quad & c_2 \mapsto \mathrm{ELBO}_d (\theta, \phi; x) + \sqrt{d} \sigma {\sqrt{2 \log N}} + {\frac{c_2 \sqrt{d} \log \log N}{\sqrt{\log N}}} \label{eq:funcThm5VAE}
\end{align}
and see how well they approximate the behavior of the VR-IWAE bound $\liren[\alpha][N,d](\theta, \phi; x)$. Although the functions above contain unknown terms, those terms can all be estimated: the VR bound can be estimated using MC sampling as in \eqref{eq:biasedVRbound} and so can the ELBO. As for $\sigma$, it can be estimated from the sample standard deviation of the log-weights (and $\gamma_{\alpha,d}^2$ can be estimated in a similar fashion).

Note as a side remark that we are considering the VR-IWAE bound as the quantity of interest whose behavior shall be mimicked by \eqref{eq:OneOverNGenDomkeVAE2} or \eqref{eq:LogNormGenDomkeVAE2} (through \eqref{eq:funcThm3VAE} 
or \eqref{eq:funcThm5VAE}). Indeed, while we were working with the variational gap in our previous numerical experiments, computing this quantity requires us to estimate both the VR-IWAE bound and the log-likelihood here, which would have incurred an additional source of randomness that we have been able to avoid in both \eqref{eq:OneOverNGenDomkeVAE2} and \eqref{eq:LogNormGenDomkeVAE2}. \newline

\begin{figure}
  \begin{tabular}{ccc}
  \includegraphics[scale=0.3]{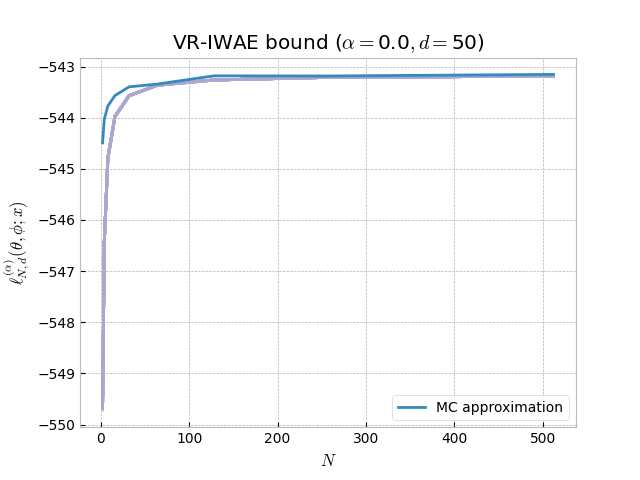} &   \includegraphics[scale=0.3]{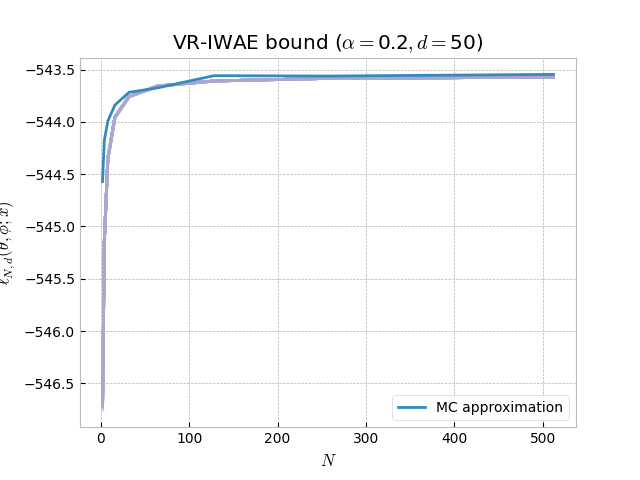} & \includegraphics[scale=0.3]{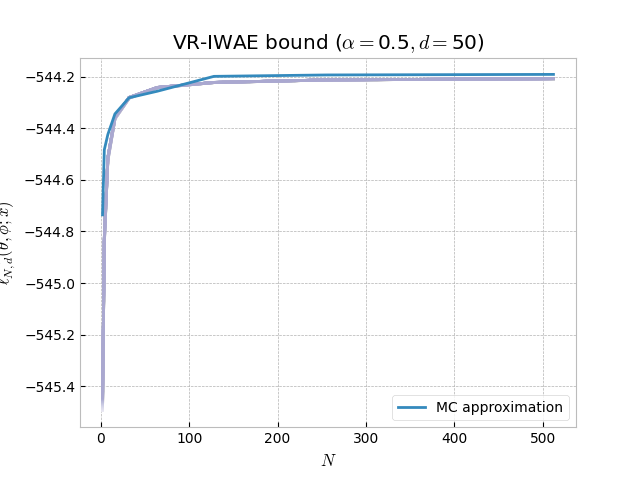} \\
  \includegraphics[scale=0.3]{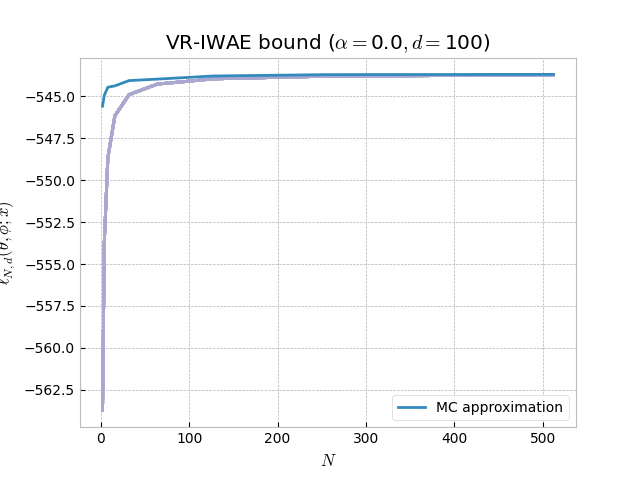} &   \includegraphics[scale=0.3]{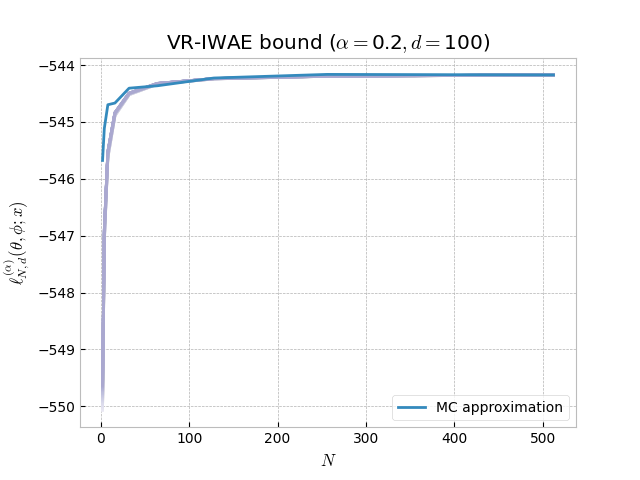} & \includegraphics[scale=0.3]{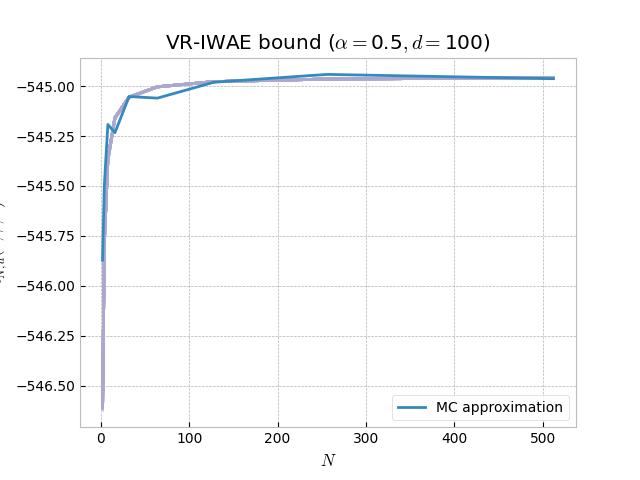} \\
  \includegraphics[scale=0.3]{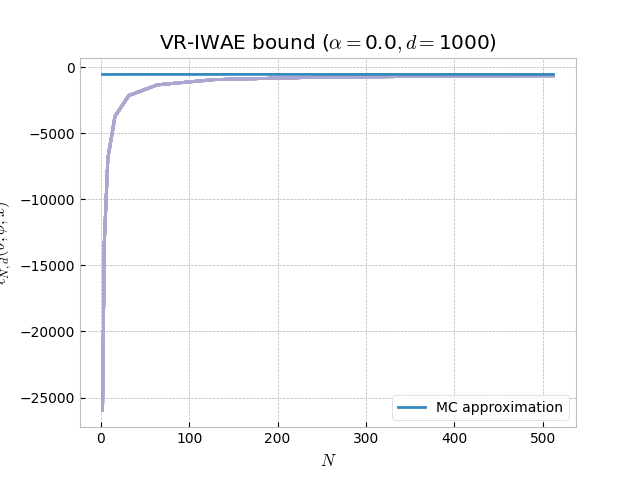} &   \includegraphics[scale=0.3]{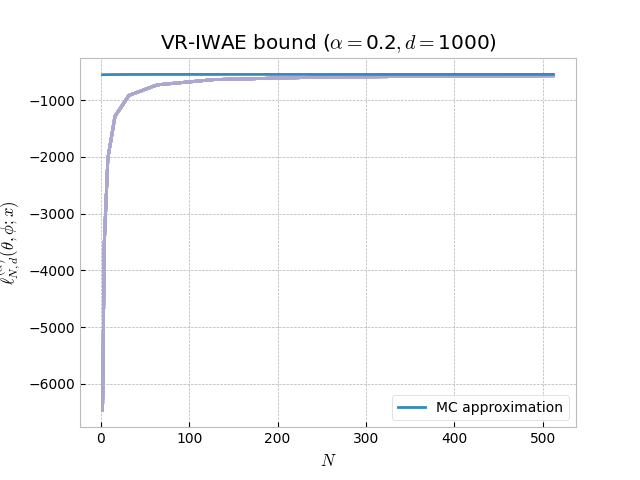} & \includegraphics[scale=0.3]{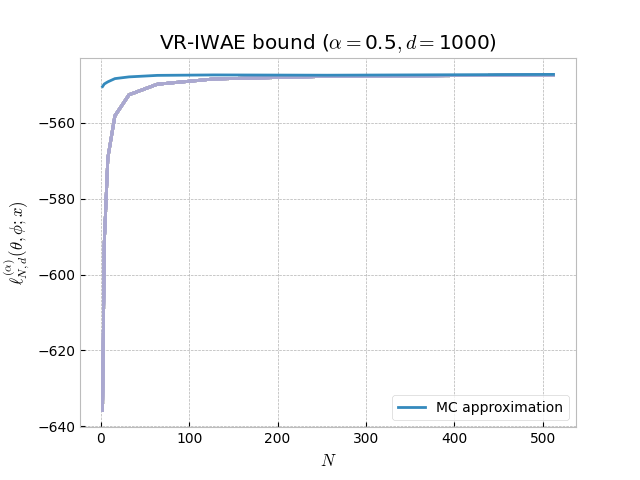} \\
  \end{tabular}
  \caption{Plotted in blue is the MC estimate of the VR-IWAE bound $\liren[\alpha][N,d](\theta, \phi; x)$ (averaged over 100 MC samples) for the VAE considered in \Cref{subsec:RealData}, for a randomly selected datapoint $x$ in the testing set, randomly generated model parameters $(\theta, \phi)$ and for varying values of $\alpha$ and of $d$. Plotted in purple are curves of the form \eqref{eq:funcThm3VAE} with tailored values of $c_1$.}
  \label{fig:vrIWAEinVAElow}
\end{figure}

Based on \Cref{fig:lognormalweightsdistribution}, we expect two situations to arise at this stage: (i) the asymptotic regime suggested by \Cref{prop:GenDomke} captures the behavior of the VR-IWAE bound in low to medium dimensions and (ii) the asymptotic regime predicted by \Cref{thm:iidRv} is accurate as $d$ increases and $N$ does not grow faster than $d^{1/3}$. 

This is exactly what we observe in Figures \ref{fig:vrIWAEinVAElow} and \ref{fig:vrIWAEinVAEhigh}, in which $\sigma$ is estimated with the $1 000 000$ (unnormalized) weight samples used to build \Cref{fig:lognormalweightsdistribution} and so are the VR bound and the ELBO (additional plots are also available in Figures \ref{fig:vrIWAEinVAElowApp} and \ref{fig:vrIWAEinVAEhighApp} of \Cref{subsec:VAEexApp}). %in which the parameters $(\theta, \phi)$ are shared across different values of $\alpha$ for each $d$, the VR-IWAE bound is estimated using $N$ samples with $N \in \lrcb{2^j ~: ~j = 1 \ldots 9}$, .
In particular, we see in Figure \ref{fig:vrIWAEinVAElow} that the asymptotic regime of \Cref{prop:GenDomke} mimics the behavior of the VR-IWAE bound in low to medium dimensions as long as $\gamma_{\alpha, d}^2$ does not grow too quickly with $d$ (and we already observe a mismatch between the two for $\alpha = 0$ and $d = 100$). %This lends credence to \Cref{lem:discussConditions} and it illustrates how $\alpha$ plays a role in ensuring that the variance term appearing in $\gamma_{\alpha, d}^2$ does not grow too quickly in low to medium dimensions. 

Nevertheless, the VR-IWAE bound ends up straying away from the behavior predicted by \Cref{prop:GenDomke} as $d$ increases unless $N$ becomes impractically large (with the particularity that this process happens slower as $\alpha$ increases). We then see on Figure \ref{fig:vrIWAEinVAEhigh} that, as $d$ increases to reach high-dimensional settings so that the ratio $\log N/d^{1/3}$ becomes small for the values of $N$ considered here, the behavior predicted by \Cref{thm:iidRv} starts to emerge. \newline

\begin{figure}
  \begin{tabular}{ccc}
    \includegraphics[scale=0.29]{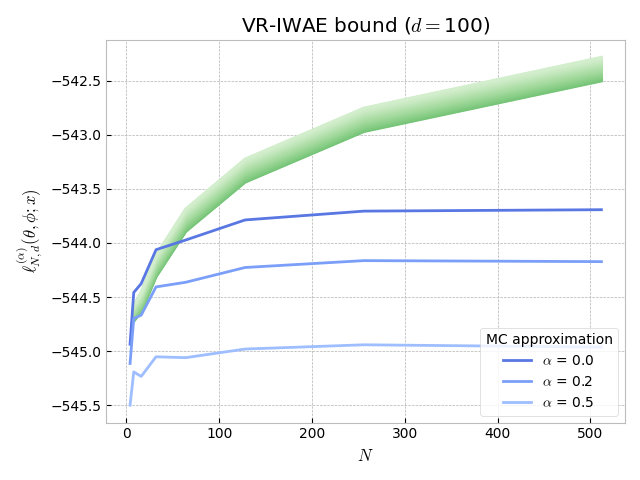}
    & \includegraphics[scale=0.29]{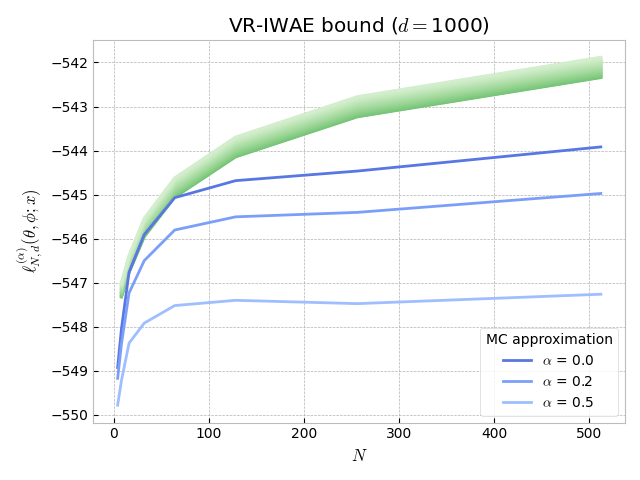}
    & \includegraphics[scale=0.29]{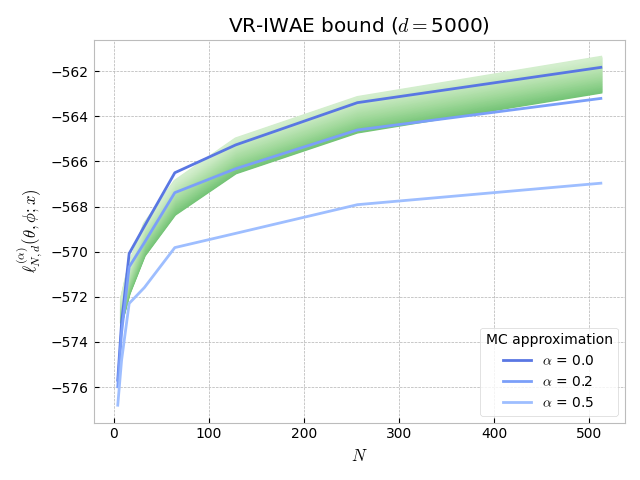} \\
  \end{tabular}
  \caption{Plotted in blue is the MC estimate of the VR-IWAE bound $\liren[\alpha][N,d](\theta, \phi; x)$ (averaged over 100 MC samples) for the VAE considered in \Cref{subsec:RealData}, for a randomly selected datapoint $x$ in the testing set, randomly generated model parameters $(\theta, \phi)$ and for varying values of $\alpha$ and of $d$. Plotted in green are curves of the form \eqref{eq:funcThm5VAE} with tailored values of $c_2$.}
  \label{fig:vrIWAEinVAEhigh}
\end{figure}

We now look into the training of the VR-IWAE bound and more specifically into the SNR in medium to high dimensions at initialization, since this scenario corresponds to situations where the VR-IWAE bound seems to resemble more and more the behavior predicted by \Cref{thm:iidRv} (as observed in Figure \ref{fig:vrIWAEinVAEhigh}). Following the methodology from the previous subsection, the results are presented in Figures~\ref{fig:vae_vr_iwae_grad_snr_theta} and \ref{fig:vae_vr_iwae_grad_snr_phi_rep}, in which we have plotted the SNR for the generative network and for the inference network respectively in the reparameterized case alongside theoretical lines that reflect the asymptotic regimes predicted by \Cref{prop:SNRconvergence}. 

\begin{figure}[t]
  \begin{tabular}{ccc}
    \includegraphics[scale=0.29]{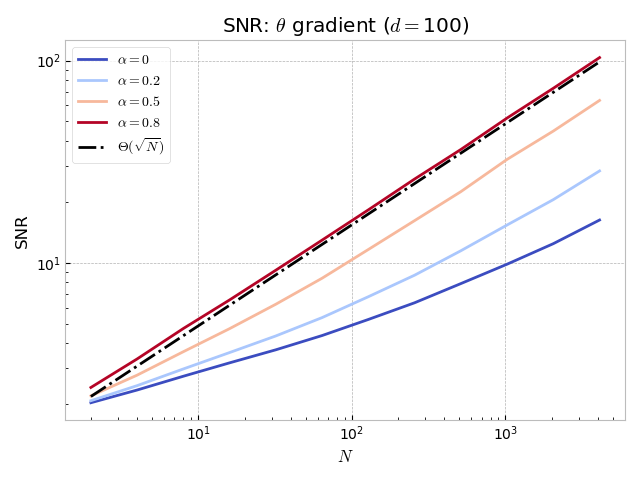}  &
    \includegraphics[scale=0.29]{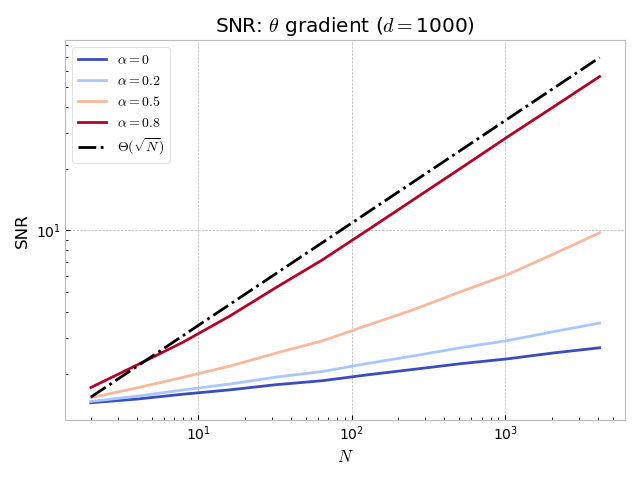}  &
    \includegraphics[scale=0.29]{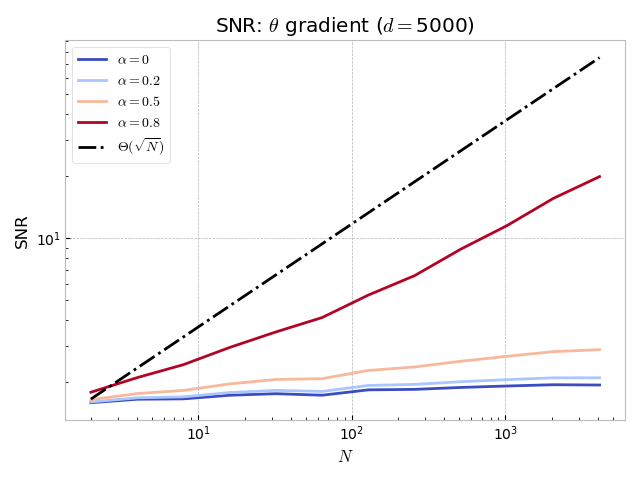} \\
  \end{tabular}
  \caption{Plotted is the SNR of the generative network ($\theta$) gradients in the reparameterized case (computed over 10000 MC samples) for the VAE considered in \Cref{subsec:RealData} as a function of $N$, for a randomly selected datapoint $x$ in the testing set, randomly generated model parameters $(\theta, \phi)$ and for varying values of $\alpha$ and of $d$. \label{fig:vae_vr_iwae_grad_snr_theta}}
\end{figure}

\begin{figure}[h]
  \begin{tabular}{ccc}
    \includegraphics[scale=0.29]{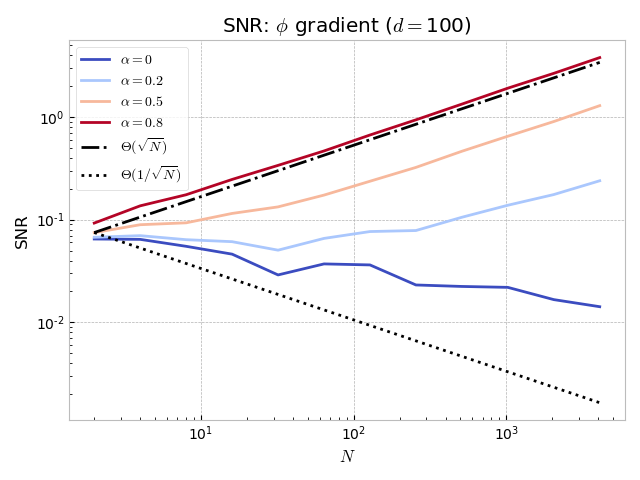} &
    \includegraphics[scale=0.29]{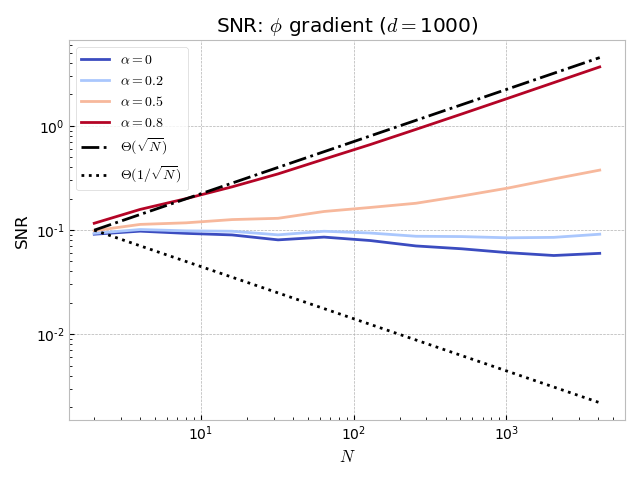} &
    \includegraphics[scale=0.29]{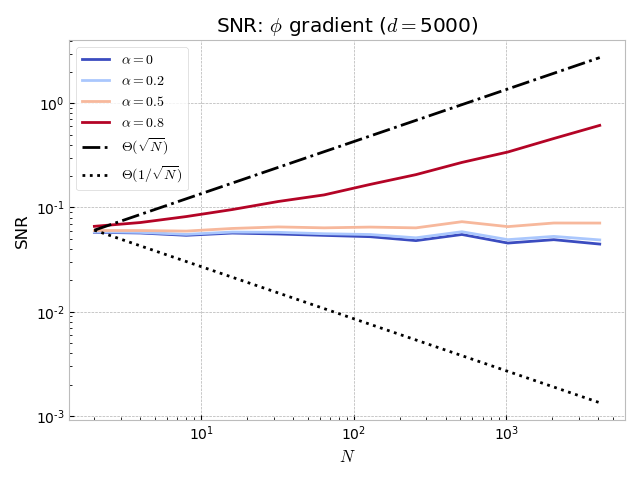}  \\
  \end{tabular}
  \caption{Plotted is the SNR of the inference network ($\phi$) gradients in the reparameterized case (computed over 10000 MC samples) for the VAE considered in \Cref{subsec:RealData} as a function of $N$, for a randomly selected datapoint $x$ in the testing set, randomly generated model parameters $(\theta, \phi)$ and for varying values of $\alpha$ and of $d$. \label{fig:vae_vr_iwae_grad_snr_phi_rep}}
\end{figure}

\begin{figure}[h]
  \begin{tabular}{ccc}
    \includegraphics[scale=0.29]{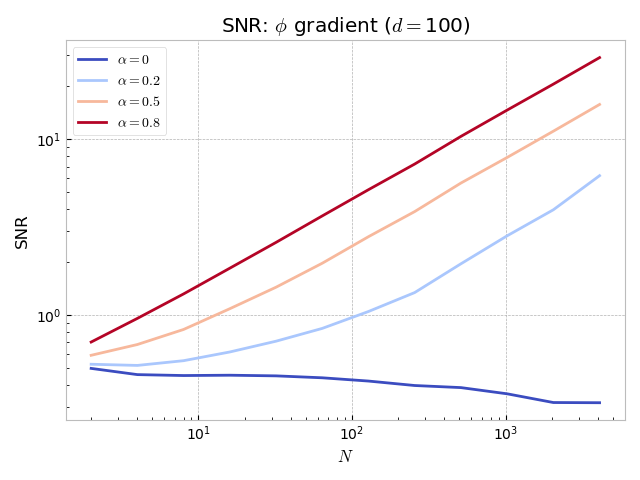} &
    \includegraphics[scale=0.29]{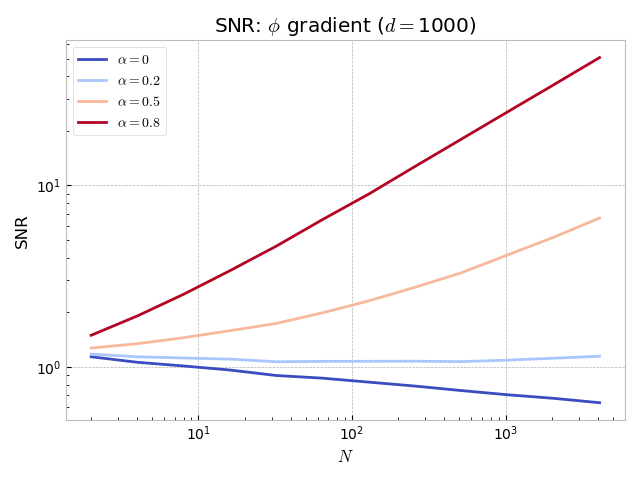} &
    \includegraphics[scale=0.29]{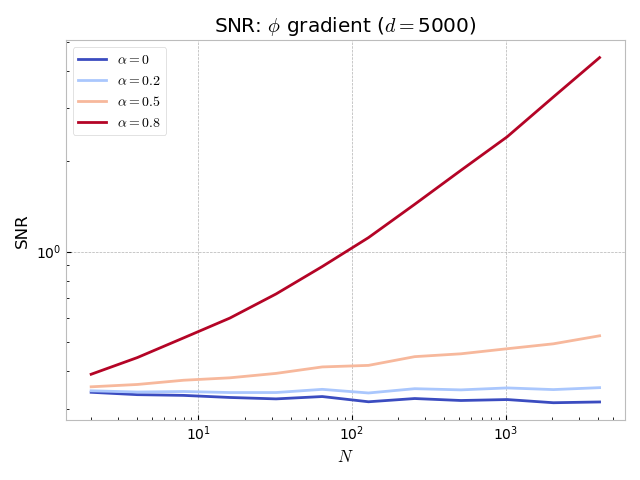}  \\
  \end{tabular}
  \caption{Plotted is the SNR of the inference network ($\phi$) gradients in the doubly-reparameterized case (computed over 10000 MC samples) for the VAE considered in \Cref{subsec:RealData} as a function of $N$, for a randomly selected datapoint $x$ in the testing set, randomly generated model parameters $(\theta, \phi)$ and for varying values of $\alpha$ and of $d$. \label{fig:vae_vr_iwae_grad_snr_phi_drep}}
\end{figure}

As already observed in \Cref{subsec:linGaussEx}, the SNR benefits from setting $\alpha > 0$ and the asymptotic rates predicted by \Cref{prop:SNRconvergence} do not capture the SNR behavior as $d$ increases (unless $N$ is unpractically large and/or we appeal to higher values of $\alpha$). Furthermore, and as we can see in Figure \ref{fig:vae_vr_iwae_grad_snr_phi_drep}, resorting to doubly-reparameterized estimators improves the SNR. \newline

We thus confirmed that approximately log-normal weights can arise in real data scenarios as $d$ increases and that our theoretical study provides a useful framework to capture the impact of the weights on the VR-IWAE bound as a function of $N$, $d$ and~$\alpha$. In line with our empirical findings for the SNR, we also obtained that the asymptotic rates predicted by \Cref{prop:SNRconvergence} match the observed rates in low to medium dimensions and we postulated that the weight collapse occuring in the VR-IWAE bound as $d$ increases may deteriorate the SNR too. \looseness=-1 % (although this deterioration appears to be less quick for doubly-reparameterized gradient estimators)

One aspect that remains unexplored empirically is the role of $M$ in the VR-IWAE bound methodology, and in particular the interplay between $M$ and $N$ in the learning outcome. Indeed, the total number of samples needed per iteration in the gradient descent procedure is $N \times M$, with $M$ being responsible for the usual $1/M$ variance reduction in gradient estimators such as \eqref{eq:deltaMNthetaell}. Intuitively, and following a similar line of reasoning as for $\alpha$, we expect to see a bias-variance tradeoff between increasing $M$ or $N$ while keeping $M \times N$ fixed (we refer to \Cref{subsec:VAEexApp} for details). 
Furthermore, if the weight collapse appearing in the VR-IWAE bound as $d$ increases ends up badly impacting the associated gradient descent, our results indicate that practitioners should either (i) set $N = 1$ and allocate the maximum computational budget to $M$, which in fact corresponds to setting $\alpha = 1$ in the VR-IWAE bound, (ii) find more suitable variational approximations that can capture the complexity within the posterior density or (iii) resort to/construct better gradient estimates (e.g. doubly-reparameterized gradient estimators). %This paves the way for further research into the VR-IWAE bound methodology. \looseness=-1

\section{Conclusion}

In this paper, we formalized the VR-IWAE bound, a variational bound depending on an hyperparameter $\alpha \in [0,1)$ which generalizes the standard IWAE bound ($\alpha = 0$). We showed that the VR-IWAE bound provides theoretical guarantees behind various VR bound-based schemes proposed in the alpha-divergence variational inference literature and identified other additional desirable properties of this bound. 

We then provided two complementary analyses of the variational gap, that is of the difference between the VR-IWAE bound and the marginal log-likelihood. The first analysis shed light on how $\alpha$ may play a role in reducing the variational gap. We then proposed a second analysis to better capture the behavior of the variational gap in high-dimensional  scenarios, establishing that the variational gap suffers in this case from a damaging weight collapse phenomenon for all $\alpha \in [0,1)$. Lastly, we illustrated our theoretical results over several toy and real-data examples.

Overall, our work provides foundations for improving the IWAE and VR methodologies and we now state potential directions of research to extend it. Firstly, one may investigate whether the weight collapse behavior applies beyond the cases we highlighted. Looking into how this weight collapse affects the gradient descent procedures associated to the VR-IWAE bound could be a second direction of research. Thirdly, and in order to improve on the VR-IWAE bound methodology beyond the weight collapse phenomenon, one may seek to further build on the fact that the VR-IWAE bound is the theoretically-sound extension of the IWAE bound that originates from the Alpha-Divergence Variational Inference methodology.

\section*{Acknowledgments} 

Kam\'{e}lia Daudel and Arnaud Doucet acknowledge support of the UK Defence Science and Technology Laboratory (Dstl) and Engineering and Physical Research Council (EPSRC) under grant EP/R013616/1. This is part of the collaboration between US DOD, UK MOD and UK EPSRC under the Multidisciplinary University Research Initiative. Joe Benton was supported by the EPSRC Centre for Doctoral Training in Modern Statistics and Statistical Machine Learning (EP/S023151/1). Arnaud Doucet also acknowledges support from the EPSRC grant EP/R034710/1.

\bibliography{references}

\appendix

\section{Deferred proofs of \Cref{sec:VR-IWAEbound}}

\subsection{Extension of $\liren$ to the case $\alpha = 1$}
\label{subsec:extAlpha1}

We prove that under common differentiability assumptions, the following limit holds:
    $$
    \lim_{\alpha \to 1} \liren  (\theta,\phi;x) = \mathrm{ELBO}(\theta, \phi;x).
    $$

\begin{proof}
Setting $f(\alpha) = \int \int \prod_{i = 1}^N q_\phi (z_i) \log \lr{ \frac{1}{N} \sum_{j = 1}^N \w(z_j) ^{1-\alpha} } \rmd z_{1:N}$, the bound $\liren  (\theta,\phi;x)$ can be rewritten as
    $$
    \liren  (\theta,\phi;x) = - \frac{f(\alpha) - f(1)}{\alpha-1}
    $$
    and hence, $\lim_{\alpha \to 1} \liren  (\theta,\phi;x) = - f'(1)$. We then get the desired result by observing that 
    $$
    f'(\alpha) = \int \int \prod_{i = 1}^N q_\phi (z_i) \lr{ \frac{\frac{1}{N} \sum_{j = 1}^N - \log( \bar{w}_{\phi, \theta}(z_j)) \bar{w}_{\phi, \theta}(z_j) ^{1-\alpha}     }{\frac{1}{N} \sum_{j = 1}^N \bar{w}_{\phi, \theta}(z_j) ^{1-\alpha} }} \rmd z_{1:N}
    $$
    and letting $\alpha \to 1$ in the quantity above.
\end{proof}

\subsection{Proof of \Cref{lem:generalBound}}
\label{sec:prooflemgeneralbound}

\begin{proof}[Proof of \Cref{lem:generalBound}] The results for the case $\alpha = 0$ follow from \cite{burda2015importance} and we focus on the case $\alpha \in (0,1)$ in the proof below.
    \begin{enumerate}
        \item One the one hand, \cite[Theorem 1]{li2016renyi} implies that: for all $\alpha \in (0,1)$,
        \begin{align} \label{eq:boundVRELBO}
        \mathcal{L}^{(\alpha)}(\theta, \phi; x) \leq \ell(\theta;x).
        \end{align}
        On the other hand, we obtain from \cite[Theorem 2]{li2016renyi} that:
        \begin{itemize}
            \item For all $N \in \mathbb{N}^\star$ and all $\alpha < 1$
            \begin{align*}
            \liren(\theta, \phi;x) \leq \liren[N+1](\theta, \phi;x) \leq \mathcal{L}^{(\alpha)}(\theta, \phi; x)
            \end{align*}
            which gives \eqref{eq:ineqELBOLnAlpha} when paired with \eqref{eq:boundVRELBO}.
            %where the case of equality is reached if and only if $p_\theta(z|x) = q_\phi(z)$ for $\nu$-almost all $z \in \rset^d$. 
            \item If the function $z \mapsto \w(z)$ is bounded, then $\liren(\theta, \phi;x)$ approaches the VR bound $\mathcal{L}^{(\alpha)}(\theta, \phi; x)$ as $N$ goes to infinity.
        \end{itemize}
    \item Let $1> \alpha_1 > \alpha_2 > 0$. Then, the functions $u \mapsto u^{\frac{1-\alpha_1}{1-\alpha_2}}$ and $u \mapsto u^{1-\alpha_2}$ are concave for all $u > 0$ and hence Jensen's inequality implies
    \begin{align*}
        \liren[\alpha_1](\theta, \phi; x) & = \frac{1}{1-\alpha_1} \int \int \prod_{i=1}^N q_\phi(z_i|x) \log \lr{ \frac{1}{N} \sum_{j=1}^N \lrb{\w(z_i)^{1-\alpha_2}}^{\frac{1-\alpha_1}{1-\alpha_2}}} \rmd z_{1:N} \\
        & \leq \liren[\alpha_2](\theta, \phi; x) \\
        & \leq \liren[0](\theta, \phi; x) 
    \end{align*}
    The desired result \eqref{eq:looseness} follows by using that $\liren[0](\theta, \phi; x) = \liwae (\theta, \phi; x) $. As for the case of equality, it is obtained as the case of equality of Jensen's inequality.

    \item Under the reparameterization trick,
\begin{align*} 
  \liren (\theta,\phi;x)  & =  \frac{1}{1-\alpha} \int \int \prod_{i = 1}^N q(\varepsilon_i) \log \lr{ \frac{1}{N} \sum_{j = 1}^N \w(f(\varepsilon_j, \phi))^{1-\alpha}} \rmd \varepsilon_{1:N}%\label{eq:def_reparam}
  \end{align*} 
leading, under common differentiability assumptions, to
\begin{align*} 
 \frac{\partial}{\partial \phi} \liren  (\theta,\phi;x)  &= \int \int \prod_{i = 1}^N q(\varepsilon_i) \lr{\sum_{j = 1}^N  \frac{\w(f(\varepsilon_j, \phi))^{-\alpha}  \frac{\partial}{\partial \phi} \w(f(\varepsilon_j, \phi)) }{\sum_{k = 1}^N \w(f(\varepsilon_k, \phi))^{1-\alpha}}}   \rmd \varepsilon_{1:N}. 
\end{align*}
The desired result \eqref{reparam:liren} is then obtained using the REINFORCE trick 
$$  \frac{\partial}{\partial \phi} \w(f(\varepsilon_j, \phi)) = \w(f(\varepsilon_j, \phi))  \frac{\partial}{\partial \phi} \log \w(f(\varepsilon_j, \phi))$$ and the unbiased estimator \eqref{reparam:liren:MC} follows immediately.

% & =  \int \int \prod_{i = 1}^N q(\varepsilon_i) \lr{\sum_{j = 1}^N  \frac{\w(f(\varepsilon_j, \phi))^{1-\alpha} }{\sum_{k = 1}^N \w(f(\varepsilon_k, \phi))^{1-\alpha}}   \frac{\partial}{\partial \phi} \log \w(f(\varepsilon_j, \phi))}   \rmd \varepsilon_{1:N}.
%  \end{align*} 
%  where we have used the
\end{enumerate}
\end{proof}

\subsection{Proof of \Cref{prop:SNRconvergence}}
\label{subsec:SNRanalysis:app}

The proof of \Cref{prop:SNRconvergence} is based on the proof of the corresponding result in \cite[Theorem 1]{rainforth2018tighter} (arxiv version of 5 Mar 2019). First, we prove the following useful lemma, which is an extension of \cite[Lemma 1]{rainforth2018tighter}.

\begin{lem}
    \label{lem:keySNRlemma}
    Suppose we have random variables $X_{i,j}$ for all $i = 1 \dots r$ and $j = 1 \dots N$ satisfying
    \begin{enumerate}[label=(\roman*)]
        \item \label{ass:meanzero} $\PE\lr{X_{i,j}} = 0$ for all $i = 1 \dots r$ and $j = 1 \dots N$;
        \item \label{ass:rthmoment} $\PE\lr{|X_{i,j}|^r} < \infty$ for all $i = 1 \dots r$ and $j = 1 \dots N$;
        \item \label{ass:iid} for each $i = 1 \dots r$, the random variables $X_{i,1}, \dots, X_{i,N}$ are i.i.d.;
        \item \label{ass:independence} for each $i = 1 \dots r$ and $j = 1 \dots N$, the random variables $X_{i,j}$ and $\{X_{i', j'}\}_{i' = 1 \dots r; \; j' \neq j}$ are independent.
    \end{enumerate}
    Then
    \begin{equation*}
        \PE\lrb{\lr{\frac{1}{N} \sum_{j=1}^N X_{1,j}} \dots \lr{\frac{1}{N} \sum_{j=1}^N X_{r,j}}} = \begin{cases} O(N^{-r/2}) & \text{if $r$ is even,} \\ O(N^{-(r+1)/2}) & \text{if $r$ is odd.} \end{cases}
    \end{equation*}
\end{lem}

\begin{proof}
We have
\begin{equation} \label{eq:rewritingProdOfSums}
    \lr{\frac{1}{N} \sum_{j=1}^N X_{1,j}} \dots \lr{\frac{1}{N} \sum_{j=1}^N X_{r,j}} = \frac{1}{N^r} \sum_{\lrcb{A_1, \dots, A_t}} \sum_{(j_1, \dots, j_t)} \lr{\prod_{i \in A_1} X_{i, j_1}} \dots \lr{\prod_{i \in A_t} X_{i, j_t}}
\end{equation}
where the first sum is over all partitions $\lrcb{A_1, \dots, A_t}$ of the set $\{1, \dots, r\}$ (so that $t$ is an integer between $1$ and $r$ for each partition) and the second sum is over all tuples $(j_1, \dots, j_t)$ with each element being a distinct integer between $1$ and $N$ (e.g. for the partition $\{A_1 \}$ with $A_1 = \{1, \dots, r\}$ and $t = 1$, the second sum reduces to the sum over $j_1 = 1 \ldots N$ and for the partition $\lrcb{A_1, \ldots, A_r}$ with $A_p = \{p\}$ for all $p = 1 \ldots r$ and $t = r$, the second sum corresponds to the sum over all permutations over the subsets of length $r$ of $(1, \ldots, N)$).

Now consider the case where $|A_p| = 1$ for some integer $p$ between $1$ and $t$ in a certain partition $\lrcb{A_1, \dots, A_t}$. Without any loss of generality, we let $|A_1| = 1$. Then, by the independence of condition \ref{ass:independence} followed by condition \ref{ass:meanzero}, we have
\begin{align*}
    \PE\lrb{\lr{\prod_{i \in A_1} X_{i, j_1}} \dots \lr{\prod_{i \in A_t} X_{i, j_t}}} & = \PE\lrb{X_{i^\ast, j_1}} \PE\lrb{\lr{\prod_{i \in A_2} X_{i, j_2}} \dots \lr{\prod_{i \in A_t} X_{i, j_t}}} \\
    & = 0
\end{align*}
where $i^\ast$ is the single element of $A_1$. Hence, we can restrict the sum over $\lrcb{A_1, \dots, A_t}$ to only consider partitions where every partition has at least two elements. Furthermore, by the generalized H\"older's inequality (i.e. given the $r$ random variables $X_1, \ldots, X_r$, it holds that $\PE(|\prod_{p = 1}^r X_p|) \leq \prod_{p=1}^r \PE(|X_p|^r)^{1/r}$) and conditions \ref{ass:rthmoment} and \ref{ass:iid}, we also have
\begin{equation*}
    \PE\lrb{\Bigg|\lr{\prod_{i \in A_1} X_{i, j_1}} \dots \lr{\prod_{i \in A_t} X_{i, j_t}}\Bigg|} \leq \prod_{i=1}^r \PE\lr{|X_{i,1}|^r}^{1/r} < \infty,
\end{equation*}
where we have used that the product on the l.h.s. of the equation above contains exactly $r$ terms since $\lrcb{A_1, \ldots, A_t}$ is a partition of $\lrcb{1, \ldots,r}$. Putting this together with \eqref{eq:rewritingProdOfSums} yields:
\begin{align*}
    \Bigg| \PE\lrb{\lr{\frac{1}{N} \sum_{j=1}^N X_{1,j}} \dots \lr{\frac{1}{N} \sum_{j=1}^N X_{r,j}}} \Bigg| & \leq \frac{1}{N^r} \sum_{\substack{\lrcb{A_1, \dots, A_t} \\ \text{all } |A_p| \geq 2}} \sum_{(j_1, \dots, j_t)} \lr{\prod_{i=1}^r \PE\lr{|X_{i,1}|^r}^{1/r}} \\
    & \leq \frac{1}{N^r} \sum_{\substack{\lrcb{A_1, \dots, A_t} \\ \text{all } |A_p| \geq 2}} N^t \lr{\prod_{i=1}^r \PE\lr{|X_{i,1}|^r}^{1/r}}.
\end{align*}
Finally, note that (i) any partition $\lrcb{A_1, \dots, A_t}$ of $\{1, \dots, r\}$ where each part has size at least 2 can have at most $\lfloor r/2 \rfloor$ parts, so $t \leq \lfloor r/2 \rfloor$ and (ii) we can crudely bound the number of partitions by $r^r$. Hence
\begin{align*}
    \Bigg| \PE\lrb{\lr{\frac{1}{N} \sum_{j=1}^N X_{1,j}} \dots \lr{\frac{1}{N} \sum_{j=1}^N X_{r,j}}} \Bigg| & \leq \frac{r^r}{N^{r-\lfloor r/2 \rfloor}} \lr{\prod_{i=1}^r \PE\lr{|X_{i,1}|^r}^{1/r}} \\
    & = \begin{cases} O(N^{-r/2}) & \text{if $r$ is even} \\ O(N^{-(r+1)/2}) & \text{if $r$ is odd.} \end{cases}
\end{align*}
\end{proof}

We next prove a second lemma. 

\begin{lem}
    \label{lem:equivlimsupconditionGen}
    Let $k$ be a positive integer. Set $\RalphaN = N^{-1} \sum_{i=1}^N \w(Z_i)^{1-\alpha}$, where $Z_1, \ldots, Z_N$ are i.i.d. samples generated according to $q_\phi$. Then, the condition 
  \begin{align} \label{ass:limsupZ}
    \limsup_{N \rightarrow \infty} \PE\lr{(1/\RalphaN)^k} < \infty
    \end{align}
    is equivalent to the statement that there exists some $N \in \mathbb{N}^\star$ for which $\PE((1/\RalphaN)^k) < \infty$.
  \end{lem}

\begin{proof}[Proof of \Cref{lem:equivlimsupconditionGen}]
    Fix a positive integer $N \geq 2$. For all $x \in [0,1)$, we have by convexity of the function $x \mapsto (1-x)^{-k}$ that
    \begin{equation*}
        \lr{\frac{1}{1-x}}^k \geq \lr{\frac{N}{N-1}}^k + k \lr{\frac{N}{N-1}}^{k+1} \lr{x - \frac{1}{N}}.
    \end{equation*}
    It follows that if $x_1, \dots, x_N \in (0,1)$ are such that $\sum_{i=1}^{N} x_i = 1$, then
    \begin{equation*}
        \frac{1}{N} \sum_{i=1}^N \lr{\frac{1}{1 - x_i}}^k \geq \lr{\frac{N}{N-1}}^k.
    \end{equation*}
    Given $\alpha \in [0,1)$ and $N$ positive reals $w_1, \dots, w_N$, we may set $x_i = w_i^{1-\alpha} / \lr{\sum_{i=1}^n w_i^{1-\alpha}}$ in the above to get
    \begin{equation*}
        \frac{1}{N} \sum_{i=1}^N \lr{\dfrac{N-1}{{\sum_{\substack{j = 1 \\ j \neq i}}^N w_j^{1-\alpha}}}}^k \geq \lr{\frac{N}{{\sum_{j=1}^N w_j^{1-\alpha}}}}^k.
    \end{equation*}
    Now consider setting $w_i = \w(Z_i)$ where the $Z_i$ are i.i.d. samples generated according to $q_\phi$. We see that the r.h.s. of the above expression is distributed as $(1/\RalphaN)^k$, while each term in the sum on the l.h.s. is distributed as $(1/R_{\alpha, N-1})^k$. We conclude that
    \begin{equation*}
        \PE\lr{(1/R_{\alpha,N-1})^k} \geq \PE\lr{(1/\RalphaN)^k}
    \end{equation*}
    for all $N \geq 2$. Hence $\PE\lr{(1/\RalphaN)^k}$ is decreasing in $N$, so $\limsup_{N \rightarrow \infty} \PE\lr{(1/\RalphaN)^k} < \infty$ if and only if there exists some $N$ such that $\PE\lr{(1/\RalphaN)^k} < \infty$.
\end{proof}
We now move on to the proof of \Cref{prop:SNRconvergence}.

\begin{proof}[Proof of \Cref{prop:SNRconvergence}] We use the following shorthand notation
\begin{align*}
    & \tilde{w}_{m,j} = \w(f(\varepsilon_{m,j}, \phi)), \quad m = 1 \ldots M, ~ j = 1 \ldots N \\
    %& \delta_{M, N}^{(\alpha)}(\theta, \phi) = \frac{1}{M} \sum_{m=1}^M \nabla_{\theta, \phi} \log \lr{\frac{1}{N} \sum_{j=1}^N \w(f(\varepsilon_{m,j}, \phi))^{1-\alpha}} \\
    %& \mathrm{SNR}_{M,N}(\theta, \phi) = \frac{|\PE(\delta_{M, N}^{(\alpha)} )|}{\sqrt{\mathbb{V}(\delta_{M, N}^{(\alpha)})}} \\
    & Z_\alpha = \PE_{\varepsilon \sim q} (\w(f(\varepsilon, \phi))^{1-\alpha}) 
\end{align*} 
and we also recall the notation
\begin{align*}    
    & \hat{Z}_{1, N, \alpha} = \frac{1}{N} \sum_{j=1}^N \tilde{w}_{1,j}^{1-\alpha}.
\end{align*} 
We will first prove that 
\begin{align}
\mathrm{SNR}[\delta_{M,N}^{(\alpha)}(\theta_\ell)] & =  \sqrt{M} \frac{\left| \sqrt{N} \frac{\partial  Z_\alpha}{\partial \theta_\ell}  -\frac{Z_\alpha}{2 \sqrt{N}} \frac{\partial}{\partial \theta_\ell} \lrb{ \frac{\mathbb{V}(\tilde{w}_{1, 1}^{1-\alpha})}{Z_\alpha^2} } + O \lr{\frac{1}{N^{3/2}}} \right|}{ \sqrt{\PE \lr{ \tilde{w}_{1,1}^{2(1-\alpha)} \lrb{(1-\alpha) \frac{\partial \log \tilde{w}_{1,1}}{\partial \theta_\ell}  - \frac{\partial \log Z_\alpha}{\partial \theta_\ell} }^2}} + O \lr{\frac{1}{N}}} \label{eq:SNRexpression1} \\
\mathrm{SNR}[\delta_{M,N}^{(\alpha)}(\phi_{\ell'})] & =  \sqrt{M} \frac{\left| \sqrt{N} \frac{\partial Z_\alpha}{\partial \phi_{\ell'}}  -\frac{Z_\alpha}{2 \sqrt{N}} \frac{\partial}{\partial \phi_{\ell'}} \lrb{ \frac{\mathbb{V}(\tilde{w}_{1, 1}^{1-\alpha})}{Z_\alpha^2} } + O \lr{\frac{1}{N^{3/2}}} \right|}{ \sqrt{\PE \lr{ \tilde{w}_{1,1}^{2(1-\alpha)} \lrb{(1-\alpha) \frac{\partial \log \tilde{w}_{1,1}}{\partial \phi_{\ell'}}  - \frac{\partial \log Z_\alpha}{\partial \phi_{\ell'}} }^2}} + O \lr{\frac{1}{N}}}. \label{eq:SNRexpression2}
\end{align}
As the two expressions above follow the same form, it is in fact enough to only prove \eqref{eq:SNRexpression1}. We will do so by studying the asymptotic variance and expected value of $\tilde{\delta}_{M,N}^{(\alpha)}(\theta_\ell) \eqdef (1-\alpha)\delta_{M,N}^{(\alpha)}(\theta_\ell)$ separately, before combining them to deduce \eqref{eq:SNRexpression1}.
    \begin{itemize}
      \item \textbf{Study of $\mathbb{V}(\tilde{\delta}_{M,N}^{(\alpha)}(\theta_\ell))$.}
    \end{itemize}
    \noindent We start from the identity
    \begin{equation*}
    \frac{\partial \log \hat{Z}_{1, N, \alpha}}{\partial \theta_\ell}  = \frac{\partial \log Z_\alpha }{\partial \theta_\ell} + \frac{\partial}{\partial \theta_\ell} \lr{ \frac{\hat{Z}_{1, N, \alpha} - Z_\alpha}{Z_\alpha} } - \lr{\frac{\hat{Z}_{1, N, \alpha} - Z_\alpha}{\hat{Z}_{1, N, \alpha}}} \cdot \frac{\partial}{\partial \theta_\ell} \lr{ \frac{\hat{Z}_{1, N, \alpha} - Z_\alpha}{Z_\alpha} },
    \end{equation*}
    which can for example be verified by using the following identity (which is a version of the Taylor expansion to first order with an explicit form for the remainder) %either by expanding directly, or
    \begin{equation*}
        \log (1+x) = x - \int_0^x \frac{t}{1+t} dt,
    \end{equation*}
    substituting $x = (\hat{Z}_{1, N, \alpha} - Z_\alpha)/Z_\alpha$, differentiating with respect to $\theta_\ell$ and using the chain rule where necessary. Hence,
    \begin{align*}
        M \cdot \mathbb{V} \lr{\tilde{\delta}_{M,N}^{(\alpha)} (\theta_\ell)} & = \mathbb{V} \lr{\tilde{\delta}_{1, N}^{(\alpha)}(\theta_\ell)} = \mathbb{V}\lr{  \frac{\partial \log(\hat{Z}_{1, N, \alpha})}{\partial \theta_\ell} } \\
        & = \mathbb{V} \lr{ \frac{\partial \log Z_\alpha }{\partial \theta_\ell} + \frac{\partial}{\partial \theta_\ell} \lr{ \frac{\hat{Z}_{1, N, \alpha} - Z_\alpha}{Z_\alpha} } - \lr{\frac{\hat{Z}_{1, N, \alpha} - Z_\alpha}{\hat{Z}_{1, N, \alpha}}} \cdot \frac{\partial}{\partial \theta_\ell} \lr{ \frac{\hat{Z}_{1, N, \alpha} - Z_\alpha}{Z_\alpha} }} \\
        & = \mathbb{V} \lr{ \frac{\partial}{\partial \theta_\ell} \lr{ \frac{\hat{Z}_{1, N, \alpha} - Z_\alpha}{Z_\alpha} } - \lr{\frac{\hat{Z}_{1, N, \alpha} - Z_\alpha}{\hat{Z}_{1, N, \alpha}}} \cdot \frac{\partial}{\partial \theta_\ell} \lr{ \frac{\hat{Z}_{1, N, \alpha} - Z_\alpha}{Z_\alpha} }} \\
        & = \mathbb{V} \lr{\frac{\partial}{\partial \theta_\ell} \lr{ \frac{\hat{Z}_{1, N, \alpha} - Z_\alpha}{Z_\alpha} }} \\
        & \hspace{5mm} + 2 \cov \lr{\frac{\partial}{\partial \theta_\ell} \lr{ \frac{\hat{Z}_{1, N, \alpha} - Z_\alpha}{Z_\alpha} } , \lr{\frac{\hat{Z}_{1, N, \alpha} - Z_\alpha}{\hat{Z}_{1, N, \alpha}}} \cdot \frac{\partial}{\partial \theta_\ell} \lr{ \frac{\hat{Z}_{1, N, \alpha} - Z_\alpha}{Z_\alpha} } } \\
        & \hspace{5mm} + \mathbb{V} \lr{\lr{\frac{\hat{Z}_{1, N, \alpha} - Z_\alpha}{\hat{Z}_{1, N, \alpha}}} \cdot \frac{\partial}{\partial \theta_\ell} \lr{ \frac{\hat{Z}_{1, N, \alpha} - Z_\alpha}{Z_\alpha} }}
    \end{align*}
    Furthermore, observe that 
    \begin{align}
        \frac{\partial}{\partial \theta_\ell} \lr{\frac{\hat{Z}_{1, N, \alpha} - Z_\alpha}{Z_\alpha}}
        & = \frac{1}{N} \sum_{j=1}^N \frac{\partial}{\partial \theta_\ell} \lr{\frac{\tilde{w}_{1,j}^{1-\alpha} - Z_\alpha}{Z_\alpha}} \nonumber \\
        & = \frac{1}{N} \sum_{j=1}^N \frac{Z_\alpha \frac{\partial  (\tilde{w}_{1,j}^{1-\alpha})}{\partial \theta_\ell} - \tilde{w}_{1,j}^{1-\alpha} \frac{\partial Z_\alpha}{\partial \theta_\ell} }{Z_\alpha^2}. \label{eq:GradientAsASum}
    \end{align}
    As a result, 
    \begin{align} \label{eq:PENablaZero}
    \PE\lrb{\frac{\partial}{\partial \theta_\ell} \lr{ \frac{\hat{Z}_{1, N, \alpha} - Z_\alpha}{Z_\alpha} }} = \frac{\PE[\frac{\partial  (\tilde{w}_{1,1}^{1-\alpha})}{\partial \theta_\ell}] - \frac{\partial Z_\alpha}{\partial \theta_\ell}}{Z_\alpha} = 0,
    \end{align}
    where we have used that under common differentiability assumptions $\PE[{\partial  (\tilde{w}_{1,1}^{1-\alpha})} / {\partial \theta_\ell}] = {\partial Z_\alpha} / {\partial \theta_\ell}$, that is we can interchange the order of integration and differentiation. Consequently, we can simplify the expression of $M \cdot \mathbb{V} \lr{\tilde{\delta}_{M,N}^{(\alpha)}(\theta_\ell)}$ to obtain that
      \begin{align}
        M \cdot \mathbb{V} \lr{\tilde{\delta}_{M,N}^{(\alpha)}(\theta_\ell)} & = \PE\lr{\lrb{\frac{\partial}{\partial \theta_\ell} \lr{ \frac{\hat{Z}_{1, N, \alpha} - Z_\alpha}{Z_\alpha} }}^2} \nonumber \\
        & \quad + 2 \PE\lr{\lr{\frac{\hat{Z}_{1, N, \alpha} - Z_\alpha}{\hat{Z}_{1, N, \alpha}}} \cdot \lrb{\frac{\partial}{\partial \theta_\ell} \lr{ \frac{\hat{Z}_{1, N, \alpha} - Z_\alpha}{Z_\alpha} }}^2} \nonumber \\
        & \quad + \mathbb{V} \lr{\lr{\frac{\hat{Z}_{1, N, \alpha} - Z_\alpha}{\hat{Z}_{1, N, \alpha}}} \cdot \frac{\partial}{\partial \theta_\ell} \lr{ \frac{\hat{Z}_{1, N, \alpha} - Z_\alpha}{Z_\alpha} }} \label{eq:SNRvariance}
    \end{align}
    We now control the three terms in the r.h.s. of \eqref{eq:SNRvariance} separately.

    \begin{enumerate}
        \item \textit{First term in the r.h.s. of \eqref{eq:SNRvariance}.} To control the first term in the r.h.s. of \eqref{eq:SNRvariance}, we use \eqref{eq:GradientAsASum} to get that
    \begin{align}
        \PE \lr{ \lrb{  \frac{\partial}{\partial \theta_\ell} \lr{\frac{\hat{Z}_{1, N, \alpha} - Z_\alpha}{Z_\alpha}}}^2} & = \frac{1}{N^2} \PE \lr{ \sum_{j=1}^N \lrb{\frac{\partial}{\partial \theta_\ell} \lr{\frac{\tilde{w}_{1,j}^{1-\alpha} - Z_\alpha}{Z_\alpha}}}^2} \nonumber \\
        & = \frac{1}{N^2} \PE \lr{\sum_{j=1}^N \lrb{\frac{Z_\alpha \frac{\partial (\tilde{w}_{1,j}^{1-\alpha})}{\partial \theta_\ell}  - \tilde{w}_{1,j}^{1-\alpha} \frac{\partial Z_\alpha}{\partial \theta_\ell} }{Z_\alpha^2} }^2} \nonumber \\
        & = \frac{1}{N Z_\alpha^4} \PE \lr{ \lrb{Z_\alpha \frac{\partial (\tilde{w}_{1,j}^{1-\alpha})}{\partial \theta_\ell}  - \tilde{w}_{1,j}^{1-\alpha} \frac{\partial Z_\alpha}{\partial \theta_\ell} }^2} \nonumber \\
        & = \frac{1}{N Z_\alpha^4} \PE \lr{ \lrb{\tilde{w}_{1,1}^{-\alpha} \lrcb{(1-\alpha) Z_\alpha \frac{\partial \tilde{w}_{1,1}}{\partial \theta_\ell}  - \tilde{w}_{1,1} \frac{\partial Z_\alpha}{\partial \theta_\ell} } }^2}, \label{eq:FirstTermDom}
    \end{align}
    where the cross-terms disappeared due to the independence of the  $(\varepsilon_{1,j})_{1 \leq j \leq N}$ paired up with \eqref{eq:PENablaZero}.

    \item \textit{Second term in the r.h.s of \eqref{eq:SNRvariance}.} We deal with the cross term in \eqref{eq:SNRvariance} by splitting it into two parts
    \begin{align}
        & \PE\lr{\lr{\frac{\hat{Z}_{1, N, \alpha} - Z_\alpha}{\hat{Z}_{1, N, \alpha}}} \cdot \lrb{\frac{\partial}{\partial \theta_\ell} \lr{ \frac{\hat{Z}_{1, N, \alpha} - Z_\alpha}{Z_\alpha} }}^2} \nonumber \\
        & \quad \quad = \PE\lr{\lr{\frac{\hat{Z}_{1, N, \alpha} - Z_\alpha}{Z_\alpha}} \cdot \lrb{\frac{\partial}{\partial \theta_\ell} \lr{ \frac{\hat{Z}_{1, N, \alpha} - Z_\alpha}{Z_\alpha} }}^2} \nonumber \\ 
        & \quad \quad + \PE\lr{\lr{\frac{\hat{Z}_{1, N, \alpha} - Z_\alpha}{Z_\alpha}} \cdot \lrb{\frac{\partial}{\partial \theta_\ell} \lr{ \frac{\hat{Z}_{1, N, \alpha} - Z_\alpha}{Z_\alpha} }}^2 \lr{\frac{Z_\alpha}{\hat{Z}_{1, N, \alpha}} - 1}} \label{eq:SNRvariancecross}.
    \end{align}
    Using the expression of $\frac{\partial}{\partial \theta_\ell} \lr{\frac{\hat{Z}_{1, N, \alpha} - Z_\alpha}{Z_\alpha}}$ given in \eqref{eq:GradientAsASum} and the fact that
    \begin{equation} \label{eq:DiffAsASum}
        \hat{Z}_{1, N, \alpha} - Z_\alpha = \frac{1}{N} \sum_{j=1}^N \lr{\tilde{w}_{1,j}^{1-\alpha} - Z_\alpha},
    \end{equation} 
    we set: for all $j = 1 \ldots J$,
    \begin{align*}
        & X_{1,j} = \tilde{w}_{1,j}^{1-\alpha} - Z_\alpha, \\
        & X_{2,j} = X_{3,j} = \frac{Z_\alpha \frac{\partial (\tilde{w}_{1,j}^{1-\alpha})}{\partial \theta_\ell}  - \tilde{w}_{1,j}^{1-\alpha} \frac{\partial Z_\alpha}{\partial \theta_\ell} }{Z_\alpha^2}
    \end{align*}
    and we can then apply \Cref{lem:keySNRlemma} with $r = 3$ (by noting in particular that the required moments are finite under our assumptions): we thus obtain that 
    \begin{equation*}
        \PE\lr{\lr{\hat{Z}_{1, N, \alpha} - Z_\alpha} \lrb{\frac{\partial}{\partial \theta_\ell} \lr{ \frac{\hat{Z}_{1, N, \alpha} - Z_\alpha}{Z_\alpha} }}^2} = O \lr{\frac{1}{N^2}}
    \end{equation*}
    which controls the first term in the r.h.s. of \eqref{eq:SNRvariancecross}. The second term in the r.h.s. of \eqref{eq:SNRvariancecross} can be bounded as follows
    \begin{align}
        & \hspace{4mm} \PE\lr{\lr{\frac{\hat{Z}_{1, N, \alpha} - Z_\alpha}{Z_\alpha}} \lrb{\frac{\partial}{\partial \theta_\ell} \lr{ \frac{\hat{Z}_{1, N, \alpha} - Z_\alpha}{Z_\alpha} }}^2 \lr{\frac{Z_\alpha}{\hat{Z}_{1, N, \alpha}} - 1}} \nonumber \\
        & \leq \PE\lr{\lr{\frac{\hat{Z}_{1, N, \alpha} - Z_\alpha}{Z_\alpha}}^2 \lrb{\frac{\partial}{\partial \theta_\ell} \lr{ \frac{\hat{Z}_{1, N, \alpha} - Z_\alpha}{Z_\alpha} }}^4 }^{1/2} \PE \lr{\lr{\frac{Z_\alpha}{\hat{Z}_{1, N, \alpha}} - 1}^2}^{1/2}. \label{eq:SNRvariancecrosssplit}
    \end{align}
    By taking this time: for all $j = 1 \ldots N$,
    \begin{align*}
        & X_{1,j} = X_{2,j} = \tilde{w}_{1,j}^{1-\alpha} - Z_\alpha, \\
        & X_{3,j} = X_{4,j} = X_{5,j} = X_{6,j} = \frac{Z_\alpha \frac{\partial (\tilde{w}_{1,j}^{1-\alpha})}{\partial \theta_\ell}  - \tilde{w}_{1,j}^{1-\alpha} \frac{\partial Z_\alpha}{\partial \theta_\ell} }{Z_\alpha^2} ,
    \end{align*}
    in \Cref{lem:keySNRlemma} with $r = 6$ (and noting once again that the required moments are finite under our assumptions), we see that the first term in the r.h.s. of \eqref{eq:SNRvariancecrosssplit} is $O(N^{-3/2})$. As for the second term of the r.h.s. of \eqref{eq:SNRvariancecrosssplit}, Cauchy-Schwarz implies that
    \begin{equation*}
        \PE \lr{\lr{\frac{Z_\alpha}{\hat{Z}_{1, N, \alpha}} - 1}^2} \leq \PE\lr{\frac{1}{\hat{Z}_{1, N, \alpha}^4}}^{1/2} \PE\lr{\lr{Z_\alpha - \hat{Z}_{1, N, \alpha}}^4}^{1/2}.
    \end{equation*}
    Now note that under our assumptions, \Cref{lem:equivlimsupconditionGen} can be applied with $k = 4$ so that \eqref{ass:limsupZ} with $k = 4$ holds and controls the first term in the r.h.s. above. Furthermore, applying \Cref{lem:keySNRlemma} with $r = 4$ and for all $j = 1 \ldots N$,
    \begin{equation*}
        X_{1,j} = X_{2,j} = X_{3,j} = X_{4,j} = \tilde{w}_{1,j}^{1-\alpha} - Z_\alpha
    \end{equation*}
    yields
    \begin{equation*}
        \PE\lr{\lr{Z_\alpha - \hat{Z}_{1, N, \alpha}}^4} = O\lr{\frac{1}{N^2}}.
    \end{equation*}
    We can then conclude that
    \begin{equation}
        \label{eq:SNR_CSbound}
        \PE \lr{\lr{\frac{Z_\alpha}{\hat{Z}_{1, N, \alpha}} - 1}^2}^{1/2} = O\lr{\frac{1}{N^{1/2}}},
    \end{equation}
    and so the r.h.s. of \eqref{eq:SNRvariancecrosssplit} is bounded above by $O(N^{-2})$. It follows that the second term in the r.h.s. of \eqref{eq:SNRvariancecross} is $O(N^{-2})$ too and we can conclude that
    \begin{align} \label{eq:secondTermON2}
        & \PE\lr{\lr{\frac{\hat{Z}_{1, N, \alpha} - Z_\alpha}{\hat{Z}_{1, N, \alpha}}} \cdot \lrb{\frac{\partial}{\partial \theta_\ell} \lr{ \frac{\hat{Z}_{1, N, \alpha} - Z_\alpha}{Z_\alpha} }}^2} =  O\lr{\frac{1}{N^2}}
    \end{align}
    that is the second term of the r.h.s. of \eqref{eq:SNRvariance} is $O(N^{-2})$. 
    
    \item \textit{Third term of the r.h.s. in \eqref{eq:SNRvariance}.} For the third term of the r.h.s. in \eqref{eq:SNRvariance}, note that
    \begin{align*}
       & \hspace{4mm} \mathbb{V} \lr{\lr{\frac{\hat{Z}_{1, N, \alpha} - Z_\alpha}{\hat{Z}_{1, N, \alpha}}} \cdot \frac{\partial}{\partial \theta_\ell} \lr{ \frac{\hat{Z}_{1, N, \alpha} - Z_\alpha}{Z_\alpha} }} \\
       & \leq \PE \lr{\lrb{\lr{\frac{\hat{Z}_{1, N, \alpha} - Z_\alpha}{\hat{Z}_{1, N, \alpha}}} \cdot \frac{\partial}{\partial \theta_\ell} \lr{ \frac{\hat{Z}_{1, N, \alpha} - Z_\alpha}{Z_\alpha} }}^2} \\
        & \leq \PE \lr{\lrb{\lr{\hat{Z}_{1, N, \alpha} - Z_\alpha} \cdot \frac{\partial}{\partial \theta_\ell} \lr{ \frac{\hat{Z}_{1, N, \alpha} - Z_\alpha}{Z_\alpha} }}^4}^{1/2} \PE\lr{\frac{1}{\hat{Z}_{1, N, \alpha}^4}}^{1/2}
    \end{align*}
    where the final line follows from Cauchy--Schwarz. As a result, using \eqref{eq:GradientAsASum} and \eqref{eq:DiffAsASum}, taking for all $j = 1 \ldots N$
    \begin{align*}
        & X_{1,j} = X_{2,j} = X_{3,j} = X_{4,j} = \tilde{w}_{1,j}^{1-\alpha} - Z_\alpha \\
        & X_{5,j} = X_{6,j} = X_{7,j} = X_{8,j} = \frac{Z_\alpha \frac{\partial (\tilde{w}_{1,j}^{1-\alpha})}{\partial \theta_\ell}  - \tilde{w}_{1,j}^{1-\alpha} \frac{\partial Z_\alpha}{\partial \theta_\ell} }{Z_\alpha^2},
    \end{align*}
    and since the required moments are finite under our assumptions, \Cref{lem:keySNRlemma} with $r = 8$ implies that
    \begin{equation*}
        \PE \lr{\lrb{\lr{\hat{Z}_{1, N, \alpha} - Z_\alpha} \cdot \frac{\partial}{\partial \theta_\ell} \lr{ \frac{\hat{Z}_{1, N, \alpha} - Z_\alpha}{Z_\alpha} }}^4} = O\lr{\frac{1}{N^4}}.
    \end{equation*}
    Combined with \eqref{ass:limsupZ} with $k = 4$ (which holds under our assumptions by \Cref{lem:equivlimsupconditionGen} with $k = 4$), this implies that
    \begin{equation} \label{eq:thirdTermON2}
        \mathbb{V} \lr{\lr{\frac{\hat{Z}_{1, N, \alpha} - Z_\alpha}{\hat{Z}_{1, N, \alpha}}} \cdot \frac{\partial}{\partial \theta_\ell} \lr{ \frac{\hat{Z}_{1, N, \alpha} - Z_\alpha}{Z_\alpha} }} = O\lr{\frac{1}{N^2}}.
    \end{equation}

\end{enumerate}
Putting \eqref{eq:SNRvariance}, \eqref{eq:FirstTermDom}, \eqref{eq:secondTermON2} and \eqref{eq:thirdTermON2} together, we see that
    \begin{align*}
        \mathbb{V} \lr{\tilde{\delta}_{M,N}^{(\alpha)}(\theta_\ell)} & = \frac{1}{MN Z_\alpha^4} \PE \lr{ \lrb{\tilde{w}_{1,1}^{-\alpha} \lrcb{(1-\alpha) Z_\alpha \frac{\partial \tilde{w}_{1,1}}{\partial \theta_\ell}  - \tilde{w}_{1,1} \frac{\partial Z_\alpha}{\partial \theta_\ell} }  }^2} + O\lr{\frac{1}{MN^2}} \\
        & = \frac{1}{M N Z_\alpha^2} \PE \lr{ \tilde{w}_{1,1}^{2(1-\alpha)} \lrb{(1-\alpha)  \frac{\partial \log \tilde{w}_{1,1}}{\partial \theta_\ell} - \frac{\partial \log Z_\alpha}{\partial \theta_\ell}}^2} + O \lr{\frac{1}{MN^2}}
    \end{align*}
    and it follows that
    \begin{align}
        \sqrt{\mathbb{V}(\tilde{\delta}_{M,N}^{(\alpha)}(\theta_\ell))} & = \frac{1}{\sqrt{M N} Z_\alpha} \sqrt{\PE \lr{\tilde{w}_{1,1}^{2(1-\alpha)} \lrb{(1-\alpha)  \frac{\partial \log \tilde{w}_{1,1}}{\partial \theta_\ell} - \frac{\partial \log Z_\alpha}{\partial \theta_\ell}}^2} + O \lr{\frac{1}{N}}} \nonumber \\
        & = \frac{1}{\sqrt{M N} Z_\alpha} \sqrt{\PE \lr{\tilde{w}_{1,1}^{2(1-\alpha)} \lrb{(1-\alpha)  \frac{\partial \log \tilde{w}_{1,1}}{\partial \theta_\ell} - \frac{\partial \log Z_\alpha}{\partial \theta_\ell}}^2}} \sqrt{1 + O \lr{\frac{1}{N}}} \nonumber \\
        & = \frac{1}{\sqrt{M N} Z_\alpha} \lr{ \sqrt{\PE \lr{ \tilde{w}_{1,1}^{2(1-\alpha)} \lrb{(1-\alpha)  \frac{\partial \log \tilde{w}_{1,1}}{\partial \theta_\ell} - \frac{\partial \log Z_\alpha}{\partial \theta_\ell}}^2}} + O \lr{\frac{1}{N}}}. \label{eq:VarDeltaMN}
      \end{align}
    
    \begin{itemize}
      \item \textbf{Study of $\PE(\tilde{\delta}_{M,N}^{(\alpha)}(\theta_\ell))$.}
    \end{itemize}
    
    \noindent We start from the identity
    \begin{multline*}
        \frac{\partial \log \hat{Z}_{1, N, \alpha}}{\partial \theta_\ell}  = \frac{\partial \log Z_\alpha}{\partial \theta_\ell}  + \frac{\partial}{\partial \theta_\ell} \lr{ \frac{\hat{Z}_{1, N, \alpha} - Z_\alpha}{Z_\alpha} } - \frac{1}{2} \frac{\partial}{\partial \theta_\ell} \lr{ \lrb{ \frac{\hat{Z}_{1, N, \alpha} - Z_\alpha}{Z_\alpha} }^2 } \\ + \lr{\frac{(\hat{Z}_{1, N, \alpha} - Z_\alpha)^2}{Z_\alpha \cdot \hat{Z}_{1, N, \alpha}}} \frac{\partial}{\partial \theta_\ell} \lr{ \frac{\hat{Z}_{1, N, \alpha} - Z_\alpha}{Z_\alpha}},
    \end{multline*}
    which can for example be proved using the following identity (which is a version of the Taylor expansion to second order with an explicit form for the remainder)
    \begin{equation*}
        \log (1+x) = x - \frac{x^2}{2} + \int_0^x \frac{t^2}{1+t}\rmd t,
    \end{equation*}
    substituting $x = (\hat{Z}_{1, N, \alpha} - Z_\alpha)/Z_\alpha$, differentiating with respect to $\theta_\ell$ and using the chain rule where necessary. It follows that
    \begin{align}
        \PE \lr{\tilde{\delta}_{M, N}^{(\alpha)}(\theta_\ell)} & = \PE \lr{\tilde{\delta}_{1, N}^{(\alpha)}(\theta_\ell)} = \PE \lr{ \frac{\partial \log \hat{Z}_{1, N, \alpha}}{\partial \theta_\ell}} \nonumber \\
        & = \PE \Bigg( \frac{\partial \log Z_\alpha}{\partial \theta_\ell} + \frac{\partial}{\partial \theta_\ell} \lr{ \frac{\hat{Z}_{1, N, \alpha} - Z_\alpha}{Z_\alpha} } - \frac{1}{2} \frac{\partial}{\partial \theta_\ell} \lr{ \lrb{ \frac{\hat{Z}_{1, N, \alpha} - Z_\alpha}{Z_\alpha} }^2 } \nonumber \\
        & \hspace{10mm} + \lr{\frac{(\hat{Z}_{1, N, \alpha} - Z_\alpha)^2}{Z_\alpha \cdot \hat{Z}_{1, N, \alpha}}} \frac{\partial}{\partial \theta_\ell} \lr{ \frac{\hat{Z}_{1, N, \alpha} - Z_\alpha}{Z_\alpha}} \Bigg) \nonumber \\
        & =  \frac{\partial \log Z_\alpha}{\partial \theta_\ell}  - \frac{1}{2} \PE\lr{\frac{\partial}{\partial \theta_\ell} \lr{ \lrb{ \frac{\hat{Z}_{1, N, \alpha} - Z_\alpha}{Z_\alpha} }^2 }} + R_2(\hat{Z}_{1, N, \alpha}) \nonumber \\
        %& =\frac{\partial \log Z_\alpha}{\partial \theta_\ell}  - \frac{1}{2} \frac{\partial}{\partial \theta_\ell} \lrb{ \frac{\mathbb{V}(\hat{Z}_{1, N, \alpha})}{Z_\alpha^2} } + R_2(\hat{Z}_{1, N, \alpha}) \nonumber \\
        & = \frac{\partial \log Z_\alpha}{\partial \theta_\ell}  - \frac{1}{2 N} \frac{\partial}{\partial \theta_\ell} \lrb{ \frac{\mathbb{V}(\tilde{w}_{1, 1}^{1-\alpha})}{Z_\alpha^2} } + R_2(\hat{Z}_{1, N, \alpha}) \label{eq:PEdeltaInSNR}
    \end{align}
    where we denote
    \begin{equation*}
        R_2(\hat{Z}_{1, N, \alpha}) = \PE\lr{\lr{\frac{(\hat{Z}_{1, N, \alpha} - Z_\alpha)^2}{Z_\alpha \cdot \hat{Z}_{1, N, \alpha}}}  \frac{\partial}{\partial \theta_\ell} \lr{ \frac{\hat{Z}_{1, N, \alpha} - Z_\alpha}{Z_\alpha}}}
    \end{equation*}
    and where we have used \eqref{eq:GradientAsASum}, \eqref{eq:PENablaZero}, \eqref{eq:DiffAsASum} and the fact that under common differentiability assumptions, we have that
    \begin{align*}
    \PE\lr{ \frac{\partial}{\partial \theta_\ell} \lr{ \lrb{ \frac{\hat{Z}_{1, N, \alpha} - Z_\alpha}{Z_\alpha} }^2 }} & = 2 \PE \lr{  \lrb{ \frac{\hat{Z}_{1, N, \alpha} - Z_\alpha}{Z_\alpha} }   \frac{\partial}{\partial \theta_\ell} \lr{ \lrb{ \frac{\hat{Z}_{1, N, \alpha} - Z_\alpha}{Z_\alpha} }}  } \\
    & = \frac{2}{N^2}  \PE \lr{ \sum_{j = 1}^N \frac{\tilde{w}_{1,j}^{1-\alpha} - Z_\alpha}{Z_\alpha} \cdot \frac{Z_\alpha  \frac{\partial (\tilde{w}_{1,j}^{1-\alpha})}{\partial \theta_\ell}  - \tilde{w}_{1,j}^{1-\alpha} \frac{\partial Z_\alpha}{\partial \theta_\ell} }{Z_\alpha^2}} \\
    & = \frac{1}{N} \PE \lr{  \frac{\partial}{\partial \theta_\ell}  \lr{\lrb{\frac{\tilde{w}_{1,j}^{1-\alpha} - Z_\alpha}{Z_\alpha}}^2}} \\
    & = \frac{1}{N} \frac{\partial}{\partial \theta_\ell} \lrb{ \frac{\mathbb{V}(\tilde{w}_{1, 1}^{1-\alpha})}{Z_\alpha^2} }
    \end{align*}
    (here the cross-terms disappear due to the independence of the $(\varepsilon_{1,j})_{1 \leq j \leq N}$ paired up with \eqref{eq:PENablaZero}). Notice then that we can split up $R_2(\hat{Z}_{1, N, \alpha})$ as
    \begin{multline}
        \label{eq:SNReverror}
        R_2(\hat{Z}_{1, N, \alpha}) = \PE\lr{\lr{\frac{(\hat{Z}_{1, N, \alpha} - Z_\alpha)^2}{Z_\alpha^2}} \frac{\partial}{\partial \theta_\ell} \lr{ \frac{\hat{Z}_{1, N, \alpha} - Z_\alpha}{Z_\alpha}}} \\ + \PE\lr{\lr{\frac{(\hat{Z}_{1, N, \alpha} - Z_\alpha)^2}{Z_\alpha^2}} \lr{\frac{Z_\alpha}{\hat{Z}_{1, N, \alpha}} - 1} \frac{\partial}{\partial \theta_\ell} \lr{ \frac{\hat{Z}_{1, N, \alpha} - Z_\alpha}{Z_\alpha}}}
    \end{multline}
    The first term in \eqref{eq:SNReverror} can be bounded by applying \Cref{lem:keySNRlemma} with $r = 3$ and for all $j = 1 \ldots N$, 
    \begin{align*}
        & X_{1,j} = X_{2,j} = \tilde{w}_{1,j}^{1-\alpha} - Z_\alpha, \\
        & X_{3,j} = \frac{Z_\alpha  \frac{\partial (\tilde{w}_{1,j}^{1-\alpha})}{\partial \theta_\ell}  - \tilde{w}_{1,j}^{1-\alpha} \frac{\partial Z_\alpha}{\partial \theta_\ell} }{Z_\alpha^2},
    \end{align*}
    noting the required moments are finite under our assumptions, so that
    \begin{equation*}
        \PE\lr{\lr{\hat{Z}_{1, N, \alpha} - Z_\alpha}^2 \frac{\partial}{\partial \theta_\ell}  \lr{ \frac{\hat{Z}_{1, N, \alpha} - Z_\alpha}{Z_\alpha}}} = O\lr{\frac{1}{N^2}}.
    \end{equation*}
    The second term in \eqref{eq:SNReverror} can be bounded using Cauchy-Schwarz as follows
    \begin{align*}
        & \hspace{5mm} \PE\lr{\lr{\frac{(\hat{Z}_{1, N, \alpha} - Z_\alpha)^2}{Z_\alpha^2}} \lr{\frac{Z_\alpha}{\hat{Z}_{1, N, \alpha}} - 1} \frac{\partial}{\partial \theta_\ell}  \lr{ \frac{\hat{Z}_{1, N, \alpha} - Z_\alpha}{Z_\alpha}}} \\
        & \leq \PE\lr{\lr{\frac{(\hat{Z}_{1, N, \alpha} - Z_\alpha)^2}{Z_\alpha^2}}^2 \lrb{\frac{\partial}{\partial \theta_\ell}  \lr{ \frac{\hat{Z}_{1, N, \alpha} - Z_\alpha}{Z_\alpha}}}^2}^{1/2} \PE \lr{\lr{\frac{Z_\alpha}{\hat{Z}_{1, N, \alpha}} - 1}^2 }^{1/2}
    \end{align*}
    The second term is $O(N^{-1/2})$ by \eqref{eq:SNR_CSbound}, while the first can be bounded by $O(N^{-3/2})$ using \Cref{lem:keySNRlemma} with $r = 6$ and for all $j = 1 \ldots N$:
    \begin{align*}
        & X_{1,j} = X_{2,j} = X_{3,j} = X_{4,j} = \tilde{w}_{1,j}^{1-\alpha} - Z_\alpha, \\
        & X_{5,j} = X_{6,j} = \frac{Z_\alpha  \frac{\partial (\tilde{w}_{1,j}^{1-\alpha})}{\partial \theta_\ell}  - \tilde{w}_{1,j}^{1-\alpha} \frac{\partial Z_\alpha}{\partial \theta_\ell} }{Z_\alpha^2}.
    \end{align*}
    Hence by combining with \eqref{eq:PEdeltaInSNR}, we have
    \begin{equation} \label{eq:PEdeltaMN}
        \PE \lr{\tilde{\delta}_{M, N}^{(\alpha)}} = \frac{\partial \log Z_\alpha}{\partial \theta_\ell}   - \frac{1}{2 N} \frac{\partial}{\partial \theta_\ell}  \lrb{ \frac{\mathbb{V}(\tilde{w}_{1, 1}^{1-\alpha})}{Z_\alpha^2} } + O\lr{\frac{1}{N^2}}
    \end{equation}
    
    \begin{itemize}
    \item \textbf{Deducing $\mathrm{SNR}[\delta_{M,N}^{(\alpha)}(\theta_\ell)]$.}
    \end{itemize}
    \noindent Finally, putting \eqref{eq:VarDeltaMN} and \eqref{eq:PEdeltaMN} together, we get
    \begin{align*}
        \mathrm{SNR}[\delta_{M,N}^{(\alpha)}(\theta_\ell)] & = \frac{ \left|\frac{\partial \log Z_\alpha}{\partial \theta_\ell}   - \frac{1}{2 N} \frac{\partial}{\partial \theta_\ell}  \lrb{ \frac{\mathbb{V}(\tilde{w}_{1, 1}^{1-\alpha})}{Z_\alpha^2} } + O\lr{\frac{1}{N^2}}\right|}{ \frac{1}{\sqrt{M N} Z_\alpha} \lr{ \sqrt{\PE \lr{ \tilde{w}_{1,1}^{2(1-\alpha)} \lrb{(1-\alpha)  \frac{\partial \log \tilde{w}_{1,1}}{\partial \theta_\ell} - \frac{\partial \log Z_\alpha}{\partial \theta_\ell}}^2}} + O \lr{\frac{1}{N}}}} \\
      & = \sqrt{M} \frac{\left| \sqrt{N} \frac{\partial Z_\alpha}{\partial \theta_\ell}   -\frac{Z_\alpha}{2 \sqrt{N}} \frac{\partial}{\partial \theta_\ell}  \lrb{ \frac{\mathbb{V}(\tilde{w}_{1, 1}^{1-\alpha})}{Z_\alpha^2} } + O \lr{\frac{1}{N^{3/2}}} \right|}{{ \sqrt{\PE \lr{ \tilde{w}_{1,1}^{2(1-\alpha)} \lrb{(1-\alpha)  \frac{\partial \log \tilde{w}_{1,1}}{\partial \theta_\ell} - \frac{\partial \log Z_\alpha}{\partial \theta_\ell}}^2}} + O \lr{\frac{1}{N}}}}
    \end{align*}
which is exactly \eqref{eq:SNRexpression1}. Since we have assumed that $\frac{\partial Z_\alpha}{\partial \theta_\ell}$ is non-zero and since this term corresponds to the leading order term, we then deduce that 
$$
\mathrm{SNR}[\delta_{M,N}^{(\alpha)}(\theta_\ell)] = \Theta(\sqrt{M N})
$$
and we thus recover \eqref{eq:SNRbehavior1}.

Similarly, \eqref{eq:SNRexpression2} holds for $\mathrm{SNR}[\delta_{M,N}^{(\alpha)}(\phi_{\ell'})]$ and we obtain the desired result \eqref{eq:SNRbehavior2} by splitting the cases $\alpha \in (0,1)$ and $\alpha = 0$. In the former, we have ${\partial Z_\alpha}/{\partial \phi_{\ell'}} = \partial \PE(\tilde{w}_{1,1}^{1-\alpha})/\partial \phi_{\ell'} \neq 0$ by \eqref{eq:leadingOrder}, so the leading order term is $\Theta(\sqrt{MN})$, while in the latter case we have ${\partial Z_\alpha}/{\partial \phi_{\ell'}} = 0$ while ${\partial} \mathbb{V}(\tilde{w}_{1,1}^{1-\alpha}) / {\partial \phi_{\ell'}} > 0$ and so the leading order term is $\Theta(\sqrt{M/N})$
\end{proof}

\subsection{Proof of \Cref{prop:drepIwaeAlpha}}
\label{subsec:proof:drepIwaeAlpha}

\begin{proof}[Proof of \Cref{prop:drepIwaeAlpha}]
    Recall from \Cref{lem:generalBound} that
    \begin{align*}
    \frac{\partial}{\partial \phi} \liren  (\theta,\phi; x) = & \int \int \prod_{i = 1}^N q(\varepsilon_i) \lr{\sum_{j = 1}^N  \frac{\w(z_j)^{1-\alpha} }{\sum_{k = 1}^N \w(z_k)^{1-\alpha}} \frac{\partial}{\partial \phi} \log \w(f(\varepsilon_j, \phi))}   \rmd \varepsilon_{1:N}. 
    \end{align*}
    We will now follow the reasoning of \cite{Tucker2019DoublyRG}. To do so, we expand the total derivative of $ \liren$ with respect to $\phi$ by using that %we introduce the shorthand notation $\w(z_j) = w_{j}$ for all $j = 1 \ldots J$ and
    \begin{align*}
    \frac{\partial}{\partial \phi} \log \w(f(\varepsilon_j, \phi)) = - \frac{\partial}{\partial \phi} \log q_\phi(f(\varepsilon_j, \phi'))|_{\phi' = \phi} + \frac{\partial}{\partial \phi} f(\varepsilon_j, \phi) ~ \frac{\partial}{\partial z_j} \log \w(z_j)
    \end{align*}
    which gives
    \begin{align}
        \frac{\partial}{\partial \phi} \liren  (\theta,\phi;x) & = - \int \int \prod_{i = 1}^N q(\varepsilon_i) \lr{\sum_{j = 1}^N  \frac{\w(z_j)^{1-\alpha} }{\sum_{k = 1}^N \w(z_k)^{1-\alpha}}  \frac{\partial}{\partial \phi} \log q_\phi(f(\varepsilon_j, \phi'))|_{\phi' = \phi}}   \rmd \varepsilon_{1:N} \nonumber \\
        & \quad + \int \int \prod_{i = 1}^N q(\varepsilon_i) \lr{\sum_{j = 1}^N  \frac{\w(z_j)^{1-\alpha} }{\sum_{k = 1}^N \w(z_k)^{1-\alpha}} \frac{\partial}{\partial \phi} f(\varepsilon_j, \phi) ~ \frac{\partial}{\partial z_j} \log \w(z_j)
        }   \rmd \varepsilon_{1:N} \nonumber \\
        & \eqdef - A + B. \label{eq:partialVRIWAEdrepAB}
    \end{align}
    Notice now that 
    \begin{align}
      A & = \sum_{j = 1}^N  \int \int \prod_{i = 1}^N q(\varepsilon_i) \lr{ \frac{\w(z_j)^{1-\alpha} }{\sum_{k = 1}^N \w(z_k)^{1-\alpha}}  \frac{\partial}{\partial \phi} \log q_\phi(f(\varepsilon_j, \phi'))|_{\phi' = \phi}}   \rmd \varepsilon_{1:N} \nonumber \\ 
      & = \sum_{j = 1}^N  \int \int \prod_{i = 1}^N q_\phi(z_i) \lr{ \frac{\w(z_j)^{1-\alpha} }{\sum_{k = 1}^N \w(z_k)^{1-\alpha}}  \frac{\partial}{\partial \phi} \log q_\phi(z'_j)|_{z_j = z'_j}} \rmd z_{1:N}, \label{eq:partialVRIWAEdrepA}
    \end{align}
    where we have set $z'_j = f(\varepsilon_j, \phi')$. Observe in addition that for all $j = 1 \ldots N$, the reparameterization trick implies:
    \begin{multline*}
    \int q_\phi(z_j) ~{ \frac{\w(z_j)^{1-\alpha} }{\sum_{k = 1}^N \w(z_k)^{1-\alpha}}  \frac{\partial}{\partial \phi} \log q_\phi(z'_j)|_{z_j = z'_j}}   ~\rmd z_{j} \\ =  \int q(\varepsilon_j) ~{\frac{\partial}{\partial z_j} \lr{ \frac{\w(z_j)^{1-\alpha} }{\sum_{k = 1}^N \w(z_k)^{1-\alpha}}}  \frac{\partial }{\partial \phi} f(\varepsilon_j, \phi)} ~ \rmd \varepsilon_{j} \\
    \end{multline*}
    and hence
    \begin{multline*}
        \int q_\phi(z_j) ~{ \frac{\w(z_j)^{1-\alpha} }{\sum_{k = 1}^N \w(z_k)^{1-\alpha}}  \frac{\partial}{\partial \phi} \log q_\phi(z'_j)|_{z_j = z'_j}}   ~\rmd z_{j} \\ = (1- \alpha) \int q(\varepsilon_j) {  \lrb{\frac{\w(z_j)^{1-\alpha} }{\sum_{k = 1}^N \w(z_k)^{1-\alpha}} -  \lr{\frac{\w(z_j)^{1-\alpha} }{\sum_{k = 1}^N \w(z_k)^{1-\alpha}}}^2} \frac{\partial }{\partial \phi} f(\varepsilon_j, \phi) \frac{\partial}{\partial z_j} \log \w(z_j)}  ~\rmd \varepsilon_{j}.
    \end{multline*}
    The desired equality \eqref{eq:partialVRIWAEdrep} is then obtained by combining the last equality above with \eqref{eq:partialVRIWAEdrepAB} and \eqref{eq:partialVRIWAEdrepA}.
    %\begin{multline*}
    %    \frac{\partial}{\partial \phi}  \liren  (\theta,\phi; x) \\ = \int \int \prod_{i = 1}^N q(\varepsilon_i) \lrb{\sum_{j=1}^N \lr{\alpha ~ \frac{w_{j}^{1-\alpha} }{\sum_{k = 1}^N w_{k}^{1-\alpha}} + (1-\alpha) \lr{\frac{w_{j}^{1-\alpha} }{\sum_{k = 1}^N w_{k}^{1-\alpha}}}^2} \frac{\rmd}{\rmd \phi}  f(\varepsilon_j, \phi) \frac{\partial}{\partial z_j} \log \w(z_j)}  \rmd \varepsilon_{1:N}. 
    %\end{multline*}
    \end{proof} 

\section{Deferred proofs and results of \Cref{sec:HighDim}}

\subsection{Proof of \Cref{prop:GenDomke}}
\label{subsec:proofGenDomke}

\begin{proof}[Proof of \Cref{prop:GenDomke}] 
    For convenience in the proof, let us first introduce the notation 
    $$
    \Ralpha = \w(Z)^{1-\alpha}
    $$ 
    with $Z \sim q_\phi$ and let us observe that under \ref{hyp:VRIWAEwell-defined} we have that $\PE(\Ralpha) > 0$. Furthermore,  \ref{hyp:VRIWAEwell-defined} and Jensen's inequality applied to the concave function $u \mapsto u^{1-\alpha}$ yield 
    \begin{align} \label{eq:PERalphawell-defined}
    \PE(\Ralpha) \leq p_\theta(x)^{1-\alpha} < \infty,
    \end{align}
    meaning that \eqref{eq:boundedExpectationWeights} holds. Now decompose the variational gap into the two following terms:
    \begin{align*}
      \Delta_{N}^{(\alpha)}(\theta, \phi; x) = \lrb{\liren  (\theta,\phi;x) - \mathcal{L}^{(\alpha)}(\theta, \phi; x)} + \lrb{\mathcal{L}^{(\alpha)}(\theta, \phi; x) - \ell(\theta;x)}. %\\
      %& = \frac{1}{1-\alpha} \int \int \prod_{i = 1}^N q_\phi (z_i) \log \lr{ \frac{1}{N} \sum_{j = 1}^N \barw(z_j) ^{1-\alpha} }  \rmd z_{1:N}
    \end{align*}
    To get the desired result \eqref{eq:OneOverNGenDomke}, we only need to study the behavior of the term inside the first bracket. This will be done via an adaptation of the proof of \cite[Theorem 3]{domke2018} to our more general framework, which is provided here for the sake of completeness. We write 
    $$
    \liren  (\theta,\phi;x) - \mathcal{L}^{(\alpha)}(\theta, \phi; x) = \frac{1}{1-\alpha} \PE\lr{ \log \lr{ 1 + \delta_{\alpha, N}}}
    $$
    where for all $z \in \rset^d$,
    \begin{align*}
    \delta_{\alpha, N} = \frac{\RalphaN}{\PE(\Ralpha)} - 1 \in (-1, \infty).
    \end{align*}
     The second-order Taylor expansion of $\log \lr{ 1 + \delta_{\alpha, N}}$ gives
    \begin{align*}
    \log \lr{ 1 + \delta_{\alpha, N}} = \delta_{\alpha, N} - \frac{1}{2} \delta_{\alpha, N}^2 + \int_0^{\delta_{\alpha, N}} \frac{x^2}{1+x} \rmd x.
    \end{align*}
    Now using that $\PE(\delta_{\alpha, N}) = 0$ and that $\PE(\delta_{\alpha, N}^2) = \mathbb{V}_{Z \sim q_\phi}(\Rtalpha(Z)) / N$, we deduce 
    \begin{align*}
    \liren  (\theta,\phi;x) - \mathcal{L}^{(\alpha)}(\theta, \phi; x) =  - \frac{\gamma_\alpha^2}{2N} + \frac{1}{1-\alpha}  \PE\lr{\int_0^{\delta_{\alpha, N}} \frac{x^2}{1+x} \rmd x}.
    \end{align*}
    All that is left to prove is then that 
    \begin{align}\label{eq:leftToProveDomke}
    \lim_{N \to \infty} N \left| \PE\lr{\int_0^{\delta_{\alpha, N}} \frac{x^2}{1+x} \rmd x} \right| = 0.
    \end{align}
    By \cite[Lemma 7]{domke2018}, we have that for all $\varepsilon > 0$ and all $\beta \in (0,1]$ there exist positive constants $C_\varepsilon$ and $D_\beta$ such that
    $$
    \left|\int_0^{\delta_{\alpha, N}} \frac{x^2}{1+x} \rmd x \right| \leq C_\varepsilon \left| \frac{1}{1+ \delta_{\alpha, N}}\right|^{\frac{\varepsilon}{1 + \varepsilon}} \left| \delta_{\alpha, N} \right|^{\frac{2 + 3 \varepsilon}{1 + \varepsilon}} + D_\beta |\delta_{\alpha, N}|^{2 + \beta}
    $$
    and as a result
    \begin{align}\label{eq:Ngap}
    N \left| \PE\lr{\int_0^{\delta_{\alpha, N}} \frac{x^2}{1+x} \rmd x} \right| \leq C_\varepsilon N \PE \lr{\left| \frac{1}{1+ \delta_{\alpha, N}}\right|^{\frac{\varepsilon}{1 + \varepsilon}} \left| \delta_{\alpha, N} \right|^{\frac{2 + 3 \varepsilon}{1 + \varepsilon}}} + D_\beta N \PE \lr{ |\delta_{\alpha, N}|^{2 + \beta}}.
    \end{align}
    Recall that under our assumptions, there exists $\beta > 0$ such that \eqref{eq:conditionVariance} holds. Without loss of generality, one can assume that $\beta \in (0,1]$. [Indeed, assuming that $\beta > 1$, we can find $0<\beta' \leq 1 < \beta$ so that
    $$
    \PE_{Z \sim q_\phi}(|\Rtalpha(Z) - 1|^{2+\beta'}) < \infty,
    $$
    which follows from Jensen's inequality applied to the concave function $u \mapsto u ^{(2 + \beta')/(2+\beta)}$ and from \eqref{eq:conditionVariance}.] Let us now show that the two terms in \eqref{eq:Ngap} go to $0$ as $N \to \infty$ for a suitable choice of $\varepsilon$ ($\varepsilon = \beta/3$).
    
    \begin{itemize}
      \item First term of \eqref{eq:Ngap}. Observe first that Hölder's inequality with $p = (1+\varepsilon)/\varepsilon$ and $q = 1+\varepsilon$ implies the following:
      $$
      \PE \lr{\left| \frac{1}{1+ \delta_{\alpha, N}}\right|^{\frac{\varepsilon}{1 + \varepsilon}} \left| \delta_{\alpha, N} \right|^{\frac{2 + 3 \varepsilon}{1 + \varepsilon}}} \leq \PE \lr{\left| \frac{1}{1+ \delta_{\alpha, N}}\right|} ^{\frac{\varepsilon}{1 + \varepsilon}} \PE\lr{ \left| \delta_{\alpha, N} \right|^{2 + 3 \varepsilon}}^{\frac{1}{1 + \varepsilon}}.
      $$
      From there, we deduce that
      \begin{multline*}
        \limsup_{N \to \infty} N \PE \lr{\left| \frac{1}{1+ \delta_{\alpha, N}}\right|^{\frac{\varepsilon}{1 + \varepsilon}} \left| \delta_{\alpha, N} \right|^{\frac{2 + 3 \varepsilon}{1 + \varepsilon}}} \\
        \leq \limsup_{N \to \infty} \PE \lr{\left| \frac{1}{1+ \delta_{\alpha, N}}\right|} ^{\frac{\varepsilon}{1 + \varepsilon}} \limsup_{N \to \infty} \lrb{ N \PE\lr{ \left| \delta_{\alpha, N} \right|^{2 + 3 \varepsilon}}^{\frac{1}{1 + \varepsilon}}},
      \end{multline*}
      having used that for any two sequences of non-negative real numbers $(a_N)_{N \in \mathbb{N}^*}$ and $(b_N)_{N \in \mathbb{N}^*}$, $\limsup_{N \to \infty} (a_N b_N) \leq \limsup_{N \to \infty} a_N \cdot \limsup_{N \to \infty} b_N$. 
      
      We then obtain that the first limit is bounded by a constant by appealing to \eqref{eq:conditionLimsup}. [Indeed, \eqref{eq:conditionLimsup} means that for sufficiently large $N$, $\PE(1/\RalphaN)$ is bounded by a constant, and hence so is $\PE(|1/(1 + \delta_{\alpha, N})|)$ by combining the boundedness of $\PE(1/\RalphaN)$ with \eqref{eq:PERalphawell-defined}]. 
      
      As for the second limit, \cite[Lemma 5]{domke2018} with $s = 2+ 3 \varepsilon \geq 2$ and $U_i = \Rtalpha(Z_i) - 1$ implies that there exists a constant $B_\varepsilon > 0$ such that
      $$
      \PE \lr{ |\delta_{\alpha, N}|^{2 +  3 \varepsilon}} \leq B_\varepsilon N^{-(2+ 3 \varepsilon)/2} \PE_{Z \sim q_\phi} \lr{\left|\Rtalpha(Z) -1\right|^{2+ 3 \varepsilon}}.
      $$
      Setting $\varepsilon = \beta /3$, we can rewrite the term on the r.h.s. as 
      $$
      B_{\beta/3} N^{-(2+ \beta)/2} \PE_{Z \sim q_\phi} \lr{\left|\Rtalpha(Z) -1\right|^{2+ \beta}},
      $$
      leading in particular to the inequality
      \begin{align}\label{eq:usefulForSecondterm}
      \PE \lr{ |\delta_{\alpha, N}|^{2 +  \beta}} \leq  B_{\beta/3} N^{-(2+ \beta)/2} \PE_{Z \sim q_\phi} \lr{\left|\Rtalpha(Z) -1\right|^{2+ \beta}}.
      \end{align}
    
      Hence, by \eqref{eq:conditionVariance} and since $ N^{-(2+\beta)/2} = o(N^{-1})$, we obtain
      $$
      \limsup_{N \to \infty} \lrb{ N \PE\lr{ \left| \delta_{\alpha, N} \right|^{2 + 3 \varepsilon}}^{\frac{1}{1 + \varepsilon}}} = 0 \quad \mbox{when $\varepsilon = \beta/3$}.
      $$
      As a consequence
      \begin{align} \label{eq:leftToProveDomke1}
        \limsup_{N \to \infty} N \PE \lr{\left| \frac{1}{1+ \delta_{\alpha, N}}\right|^{\frac{\varepsilon}{1 + \varepsilon}} \left| \delta_{\alpha, N} \right|^{\frac{2 + 3 \varepsilon}{1 + \varepsilon}}} = 0 \quad \mbox{when $\varepsilon = \beta/3$}.
      \end{align}
    
      \item Second term of \eqref{eq:Ngap}. Using \eqref{eq:usefulForSecondterm} combined with \eqref{eq:conditionVariance} and since $ N^{-(2+\beta)/2} = o(N^{-1})$, we deduce:
      \begin{align}\label{eq:leftToProveDomke2}
      \lim_{N \to \infty} D_\beta N \PE \lr{ |\delta_{\alpha, N}|^{2 + \beta}} = 0.
      \end{align}
    \end{itemize}
    Combining \eqref{eq:Ngap} with \eqref{eq:leftToProveDomke1} and \eqref{eq:leftToProveDomke2} yields \eqref{eq:leftToProveDomke} and the proof is concluded.
    \end{proof}

\subsection{Proof of \Cref{lem:discussConditions}}

\label{subsec:prooflem:discussConditions}

\begin{proof}[Proof of \Cref{lem:discussConditions}] We prove the two assertions separately.
    \begin{enumerate}
\item Assume that \eqref{eq:conditionVariance} holds with $\alpha = \alpha_2$, that is, there exists $\beta > 0$ such that
    \begin{align*}
    \PE_{Z \sim q_\phi} \lr{\left|\Rtalpha[\alpha_2](Z)- 1\right|^{2+\beta}} < \infty
    \end{align*}
    or equivalently using \eqref{eq:boundedExpectationWeights} with $\alpha = \alpha_2$ and setting $a_2 \eqdef \PE_{Z \sim q_\phi}(\w(Z)^{1-\alpha_2})$ so that $a_2 \in (0, \infty)$, 
    \begin{align} \label{eq:conditionAlphaZero}
    \PE_{Z \sim q_\phi} \lr{\left|\w(Z)^{1-\alpha_2}- a_2\right|^{2+\beta}} < \infty.
    \end{align}
    We now want to prove that \eqref{eq:conditionAlphaZero} implies \eqref{eq:conditionVariance} with $\alpha = \alpha_1$. Using that $|u^{\eta} - 1| \leq |u-1|$ for all $u \geq 0$ and all $\eta \in (0,1)$, we have: for all $z \in \rset^d$,
    \begin{align*}
    \left|\Rtalpha[\alpha_1](z)- 1 \right| \leq \left|\frac{\w(z)^{1-\alpha_2}}{{\PE_{Z \sim q_\phi}(\w(Z)^{1-\alpha_1})}^{\frac{1-\alpha_2}{1-\alpha_1}}}- 1\right|
    \end{align*}
    where we have set $\eta = (1-\alpha_1)/(1-\alpha_2)$. Hence, 
    \begin{align} \label{eq:boundConditionBeforeMinkowski}
    \PE_{Z \sim q_\phi}\lr{\left|\Rtalpha[\alpha_1](Z)- 1\right|^{2+\beta}} \leq \tilde{a}_1^{-1} \PE_{Z \sim q_\phi}\lr{\left|\w(Z)^{1-\alpha_2}- \tilde{a}_1\right|^{2+\beta}},
    \end{align}
    where $\tilde{a}_1 = a_1^{(1-\alpha_2)/(1-\alpha_1)}$ with $a_1 \eqdef {\PE_{Z \sim q_\phi}(\w(Z)^{1-\alpha_1})}$. Note in particular that $a_1$ belongs to $(0, \infty)$ as a consequence of \eqref{eq:boundedExpectationWeights} with $\alpha = \alpha_1$ and thus so does $\tilde{a}_1$. Now observe that, setting $p = 2 + \beta > 1$, Minkowski's inequality implies that 
    \begin{align*}
    \PE_{Z \sim q_\phi}\lr{\left|\w(Z)^{1-\alpha_2}- \tilde{a}_1\right|^p}^{\frac{1}{p}} \leq \PE_{Z \sim q_\phi}\lr{\left|\w(Z)^{1-\alpha_2} - a_2 \right|^{p}}^{\frac{1}{p}}  + \PE_{Z \sim q_\phi}\lr{\left|a_2 - \tilde{a}_1\right|^{p}}^{\frac{1}{p}} 
    \end{align*}
    that is 
    $$
    \PE_{Z \sim q_\phi}\lr{\left|\w(Z)^{1-\alpha_2}- \tilde{a}_1\right|^{2+\beta}}^{\frac{1}{2+\beta}}  \leq \PE_{Z \sim q_\phi}\lr{\left|\w(Z)^{1-\alpha_2} - a_2\right|^{2+ \beta}}^{\frac{1}{2+\beta}}  + \left|a_2 - \tilde{a}_1\right|
    $$
    We then deduce that \eqref{eq:conditionVariance} holds with $\alpha = \alpha_1$ by combining \eqref{eq:boundConditionBeforeMinkowski} with the inequality above and the fact that (i) $\tilde{a}_1^{-1} < \infty$, (ii) $\PE_{Z \sim q_\phi}(|\w(Z)^{1-\alpha_2} - a_2|^{2+ \beta})^{{1}/(2+\beta)} < \infty$ by \eqref{eq:conditionAlphaZero} and (iii) $\left|a_2 - \tilde{a}_1\right| < \infty$.

    \item Assume \eqref{eq:conditionLimsup} holds for $\alpha = \alpha_2$. Then by \Cref{lem:equivlimsupconditionGen} with $k = 1$ we may pick $N$ such that $\PE(1/\RalphaN[\alpha_2]) < \infty$. Observe now that
    \begin{align*}
        \PE\lr{1/\RalphaN[\alpha_1]} & = \PE\lr{\frac{N}{\sum_{i=1}^N \w(Z_i)^{1-\alpha_1}}} \\
        & \leq \PE\lr{\frac{N}{\sum_{i=1}^N \w(Z_i)^{1-\alpha_1}} \;\Bigg| \; \w(Z_i) \leq p_\theta(x) \text{ for all } i = 1, \dots, N}
    \end{align*}
    where we have used that $N/(\sum_{i=1}^N w_{\theta, \phi}(Z_i)^{1-\alpha_1})$ is a decreasing function of each $\w(Z_i)$, with $\w(Z_i)$ being independent random variables. Since $\alpha_1 >  \alpha_2$, it follows that
        \begin{align*}
        \PE\lr{1/\RalphaN[\alpha_1]} & \leq \PE\lr{\frac{N p_\theta(x)^{\alpha_1 - \alpha_2}}{\sum_{i=1}^N \w(Z_i)^{1-\alpha_2}} \;\Bigg| \; \w(Z_i) \leq p_\theta(x) \text{ for all } i = 1 \dots N} \\
        & \leq p_\theta(x)^{\alpha_1 - \alpha_2} \PE\lr{\frac{1}{\RalphaN[\alpha_2]} \;\bigg|\; \w(Z_i) \leq p_\theta(x) \text{ for all } i = 1 \dots N} \\
        & \leq p_\theta(x)^{\alpha_1 - \alpha_2} \frac{\PE(1/\RalphaN[\alpha_2])}{\PP\lr{\w(Z) \leq p_\theta(x)}^N} \\
        & < \infty
    \end{align*}
    since $\PE(\w(Z)) = p_\theta(x)$ which implies that $\PP\lr{\w(Z) \leq p_\theta(x)} > 0$. We see that there exists a choice of $N$ for which $\PE\lr{1/\RalphaN[\alpha_1]} < \infty$, from which \eqref{eq:conditionLimsup} follows by \Cref{lem:equivlimsupconditionGen} with $k = 1$.
\end{enumerate}
\end{proof}

\subsection{Behavior of $\gamma_\alpha^2$}
\label{subsec:behaviorGammaAlpha}

\begin{lem} \label{lem:gammaAlphaToZero} Let $\alpha \in [0,1)$. Then, under common integrability and differentiability assumptions
    $$
    \lim_{\alpha \to 1} \gamma_\alpha^2 = 0.
    $$
    \end{lem}
    
    \begin{proof}
    By definition of $\gamma_\alpha^2$, we have that 
    $$
    \gamma_\alpha^2 = \frac{1}{\PE(\w^{1-\alpha})^2} \cdot \frac{1}{1-\alpha} \PE\lr{\lrb{\w^{1-\alpha} - \PE(\w^{1-\alpha})}^2}.
    $$
    On the one hand, we have that $\PE(\w^{1-\alpha}) \to 1$ as $\alpha \to 1$ under convenient integrability assumptions. On the other hand, for all $z \in \rset^d$,
    \begin{multline*}
      \lim_{\alpha \to 1} \frac{1}{1-\alpha} \lrb{\w(z)^{1-\alpha} - \PE(\w^{1-\alpha})}^2 \\ = \lim_{\alpha \to 1} \lrcb{{2 \lrb{\w(z)^{1-\alpha} - \PE(\w^{1-\alpha})}\cdot \lrb{-\w(z)^{1-\alpha}\log \w(z) + \PE\lr{\w^{1-\alpha}\log \w}}}}
    \end{multline*}
    so that under convenient differentiability assumptions,
    \begin{align*}
      \lim_{\alpha \to 1} \frac{1}{1-\alpha} \PE \lr{\lrb{\w^{1-\alpha} - \PE(\w^{1-\alpha})}^2} = 0
    \end{align*}
    thus implying that $\lim_{\alpha \to 1} \gamma_\alpha^2 = 0$.
    \end{proof}

    \subsection{Proof of \Cref{lem:VarianceExponentialWithD}}
    \label{subsec:lem:VarianceExponentialWithD}
    
    \begin{proof}[Proof of \Cref{lem:VarianceExponentialWithD}]
    Using \eqref{eq:logNormalInLem}, we first deduce that: for all $m \in \rset$,
    \begin{align*}
        \mathbb{E}_{Z \sim q_\phi} \lr{\barw(Z)^m} & = \mathbb{E}_{S \sim \mathcal{N}(0,1)} \left[\exp \left(-\frac{m\sigma^2 d}{2}- m \sigma \sqrt{d} {S} \right) \right] \\
        & = \exp\left(-\frac{m\sigma^2 d}{2}\right) \mathbb{E}_{S \sim \mathcal{N}(0,1)}  \left[\exp\left(- m \sigma \sqrt{d} {S} \right) \right]  \\
       &= \exp\left(-\frac{m\sigma^2 d}{2}\right) \exp \left( \frac{m^2 \sigma^2 d}{2} \right) \\
        &= \exp\left( \frac{m(m-1) \sigma^2 d}{2} \right). 
    \end{align*}
    Therefore, plugging in $m=1-\alpha$, 
    \begin{align*}
            \mathcal{L}^{(\alpha)}(\theta, \phi; x) - \ell(\theta;x) & = \frac{1}{1-\alpha} \log \mathbb{E}_{Z \sim q_\phi} \lr{\barw(Z)^{1-\alpha}} = - \frac{\alpha \sigma^2 d}{2},
            \end{align*}
    which gives the desired result for $\mathcal{L}^{(\alpha)}(\theta, \phi; x) - \ell(\theta;x)$. In addition: for all $m \in \rset$, 
    \begin{align} \label{eq:computeExpectationLogNormal}
        \PE_{Z \sim q_\phi}(\w(Z)^m) = \exp\left( \frac{m(m-1) \sigma^2 d}{2} \right) p_\theta(x)^m.
    \end{align}
    Now note that \eqref{eq:logNormalInLem} can be rewritten as: for all $i = 1 \ldots N$,
    \begin{align*} %\label{eq:lognormalweightsInProof}
    \log \w(z_i)  ={-\frac{\sigma^2 d}{2}- \sigma \sqrt{d} {S_i}} + \log p_\theta(x), \quad S_i \sim \mathcal{N}(0,1).
    \end{align*}  
    Hence, we get that: for all $i = 1 \ldots N$, 
        \begin{align*}
        \log \Rtalpha(z_i) & = (1-\alpha) \log {\w(z_i)} - \log {\PE_{Z \sim q_\phi}(\w(Z)^{1-\alpha})} \\
        & = - (1-\alpha)\sigma \sqrt{d} {S_i} -  \frac{(1-\alpha)^2 \sigma^2 d}{2}
        \end{align*}
where we have used \eqref{eq:computeExpectationLogNormal} with $m = 1 - \alpha$. As a result,
        \begin{align*}
        \mathbb{V}_{Z \sim q_\phi}(\Rtalpha(Z)) & = \mathbb{V}_{S \sim \mathcal{N}(0,1)}\lr{\exp \lrcb{-(1-\alpha)\sigma \sqrt{d} {S}} \exp\lrcb{- \frac{(1-\alpha)^2 \sigma^2 d}{2}} } \\
        & = \exp \lr{(1-\alpha)^2 \sigma^2 d }   \lr{\exp \lr{(1-\alpha)^2 \sigma^2 d} - 1}  \exp\lr{- (1-\alpha)^2 \sigma^2 d}   \\
        & = \exp \lr{(1-\alpha)^2 \sigma^2 d} - 1,
        \end{align*} 
    which yields the desired result for $\gamma_\alpha^2$.   In addition, we show that \eqref{eq:conditionVariance} and \eqref{eq:conditionLimsup} both hold, meaning that we can apply \Cref{prop:GenDomke}. To see this, note that $\mathbb{E}_{Z \sim q_\phi} \lr{\w(Z)^m}$ and $\mathbb{E}_{Z \sim q_\phi} (\Rtalpha(Z)^m)$ are well-defined and finite for all $\alpha,m\in\mathbb{R}$. 
    Furthermore, observe that by convexity of the function $u \mapsto |u|^{2 + \beta}$ with $\beta >0$, it holds that
\begin{align*}
    %\label{eq:centralvsrawmoment}
    \PE_{Z \sim q_\phi}(|\Rtalpha(Z)  - 1|^{2+\beta}) \leq 2^{1+\beta} \lr{ \PE_{Z \sim q_\phi}(\Rtalpha(Z)^{2+\beta}) + 1}
\end{align*}
thus \eqref{eq:conditionVariance} holds for all $\beta>0$. Lastly, $\PE(1/\RalphaN)\leq\PE(N^{-1} \sum_{i=1}^N \w(Z_i)^{\alpha-1})$ by the HM-AM inequality and the r.h.s. is finite for all $\alpha\in\mathbb{R}$ thus \eqref{eq:conditionLimsup} also holds. 

In the particular case $p_\theta(z|x) = \mathcal{N}(z;\theta, \boldsymbol{I}_d)$ and $q_\phi(z|x) = \mathcal{N}(z; \phi, \boldsymbol{I}_d)$: for all $i = 1 \ldots N$, 
\begin{align}
  \log \barw(z_i) & = - \frac{1}{2} \lr{ \| z_i - \theta \|^2 - \| z_i - \phi\|^2 } \nonumber \\
  & = - \frac{1}{2} \lr{ \|\theta \|^2 - \|\phi\|^2 + 2 \langle z_i, \phi - \theta \rangle} \nonumber \\
  & = - \frac{1}{2} \lr{ \langle \theta - \phi, \theta + \phi \rangle  + 2 \langle z_i, \phi - \theta \rangle} \nonumber \\
  & = - \frac{1}{2} \lr{ \langle \theta - \phi, \theta - \phi + 2 \phi \rangle  + 2 \langle z_i, \phi - \theta \rangle}  \nonumber \\
  & = - \frac{1}{2} \lr{ B_d^2 + 2 \langle z_i - \phi, \phi - \theta \rangle}   \nonumber \\
  & = - \frac{B_d^2}{2} - B_d S_i \quad \mbox{with $S_i = \frac{1}{B_d} \langle z_i - \phi, \phi- \theta \rangle $,} \label{eq:logNormalBdProof}
\end{align}
where we have set $B_d = \|\phi - \theta\|$. Since $z_i - \phi \sim \mathcal{N}(0,\boldsymbol{I}_d)$, it follows that $S_i\sim \mathcal{N}(0,1)$ as required. Lastly, when $\theta = 0 \cdot \boldsymbol{u}_d$ and $\phi = \boldsymbol{u}_d$, we have that $B_d = \sqrt{d}$.
\end{proof}

\subsection{Proof of \Cref{ex:LinGaussThm3}}

\label{subsec:linGaussExApp2}

\begin{proof}
Let us first prove that $p_\theta(x)=\mathcal{N}(x;\theta, 2\boldsymbol{I}_d)$ and $p_\theta(z|x)=\mathcal{N}(z; (\theta+x)/2, 1/2 ~ \boldsymbol{I}_d)$. To see this, note that
    \begin{align*}
    p_\theta(x, z) & = \lr{\frac{1}{(2\pi)^{{d}/{2}}}}^2 \exp \lr{- \frac{1}{2} \lrcb{\| z - \theta \|^2 + \| x - z \|^2}  }.
    \end{align*}
    As a result, only considering the dependency in $z$, we have that
    \begin{align*}
      p_\theta(x, z) & \propto \exp \lr{- \frac{1}{2} \cdot 2 \Big\| z - \frac{\theta + x}{2} \Big\|^2   },
    \end{align*}
    which implies that $p_\theta(z|x)= \mathcal{N}(z; (\theta+x)/2, 1/2 ~ \boldsymbol{I}_d)$. Furthermore, 
    \begin{align*}
      p_\theta(x) & = \int p_\theta(x, z) \rmd z \\
      & = \frac{1}{(2\pi \cdot 2)^{{d}/{2}}}\int \frac{1}{(2\pi \cdot 1/2)^{{d}/{2}}} \exp \lr{ - \frac{1}{2} \cdot 2 \Big\| z - \frac{\theta + x}{2} \Big\|^2 } \rmd z \cdot \exp \lr{ - \frac{1}{2} \cdot \frac{1}{2} \| \theta - x \|^2  } \\
      & =\mathcal{N}(x;\theta, 2\boldsymbol{I}_d).
    \end{align*}
    Hence, for all $z \in \rset^d$
    \begin{align*}
    \log \barw(z) & = \log \lr{ \frac{p_\theta(z|x)}{q_\phi(z|x)} } \\
    & =  \log \lr{ \frac{(2 \pi \cdot 2/3  )^{d/2}}{(2 \pi \cdot 1/2  )^{d/2} } \exp \lrb{  - \Big\| z - \frac{\theta + x}{2} \Big\|^2  + \frac{3}{4} \|z-Ax-b\|^{2} }  } 
    \end{align*}
    and from there, we can straightforwardly deduce that: for all $i = 1 \ldots N$,
    \begin{align} \label{eq:LinGaussWeights}
        \log \barw(z_i) = \frac{d}{2} \log \lr{ \frac{4}{3}} - \Big\| z_i - \frac{\theta + x}{2} \Big\|^2 + \frac{3}{4} \|z_i-Ax-b\|^{2}. %, \quad i = 1 \ldots N.
      \end{align}
Using \eqref{eq:LinGaussWeights} we can write that 
\begin{align*}
\mathcal{L}^{(\alpha)}(\theta, \phi; x) - \ell(\theta; x) & = \frac{1}{1-\alpha} \log \lr{\int q_\phi(z|x) \barw(z)^{1-\alpha} \rmd z    } \\
& = \frac{d}{2} \log \lr{\frac{4}{3}} + \frac{1}{1-\alpha} \log \lr{ I }
\end{align*}
where
\begin{align*}
I & = \int \frac{1}{(4 \pi / 3)^{d/2}} \exp \lrb{- \frac{3}{4}  \|z-Ax-b\|^{2} + (1-\alpha) \lr{ - \Big\| z - \frac{\theta + x}{2} \Big\|^2 + \frac{3}{4} \|z-Ax-b\|^{2} } } \rmd z \\
  & = \int \frac{1}{(4 \pi / 3)^{d/2}} \exp \lrb{- \frac{3 \alpha}{4}  \|z-Ax-b\|^{2} - (1-\alpha) \Big\| z - \frac{\theta + x}{2} \Big\|^2 } \rmd z. 
\end{align*}
$I$ is well-defined and finite for all $\alpha<4$. Completing the square leads to 
\begin{align*}
    I = \lr{\frac{3}{4-\alpha}}^{d/2} \exp \lr{ \lr{\frac{4}{4 -\alpha}} \Big\|  \frac{3 \alpha}{4}  (Ax+b) + (1-\alpha) \frac{\theta + x}{2} \Big\|^2  - \frac{3 \alpha}{4}  \|Ax+b\|^{2} - (1-\alpha) \Big\|\frac{\theta + x}{2} \Big\|^2}.
\end{align*}
As a result,
\begin{multline*}
    \mathcal{L}^{(\alpha)}(\theta, \phi; x) - \ell(\theta; x) =    \frac{d}{2} \lrb{\log \lr{\frac{4}{3}} + \frac{1}
    {1-\alpha} \log \lr{\frac{3}{4-\alpha}}} \\ 
    +  \frac{4}{(4 -\alpha)(1-\alpha)} \Big\|  \frac{3 \alpha}{4}  (Ax+b) + (1-\alpha) \frac{\theta + x}{2} \Big\|^2 - \frac{3 \alpha}{4(1-\alpha)}  \|Ax+b\|^{2} - \Big\|\frac{\theta + x}{2} \Big\|^2 .
\end{multline*}
It can then be checked that the second line simplifies to $- \frac{3\alpha}{4-\alpha}\Big\|Ax+b-\frac{\theta+x}{2}\Big\|^2$, from which we deduce the desired result for $\mathcal{L}^{(\alpha)}(\theta, \phi; x) - \ell(\theta; x)$. [Notice in particular that $\mathcal{L}^{(\alpha)}(\theta, \phi; x) - \ell(\theta; x)$ is well-defined for all $\alpha\neq1,\alpha<4$ with continuous extension at $\alpha=1$.] On the other hand, we have that
\begin{equation*}
    \gamma_{\alpha}^2=\frac{1}{1-\alpha}\mathbb{V}_{Z\sim q_{\phi}}(\overline{w}_{\theta,\phi}^{(\alpha)}(Z))=\frac{1}{1-\alpha}\left(\frac{\mathbb{E}_{Z \sim q_\phi} \lr{\barw(Z)^{2-2\alpha}}}{\mathbb{E}_{Z \sim q_\phi} \lr{\barw(Z)^{1-\alpha}}^{2}}-1\right). 
\end{equation*}
Furthermore, for all $\alpha ' \in \rset \setminus \lrcb{1}$, it holds that
\begin{equation}
    \mathbb{E}_{Z \sim q_\phi} \lr{\barw(Z)^{1-\alpha'}} = \exp \left( (1-\alpha')  \lrb{\mathcal{L}^{(\alpha')}(\theta, \phi; x) - \ell(\theta;x) } \right).
    \label{eq:linear_gaussian_wbar_moment}
\end{equation}
Now using \eqref{eq:linear_gaussian_wbar_moment} with $\alpha' = \alpha$ and $\alpha' = 2 \alpha -1$ and combining with the expression of $\mathcal{L}^{(\alpha)}(\theta, \phi; x) - \ell(\theta; x)$, we obtain
\begin{align*}
    \gamma_{\alpha}^2 = \frac{1}{1-\alpha} \lr{ \exp (A) - 1}
\end{align*}
with 
\begin{align*}
    A & = 2(1-\alpha) \lrcb{  \frac{d}{2} \lrb{\log \lr{\frac{4}{3}} + \frac{1}
    {2(1-\alpha)} \log \lr{\frac{3}{5-2\alpha}}} -  \frac{3 (2 \alpha-1)}{5-2\alpha}\Big\|Ax+b-\frac{\theta+x}{2}\Big\|^2} \\
    & \quad - 2(1-\alpha) \lrcb{\frac{d}{2} \lrb{\log \lr{\frac{4}{3}} + \frac{1}
    {1-\alpha} \log \lr{\frac{3}{4-\alpha}}}   - \frac{3\alpha}{4-\alpha}\Big\|Ax+b-\frac{\theta+x}{2}\Big\|^2} \\
    & = \frac{d}{2} \log \lr{\frac{(4-\alpha)^2}{5-2\alpha}} + \frac{24(1-\alpha)^2}{(5-2\alpha)(4-\alpha)}\Big\|Ax+b-\frac{\theta+x}{2}\Big\|^2
\end{align*}
%\begin{align*}
%    \gamma_{\alpha, d}^2 = \frac{1}{1-\alpha}\left[(4-\alpha)^{d}(15-6\alpha)^{-\frac{d}{2}}\exp\left(\frac{24(1-\alpha)^2}{(5-2\alpha)(4-\alpha)}\Big\|Ax+b-\frac{\theta+x}{2}\Big\|^2\right)-1\right], 
%\end{align*} 
from which we deduce the desired result for $\gamma_{\alpha}^2$ [notice in particular that $\gamma_{\alpha}^2$ is well-defined for all $\alpha < 5/2$]. 
% first observe that \ref{hyp:VRIWAEwell-defined} clearly holds.

In addition, the assumptions made in \Cref{prop:GenDomke} are satisfied for all $\alpha \in [0,1)$. To see this, set $m = 1 -\alpha'$ in \eqref{eq:linear_gaussian_wbar_moment} and use that $\mathcal{L}^{(\alpha)}(\theta, \phi; x) - \ell(\theta; x)$ is well-defined for all $\alpha\neq1,\alpha<4$ with continuous extension at $\alpha=1$, so that $\mathbb{E}_{Z \sim q_\phi} (\barw(Z)^m)$ and thus $\mathbb{E}_{Z \sim q_\phi} \lr{\w(Z)^m}$ and $\mathbb{E}_{Z \sim q_\phi} (\Rtalpha(Z)^m)$ are well-defined and finite for all $m>-3$. Furthermore, the convexity of the function $u \mapsto |u|^{2 + \beta}$ with $\beta >0$ implies that
\begin{align*}
    \PE_{Z \sim q_\phi}(|\Rtalpha(Z)  - 1|^{2+\beta}) \leq 2^{1+\beta} \lr{ \PE_{Z \sim q_\phi}(\Rtalpha(Z)^{2+\beta}) + 1}
\end{align*}
thus \eqref{eq:conditionVariance} holds for all $\beta>0$. Lastly, $\PE(1/\RalphaN)\leq\PE(N^{-1} \sum_{i=1}^N \w(Z_i)^{\alpha-1})$ by the HM-AM inequality and the r.h.s. is finite for all $\alpha>-2$ thus \eqref{eq:conditionLimsup} also holds. 
\end{proof}

\subsection{Proof of \Cref{lem:rewritingrenyi}}
\label{sec:prooflemrewriting}
  
\begin{proof}[Proof of \Cref{lem:rewritingrenyi}]
For all $\alpha \in [0,1)$, we can rewrite the variational gap $\Delta_{N,d}^{(\alpha)}(\theta, \phi)$ as 
\begin{align*}
  \Delta_{N,d}^{(\alpha)}(\theta, \phi;x) & = \frac{1}{1-\alpha} \int\int \prod_{i = 1}^N q_\phi (z_i) \log \lr{ \frac{1}{N} \sum_{j = 1}^N \overline{w}_j ^{1-\alpha} }  \rmd\overline{w}_{1:N} \\ 
  & = \frac{1}{1-\alpha} \int \int\prod_{i=1}^{N}q_{\phi}(z_{i})\log\lr{\frac{1}{N}\sum_{j=1}^{N} (\overline{w}^{(j)})^{1-\alpha}}\rmd\overline{w}_{1:N} \nonumber \\ 
  & = \frac{1}{1-\alpha} \left[\int\int\prod_{i=1}^{N}q_{\phi}(z_{i})\log\left(\frac{1}{N}(\overline{w}^{\left(N\right)})^{1-\alpha}\right)\rmd \overline{w}_{1:N} \right. \nonumber \\
  &\quad \quad \left.+ \int\int\prod_{i=1}^{N}q_{\phi}(z_i)\log\left(1+  \sum_{j = 1}^{N-1}\lr{\frac{\overline{w}^{(j)}}{\overline{w}^{(N)}}}^{1-\alpha} \right)\rmd\overline{w}_{1:N} \right]\nonumber \\
  & = \Delta_{N,d}^{(\alpha, MAX)}(\theta, \phi;x) + R_{N,d}^{(\alpha)}(\theta, \phi;x)
\end{align*}
where we have used \eqref{eq:DeltaMaxNDExt} and where we have set
\begin{align*}
& R_{N,d}^{(\alpha)} (\theta, \phi;x) \eqdef \frac{1}{1-\alpha}\int \int \prod_{i = 1}^N q_\phi (z_i) \log \lr{1 + \sum_{j = 1}^{N-1} \lr{\frac{\overline{w}^{(j)}}{\overline{w}^{(N)}}}^{1-\alpha} }  \rmd\overline{w}_{1:N}.
\end{align*}
All that is left to do is now to prove \eqref{eq:BoundRNdExt}. Observe that by definition of $T_{N,d}^{(\alpha)}$ in \eqref{eq:expressionT_NdExt} and since $\alpha \in [0,1)$, we can write
\begin{align*}
0 \leq R_{N,d}^{(\alpha)} (\theta, \phi;x) & = \frac{1}{1-\alpha}\int \int \prod_{i = 1}^N q_\phi (z_i) \log \lr{1 + T_{N,d}^{(\alpha)}}  \rmd\overline{w}_{1:N}\\
& \leq  \frac{1}{1-\alpha} \int \prod_{i = 1}^N q_\phi (z_i) ~ T_{N,d}^{(\alpha)} ~  \rmd\overline{w}_{1:N} \\
& = \frac{1}{1-\alpha}  \PE (T_{N,d}^{(\alpha)}),
\end{align*}
which concludes the proof.
\end{proof}
    
\subsection{Deferred proofs of \Cref{subsec:lognormalExporegime}}

\subsubsection{Proof of \Cref{lem:S1approxGauss}}
\label{subsec:proofS1approxGauss}

\begin{proof}[Proof of \Cref{lem:S1approxGauss}]
    First, note that since $S_1, \ldots, S_N$ are i.i.d. normal random variables, so are $-S_1, \ldots, -S_N$. Setting $M_N = \max_{1 \leq i \leq N} -S_i$, we also have $S^{(1)} = - M_N$. A standard result \cite[obtained, for example, by combining Theorem 1.1.2 and Example 1.1.7 in][]{de2007extreme} is that for all $x \in \rset$,
    \begin{equation}\label{eq:cvDistributionGumbel}
        \lim_{N \to \infty } P\lr{ a_N^{-1} \left(M_N - b_N\right) \leq x} = \exp(-e^{-x})
    \end{equation}
    with $a_N = 1/\sqrt{2\log N}$ and $b_N = \sqrt{2\log N}-\frac{1}{2}(\log\log N+\log4\pi)/(\sqrt{2\log N})$. Since $\PE(|M_N|) \leq \PE(M_N^2)^{1/2} \leq \PE(\sum_{i=1}^N S_i^2)^{1/2} \leq N^{1/2} < \infty$ for all $N$, it follows by \cite[Theorem 2.1]{pickands1968momentconvergence} that
    \begin{equation*}
        \lim_{N \to \infty } a_N^{-1} \lr{\PE\lr{M_N} - b_N} = \PE(U),
    \end{equation*}
    where $U$ is a Gumbel random variable and $\mathbb{E}(U)$ is given by the Euler--Masceroni constant. Using that $S^{(1)} = - M_N$, we deduce
    \begin{equation*}
        \lim_{N \to \infty } -a_N^{-1} \lr{\PE(S^{(1)}) + b_N} = \PE(U).
    \end{equation*}
    Finally, plugging in the definition of $a_N$ and $b_N$, we obtain
    \begin{align*}
        \PE(S^{(1)}) &= -\sqrt{2\log N}+\frac{\log\log N+\log4\pi}{2\sqrt{2\log N}} - \frac{\PE(U)}{\sqrt{2 \log N}} + o \lr{\frac{1}{\sqrt{2 \log N}}} \\
         & = -\sqrt{2\log N} + O \lr{\frac{\log \log N}{\sqrt{\log N}}}
    \end{align*}
    and we have thus recovered \eqref{eq:ExpS1approxGen}.
\end{proof}

\subsubsection{Proof of \Cref{prop:limDeltaGaussianRenyi}}
\label{subsec:prooflimDeltaGaussian}

\begin{proof}[Proof of \Cref{prop:limDeltaGaussianRenyi}]
    First note that 
    \begin{equation*}%\label{eq:logweights}
        \log \overline{w}^{(N)} = -\frac{d \sigma^2}{2}- \sqrt{d}\sigma S^{(1)}.
    \end{equation*}
    Combining this result with the definition of $\Delta_{N,d}^{MAX}(\theta, \phi;x)$ in \eqref{eq:DeltaMaxNDExt} yields
    \begin{equation*}
        \Delta_{N,d}^{MAX}(\theta, \phi) = - \frac{d \sigma^2}{2} - \sqrt{d} \sigma \PE(S^{(1)}) + \frac{\log N}{\alpha -1}.
    \end{equation*}
    Now using \eqref{eq:ExpS1approxGen}, we deduce
    \begin{align*}
        \Delta_{N,d}^{MAX}(\theta, \phi) & = - \frac{d \sigma^2}{2} + \sqrt{d} \sigma \lr{ \sqrt{2 \log N} + O \lr{\frac{\log \log N}{\sqrt{\log N}}}} + \frac{\log N}{\alpha -1} \\
        & = - \frac{d \sigma^2}{2} \lrcb{1 - 2 \sqrt{\frac{2 \log N}{ d\sigma^2}} + \frac{1}{1-\alpha} \frac{2\log N}{d \sigma^2} + O \lr{\frac{\log \log N}{\sqrt{d \log N}}}},
    \end{align*}
    which concludes the proof.    
\end{proof}

\subsubsection{Proof of \Cref{prop:limitTGaussianRenyi}}   
\label{sec:Proofgaussian}

We first prove a useful intermediate lemma regarding the concentration of $S^{(1)}$.

\begin{lem}\label{lem:exactnormalminconcentration}
    Let $S_1, \ldots, S_N$ be i.i.d. normal random variables, set $S^{(1)} = \min_{1 \leq i \leq N} S_i$, and define $I_N = [-4 \sqrt{\log N}, - \sqrt{\log N}]$. Then as $N \rightarrow \infty$, we have
    \begin{equation*}
        \PP(S^{(1)} \not \in I_N) = O\lr{\frac{1}{N^4}}.
    \end{equation*}
\end{lem}

\begin{proof}
    We control the probability of the events $\{S^{(1)} > - \sqrt{\log N} \}$ and $\{S^{(1)} < - 4 \sqrt{\log N}\}$ separately. First, note that
    \begin{align}
        \log \PP(S^{(1)} > - \sqrt{\log N}) & = N \log (\overline{\Phi}(- \sqrt{\log N})) \nonumber \\
        & = N \log \left(1 - \overline{\Phi}(\sqrt{\log N}) \right) \nonumber \\
        & = - N \overline{\Phi}(\sqrt{\log N})(1 + o(1)) \nonumber \\
        & = - N \frac{\phi(\sqrt{\log N})}{\sqrt{\log N}}(1 + o(1)) \label{eq:lowerboundapprox}
    \end{align}
    where in the final line we have used the standard approximation
    \begin{equation}
        \label{eq:NormalCDFAsymptotics}
        \overline{\Phi}(x) = \frac{\phi(x)}{x}(1+o(1))
    \end{equation}
    as $x \rightarrow \infty$. We deduce that
    \begin{equation}
        \label{eq:lowerboundasymp}
        \PP(S^{(1)} > -\sqrt{\log N}) = \exp\left\{ -N \frac{N^{-1/2}}{\sqrt{2\pi}\sqrt{\log N}}(1 + o(1)) \right\} = O\left(\frac{1}{N^4}\right)
    \end{equation}
    as $N \rightarrow \infty$. Second, as $N \rightarrow \infty$ we have, by a union bound, that
    \begin{align}
        \PP(S^{(1)} < - 4 \sqrt{\log N}) & \leq N \Phi(- 4 \sqrt{\log N}) \nonumber \\
        & = N \frac{\phi(4 \sqrt{\log N})}{4 \sqrt{\log N}}(1 + o(1)) \nonumber \\
        &= \frac{N^{-7}}{4 \sqrt{2\pi} \sqrt{\log N}}(1 + o(1)) \nonumber \\
        &= O\left(\frac{1}{N^4}\right) \label{eq:upperboundasymp}
    \end{align}
    and so the result follows.
\end{proof}
We now prove \Cref{prop:limitTGaussianRenyi} by building on the proof from \cite{ObstaclestoHighDimensionalParticleFiltering} and on \Cref{lem:exactnormalminconcentration}. 

\begin{proof}[Proof of \Cref{prop:limitTGaussianRenyi}]
Denote $\sigma_\alpha = (1-\alpha) \sigma$ for all $\alpha \in [0,1)$. A first remark is that, conditional upon $S^{(1)}$, we can think of the sum in \eqref{eq:expressionT_NdExt} as the sum over $N-1$ i.i.d. random variables
    \begin{align*}
        \PE(T_{N,d}^{(\alpha)} | S^{(1)}) = (N-1) \PE \lr{ \exp \lr{- \sigma_\alpha \sqrt{d} ( S - S^{(1)})}  }
    \end{align*}    where the expectation is w.r.t. the density of $S$ given by
    \begin{align*}
        p(z)=\frac{\phi(z)}{\overline{\Phi}(S^{(1)})} \mathbb{I}(z\geq S^{(1)}),
    \end{align*}
    with $\phi(z)$ denoting the standard normal density and $\overline{\Phi}(x)=\int_{x}^{\infty}\phi(z)\rmd z$ denoting the normalizing constant. Then,
    \begin{equation*}
        \mathbb{E}(T_{N,d}^{(\alpha)}|S^{(1)}) = \frac{(N-1)\int_{S^{(1)}}^\infty \exp \lr{- \sigma_\alpha \sqrt{d} \lr{z - S^{(1)}}} \phi(z) \rmd z}{\overline{\Phi}(S^{(1)})}.
    \end{equation*}
    We can then calculate explicitly
    \begin{align*}
        & \hspace{5mm} \int_{S^{(1)}}^\infty \exp \lr{- \sigma_\alpha \sqrt{d} \lr{z - S^{(1)}}} \phi(z) \rmd z \\
        & = \exp(\sigma_\alpha\sqrt{d}S^{(1)}+\sigma_\alpha^2 d/2) \int_{S^{(1)}}^\infty (\sqrt{2\pi})^{-1} \exp \lr{-\frac{1}{2} (z + \sigma_\alpha \sqrt{d})^2 } \rmd z \\
        & = \exp(\sigma_\alpha\sqrt{d}S^{(1)}+\sigma_\alpha^2 d/2) \overline{\Phi}(\sigma_\alpha\sqrt{d}+S^{(1)}).
    \end{align*}
    Denoting $I_N = [- 4 \sqrt{\log N}, - \sqrt{\log N}]$, on the event $\{S^{(1)} \in I_N\}$, as $N,d \rightarrow \infty$ with $\log N/d \rightarrow 0$, we have
    $$\sigma_\alpha \sqrt{d} + S^{(1)} = \sigma_\alpha \sqrt{d}(1 + o_{N,d}(1)),$$
    where we use the notation $o_{N,d}(1)$ to denote that the implicit constant, which goes to zero as $N,d \rightarrow \infty$ with $\log N/d \rightarrow 0$, does not depend on $S^{(1)}$. Using the approximation \eqref{eq:NormalCDFAsymptotics} for $\Phi(x)$ as $x \to \infty$,
    \begin{equation*}
        \overline{\Phi}(\sigma_\alpha\sqrt{d}+S^{(1)})=\frac{\phi(\sigma_\alpha \sqrt{d} +S^{(1)})}{\sigma_\alpha \sqrt{d} +S^{(1)}}(1+o_{N,d}(1)).
    \end{equation*}
    Hence, observing that $\exp(\sigma_\alpha\sqrt{d}S^{(1)}+\sigma_\alpha^2 d/2) \phi(\sigma_\alpha \sqrt{d} +S^{(1)}) = \phi(S^{(1)})$, it follows that
    \begin{equation*}
        \int_{S^{(1)}}^\infty \exp \lr{- \sigma_\alpha \sqrt{d} \lr{z - S^{(1)}}} \phi(z) \rmd z = \frac{\phi(S^{(1)})}{\sigma_\alpha \sqrt{d}} (1+ o_{N,d}(1))
    \end{equation*}
    on the event $\{S^{(1)} \in I_N\}$. Using \eqref{eq:NormalCDFAsymptotics}, we can also write
    \begin{equation*}
        \Phi(S^{(1)}) = \overline{\Phi}(-S^{(1)}) = \frac{\phi(-S^{(1)})}{-S^{(1)}} (1 + o_{N,d}(1)).
    \end{equation*}
    Combined with $\phi(-S^{(1)}) = \phi(S^{(1)})$, this allows us to deduce that
    \begin{align}
        \int_{S^{(1)}}^\infty \exp \lr{- \sigma_\alpha \sqrt{d} \lr{z - S^{(1)}}} \phi(z) \rmd z & = \frac{(-S^{(1)})\Phi(S^{(1)})}{\sigma_\alpha \sqrt{d}}(1 + o_{N,d}(1)) \nonumber \\
        & \leq \Phi(S^{(1)}) \frac{4 \sqrt{\log N}}{\sigma_\alpha \sqrt{d}}(1 + o_{N,d}(1)), \label{eq:conditionalexpbound}
    \end{align}
    all on the event $\{S^{(1)} \in I_N\}$. Finally, since $S^{(1)} < - \sqrt{\log N}$ implies
    \begin{equation*}
        \overline{\Phi}(S^{(1)}) = 1 + o_{N,d}(1),
    \end{equation*}
    we have
    \begin{equation*}
        \mathbb{E}(T_{N,d}^{(\alpha)}|S^{(1)}) \leq (N-1)\Phi(S^{(1)}) \frac{4 \sqrt{\log N}}{\sigma_\alpha \sqrt{d}}(1 + o_{N,d}(1)).
    \end{equation*}
    We conclude by using the tower law of expectation and splitting according to whether $S^{(1)} \in I_N$, giving
    \begin{align*}
        \mathbb{E}(T_{N,d}^{(\alpha)}) &= \mathbb{E} \left( \mathbb{E}(T_{N,d}^{(\alpha)}|S^{(1)}) \right) \\
        & \leq \mathbb{E} \left( \mathbbm{1}_{\{S^{(1)} \in I_N\}}(N-1) \Phi(S^{(1)}) \frac{4 \sqrt{\log N}}{\sigma_\alpha \sqrt{d}}(1 + o_{N,d}(1)) \right) + \mathbb{E} \left( \mathbbm{1}_{\{S^{(1)} \not \in I_N\}} \right) \\
        & \leq (N-1) \frac{4 \sqrt{\log N}}{\sigma_\alpha \sqrt{d}} \mathbb{E}\left(\Phi(S^{(1)})\right)(1 + o(1)) + \PP(S^{(1)} \not \in I_N) \\
        & \leq \frac{4 \sqrt{\log N}}{\sigma_\alpha \sqrt{d}}(1 + o(1)) + O\left(\frac{1}{N^4}\right) \rightarrow 0
    \end{align*}
    where in the final line we have used \Cref{lem:exactnormalminconcentration} and that $\Phi(S^{(1)})$ is distributed as the minimum of $N$ independent uniform random variables on $[0,1]$ and so $\mathbb{E}(\Phi(S^{(1)}))=\frac{1}{N+1}$.

\end{proof}

\subsubsection{Proof of \Cref{thm:MainGauss2}}
\label{subsub:proof:thm:MainGauss2}

    First note that \Cref{lem:S1approxGauss} and \Cref{lem:exactnormalminconcentration} are not affected by the change in the distribution of the weights we have made in \eqref{eq:lognormalweights2}. As for \Cref{prop:limDeltaGaussianRenyi} and \Cref{prop:limitTGaussianRenyi}, they are modified according to \Cref{prop:limDeltaGaussianRenyi2} and \Cref{prop:limitTGaussianRenyi2} below.

    \begin{prop}\label{prop:limDeltaGaussianRenyi2} Let $S_1, \ldots, S_N$ be i.i.d. normal random variables. % and let $S^{(1)} \leq \ldots \leq S^{(N)}$  be the ordered sequence of $S_1, \ldots, S_N$. 
      Further assume that the weights $\overline{w}_1, \ldots, \overline{w}_N$ satisfy \eqref{eq:lognormalweights2} and that there exists $\sigma_- > 0$ such that $B_d \geq \sigma_- \sqrt{d}$. Then, for all $\alpha \in [0,1)$, 
      \begin{align*}
      \lim_{N,d \to \infty} \Delta_{N,d}^{(\alpha, MAX)}(\theta, \phi;x) + \frac{B_d^2}{2} \lrcb{ 1 -  2 ~ \frac{\sqrt{2 \log N}}{B_d} + \frac{1}{1-\alpha} ~\frac{2\log N}{B_d^2} + O \lr{\frac{\log \log N}{B_d \sqrt{\log N}}} } = 0.
      \end{align*}
    \end{prop}

    \begin{proof}
      First note that 
      \begin{equation*}%\label{eq:logweights}
          \log \overline{w}^{(N)} = -\frac{B_d^2}{2}- B_d S^{(1)}.
      \end{equation*}
      Combining this result with the definition of $\Delta_{N,d}^{MAX}(\theta, \phi;x)$ in \eqref{eq:DeltaMaxNDExt} yields
      \begin{equation*}
          \Delta_{N,d}^{MAX}(\theta, \phi) = - \frac{B_d^2}{2} - B_d \PE(S^{(1)}) + \frac{\log N}{\alpha -1}.
      \end{equation*}
      Now using \eqref{eq:ExpS1approxGen} of \Cref{lem:S1approxGauss}, we deduce
      \begin{align*}
          \Delta_{N,d}^{MAX}(\theta, \phi) & = - \frac{B_d^2}{2} + B_d \lr{ \sqrt{2 \log N} + O \lr{\frac{\log \log N}{\sqrt{\log N}}}} + \frac{\log N}{\alpha -1},
      \end{align*}
      which concludes the proof.    
  \end{proof}
  
  \begin{prop}\label{prop:limitTGaussianRenyi2} Let $S_1, \ldots, S_N$ be i.i.d. normal random variables. %and let $S^{(1)} \leq \ldots \leq S^{(N)}$  be the ordered sequence of $S_1, \ldots, S_N$. 
  Further assume that the weights $\overline{w}_1, \ldots, \overline{w}_N$ satisfy  \eqref{eq:lognormalweights2} and that there exists $\sigma_- > 0$ such that $ B_d \geq \sigma_- \sqrt{d} $. Then, for all $\alpha \in [0,1)$, we have
  \begin{align*}
        \lim_{\substack{N, d \to \infty \\ \log N / d \to 0}} \PE(T^{(\alpha)}_{N,d}) = 0.
  \end{align*}
  \end{prop}

  \begin{proof}
Conditional upon $S^{(1)}$, we can think of the sum in \eqref{eq:expressionT_NdExt} as the sum over $N-1$ i.i.d. random variables
    \begin{align*}
        \PE(T_{N,d}^{(\alpha)} | S^{(1)}) = (N-1) \PE \lr{ \exp \lr{- (1-\alpha) B_d ( S - S^{(1)})}  }
    \end{align*}    where the expectation is w.r.t. the density of $S$ given by
    \begin{align*}
        p(z)=\frac{\phi(z)}{\overline{\Phi}(S^{(1)})} \mathbb{I}(z\geq S^{(1)}),
    \end{align*}
    with $\phi(z)$ denoting the standard normal density and $\overline{\Phi}(x)=\int_{x}^{\infty}\phi(z)\rmd z$ denoting the normalizing constant. Now denoting $\sigma_\alpha = (1-\alpha) \sigma_-$ for all $\alpha \in [0,1)$ and using that 
    $$
    \PE(T_{N,d}^{(\alpha)} | S^{(1)}) \leq (N-1) \PE \lr{ \exp \lr{- \sigma_\alpha \sqrt{d} ( S - S^{(1)})}  }
    $$
    we obtain by following the proof of \Cref{prop:limitTGaussianRenyi} that the term on the r.h.s. above goes to $0$ as $\log N/d \to 0$ with $N,d \to \infty$. We deduce the desired result by combining this with the fact that $\PE(T_{N,d}^{(\alpha)} | S^{(1)}) \geq 0$.
    \end{proof}
  The proof of \Cref{thm:MainGauss2} then follows immediately from \Cref{prop:limDeltaGaussianRenyi2} and \Cref{prop:limitTGaussianRenyi2}.

\subsection{Deferred proofs and results of \Cref{subsec:BeyondLogNormal}} % of \Cref{Prop:limitT}

\subsubsection{Large deviations for sums of independent random variables}
\label{app:subsub:Saulis}

We start by recalling some useful results from \cite{saulis2000} regarding large deviations for sums of independent random variables. The random variable $\xi$ is said to satisfy the assumption \ref{hyp:Gammak} if the following holds.
\begin{hypxi}{A}{$\xi$}
    \item \label{hyp:Gammak} There exists $\Delta > 0$ such that $|\Gamma_k(\xi)| \leq \frac{k!}{\Delta^{k-2}}$ for all integer $k \geq 3$, where $\Gamma_k(\xi)$ denotes the $k$-th cumulant of $\xi$.
\end{hypxi}
We now state without proof \cite[Lemma 2.3]{saulis2000} and \cite[Theorem 3.1]{saulis2000} in the particular case $\gamma = 0$.

\begin{lem}[{\cite[Lemma 2.3]{saulis2000}} with $\gamma = 0$] \label{lem:saulislargeDeviations}
    Let $\xi$ be a random variable with $\PE(\xi) = 0$ and $\PE(\xi^2) = 1$. Denote by $G(\cdot)$ the cdf of $\xi$. Assume that \ref{hyp:Gammak} holds and set 
    \begin{equation*}
        \Delta_0 = \frac{\sqrt{2}}{36} \Delta.
    \end{equation*}
    Then, in the interval $0 \leq x < \Delta_0$, the relations of large deviations
    \begin{align*}
        1 - G(x) & = (1 - \Phi(x)) \exp(P(x)) \lr{1+ \theta_1 f(x) \frac{x+1}{\Delta_0}} \\
        G(-x) & = \Phi(-x) \exp(P(-x))\lr{1+ \theta_2 f(x) \frac{x+1}{\Delta_0}}
    \end{align*}
    are valid, with $\Phi$ denoting the standard normal distribution. Here, $P$ and $f$ are defined by
    \begin{align*}
        &P(x) = \sum_{k=3}^\infty \lambda_k x^k + \theta \lr{x/\Delta_0}^3 \nonumber \\
        &f(x) = \frac{60(1+10 \Delta_0^2 \exp\lrcb{- \lr{1 - x/\Delta_0} \sqrt{\Delta_0} })}{1 - x/\Delta_0},
    \end{align*}
    where $\theta, \theta_1, \theta_2$ are some variables not exceeding $1$ in absolute value and where for all $k \geq 3$
    \begin{align*}
    |\lambda_k| \leq \frac{2}{k} \lr{16/ \Delta}^{k-2}
    \end{align*}
    so that
    \begin{align*}
    P(x) \leq \frac{x^3}{2(x + 8 \Delta_0)} \quad \mbox{and} \quad P(-x) \geq - \frac{x^3}{3 \Delta_0}.
    \end{align*}
\end{lem}

\begin{theorem}[{\cite[Theorem 3.1]{saulis2000} with $\gamma =0$}] \label{thm:saulisBernstein}
    Let $\xi_1, \ldots,$ $\xi_d$ be independent random variables with $\PE(\xi_j) = 0$ and $\sigma_j^2 = \mathbb{V}(\xi_j)  < \infty$. Set 
    \begin{equation*}
        \sd = \frac{1}{B_d} \lr{ \xi_1 + \ldots + \xi_d},
    \end{equation*}
    where $B_d^2 = \sum_{j=1}^d \sigma_j^2$. Assume that there exists $K > 0$ such that: for all $j = 1 \ldots d$, 
    \begin{align}\label{eq:bernsteinCond}
        |\PE(\xi_j^k)| \leq k! K^{k-2} \sigma_j^2, \quad k\geq 3. 
    \end{align}
    Then,
    \begin{equation*}
        |\Gamma_k(\sd)| \leq \frac{k!}{\Delta_d^{k-2}}, \quad k \geq 3
    \end{equation*}
    with
    \begin{equation*}
        \Delta_d = \frac{B_d}{K_d}, \quad \mbox{where} \quad K_d = 2 \max \lrcb{K, \max_{1 \leq j \leq d} \sigma_j},
    \end{equation*}
    that is, \ref{hyp:Gammak} holds with $\xi = \sd$ and $\Delta = \Delta_d$.    
\end{theorem}

\subsubsection{Preliminary results}
\label{app:prelimResults}

Building on \Cref{lem:saulislargeDeviations} and \Cref{thm:saulisBernstein}, we can now state some preliminary results that will come in handy when proving the results from \Cref{subsec:BeyondLogNormal}.

\begin{lem}%[{\cite[Lemma A.1]{bengtsson2008curse}}] 
    \label{lem:A1bickelcurated}
    Let $\xi_1, \ldots, \xi_d$ be i.i.d. random variables with $\PE(\xi_1) = 0$ and $\sigma^2 = \mathbb{V}(\xi_1)  < \infty$. Set 
    $$
    \sd = \frac{1}{B_d} \lr{ \xi_1 + \ldots + \xi_d},
    $$
    where $B_d = \sigma\sqrt{d}$. Assume that there exists $K > 0$ such that: 
    $$
    |\PE(\xi_1^k)| \leq k! K^{k-2} \sigma^2, \quad k \geq 3.
    $$
    Set $\Delta_d = B_d /K_d$ where $K_d = 2 \max \lrcb{K, \sigma}$. % and assume that $\Delta_d \to \infty$ as $d \to \infty$.
    %Further assume that the random variable $\xi_{1}$ has a density $g$ which is bounded, that is there exists $M > 0$ such that
    %$$
    %\sup_x g(x) \leq M < \infty.
    %$$
    Then, as $d \to \infty$, there exists an analytic function $P_d$ such that the cdf %and the pdf 
    of $\sd$, denoted $G_d(\cdot)$%and $g_d(\cdot)$ respectively
    , satisfies
    \begin{align*}
    1 - G_d(x) & = (1 - \Phi(x)) \exp(P_d(x)) (1+o(1)) \\
    G_d(-x) & = \Phi(-x) \exp(P_d(-x))(1+o(1)) %\\
    %g_d(x) & = \phi(x) \exp(P_d(x))(1 + o(1)) \\
    %g_d(-x) & = \phi(-x) \exp(P_d(-x))(1 + o(1))
    \end{align*}
    uniformly for all $x \geq 0$ and $x = o(\sqrt{d})$.
    Here, $\Phi$ denotes the standard normal distribution, $P_d$ is such that
    \begin{align*}
        P_d(x) = \sum_{k=3}^\infty \lambda_{k,d} x^k
    \end{align*}
    with 
    \begin{align*}
    &|\lambda_{k,d}| \leq A (c / \sqrt{d})^{k-2}, \quad k\geq 3 %\\
    %&|P_d(x)| \leq c' x^3 /\Delta_d
    \end{align*}
    for some constants $A, c> 0$. %, c' > 0$. 
    \end{lem}

    \begin{proof}
    Observe first that $\sd$ satisfies $\PE(\sd) = 0$ and $\PE(\sd^2) = 1$ with $B_d = \sigma \sqrt{d}$ and $K_d = 2 \max (K, \sigma)$. Furthermore, \ref{hyp:Gammak} holds with $\xi = \sd$ and $\Delta = \Delta_d$ by \Cref{thm:saulisBernstein}. Then, we can apply \Cref{lem:saulislargeDeviations} with $\xi = \sd$ and $\Delta = \Delta_d$ to obtain that in the interval $0 \leq x < \Delta_{0,d}$, the relations of large deviations
        \begin{align*}
            1 - G_d(x) & = (1 - \Phi(x)) \exp(P(x)) \lr{1+ \theta_1 f(x) \frac{x+1}{\Delta_{0,d}}}, \\
            G_d(-x) & = \Phi(-x) \exp(P(-x))\lr{1+ \theta_2 f(x) \frac{x+1}{\Delta_{0,d}}}
        \end{align*}
        are valid. Here, $\Delta_{0,d} = \frac{\sqrt{2}}{36} \Delta_d$ and $P$, $f$ are defined by
        \begin{align*}
        &P(x) = P_d(x) + \theta \lr{x/\Delta_{0,d}}^3 \\
        &f(x) = \frac{60(1+10 \Delta_{0,d}^2 \exp\lrcb{- \lr{1 - x/\Delta_{0,d}} \sqrt{\Delta_{0,d}} })}{1 - x/\Delta_{0,d}} \\
        &P_d(x) = \sum_{k=3}^\infty \lambda_{k,d} x^k,
        \end{align*}
        where $\theta, \theta_1, \theta_2$ are some variables not exceeding $1$ in absolute value and
        $$
        |\lambda_{k,d}| \leq  \frac{2}{k} (16 /\Delta_{d})^{k-2} \leq A (c/ \sqrt{d})^{k-2}, \quad k \geq 3
        $$
        for some constants $A, c > 0$. %In addition, 
        %$$
        %P(x) \leq \frac{x^3}{2(x + 8 \Delta_{0,d})} \quad \mbox{and} \quad P(x) \geq - \frac{x^3}{3 \Delta_{0,d}}.
        %$$
        Under the assumption $x = o(\sqrt{d})$, $P(x) = P_d(x) + o(1)$, $f(x) \frac{x+1}{\Delta_{0,d}} = o(1)$ and we can thus deduce that as $d \to \infty$ the relations of large deviations become
        \begin{align*}
        1 - G_d(x) & = (1 - \Phi(x)) \exp(P_d(x)) \lr{1+ o(1)} \\
        G_d(-x) & = \Phi(-x) \exp(P_d(-x))\lr{1+ o(1)}
        \end{align*}
        uniformly for all $x \geq 0$ and $x = o(\sqrt{d})$.
    \end{proof}
The corollary below then follows from \Cref{lem:A1bickelcurated}. 

\begin{coro}\label{key:coro}
 Under the assumptions of \Cref{lem:A1bickelcurated}, as $d \to \infty$, 
 \begin{align*}
     1 - G_d(x) & = (1 - \Phi(x))(1+o(1)) \\
     G_d(-x) & = \Phi(-x)(1+o(1)) %\\
     %g_d(x) & = \phi(x)(1 + o(1)) \\
     %g_d(-x) & = \phi(-x)(1 + o(1))
 \end{align*}
uniformly for all $x \geq 0$ and $x = o(d^{1/6})$.   
\end{coro}

\begin{proof} Since we consider the case $x \geq 0$ and $x = o(d^{1/6})$, we can apply \Cref{lem:A1bickelcurated} to get: as $d \to \infty$,
\begin{align*}
        1 - G_d(x) & = (1 - \Phi(x)) \exp(P_d(x)) (1+o(1)) \\
        G_d(-x) & = \Phi(-x) \exp(P_d(-x))(1+o(1)) %\\
        %g_d(x) & = \phi(x) \exp(P_d(x))(1 + o(1)) \\
        %g_d(-x) & = \phi(-x) \exp(P_d(-x))(1 + o(1)),
\end{align*}
where $P_d$ is defined in \Cref{lem:A1bickelcurated}. In addition, using successively that (i) $|\lambda_{k,d}| \leq A (c/\sqrt{d})^{k-2}$ by \Cref{lem:A1bickelcurated} (ii) $x = o (\sqrt{d})$ and (iii) $x^3 = o(\sqrt{d})$, we have that:
\begin{align*}
|P_d(x)| &\leq \sum_{k = 3}^\infty |\lambda_{k,d}| x^k \\
& \leq Ac {x^3} d^{-1/2} \sum_{k = 3}^\infty \lr{c x d^{-1/2}}^{k-3} \\
& \leq A c x^3 d^{-1/2} (1 + o(1)) \\
& = o(1).
\end{align*}
Similarly, $|P_d(-x)| = o(1)$ and consequently, 
\begin{align*}
    1 - G_d(x) & = (1 - \Phi(x))(1+o(1)) \\
    G_d(-x) & = \Phi(-x)(1+o(1)) %\\
    %g_d(x) & = \phi(x)(1 + o(1)) \\
    %g_d(-x) & = \phi(-x)(1 + o(1))
\end{align*}
uniformly for all $x \geq 0$ and $x = o(d^{1/6})$. 
\end{proof}
We also prove the following concentration result, which parallels the corresponding result \cref{lem:exactnormalminconcentration} from the exact log-normal case and which will be useful in subsequent proofs.

\begin{lem}\label{lem:approxnormalminconcentration}
    Let $S_1, \ldots, S_N$ be i.i.d. distributed according to \eqref{eq:expressionSi}, set $S^{(1)} = \min_{1 \leq i \leq N}S_i$ and define $I_N = [- 4 \sqrt{\log N}, - \sqrt{\log N}]$. Then as $N, d \rightarrow \infty$ with $\log N / d^{1/3} \rightarrow 0$, we have
    \begin{equation*}
        \PP(S^{(1)} \not \in I_N) = O\lr{\frac{1}{N^4}}.
    \end{equation*}
\end{lem}

\begin{proof}
    The proof follows the same structure as the proof of \cref{lem:exactnormalminconcentration}, using \Cref{key:coro} to relate the approximately log-normal case to the exact case.
    
    We control the probability of the events $\{S^{(1)} > - \sqrt{\log N}\}$ and $\{S^{(1)} < - 4 \sqrt{\log N}\}$ separately. First, since $\sqrt{\log N} = o(d^{1/6})$ as $N,d \rightarrow \infty$ with $\log N / d^{1/3} \rightarrow 0$, by \Cref{key:coro} we have
    \begin{align*}
        \log \PP(S^{(1)} > -\sqrt{\log N}) & = N \log \left( 1 - G_d(-\sqrt{\log N}) \right) \\
        & = N \log \left( 1 - (1 + o(1))\overline{\Phi}(\sqrt{\log N}) \right) \\
        & = - N \frac{\phi(\sqrt{\log N})}{\sqrt{\log N}}(1 + o(1)) 
    \end{align*}
    using the same method as in \eqref{eq:lowerboundapprox}. Following \eqref{eq:lowerboundasymp}, we deduce that 
    \begin{equation*}
        \PP(S^{(1)} > -\sqrt{\log N}) = O\left(\frac{1}{N^4}\right)
    \end{equation*}
    Second, we can write
    \begin{align*}
        \PP(S^{(1)} < - 4 \sqrt{\log N}) & \leq N G_d(- 4 \sqrt{\log N}) \\
        & = N \overline{\Phi}(4 \sqrt{\log N})(1 + o(1)) \\
        &= O\left(\frac{1}{N^4}\right)
    \end{align*}
    using the same method as in \eqref{eq:upperboundasymp}, from which the result follows.
\end{proof}

\subsubsection{Proof of \Cref{lem:S1approxGen}}
\label{subsec:prooflemS1approxGen}

\begin{proof}[Proof of \Cref{lem:S1approxGen}]
    The idea of the proof will be to relate it to the case where $S_1, \ldots, S_N$ are exactly normally distributed, which was proved in Section \ref{subsec:proofS1approxGauss}. Recall that we have $S_1, \ldots, S_N$ with cdf $G_d$ and $S^{(1)} = \min_{1 \leq i \leq N} S_i$. Also, let $\tilde S_1, \ldots, \tilde S_N$ be auxiliary i.i.d. standard Gaussian random variables, and set $\tilde S^{(1)} = \min_{1 \leq i \leq N} S_i$.

    By the assumption that the $\xi_{i,j}$ are absolutely continuous with respect to the Lebesgue measure, $G_d$ is continuous and hence we can construct $S_1, \ldots, S_N$ and $\tilde S_1, \ldots, \tilde S_N$ on a common probability space by drawing $N$ uniform random variables $U_1, \ldots, U_N \sim U[0,1]$ and setting $S_i = G_d^{-1}(U_i)$, $\tilde S_i = \Phi^{-1}(U_i)$. We then have that $S^{(1)} = G_d^{-1}(U^{(1)})$ and $\tilde S^{(1)} = \Phi^{-1}(U^{(1)})$.

    From \Cref{lem:S1approxGauss} we know that
    \begin{equation*}
        \PE(\tilde S^{(1)}) = - \sqrt{2 \log N} + O \left( \frac{\log \log N}{\sqrt{\log N}} \right)
    \end{equation*}
    so it suffices to prove that
    \begin{equation}\label{eq:S1proofapproxerror}
        \PE(|\tilde S^{(1)} - S^{(1)}|) = O \left( \frac{\log \log N}{\sqrt{\log N}} \right).
    \end{equation}
    Letting $I_N = [- 4\sqrt{\log N}, - \sqrt{\log N}]$, we will split the above expectation according to whether $\tilde S^{(1)} \in I_N$. 
    
    \begin{itemize}
        \item
    Assuming first that $\tilde S^{(1)} \in I_N$, so that $\tilde S^{(1)} = o(d^{1/6})$, if we let $h \in \mathbb{R}$ be an arbitrary real satisfying $h = o_{N,d}(1)$, then using \Cref{key:coro} we can write
    \begin{align*}
        G_d(\tilde S^{(1)} + h) & = \Phi(\tilde S^{(1)} + h)(1 + o_{N,d}(1)) \\
        & = - \frac{\phi(\tilde S^{(1)} + h)}{\tilde S^{(1)} + h}(1 + o_{N,d}(1)) \\
        & = \Phi(\tilde S^{(1)}) \frac{\phi(\tilde S^{(1)} + h)}{\phi(\tilde S^{(1)})}(1 + o_{N,d}(1)) \\
        & = U^{(1)} \exp\left\{ - h \tilde S^{(1)} - h^2 / 2 \right\} (1 + o_{N,d}(1))
    \end{align*}

    and so it follows by the continuity of $G_d$ that there is a choice of $h$, satisfying $h = O_{N,d}(1/\sqrt{\log N})$, such that $G_d(\tilde S^{(1)} + h) = U^{(1)}$. We conclude that
    \begin{equation*}
        |\tilde S^{(1)} - S^{(1)}| \leq O_{N,d}\left(\frac{1}{\sqrt{\log N}}\right)
    \end{equation*}
    and so 
    \begin{equation} \label{eq:diffInIN}
        \PE\left(|\tilde S^{(1)} - S^{(1)}| \mathbbm{1}_{\{\tilde S^{(1)} \in I_N\}}\right) \leq O\left(\frac{1}{\sqrt{\log N}}\right).
    \end{equation}

    \item On the other hand, we may also write
    \begin{align*}
        \PE\left(|S_1| \mathbbm{1}_{\{\tilde S^{(1)} \not \in I_N\}}\right) & \leq \PE\left(|S_1| \mathbbm{1}_{\{|S_1| \geq N^2\}}\right) + \PE\left(|S_1| \mathbbm{1}_{\{|S_1| < N^2\} \cap \{\tilde S^{(1)} \not \in I_N\}}\right) \\
        & \leq \frac{1}{N^2} \PE(|S_1|^2) + N^2 \PP\left(\tilde S^{(1)} \not \in I_N \right) \\
        & \leq O\left(\frac{1}{N^2}\right)
    \end{align*}
    where we have used $\PE(|S_1|^2) = 1$ and \Cref{lem:exactnormalminconcentration} to bound the second term.

    The same result also holds with $\tilde S_1$ in place of $S_1$ (e.g. by considering taking the $\xi_i$ to be i.i.d. Gaussians), and so we see that
    \begin{equation} \label{eq:diffNotInIN}
        \PE\left(|\tilde S^{(1)} - S^{(1)}| \mathbbm{1}_{\{\tilde S^{(1)} \not \in I_N\}}\right) \leq \sum_{i=1}^N \PE\left(|\tilde S_i - S_i| \mathbbm{1}_{\{\tilde S^{(1)} \not \in I_N\}}\right) = O\left(\frac{1}{N}\right)
    \end{equation}
\end{itemize}
    Combining \eqref{eq:diffInIN} and \eqref{eq:diffNotInIN} yields \eqref{eq:S1proofapproxerror} and the proof is concluded.
\end{proof}

\subsubsection{Proof of \Cref{prop:DeltaGen}}
\label{subsec:proof:prop:deltaGen}

\begin{proof}[Proof of \Cref{prop:DeltaGen}]
    First, note that since the weights satisfy (\ref{eq:approxlognormalweights}), we may write
    $$
    \log \overline{w}^{(N)} = - \log \lr{ \PE(\exp(-\sigma \sqrt{d} S_1))} - \sigma \sqrt{d} S^{(1)}.
    $$
    In addition, using the definition of $S_1$ written in \eqref{eq:expressionSi}, that is
    $$
    S_1 = \frac{1}{\sigma \sqrt{d}} \sum_{j = 1}^d \xi_{1,j},
    $$
    where the $\xi_{1,1}, \ldots, \xi_{1,d}$ are i.i.d. random variables, we have that
    \begin{align*}
        \PE(\exp(-\sigma \sqrt{d} S_1)) & = \prod_{j=1}^d \PE(\exp(-\xi_{1,j})) \\
        & = \lr{\PE(\exp(-\xi_{1,1}))}^d.
    \end{align*}
    Thus,
    \begin{align} \label{eq:checkALognormal}
        - \log \lr{\PE(\exp(-\sigma \sqrt{d} S_1))} & = - d \log \PE(\exp(-\xi_{1,1})) = - d a
    \end{align}
    By Jensen's inequality applied to the strictly convex function $u \mapsto - \log(u)$, we have that
    $$
    a < \PE(\xi_{1,1}) = 0.
    $$
    Hence,
    \begin{align}\label{eq:WNGen}
        \log \overline{w}^{(N)} = -d a - \sigma \sqrt{d}S^{(1)}
    \end{align}
    with $a > 0$. Following the proof of \Cref{prop:limDeltaGaussianRenyi} in \Cref{subsec:prooflimDeltaGaussian}, we can then conclude by combining \eqref{eq:WNGen} with the definition of $\Delta_{N,d}^{MAX}(\theta, \phi)$ in \eqref{eq:DeltaMaxNDExt}. Indeed,
    $$
    \Delta_{N,d}^{MAX}(\theta, \phi;x) = - da - \sigma \sqrt{d} \PE(S^{(1)}) + \frac{\log N}{\alpha -1}
    $$
    and using \eqref{eq:ExpS1approxGen}, we deduce:
    \begin{align*}
        \Delta_{N,d}^{MAX}(\theta, \phi;x) & = - da + \sigma \sqrt{d} \lr{ \sqrt{2 \log N} + O \lr{\frac{\log \log N}{\sqrt{\log N}}}} + \frac{\log N}{\alpha -1} \\
        & = - da \lrcb{1 - \frac{\sigma}{a} \sqrt{\frac{2 \log N}{ d}} + \frac{1}{1-\alpha} \frac{\log N}{d a} + O \lr{\frac{\log \log N}{\sqrt{d \log N}}}},
    \end{align*}
    Hence, using now that $\frac{1}{1-\alpha} \frac{\log N}{d a} = O \lr{\frac{\log \log N}{\sqrt{d \log N}}}$ under the assumption $\log N/d^{1/3} \to 0$, we can deduce
    \begin{align*}
     \lim_{\substack{N,d \to \infty \\ \log N / d^{1/3} \to 0}} \Delta_{N,d}^{(\alpha, MAX)}(\theta, \phi) + d a \lrcb{1 -  \frac{\sigma}{a} \sqrt{\frac{2 \log N}{ d}} + O \lr{\frac{\log \log N}{\sqrt{d \log N}}}} = 0,
    \end{align*}
    which yields the desired result.    
\end{proof}
    
\subsubsection{Proof of \Cref{Prop:limitT}}
\label{subsec:ProofLimitT}

\begin{proof}[Proof of \Cref{Prop:limitT}]
    The proof will build on the proof of \Cref{prop:limitTGaussianRenyi}. As in that proof, we denote $\sigma_\alpha = (1-\alpha) \sigma$ for all $\alpha \in [0,1)$ and observe that, conditional upon $S^{(1)}$, we can think of the sum in \eqref{eq:expressionT_NdExt} as the sum over $N-1$ i.i.d. random variables
    \begin{align*}
        \PE(T_{N,d}^{(\alpha)} | S^{(1)}) = (N-1) \PE \lr{ \exp \lr{- \sigma_\alpha \sqrt{d} ( S - S^{(1)})}  }
    \end{align*} where the expectation is w.r.t. the density of $S$ given by
    \begin{align*}
        p(z)=\frac{g_d(z)}{\overline{G_d}(S^{(1)})} \mathbb{I}(z\geq S^{(1)}),
    \end{align*}
    with $g_d$ denoting the pdf of $S_1$ and $\overline{G_d}(x) = \int_x^\infty g_d(z) \rmd z$ for all $x \in \rset$, that is
    \begin{equation*}\label{eq:condExp}
        \mathbb{E}(T_{N,d}|S^{(1)})=(N-1) \frac{\int_{S^{(1)}}^{\infty} \exp(-\sigma_\alpha \sqrt{d}(z-S^{(1)})) g_d(z) \rmd z}{\overline{G_d}(S^{(1)})}.
    \end{equation*}
    We are thus required to show that
    \begin{equation}\label{eq:expandedT}
        (N-1) \mathbb{E} \left[ \frac{\int_{S^{(1)}}^{\infty} \exp(-\sigma_\alpha \sqrt{d}(z-S^{(1)})) g_d(z) \rmd z}{\overline{G_d}(S^{(1)})} \right] \rightarrow 0
    \end{equation}
    as $N,d \rightarrow 0$ with $\log N / d^{1/3} \rightarrow 0$. First, we show that contributions due to extreme values of $S^{(1)}$ are negligible. To see this, note that
    \begin{equation*}
        \left| \frac{\int_{S^{(1)}}^{\infty} \exp(-\sigma_\alpha \sqrt{d}(z-S^{(1)})) g_d(z) \rmd z}{\overline{G_d}(S^{(1)})} \right| \leq 1,
    \end{equation*}
    so that \Cref{lem:approxnormalminconcentration} implies
    \begin{equation*}
        (N-1) \mathbb{E} \left[ \mathbbm{1}_{\{S^{(1)} \not \in I_N\}} \frac{\int_{S^{(1)}}^{\infty} \exp(-\sigma_\alpha \sqrt{d}(z-S^{(1)})) g_d(z) \rmd z}{\overline{G_d}(S^{(1)})}  \right] \leq (N-1)\PE\left(\mathbbm{1}_{\{S^{(1)} \not \in I_N\}}\right) \rightarrow 0.
    \end{equation*}
    Hence it suffices to show that 
    \begin{equation*}
        (N-1) \mathbb{E} \left[ \mathbbm{1}_{\{S^{(1)} \in I_N\}} \frac{\int_{S^{(1)}}^{\infty} \exp(-\sigma_\alpha \sqrt{d}(z-S^{(1)})) g_d(z) \rmd z}{\overline{G_d}(S^{(1)})} \right] \rightarrow 0.
    \end{equation*}
    Note that by \Cref{key:coro}, we have $\overline{G}_d(S^{(1)}) \geq \overline{G}_d(0) = 1 - \Phi(0)(1 + o(1))$ on the event $\{S^{(1)} \in I_N\}$ as $N,d \rightarrow \infty$ with $\log N / d^{1/3} \rightarrow 0$, so $\overline{G}_d(S^{(1)})$ is uniformly bounded below. It thus suffices to prove
    \begin{equation*}
        (N-1) \mathbb{E} \left[ \mathbbm{1}_{\{S^{(1)} \in I_N\}} \int_{S^{(1)}}^{\infty} \exp(-\sigma_\alpha \sqrt{d}(z-S^{(1)})) g_d(z) \rmd z \right] \rightarrow 0.
    \end{equation*}
    We will in fact show that 
    \begin{equation}\label{eq:expandedTmain}
        (N-1) \mathbb{E} \left[ \mathbbm{1}_{\{S^{(1)} \in I_N\}} \int_{S^{(1)}}^{\infty} \exp(-\sigma_\alpha \sqrt{d}(z-S^{(1)})) \phi(z) \rmd z \right] \rightarrow 0
    \end{equation}
    and
    \begin{equation}\label{eq:expandedTerror}
        (N-1) \mathbb{E} \left[ \mathbbm{1}_{\{S^{(1)} \in I_N\}} \left | \int_{S^{(1)}}^{\infty} \exp(-\sigma_\alpha \sqrt{d}(z-S^{(1)})) \left( g_d(z) - \phi(z) \right) \rmd z \right | \right] \rightarrow 0.
    \end{equation}

\begin{itemize}
\item \textbf{Proof of \eqref{eq:expandedTmain}.} Following the proof of \Cref{prop:limitTGaussianRenyi}, we see that \eqref{eq:conditionalexpbound} holds whenever $S^{(1)} \in I_N$. Restricting to the event $\{S^{(1)} \in I_N\}$ and taking expectations over $S^{(1)}$, we get
    \begin{align*}
        & (N-1) \mathbb{E} \left[ \mathbbm{1}_{\{S^{(1)} \in I_N\}} \int_{S^{(1)}}^{\infty} \exp(-\sigma_\alpha \sqrt{d}(z-S^{(1)})) \phi(z) \rmd z \right] \\ 
        \leq \; & (N-1) \frac{4 \sqrt{\log N}}{\sigma_\alpha \sqrt{d}} \PE \left[ \mathbbm{1}_{\{S^{(1)} \in I_N\}} \Phi(S^{(1)}) \right] (1 + o(1)).
    \end{align*}
    Since $S^{(1)} = o(d^{1/6})$, \Cref{key:coro} implies that
    \begin{equation*}
        \Phi(S^{(1)}) = G_d(S^{(1)})(1 + o_{N,d}(1)),
    \end{equation*}
    where the uniformity over $x$ in the statement of the theorem implies that the implicit constant is independent of $S^{(1)}$. Restricting to $\{S^{(1)} \in I_N\}$ and taking expectations again, noting that $G_d(S^{(1)})$ is distributed as the minimum of $N$ uniform random variables on $[0,1]$, we see
    \begin{equation}\label{eq:expectedphiapprox}
        \PE \left[ \mathbbm{1}_{\{S^{(1)} \in I_N\}} \Phi(S^{(1)}) \right] = \PE \left[ \mathbbm{1}_{\{S^{(1)} \in I_N\}} G_d (S^{(1)}) \right](1 + o(1)) \leq \frac{1}{N+1}(1 + o(1)).
    \end{equation}
    We conclude that
    \begin{equation*}
        (N-1) \mathbb{E} \left[ \mathbbm{1}_{\{S^{(1)} \in I_N\}} \int_{S^{(1)}}^{\infty} \exp(-\sigma_\alpha \sqrt{d}(z-S^{(1)})) \phi(z) \rmd z \right] \leq \frac{4 \sqrt{\log N}}{\sigma_\alpha \sqrt{d}}(1 + o(1)) \rightarrow 0,
    \end{equation*}
    proving \eqref{eq:expandedTmain}.

    \item \textbf{Proof of \eqref{eq:expandedTerror}.} Two applications of integration by parts give
    \begin{multline*}
        \int_{S^{(1)}}^\infty \exp\{-\sigma_\alpha \sqrt{d}(z - S^{(1)})\} g_d(z) \rmd z = -G_d(S^{(1)}) \\ + \int_{S^{(1)}}^\infty \sigma_\alpha \sqrt{d} \exp\{-\sigma_\alpha \sqrt{d}(z - S^{(1)})\} G_d(z) \rmd z
    \end{multline*}
    and 
    \begin{multline*}
        \int_{S^{(1)}}^\infty \exp\{-\sigma_\alpha \sqrt{d}(z - S^{(1)})\} \phi(z) \rmd z = -\Phi(S^{(1)}) \\ + \int_{S^{(1)}}^\infty \sigma_\alpha \sqrt{d} \exp\{-\sigma_\alpha \sqrt{d}(z - S^{(1)})\} \Phi(z) \rmd z.
    \end{multline*}
    It follows that
    \begin{multline*}
        \left | \int_{S^{(1)}}^{\infty} \exp(-\sigma_\alpha \sqrt{d}(z-S^{(1)})) \left( g_d(z) - \phi(z) \right) \rmd z \right | \\
        \leq \; \left|\Phi(S^{(1)}) - G_d(S^{(1)})\right| + \left| \int_{S^{(1)}}^\infty \sigma_\alpha \sqrt{d} \exp\{-\sigma_\alpha \sqrt{d}(z - S^{(1)})\} (G_d(z) - \Phi(z)) \rmd z \right|.
    \end{multline*}
    We now deal with each of these terms separately. On the event $\{S^{(1)} \in I_N\}$, we know $\Phi(S^{(1)}) = G_d(S^{(1)})(1 + o_{N,d}(1))$ so we have
    \begin{equation*}
        \left|\Phi(S^{(1)}) - G_d(S^{(1)})\right| \leq o_{N,d}(1) G_d(S^{(1)})
    \end{equation*}
    and so by restricting to $\{S^{(1)} \in I_N\}$ and taking expectations
    \begin{equation*}
        (N-1)\mathbb{E}\left( \mathbbm{1}_{\{S^{(1)} \in I_N\}} \left|\Phi(S^{(1)}) - G_d(S^{(1)})\right|\right) \leq (N-1) \cdot o(1) \mathbb{E}\left( \mathbbm{1}_{\{S^{(1)} \in I_N\}} G_d(S^{(1)})\right) \rightarrow 0.
    \end{equation*}
    along the same lines as \eqref{eq:expectedphiapprox}.

    We split the second term as an integral from $S^{(1)}$ to $0$ and an integral from $0$ to $\infty$ and we write:
    \begin{align*}
        \left| \int_{S^{(1)}}^\infty \sigma_\alpha \sqrt{d} \exp\{-\sigma_\alpha \sqrt{d}(z - S^{(1)})\} (G_d(z) - \Phi(z)) \rmd z \right| \leq A_1 + A_2 
    \end{align*}
    with
    \begin{align*}
        & A_1 = \int_{S^{(1)}}^0 \sigma_\alpha \sqrt{d} \exp\{-\sigma_\alpha \sqrt{d}(z - S^{(1)})\} \left| G_d(z) - \Phi(z) \right| \rmd z \\
        & A_2 = \int_0^\infty \sigma_\alpha \sqrt{d} \exp\{-\sigma_\alpha \sqrt{d}(z - S^{(1)})\} \rmd z.
    \end{align*}
    Again, we bound each term individually. For $A_2$, assuming that $S^{(1)} \in I_N$ we have the bound
    \begin{align*}
        A_2 = \exp\{\sigma_\alpha \sqrt{d} S^{(1)} \} \leq \exp\{- \sigma_\alpha \sqrt{d \log N}\} \leq \exp\{- \sigma_\alpha (\log N)^2 \}
    \end{align*}
    for sufficiently large $N,d$ with $\log N / d^{1/3} \to \infty$ and so $(N-1)\mathbb{E}(\mathbbm{1}_{\{S^{(1)} \in I_N\}} A_2) \rightarrow 0$. To bound $A_1$, note that by \Cref{key:coro}
    \begin{equation*}
        \left| G_d(z) - \Phi(z) \right| \leq o_{N,d}(1) \Phi(z)
    \end{equation*}
    for all $z \in [S^{(1)},0]$ so long as $S^{(1)} \in I_N$, and hence under this assumption we can write
    \begin{align*}
        A_1 &\leq o_{N,d}(1) \int_{S^{(1)}}^0 \sigma_\alpha \sqrt{d} \exp\{-\sigma_\alpha \sqrt{d}(z - S^{(1)})\} \Phi(z) \rmd z. \\
    \intertext{Changing the upper limit from $0$ to $\infty$, which can only weaken the bound, and then integrating by parts gives}
        &\leq o_{N,d}(1) \int_{S^{(1)}}^\infty \sigma_\alpha \sqrt{d} \exp\{-\sigma_\alpha \sqrt{d}(z - S^{(1)})\} \Phi(z) \rmd z \\
        &\leq o_{N,d}(1) \left\{ \Phi(S^{(1)}) + \int_{S^{(1)}}^\infty \exp\{-\sigma_\alpha \sqrt{d}(z - S^{(1)})\} \phi(z) \rmd z \right\}.
    \end{align*}
    We conclude that 
    \begin{multline*}
        (N-1)\PE\lr{\mathbbm{1}_{\{S^{(1)} \in I_N\}} A_1} \leq o(1) \Big\{ (N-1 )\PE\lr{\mathbbm{1}_{\{S^{(1)} \in I_N\}} \Phi(S^{(1)})} \\ 
         + (N-1) \PE \lr{\mathbbm{1}_{\{S^{(1)} \in I_N\}} \int_{S^{(1)}}^\infty \exp\{-\sigma_\alpha \sqrt{d}(z - S^{(1)})\} \phi(z) \rmd z} \Big\}
    \end{multline*}
    The former term tends to zero by \eqref{eq:expectedphiapprox} and the latter tends to zero by \eqref{eq:expandedTmain}. We thus see that \eqref{eq:expandedTerror} holds, completing the proof.
\end{itemize}
\end{proof}

\subsubsection{Proof of \Cref{ex:linGauss}}
\label{subsec:ex:linGaussApp}

\begin{proof}[Proof of \Cref{ex:linGauss}] 
Recall that by \eqref{eq:LinGaussWeights}: for all $i = 1 \ldots N$,
    \begin{align*}
        \log \overline{w}_i = \frac{d}{2} \log \lr{ \frac{4}{3}} - \Big\| z_i - \frac{\theta + x}{2} \Big\|^2 + \frac{3}{4} \|z_i-Ax-b\|^{2}. %, \quad i = 1 \ldots N.
      \end{align*}
We want to show that if $z_i \sim q_\phi(\cdot|x) = \mathcal{N}(Ax+b, 2/3 ~ \boldsymbol{I}_d)$, %and \eqref{eq:LinGaussWeights} holds for some fixed parameters $\theta, A, b$ and datapoint $x$, 
then $\log \overline{w}_i$ can be written in the form of \eqref{eq:approxlognormalweights}. % with $\sigma^2$ as defined in \Cref{ex:linGauss}. %We also want to obtain the analytical expression of $a$.
For this purpose, denote $\mathbf{1} = (1, \dots, 1)^T$ and observe that there exists an orthogonal matrix $U$ such that $U\lr{\frac{\theta + x}{2} - Ax - b} = \lambda \mathbf{1}$. We can then sample $z_i \sim \mathcal{N}(Ax+b, 2/3 ~ \boldsymbol{I}_d)$ by setting $z_i = U^{-1}y_i + Ax + b$ where $y_i \sim \mathcal{N}(0, 2/3 ~ \boldsymbol{I}_d)$. With this parameterization, \eqref{eq:LinGaussWeights} becomes
\begin{align*}
    \log \overline{w}_i & = - \Big\|U^{-1}y_i + Ax + b - \frac{\theta + x}{2} \Big\|^2 + \frac{3}{4} \|U^{-1} y_i\|^2 + const. \\
    & = - \|y_i - \lambda \mathbf{1}\|^2 + \frac{3}{4} \|y_i\|^2 + const. \\
    & = - \sum_{j=1}^d \lrcb{(y_{ij} - \lambda)^2 - \frac{3}{4}y_{ij}^2} + const.
\end{align*}
where $const.$ denotes a fixed constant which depends only on $d, \theta, A, b$ and $x$.

Let us now set $\zeta_{ij} = (y_{ij} -\lambda)^2 - {3} y_{ij}^2/ {4}$ and $\xi_{i,j} = \zeta_{ij} - \PE\lr{\zeta_{ij}}$. Since $y_i \sim \mathcal{N}(0, 2/3 \boldsymbol{I}_d)$ it follows that $\xi_{i,1}, \dots, \xi_{i,d}$ are i.i.d. random variables with $\PE\lr{\xi_{i,j}} = 0$. Now defining $\sigma^2 = \mathbb{V}(\xi_{1,1}) < \infty$, we have that 
$\sigma^2$ can be computed analytically by observing that 
\begin{align*}
    \PE(\zeta_{ij}) = \PE\lr{(y_{ij} -\lambda)^2} - \PE\lr{\frac{3}{4}y_{ij}^2} = \frac{1}{4} \PE(y_{ij}^2) + \lambda^2  = \frac{1}{6} + \lambda^2
\end{align*}
from which we can deduce that
\begin{align*}
    \sigma^2 & = \PE\lr{\lrb{\zeta_{ij} - \PE(\zeta_{ij})}^2} = \PE\lr{\lrb{\frac{1}{4}y_{ij}^2 - 2 \lambda y_{ij} - \frac{1}{6}}^2} \\
    & = \frac{1}{16} \PE(y_{ij}^4) + \lr{4 \lambda^2 - \frac{1}{12}} \PE(y_{ij}^2) + \frac{1}{36} \\
    & =  \frac{1}{18} + \frac{8}{3} \lambda^2 < \infty.
\end{align*}
Hence, \eqref{eq:approxlognormalweights} holds by defining $S_i$ as in \eqref{eq:expressionSi} (noting that the constant terms must match since $\overline{w}_i$ is normalised and so has expected value 1). In addition, we can also analytically compute the quantity $a$ defined in \eqref{eq:defALognormal} by noting that
\begin{align*}
    \PE\lr{\exp\lr{- \xi_{1,1}}} & = \int_{-\infty}^\infty \exp\lr{-\lrb{\frac{1}{4}u^2 - 2 \lambda u - \frac{1}{6}}} \cdot \frac{1}{\sqrt{2 \pi \cdot \frac{2}{3}}} e^{-\frac{3}{4}u^2} \; \rmd u \\
    & = \sqrt{\frac{3}{4}} \exp\lr{\lambda^2 + \frac{1}{6}}
\end{align*}
so that 
\begin{equation*}
    a = \lambda^2 + \frac{1}{6} + \frac{1}{2} \log \lr{\frac{3}{4}}.
\end{equation*} %with this choice of $\xi_{i,j}$ 
Finally, we check that \ref{hypBickel} holds in the case where $(\theta, \phi) = (\theta^\star, \phi^\star)$. For this choice of parameters, $\lambda = 0$, hence $\sigma$ is independent of $d$. Furthermore, $\xi_{i,j}$ is clearly absolutely continuous with respect to the Lebesgue measure, and the distribution of $\xi_{i,j}$ is independent of $d$. It follows that \ref{hypBickel:a} holds using our previous observations.

To now check \ref{hypBickel:b}, we let $k \geq 3$. In that case, $u \mapsto |u|^k$ is convex and for all real-valued $u_1, u_2$ and $u_3$, we have that: 
\begin{align*}
\Big| \frac{1}{3} u_1 + \frac{1}{3} u_2 + \frac{1}{3} u_3 \Big|^k & \leq \Big| \frac{1}{3} |u_1| + \frac{1}{3} |u_2| + \frac{1}{3} |u_3| \Big|^k  \\
& \leq \frac{1}{3} \lr{|u_1|^k + |u_2|^k + |u_3|^k}
\end{align*}
so that, setting $u_1 = (y_{ij} -\lambda)^2$, $u_2 = -\frac{3}{4}y_{ij}^2$ and $u_3 = -\PE\lr{\zeta_{ij}}$, it holds that
$$
\Big|(y_{ij} -\lambda)^2 - \frac{3}{4}y_{ij}^2 - \PE\lr{\zeta_{ij}}\Big|^k \leq 3^{k-1} \lr{  (y_{ij} -\lambda)^{2k} + \lr{\frac{3}{4}y_{ij}}^{2k} + |\PE\lr{\zeta_{ij}}|^k }.
$$
Using a similar argument applied to $(y_{ij} -\lambda)^{2k}$, we then deduce that
$$
\Big|(y_{ij} -\lambda)^2 - \frac{3}{4}y_{ij}^2 - \PE\lr{\zeta_{ij}}\Big|^k \leq 3^{k-1} \lr{2^{2k-1}\lr{y_{ij}^{2k} + \lambda^{2k}} + \lr{\frac{3}{4}y_{ij}}^{2k} + |\PE\lr{\zeta_{ij}}|^k }.
$$
Hence,
\begin{align*}
    \PE\lr{|\xi_{i,1}|^k} & = \PE\lr{\Big|(y_{ij} -\lambda)^2 - \frac{3}{4}y_{ij}^2 - \PE\lr{\zeta_{ij}}\Big|^k} \\
    & \leq 3^{k-1} \lr{2^{2k-1}\lr{\PE(y_{ij}^{2k}) + \lambda^{2k}} + \frac{3^{2k}}{4^{2k}} \PE(y_{ij}^{2k}) + |\PE\lr{\zeta_{ij}}|^k } \\
    & \leq 3^{k-1} \lr{ (2^{2k-1} + \frac{3^{2k}}{4^{2k}}) \cdot (2k-1)!! \cdot \frac{2^k}{3^k} + 2^{2k-1} \lambda^{2k} + \big|\PE\lr{\zeta_{ij}}\big|^k} \\
    & \leq 3^{k-1} \lr{ (2^{2k-1} + \frac{3^{2k}}{4^{2k}}) \cdot \frac{2^{2k}}{3^k} \cdot k! + 2^{2k-1} \lambda^{2k} + \lr{\frac{1}{6} + \lambda^2}^k}
\end{align*}
where we have used that the $(2k)$-th moment of a standard Gaussian random variable is $(2k-1)!!$. Finally, since $\lambda = 0$ in our case, we obtain that
\begin{align*}
\PE\lr{|\xi_{i,1}|^k}  & \leq k! K^{k-2} \sigma^2
\end{align*}
for some sufficiently large choice of $K$ which is independent of $d$, and \ref{hypBickel} thus holds.
\end{proof}

\begin{rem}
We have obtained that \ref{hypBickel} holds when $(\theta, \phi)$ are equal to the optimal parameters. If we now assume that $x$, $\theta$ and $\phi$ are initially drawn from Gaussian distributions with bounded covariance matrices \cite[like it is the case in][]{rainforth2018tighter}, we anticipate that \ref{hypBickel} should approximately hold even for values of the parameters other than the optimal choice. Notice indeed that in that case we inuitively expect $\big\| \frac{\theta + x}{2} - Ax - b \big\| = O(\sqrt{d})$ for most values of $x$, $\theta$ and $\phi$. It follows that we should expect $\lambda = O(1)$ and $\sigma = \Theta(1)$ in practice as $d \rightarrow \infty$, and so \ref{hypBickel} should approximately hold.
\end{rem}

\section{Futher details regarding related proof techniques}
\label{app:sec:relatedWork}

As mentioned in \Cref{sec:relatedWork}, a number of our proof techniques differs significantly from/alter parts of known proofs, which in some cases impacts the corresponding theoretical results. Namely,

\begin{itemize}
    \item \textbf{\Cref{prop:SNRconvergence}}. The proof of this result is based on the proof for the case $\alpha = 0$ written in the arxiv version of 5 Mar 2019 of \cite[Theorem 1]{rainforth2018tighter}, which was the latest version available to us. Nevertheless, and contrary to \cite{rainforth2018tighter}, we (i) use an explicit form for the remainder term in Taylor's theorem rather than the mean value form of the remainder, allowing us to get more precise control on the magnitude of the remainder and its gradients, and (ii) we consequently rely on \Cref{lem:keySNRlemma}, which is a non-immediate extension of \cite[Lemma 1]{rainforth2018tighter}. 
    
    This significantly impacts the proof technique and as a result, the main difference in terms of assumptions compared to \cite[Theorem 1]{rainforth2018tighter} is that, for a given $\alpha \in [0,1)$, we are requiring the eighth moments of $\tilde{w}_{1,1}^{1-\alpha}$, ${\partial} \tilde{w}_{1,1}^{1-\alpha} / {\partial \theta_\ell}$ and ${\partial} \tilde{w}_{1,1}^{1-\alpha} / {\partial \phi_{\ell'}}$ to be finite in \Cref{prop:SNRconvergence}, where \cite[Theorem 1]{rainforth2018tighter} asked for the fourth moments to be finite with $\alpha = 0$. In addition, we need to further assume that there exists some $N \in \mathbb{N}^\star$ for which $\PE((1/\hat{Z}_{1, N, \alpha})^4) < \infty$. 
    
    \item \textbf{\Cref{prop:limitTGaussianRenyi}}. The proof of this result mostly mirrors the proof written in \cite[Section 4.a]{ObstaclestoHighDimensionalParticleFiltering} which considers the case $\alpha = 0$ and is used in the context of particle filtering. The main difference is that we require an additional lemma (\Cref{lem:exactnormalminconcentration} of \Cref{sec:Proofgaussian}) to provide us with a more precise concentration result for $S^{(1)}$. This lemma will also have other uses in the rest of the paper. This does not result in any change of assumptions compared to \cite[Section 4.a]{ObstaclestoHighDimensionalParticleFiltering}.
    
    \item \textbf{\Cref{Prop:limitT}}. The proof of this result is arguably the one that required the most alterations out of the three discussed here. It borrows some ideas from the proofs written in the context of particle filtering in \cite{li2008Curse, bengtsson2008curse, li2008Curse}, which all aimed at establishing that $\PE(T_{N,d}^{(0)}) \to 0$ in the approximate log-normal case under some conditions on $N$ and $d$. However, we are significantly more thorough in our control of the error terms, which we bound precisely with the aid of two results from \cite{saulis2000} and \Cref{lem:approxnormalminconcentration}, rather than simply working under convergence in probability. 
    
    In terms of assumptions, it is closest to \cite{li2008Curse}, except for the fact that our result relies on a Bernstein condition while \cite{li2008Curse} uses Cramer's conditions. Both are in fact equivalent in the i.i.d. setting, and we choose to use Bernstein condition as we believe it might make it easier to generalize our result beyond the i.i.d. case using the results from \cite{saulis2000} recalled in \Cref{app:subsub:Saulis}.
  \end{itemize}
  
\section{Additional numerical experiments and derivation details}

\subsection{Gaussian example from \Cref{subsec:Toy}}

\label{subseb:app:toy}

\begin{figure}[h!]
    \begin{tabular}{ccc}
        \includegraphics[scale=0.3]{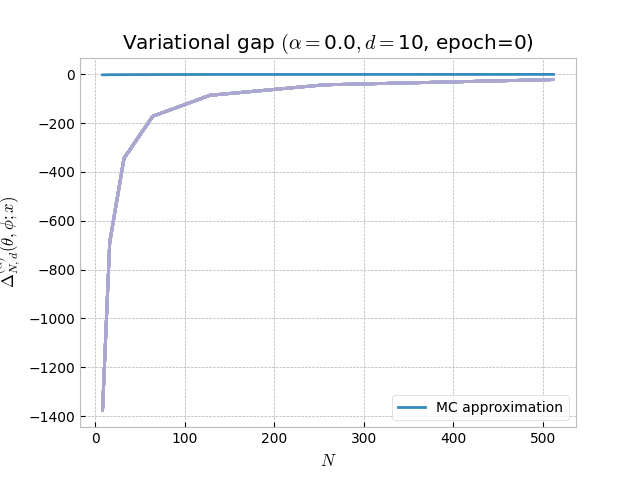} &
        \includegraphics[scale=0.3]{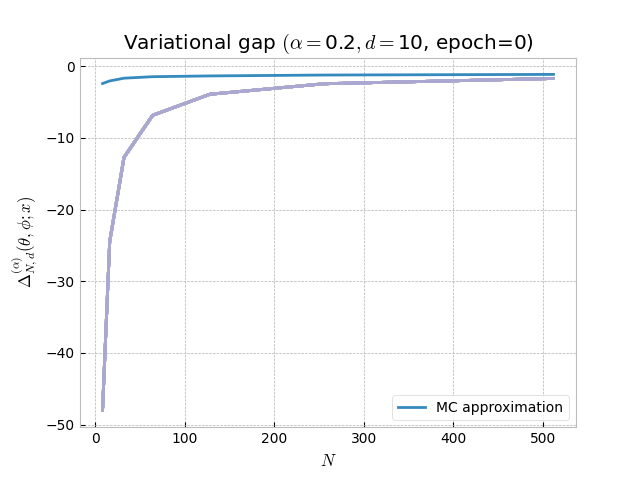} &
        \includegraphics[scale=0.3]{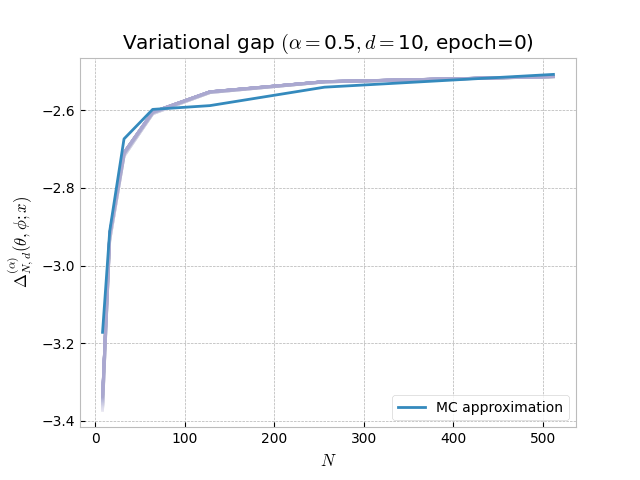} \\
      \includegraphics[scale=0.3]{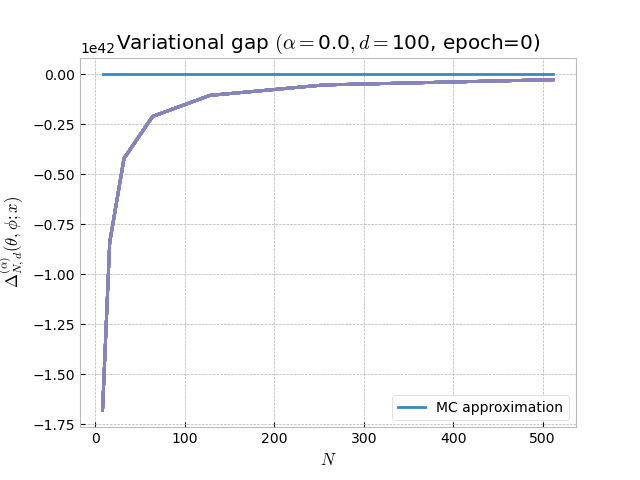} &
      \includegraphics[scale=0.3]{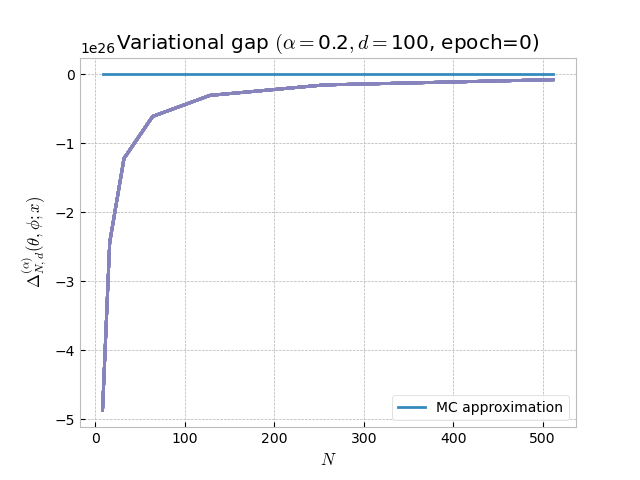} &
      \includegraphics[scale=0.3]{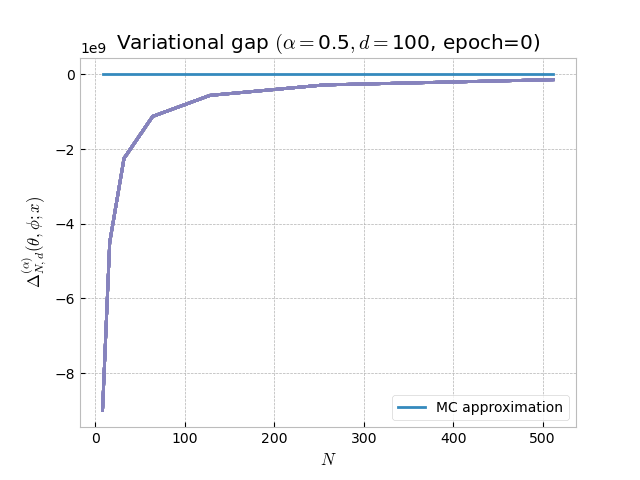}
    \end{tabular}
    \caption{Plotted in blue is the MC estimate of the variational gap $\Delta_{N,d}^{(\alpha)}(\theta, \phi; x)$ (averaged over 1000 MC samples) for the toy example described in \Cref{subsec:Toy} as a function of $N$, for varying values of $\alpha$ and of $d$ and with $(\theta, \phi) = (0 \cdot \boldsymbol{u_d}, \boldsymbol{u}_d)$ so that $B_d = \sqrt{d}$. Plotted in purple are curves of the form \eqref{eq:functionFormDomke} with tailored values of $c_1$.} \label{fig:ToyLogNormalApp}
  \end{figure}

 Figure \ref{fig:ToyLogNormalApp} empirically confirms that the asymptotic regime predicted by \Cref{prop:GenDomke} does not reflect what is happening in reality in the variational gap $\Delta_{N,d}^{(\alpha)}(\theta, \phi; x)$ when the dimension $d$ increases, $N$ is small and the distribution of the weight is log-normal. 
 
 Note that we only plotted the variational gap $\Delta_{N,d}^{(\alpha)}(\theta, \phi; x)$ for the cases $d = \lrcb{10, 100}$ in the figure above. This is due to the fact that when $d = 1000$, computing $\gamma_\alpha^2$ the $1/N$ term returns an overflow, further illustrating the limitations of the approach from \Cref{prop:GenDomke} in the specific setting considered here. Note also that since the variance term is exponential in $(1-\alpha)^2 d$, increasing $\alpha$ does play a role in decreasing $\gamma_\alpha^2$ so that the asymptotic regime predicted by \Cref{prop:GenDomke} applies in lower dimensions (e.g. $d = 10$ with $\alpha = 0.5$).

 \begin{figure}[!ht]
    \begin{center}
    \includegraphics[scale=0.33]{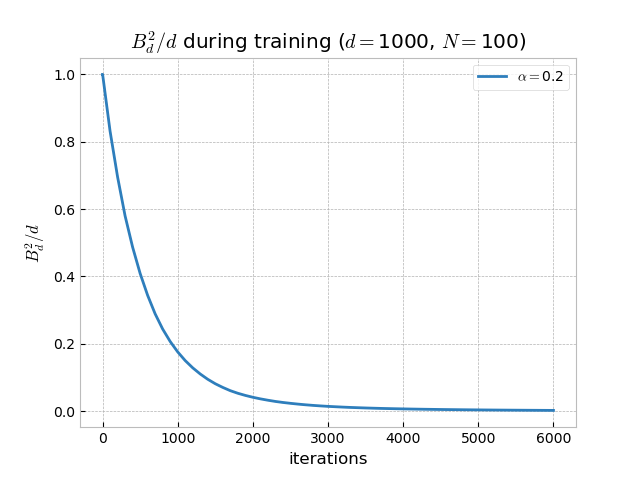}
    \end{center}
    \caption{Evolution of $B_d^2/d$ during the training of the $\phi$ parameter for the toy example described in \Cref{subsec:Toy}.}
    \label{fig:ToyLogNormal3}
  \end{figure}

 \subsection{Linear Gaussian example from \Cref{subsec:linGaussEx}}

\subsubsection{Empirical experiments for \Cref{prop:GenDomke} in the context of \Cref{subsec:linGaussEx}}
\label{subsub:DomkeLinGaussExApp}

Figure \ref{fig:Thm3LinGaussSigmapertZero} empirically confirms that we need an unpractical amount of samples $N$ for the asymptotic regime predicted by \Cref{prop:GenDomke} to capture the behavior of the variational gap as $d$ increases when $\sigmapert = 0$. Note that similar plots and conclusions can be obtained for $\sigmapert \in \lrcb{0.01, 0.5}$. Those are not given here for the sake of conciseness. 

\begin{figure}[ht!]
    \begin{tabular}{ccc}
        \includegraphics[scale=0.3]{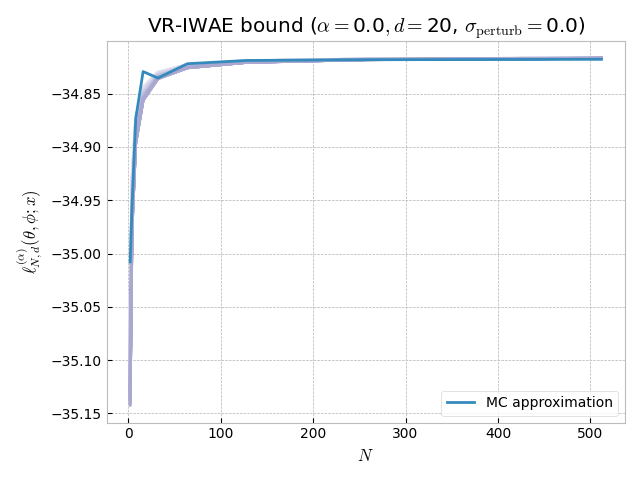} & \includegraphics[scale=0.3]{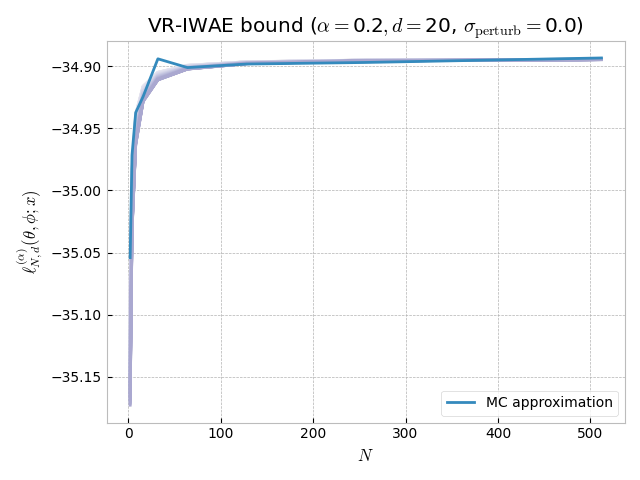} &  \includegraphics[scale=0.3]{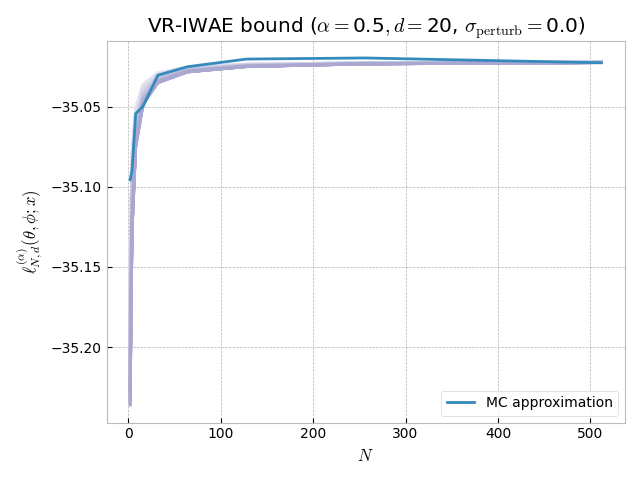} \\
        \includegraphics[scale=0.3]{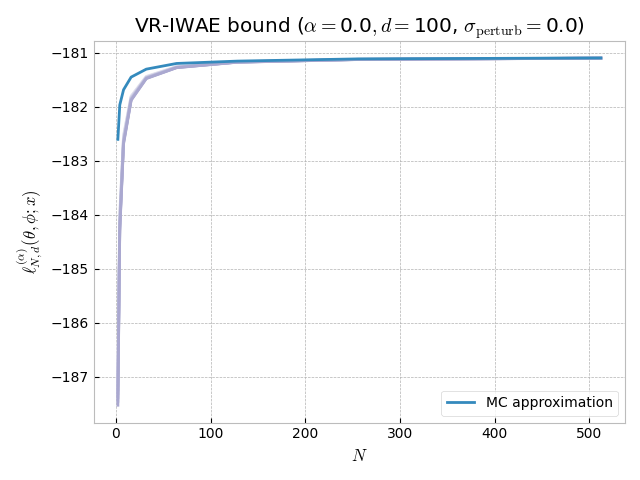} & \includegraphics[scale=0.3]{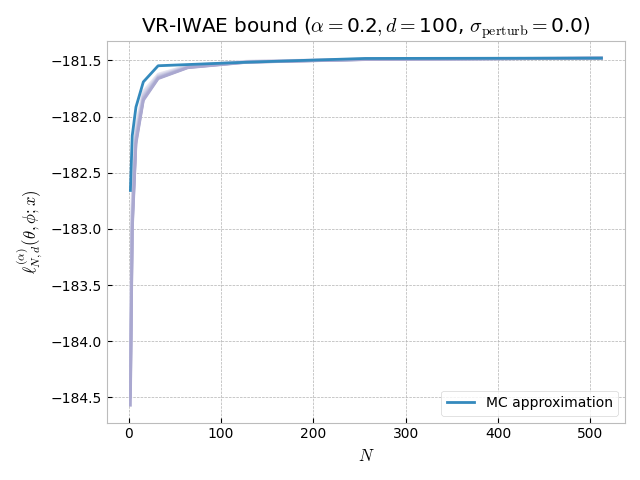} & \includegraphics[scale=0.3]{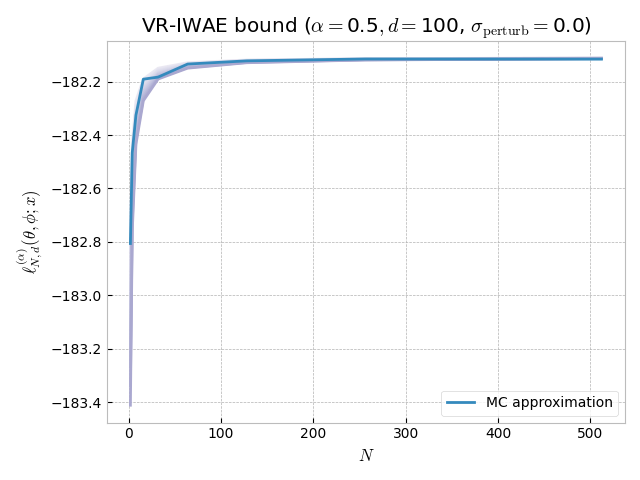} \\
        \includegraphics[scale=0.3]{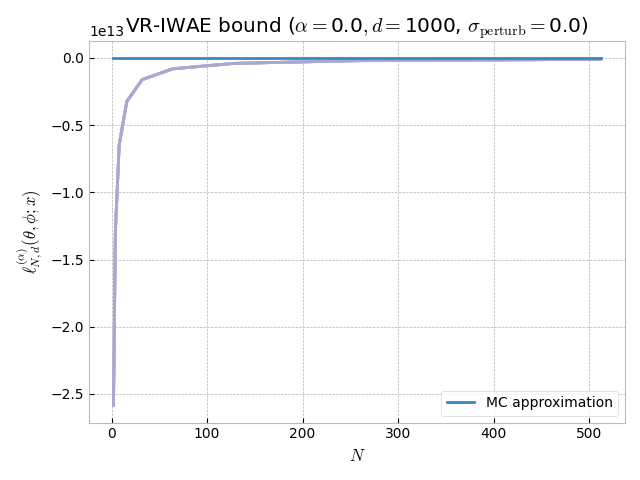} & \includegraphics[scale=0.3]{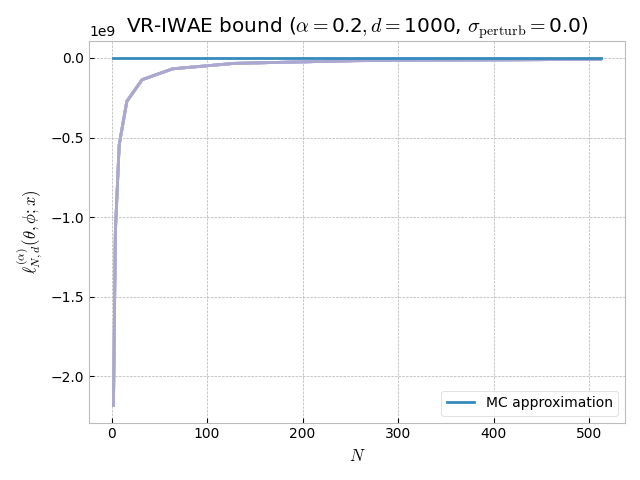} & \includegraphics[scale=0.3]{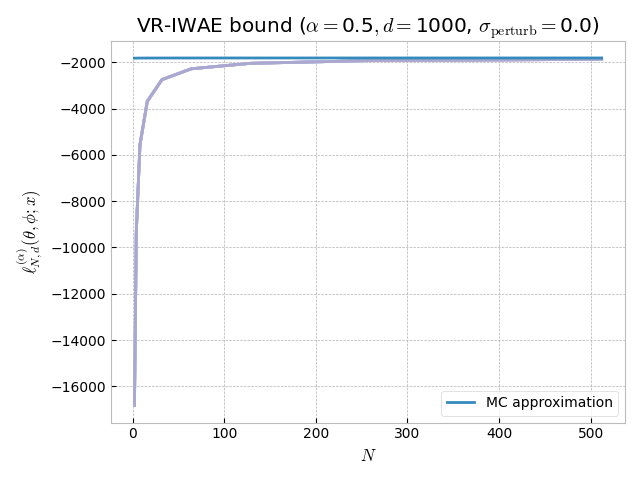}
    \end{tabular}
    \caption{Plotted in blue is the MC estimate of the VR-IWAE bound $\liren[\alpha][N,d](\theta, \phi; x)$ (averaged over 1000 MC samples) for the linear Gaussian example described in \Cref{subsec:linGaussEx} as a function of $N$, for varying values of $\alpha$ and of $d$. Plotted in purple are curves of the form \eqref{eq:funcThm3LinGauss} with tailored values of $c_1$.} \label{fig:Thm3LinGaussSigmapertZero}
\end{figure}

\subsubsection{Additional experimental results for \Cref{subsec:linGaussEx}}
\label{subsub:AddExpLinGauss}

We provide some additional results in the context of \Cref{subsec:linGaussEx} regarding the Signal-to-Noise Ratio (SNR) in the doubly-reparameterized case and the Mean Squared Error (MSE) for the VR-IWAE bound and its $\theta,\phi$ gradients.

\begin{itemize} 
    \item \textbf{SNR in the doubly-reparameterized case.} In line with \cite{Tucker2019DoublyRG}, we observe in \Cref{fig:linear_gaussian_vr_iwae_grad_snr_phi_drep} that using the doubly-reparameterized gradient estimator for $\phi$ increases the SNR when $\alpha = 0$. We in fact see that the SNR is increased for all values of $\alpha$, extending the conclusions from \cite{Tucker2019DoublyRG} to $\alpha \in [0,1)$ in the example considered here. %Furthermore, in the favourable case of low dimensions and small perturbations near the optimum, the inference network gradients do not suffer from the vanishing SNR issue highlighted in \Cref{prop:SNRconvergence} and \cite{rainforth2018tighter} when $\alpha=0$. 

However, as we get further away from the optimum ($\sigmapert = 0.5$) and/or increase the dimension ($d = 1000$), we observe that it still remains challenging to obtain an increasing SNR for the $\phi$ gradients for small values of $\alpha$, even when using doubly-reparameterized gradient estimators.

\item \textbf{MSE for the VR-IWAE bound and its $\theta,\phi$ gradients.} We observe on Figures~\ref{fig:linear_gaussian_vr_iwae_p_grad_mse_against_N} and \ref{fig:linear_gaussian_vr_iwae_mse_against_N} that while increasing $\alpha$ does not lower the MSE of the VR-IWAE estimator
\begin{equation*} 
    \frac{1}{1-\alpha} \log \lr{ \frac{1}{N} \sum_{j = 1}^N \w(Z_j;x)^{1-\alpha} }
\end{equation*}
for log-likelihood estimation, it can be useful in lowering the MSE of its $\theta$ gradients 
\begin{equation*} 
    \frac{1}{1-\alpha} \nabla_\theta \log \lr{ \frac{1}{N} \sum_{j = 1}^N \w(Z_j;x)^{1-\alpha} }
\end{equation*} 
compared to the $\theta$ gradients of the true log-likelihood $\nabla_\theta \ell_d(\theta;x)$.

In the low perturbation regime ($\sigmapert = 0.01$) and in medium to high dimensions ($d = 100, 1000$), we indeed see in \Cref{fig:linear_gaussian_vr_iwae_p_grad_mse_against_N} that 
every tested value of $\alpha>0$ achieves lower $\theta$ gradient MSE than $\alpha=0$ for low values of $N$. As we increase to $N=2^9$, the value of $\alpha$ achieving the lowest MSE is $\alpha=0.3$ for $d=100$, and $\alpha=0.8$ for $d=1000$. This sheds light on a bias-variance tradeoff between low bias at $\alpha=0$ and low variance at $\alpha=1$, and is in line with the findings of \Cref{prop:GenDomke}.  

In the high perturbation regime ($\sigmapert = 0.5$), we see in \Cref{fig:linear_gaussian_vr_iwae_p_grad_mse_against_N} that the choice of $\alpha$ appears to make less of a difference, especially when the dimension $d$ is high. %This appears to be consistent with the findings in \Cref{fig:linear_gaussian_vr_iwae_mse_against_N} for the VR-IWAE estimator. We also do not observe benefits of using $\alpha>0$ in this case. 
This suggests that bias reduction may be more important when the inference distribution $q_\phi(z|x)$ is far from the optimum. 

\end{itemize}

\begin{figure}[ht!]
    \begin{tabular}{ccc}
      \includegraphics[scale=0.29]{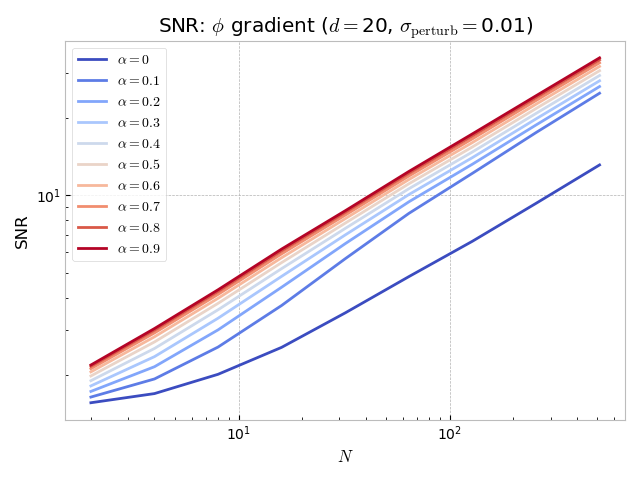} & 
      \includegraphics[scale=0.29]{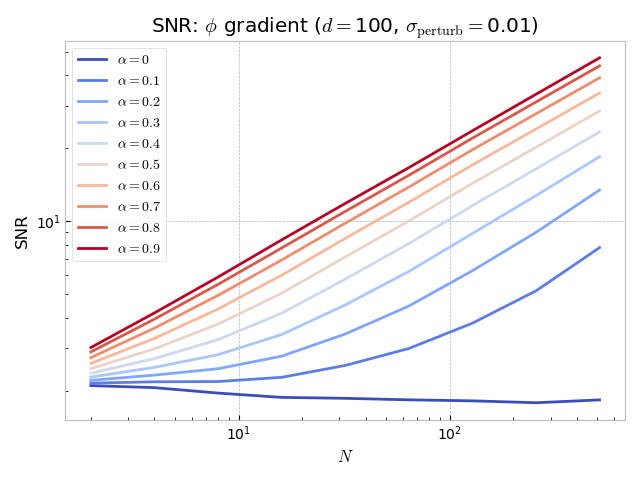} & 
      \includegraphics[scale=0.29]{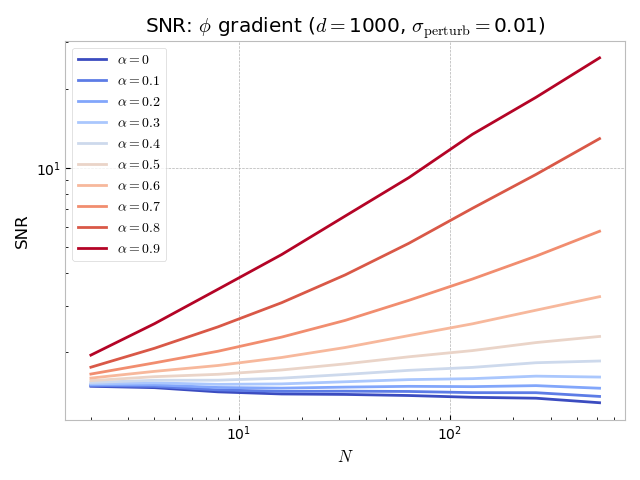} \\
      \includegraphics[scale=0.29]{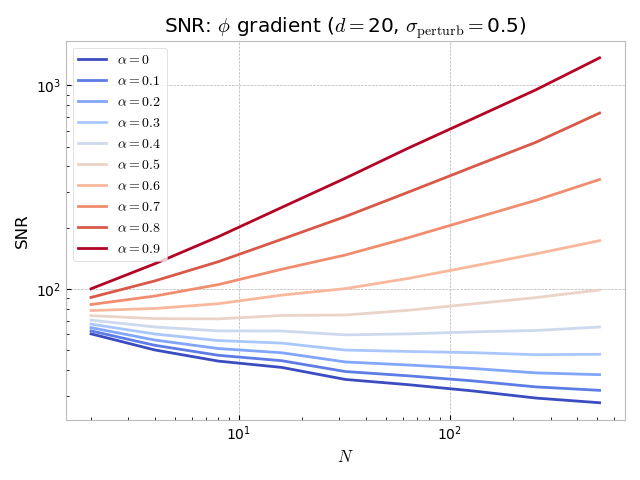} & 
      \includegraphics[scale=0.29]{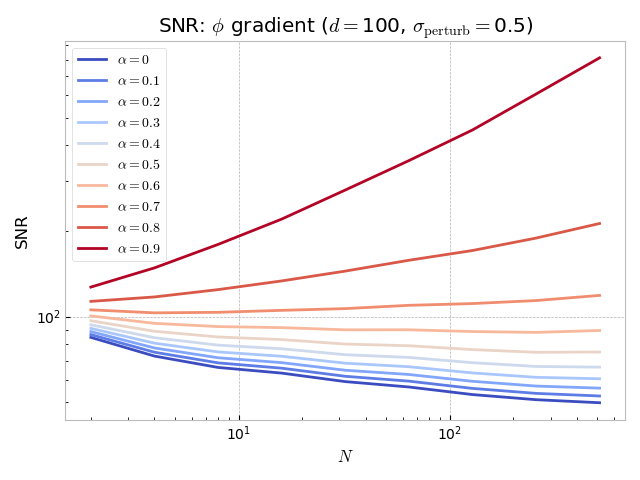} & 
      \includegraphics[scale=0.29]{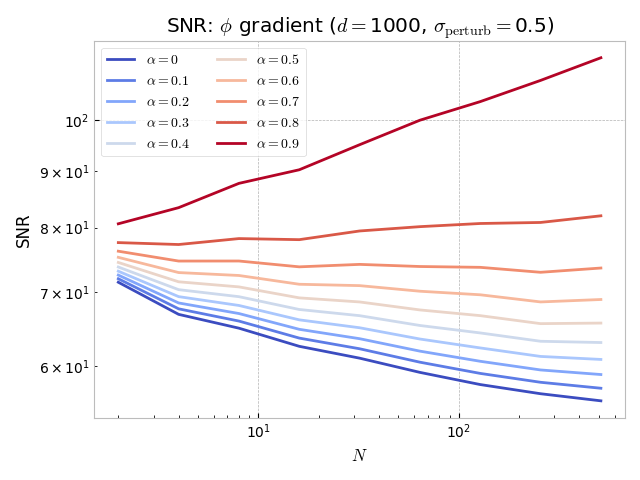}
    \end{tabular}
    \caption{Plotted is the SNR of the inference network ($\phi$) gradients in the doubly-reparameterized case (computed over 1000 MC samples) for the linear Gaussian example described in \Cref{subsec:linGaussEx} as a function of $N$, for varying values of $\alpha$ and of $d$, for a randomly selected datapoint $x$ and for 10 different initializations of the parameters $(\theta, \phi)$. %VR-IWAE q gradient SNR $\alpha\in[0,1]$ when using the doubly-reparameterized gradient estimator for the linear Gaussian example with. 
    \label{fig:linear_gaussian_vr_iwae_grad_snr_phi_drep}}
  \end{figure}

  \begin{figure}
    \begin{tabular}{ccc}
      \includegraphics[scale=0.29]{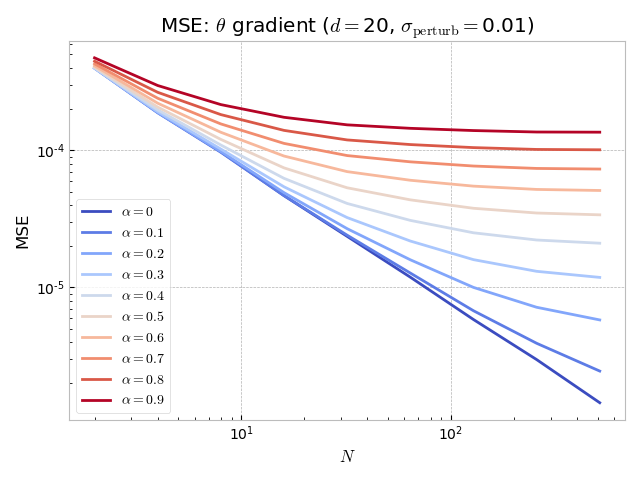} & 
      \includegraphics[scale=0.29]{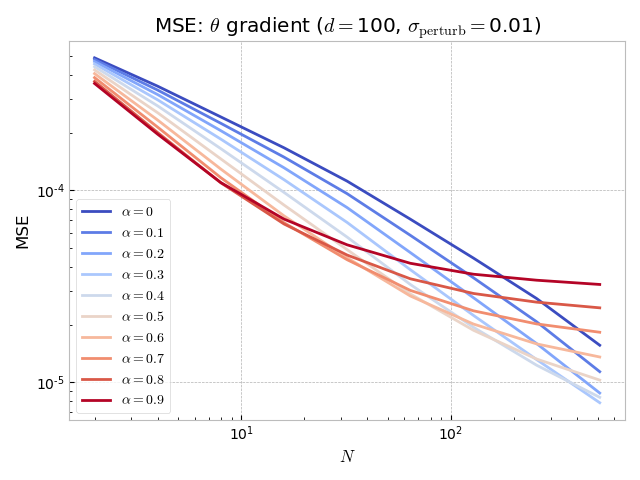} & 
      \includegraphics[scale=0.29]{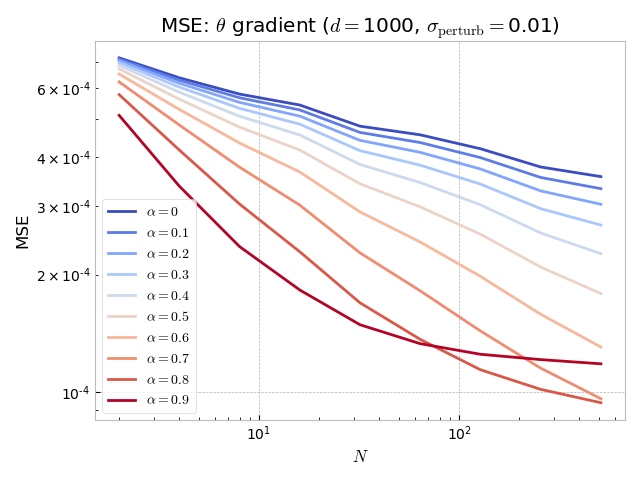} \\
      \includegraphics[scale=0.29]{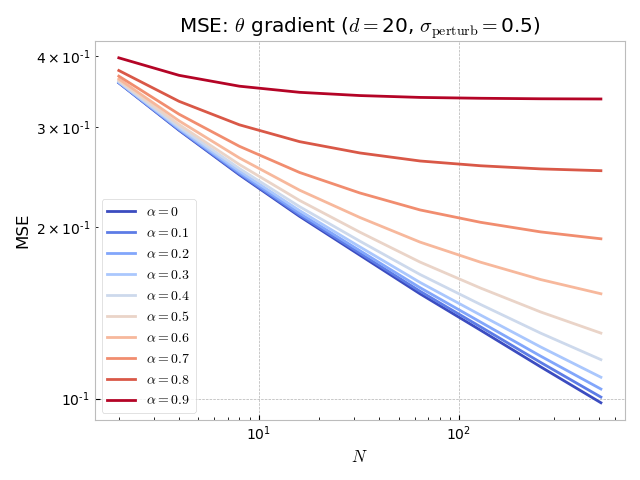} & 
      \includegraphics[scale=0.29]{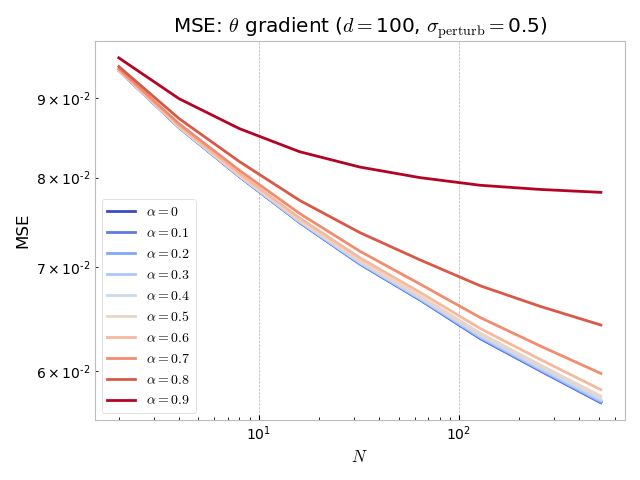} & 
      \includegraphics[scale=0.29]{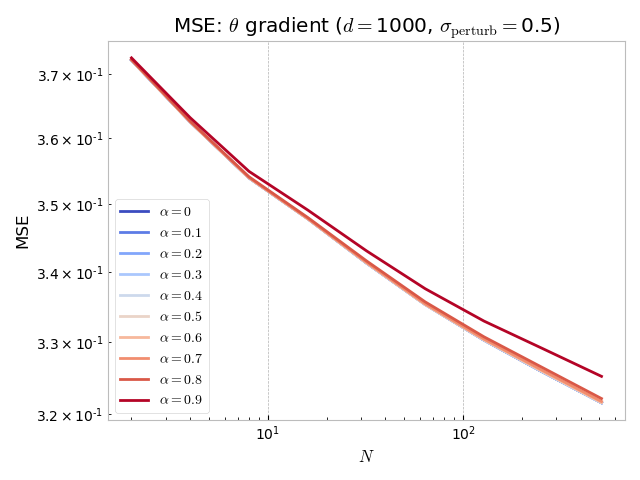}
    \end{tabular}
    \caption{Plotted is the MSE of the generative network ($\theta$) gradients (computed over 1000 MC samples) compared to the log-likelihood gradients for the linear Gaussian example described in \Cref{subsec:linGaussEx} as a function of $N$, for varying values of $\alpha$ and of $d$, for a randomly selected datapoint $x$ and for 10 different initializations of the parameters $(\theta, \phi)$. \label{fig:linear_gaussian_vr_iwae_p_grad_mse_against_N}}
  \end{figure}
  
  \begin{figure}
    \begin{tabular}{ccc}
      \includegraphics[scale=0.29]{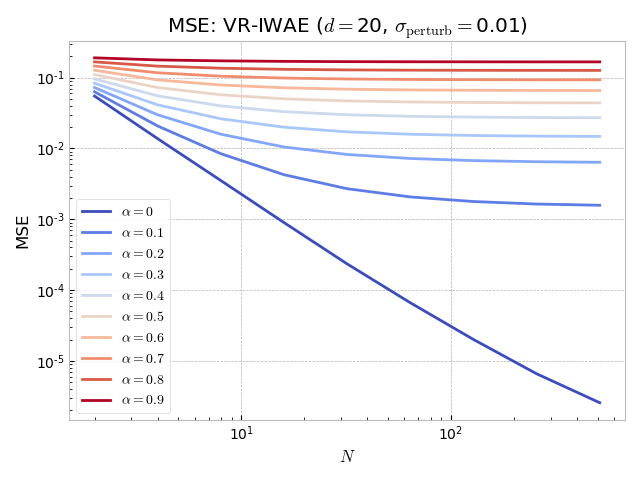} & 
      \includegraphics[scale=0.29]{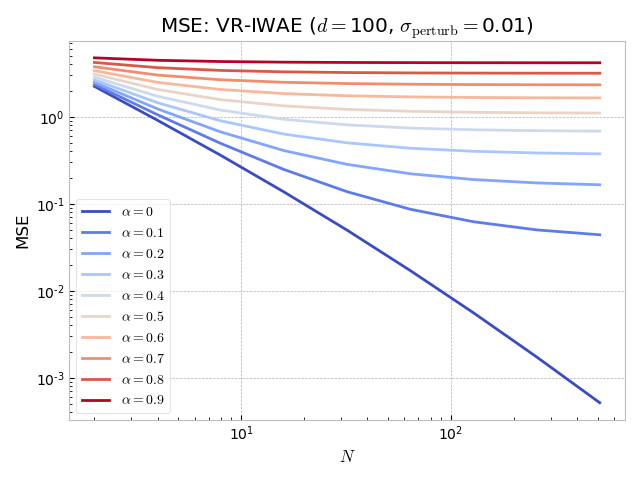} & 
      \includegraphics[scale=0.29]{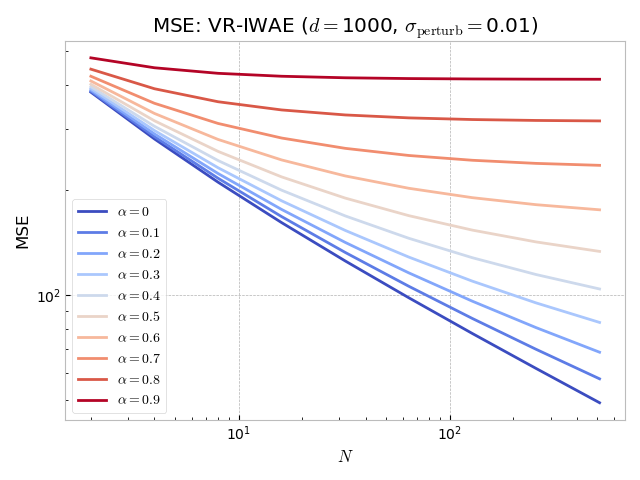} \\
      \includegraphics[scale=0.29]{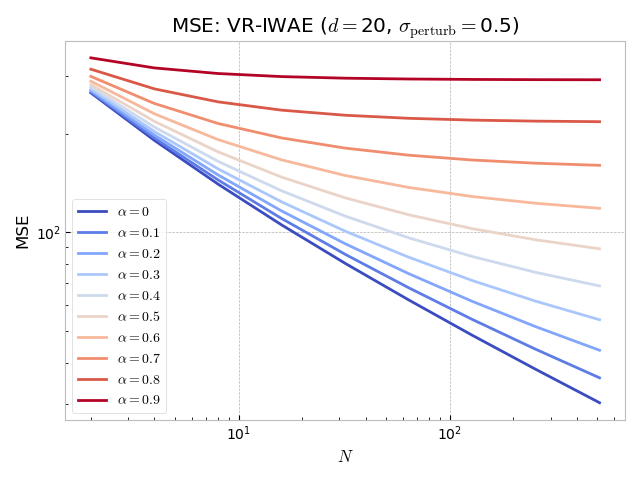} & 
      \includegraphics[scale=0.29]{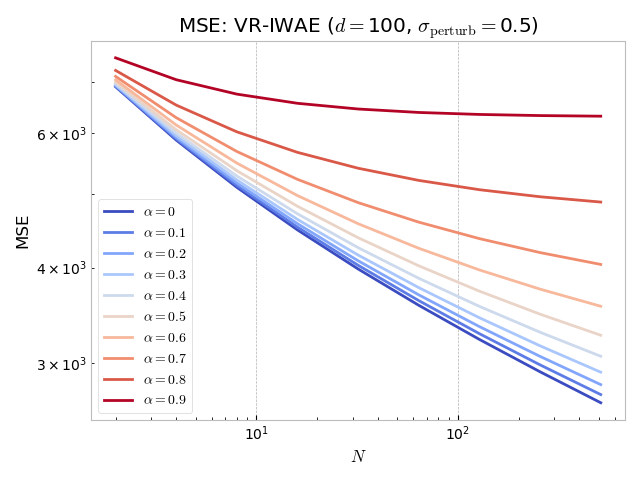} & 
      \includegraphics[scale=0.29]{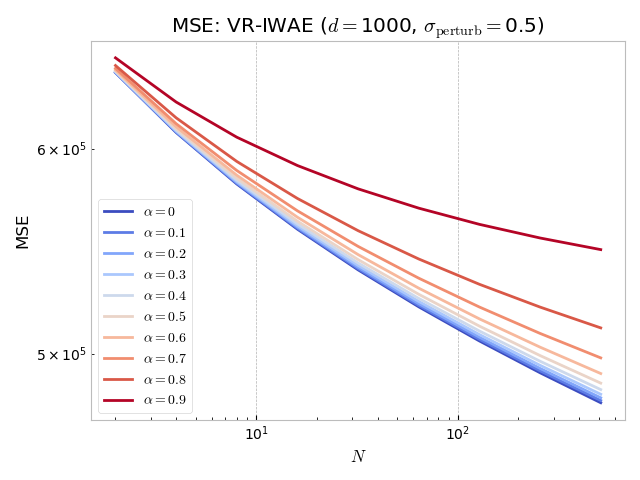}
    \end{tabular}
    \caption{Plotted is the MSE of the VR-IWAE estimate (computed over 1000 MC samples) compared to the log-likelihood gradients for the linear Gaussian example described in \Cref{subsec:linGaussEx} as a function of $N$, for varying values of $\alpha$ and of $d$, for a randomly selected datapoint $x$ and for 10 different initializations of the parameters $(\theta, \phi)$. \label{fig:linear_gaussian_vr_iwae_mse_against_N}}
  \end{figure}

  \subsection{Variational auto-encoder from \Cref{subsec:RealData}}
  We present additional results for the VAE example discussed in \Cref{subsec:RealData}. 

  \label{subsec:VAEexApp}
  \subsubsection{Complementary plots for the VR-IWAE bound}
  \label{subsec:VAEexApp}

Figures \ref{fig:vrIWAEinVAElowApp} and \ref{fig:vrIWAEinVAEhighApp} provide additional plots to Figures \ref{fig:vrIWAEinVAElow} and \ref{fig:vrIWAEinVAEhigh} reinforcing the conclusions drawn in \Cref{subsec:RealData}.
  \begin{figure}[ht!]
    \begin{tabular}{ccc}
      \includegraphics[scale=0.3]{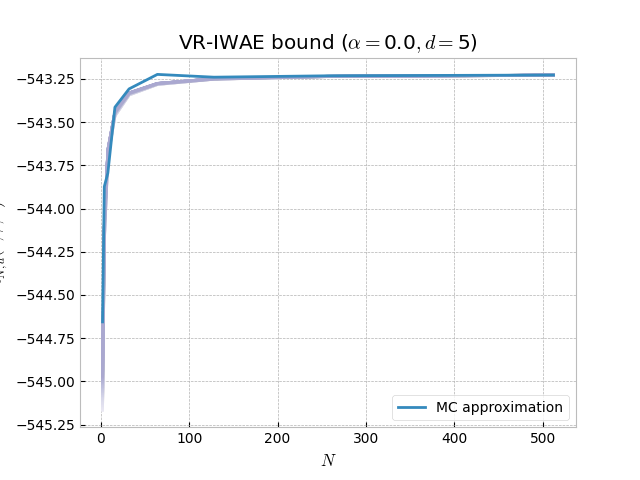} &   \includegraphics[scale=0.3]{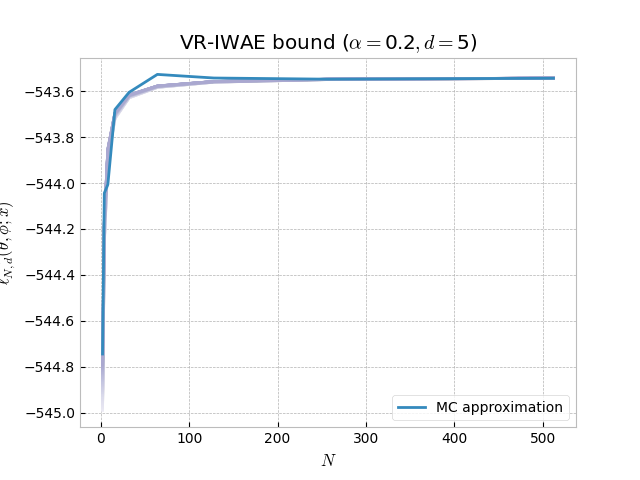} & \includegraphics[scale=0.3]{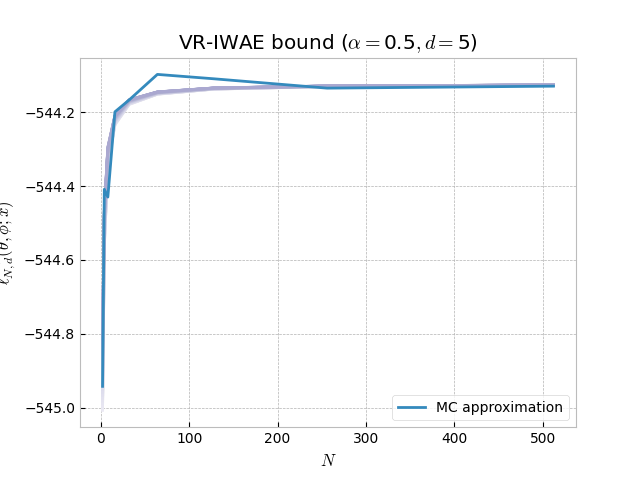} \\
      \includegraphics[scale=0.3]{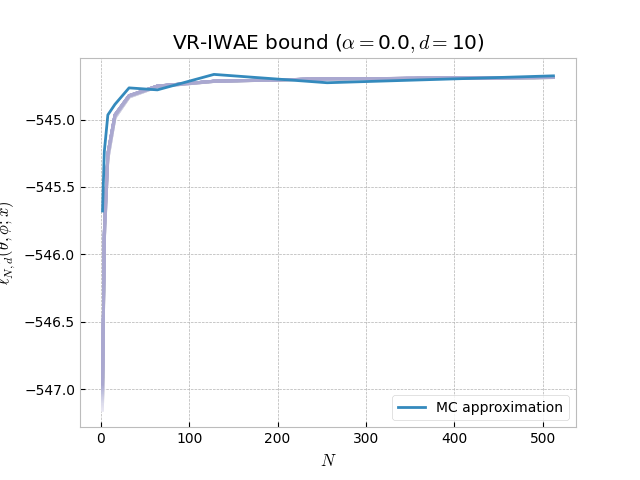} &   \includegraphics[scale=0.3]{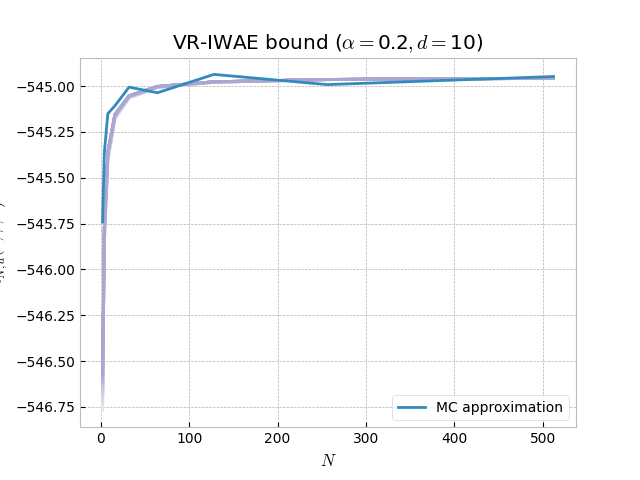} & \includegraphics[scale=0.3]{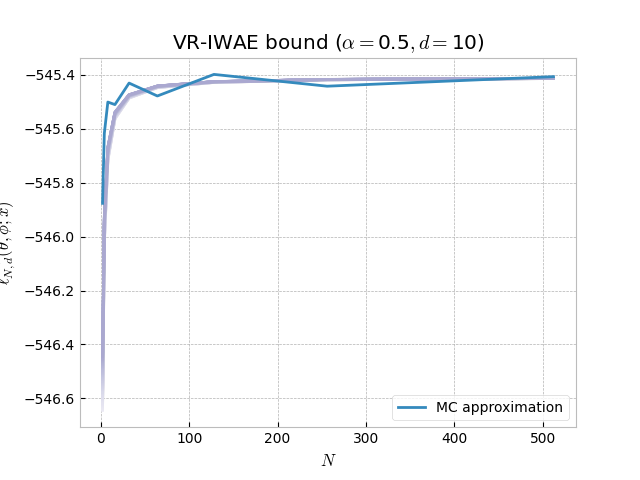} \\
      \includegraphics[scale=0.3]{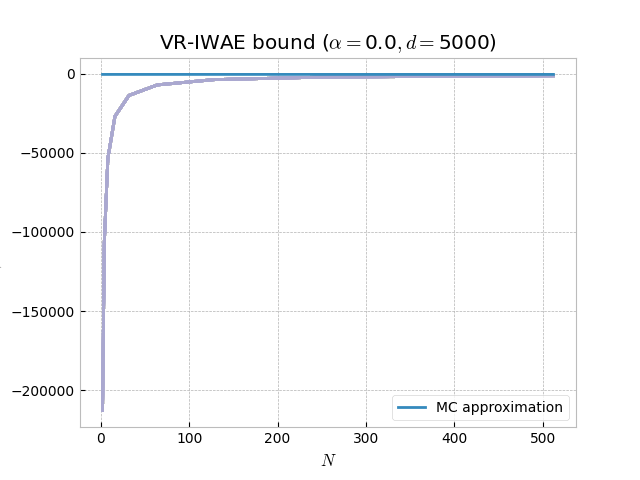} &   \includegraphics[scale=0.3]{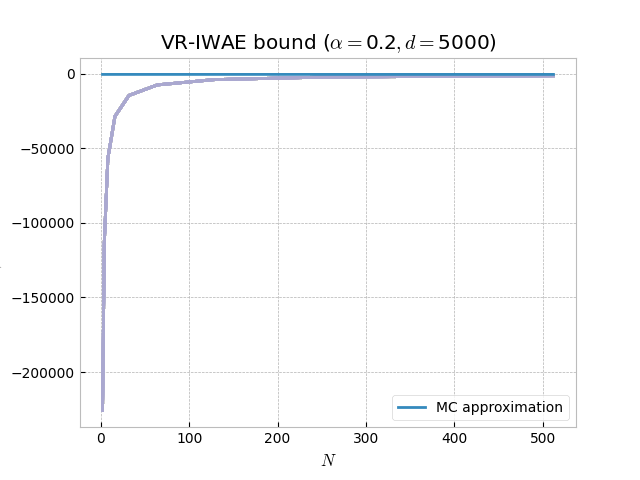} & \includegraphics[scale=0.3]{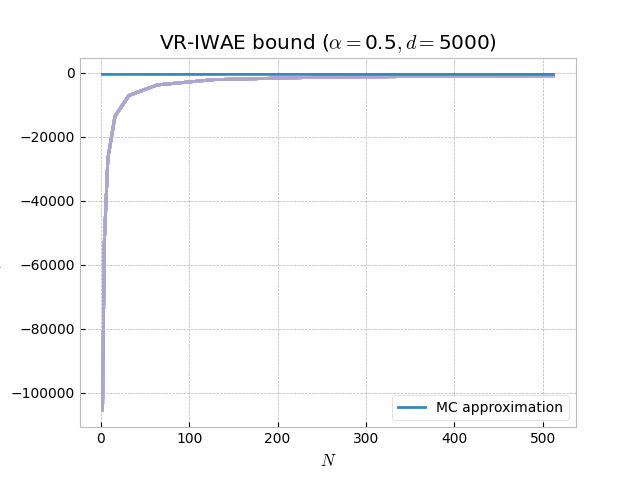}
    \end{tabular}
    \caption{Plotted in blue is the MC estimate of the VR-IWAE bound $\liren[\alpha][N,d](\theta, \phi; x)$ (averaged over 100 MC samples) for the VAE considered in \Cref{subsec:RealData}, for a randomly selected datapoint $x$ in the testing set, randomly generated model parameters $(\theta, \phi)$ and for varying values of $\alpha$ and of $d$. Plotted in purple are curves of the form \eqref{eq:funcThm3VAE} with tailored values of $c_1$.} \label{fig:vrIWAEinVAElowApp}
  \end{figure}

  \begin{figure}[ht!]
    \begin{tabular}{ccc}
      \includegraphics[scale=0.29]{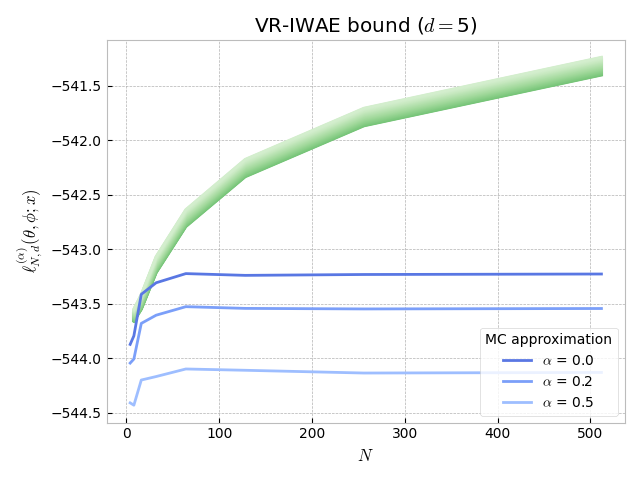}
      & \includegraphics[scale=0.29]{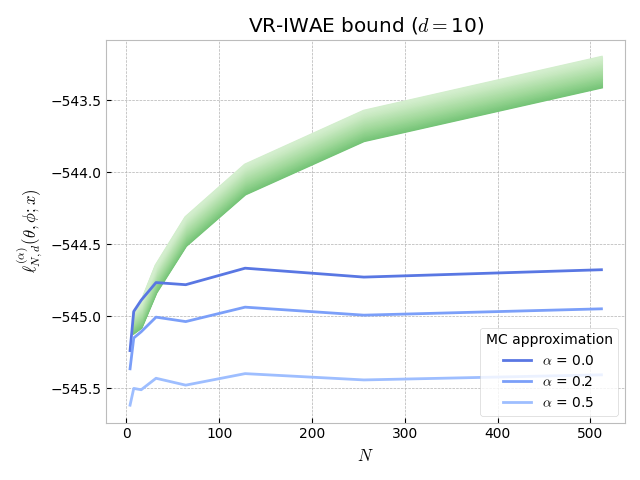}
      & \includegraphics[scale=0.29]{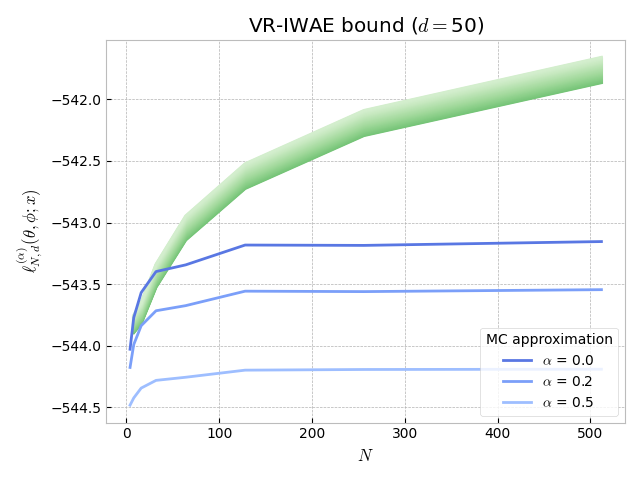} \\
    \end{tabular}
    \caption{Plotted in blue is the MC estimate of the VR-IWAE bound $\liren[\alpha][N,d](\theta, \phi; x)$ (averaged over 100 MC samples) for the VAE considered in \Cref{subsec:RealData}, for a randomly selected datapoint $x$ in the testing set, randomly generated model parameters $(\theta, \phi)$ and for varying values of $\alpha$ and of $d$. Plotted in green are curves of the form \eqref{eq:funcThm5VAE} with tailored values of $c_2$.}
    \label{fig:vrIWAEinVAEhighApp}
  \end{figure}

  \subsubsection{Impact of $\alpha$ and of $M$ on empirical performances}

  \label{subsub:appVAEImpact}

  We discuss here the impact of $\alpha$ and of $M$ on the empirical performances of the VR-IWAE bound metholodogy in the reparameterized and doubly-reparameterized cases.

  \begin{itemize}
    \item \textbf{Impact of $\alpha$ on the empirical performances.} We investigate how the choice of $\alpha$ impacts the Negative Log Likelihood (NLL) after training the VAE with the VR-IWAE bound. The NLL can indeed be used to evaluate the empirical performances of VAEs (since a lower NLL corresponds to a higher likelihood of the data under the VAE model, which indicates better training of the generative network $\theta$). 
    Furthermore, although the NLL is intractable, following \cite{burda2015importance} it can be approximated using the negative IWAE bound with $N=5000$. 
    
    We plot in \Cref{fig:vrIWAEinVAEalphatuning} the NLL estimate on the MNIST test set as a function of $\alpha$ after training VAEs on the MNIST training set using either the reparameterized (``rep") or the doubly-reparameterized (``drep") gradient estimators of the VR-IWAE objective with $N = 10,100$ and $d = 50$. Here, all the models are trained for 1000 epochs using the Adam optimizer with learning rate $1e-3$ and batch size 100. 
    
    We observe that the doubly-reparameterized gradient estimator generally achieves better NLL results than the reparameterized one when $\alpha$ is fixed. In addition, for both cases the value of $\alpha$ achieving the best NLL performance lies in the middle of $(0,1)$, around $\alpha=0.5$. In line with \Cref{prop:GenDomke}, this suggests that there is a bias-variance tradeoff to consider when choosing $\alpha$, and that the best setting can lie between the standard IWAE ($\alpha=0$, low bias) and ELBO ($\alpha=1$, low variance) objectives, with the optimal choice of $\alpha$ being dependent on the dataset, model architecture, as well as the stochastic gradient descent procedure used for training. 
  \end{itemize}

  \begin{figure}[t]
    \begin{center}
    \begin{tabular}{cc}
      \includegraphics[scale=0.38]{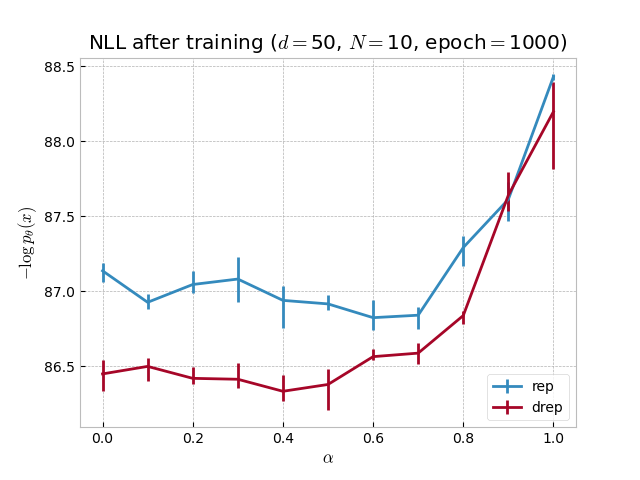}
      & \includegraphics[scale=0.38]{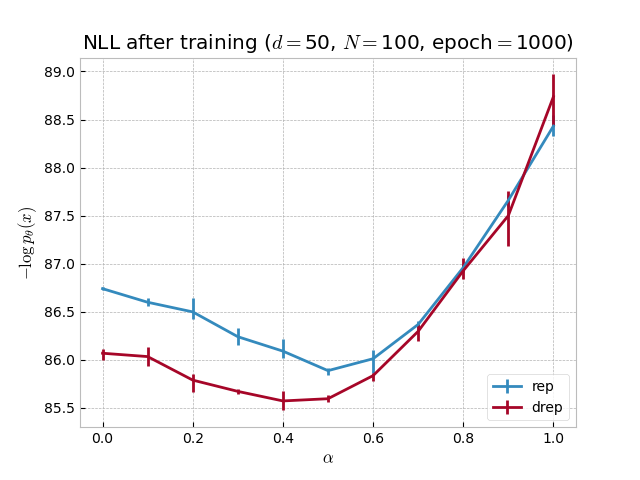}
    \end{tabular}
    \end{center}
    \caption{Plotted is the Negative Log 
    Likelihood (NLL) estimate on the test set of the MNIST dataset as described in \Cref{subsec:RealData} as a function of $\alpha$, after training on the train set for 1000 epochs with $N\in \{10,100\}$. The error bars are computed over 3 trials with different network initialisations and seeds during training. }
    % \caption{Plotted is the Negative Log 
    % Likelihood (NLL) estimate on the test set of the MNIST dataset as described in \Cref{subsec:RealData}, after training on the train set for 1000 epochs. On the left, we vary the value of $\alpha$ in the VR-IWAE objective with $N=100$. On the right, we vary the value of $M$ and $N$ while fixing $M \times N = 100$ and $\alpha = 0$. The error bars are computed over 3 trials with different network initialisations and seeds during training. }
    \label{fig:vrIWAEinVAEalphatuning}
\end{figure}

\begin{figure}[t]
    \begin{center}
    \begin{tabular}{cc}
      \includegraphics[scale=0.38]{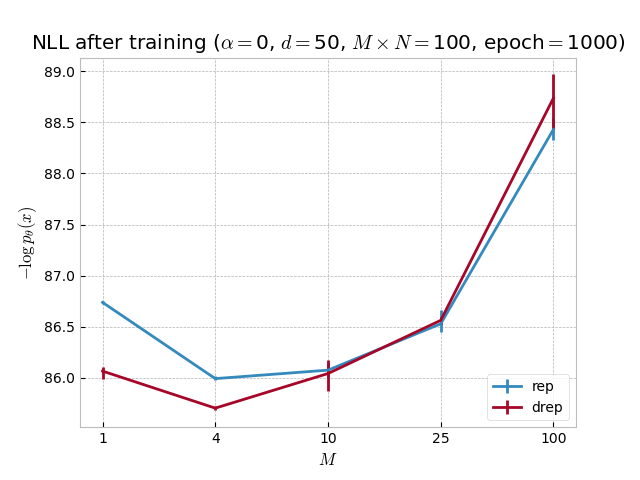}
      & \includegraphics[scale=0.38]{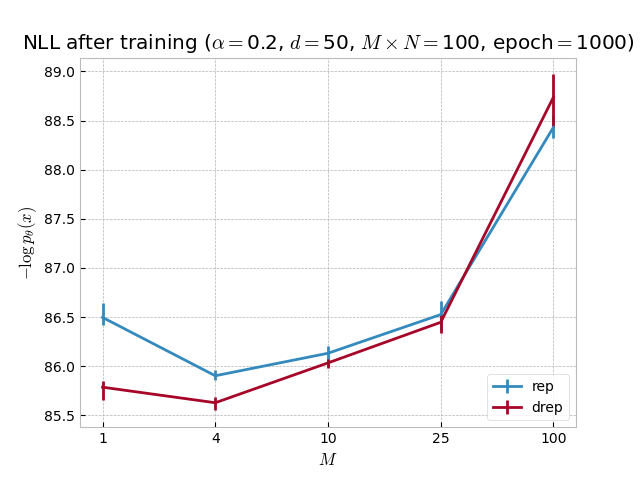}
    \end{tabular}
    \end{center}
    \caption{Plotted is the Negative Log 
    Likelihood (NLL) estimate on the test set of the MNIST dataset as described in \Cref{subsec:RealData} as a function of $M$ while fixing $M \times N = 100$, after training on the train set for 1000 epochs with $\alpha = 0,0.2$. The error bars are computed over 3 trials with different network initialisations and seeds during training. }
    \label{fig:vrIWAEinVAEMtuning}
\end{figure}

\begin{itemize}    

    \item \textbf{Impact of $M$ on the empirical performances.} We investigate how the choice of $M$ and $N$ affects the training of VAE when $M \times N$ is fixed in the VR-IWAE bound methodology. We plot in \Cref{fig:vrIWAEinVAEMtuning} the NLL on the MNIST test set after training the VAE on the MNIST training set for 1000 epochs, with $M \times N = 100$, $d = 50$ and $\alpha \in \{0, 0.2 \}$. \looseness=-1
    
    We observe a bias-variance tradeoff that is similar to the analysis above for the impact of $\alpha$. Indeed, the cases $M = 100$ and $M = 1$ have a particular meaning in the plots of \Cref{fig:vrIWAEinVAEMtuning}: $M = 100$ corresponds to the ELBO with maximum computational budget for $M$ (i.e. $\alpha = 1$, low variance), while $M = 1$ corresponds to the VR-IWAE bound with maximum computational budget for $N$ (i.e. low bias, with the lowest bias being achieved for $\alpha = 0$). Here, the best value of test NLL is obtained for $\alpha=0.2, M=4, N=25$ among our tested combinations. Note that one potential advantage of tuning $\alpha$ instead of $M$ and $N$ is that $\alpha$ resides on a one-dimensional continuous interval, whereas $M$ and $N$ are integer values and so their choice is more limited (that is, they can be more difficult to tune for a given computational budget). \looseness=-1

  \end{itemize}

\end{document}